%% file: main.tex
\documentclass[a4paper, 11pt]{article}
\usepackage[margin=2.cm]{geometry}

\usepackage{algorithmic}
\usepackage{appendix}
\usepackage{authblk}
\usepackage[ruled]{algorithm2e}
\usepackage{amsfonts}       
\usepackage{amsmath}
\usepackage{amssymb}
\usepackage{amsthm}
\usepackage{bm}
\usepackage{booktabs}       
\usepackage{breakcites}
\usepackage{caption}
\usepackage[overload]{empheq}
\usepackage{enumitem}
\usepackage{graphicx}
\usepackage[colorlinks=true, citecolor=blue, breaklinks=true]{hyperref}
\usepackage[utf8]{inputenc} 
\usepackage[T1]{fontenc}    
\usepackage{microtype}      
\usepackage{nicefrac}       
\usepackage{relsize}
\usepackage{scalefnt}
\usepackage{subcaption}
\usepackage{tikz}
\usepackage{url}            
\usepackage{wrapfig}
\usepackage{xcolor}         
\usepackage{xfrac}

\usepackage{pgfplots}
\pgfplotsset{compat=newest}
\pgfplotsset{scaled y ticks=false}
\usepgfplotslibrary{groupplots}
\usepgfplotslibrary{dateplot}

\theoremstyle{plain}
\newtheorem{theorem}{Theorem}[section]

\newtheorem{lemma}[theorem]{Lemma}
\newtheorem{corollary}[theorem]{Corollary}
\newtheorem{conjecture}[theorem]{Conjecture}
\newtheorem{claim}[theorem]{Claim}
\theoremstyle{definition}

\theoremstyle{remarwk}

\input{macros}

\begin{document}

\title{On double-descent in uncertainty quantification in overparametrized models}

\author[1]{Lucas Clart\'e}
\author[2,3]{Bruno Loureiro}
\author[3]{\\Florent Krzakala}
\author[1]{Lenka Zdeborov\'a}

\affil[1]{
\'Ecole Polytechnique F\'ed\'erale de Lausanne (EPFL)\\
Statistical Physics of Computation lab.\\
CH-1015 Lausanne, Switzerland
}
\affil[2]{
D\'epartement d'Informatique, \'Ecole Normale Sup\'erieure - PSL \& CNRS, 45 rue d’Ulm, F-75230 Paris cedex 05, France
}
\affil[3]{
\'Ecole Polytechnique F\'ed\'erale de Lausanne (EPFL)\\
Information, Learning and Physics lab.\\
CH-1015 Lausanne, Switzerland
}

\date{}

\maketitle

\begin{abstract}
Uncertainty quantification is a central challenge in reliable and trustworthy machine learning. Naive measures such as last-layer scores are well-known to yield overconfident estimates in the context of overparametrized neural networks. Several methods, ranging from temperature scaling to different Bayesian treatments of neural networks, have been proposed to mitigate overconfidence, most often supported by the numerical observation that they yield better calibrated uncertainty measures. In this work, we provide a sharp comparison between popular uncertainty measures for binary classification in a mathematically tractable model for overparametrized neural networks: the random features model.  We discuss a trade-off between classification accuracy and calibration, unveiling a double descent like behavior in the calibration curve of optimally regularized estimators as a function of overparametrization. This is in contrast with the empirical Bayes method, which we show to be well calibrated in our setting despite the higher generalization error and overparametrization.
\end{abstract}

\section{Introduction}
Uncertainty estimation is the cornerstone of reliable data processing. A large body of literature in classical statistical theory is dedicated to providing solid mathematical guarantees on a model's uncertainty, such as confidence scores for classification and confidence intervals for regression \cite{wasserman2013all}. Yet, when it comes to modern machine learning methods such as deep neural networks our mathematical understanding of the uncertainty associated with prediction falls short.
A key aspect in current machine learning practice is that, in contrast to classical wisdom,  models often operate in a regime where the complexity of the hypothesis class (e.g. as measured by the number of parameters in the model) is comparable or larger than the quantity of data available for training. This modern, overparametrized regime defies the common intuition rooted on classical statistics, therefore posing interesting challenges to their mathematical treatments. sol
For example, deep neural networks are able to achieve optimal generalization performance even when the training data are perfectly interpolated \cite{geman1992neural, PhysRevE.100.012115, Nakkiran2020}, a behaviour at odds with the  bias-variance intuition. This \emph{benign overfitting} property was recently shown to be common among  overparametrized convex methods, such as linear regression \cite{Bartlett2020, 10.1214/21-AOS2133}, random features regression \cite{mei_generalization_2022} and classification \cite{gerace_generalisation_2020}. 

While much of the theoretical effort has focused on the generalization properties of point estimates from overparametrized models, less is understood about their confidence. Indeed, a popular method to estimate uncertainty in neural networks consists of interpreting the last layer pre-activations as class probabilities. Numerical experiments suggest that deep neural networks tend to suffer from \emph{overconfidence} with respect this notion \cite{guo_calibration_2017}, a problem which has motivated many empirical calibration methods in the literature \cite{Hein_2019_CVPR, kristiadi_being_2020, NEURIPS2020_aeb7b30e, NEURIPS2020_543e8374}. Recently, it has been shown that actually overconfidence is a common problem in high-dimensional classification \cite{bai_dont_2021}, although it can be considerably mitigated by properly regularising the risk \cite{clarte_theoretical_2022}. An alternative to the pre-activation scores consists in applying a Bayesian treatment to neural networks, for instance by averaging the last layer weights over the measure induced by the empirical risk. In some contexts, these techniques were shown to provide better calibrated uncertainty measures than pre-activation score. A priori, Bayesian techniques require sampling from a high-dimensional measure, and therefore can be computationally demanding \cite{alexos_structured_2022}. Despite the success and widespread use of these uncertainty measures, mathematical guarantees relating these notions to intrinsic uncertainty measures such as the true class probabilities or the best uncertainty estimation given the available data (i.e. the true posterior uncertainty given the features) are scarce. In this work, we provide a sharp mathematical comparison between these different uncertainty notions in the context of a simple, solvable model for binary classification on structured features - such as the ones given by the first layers of neural networks. To the best of our knowledge, our work is the first to provide a sharp asymptotic analysis of uncertainty in overparametrized high-dimensional models. 

\paragraph{Related work --} Uncertainty quantification in deep learning is an active and rapidly evolving field, with many coexisting metrics and methods in the literature, see e.g. \cite{ABDAR2021243, gawlikowski_survey_2022} for two recent reviews. \cite{7298640, guo_calibration_2017} empirically observed that different from "small" networks \cite{ 10.1145/1102351.1102430}, modern deep neural networks tend to give overconfident predictions. \cite{guo_calibration_2017} proposed \emph{temperature scaling}, a simple post-processing variant of Platt scaling \cite{platt_2000} consisting of rescaling \& cross-validating the norm of the last-layer weights, and showed it can effectively calibrate them. Alternatively, \cite{kristiadi_being_2020} has argued that a Bayesian treatment of the last layer of deep networks fixes overconfidence. Bayesian methods typically involve sampling from a high-dimensional posterior \cite{Mattei2019}, and different methods have been proposed to compute them efficiently \cite{NIPS2011_7eb3c8be,gal_dropout_2016, NIPS2017_9ef2ed4b,maddox_simple_2019}. Of particular interest to our work is the Laplace approximation introduced in \cite{10.1162/neco.1992.4.3.415} for Gaussian process classification and adapted to  Bayesian deep learning in \cite{ritter_scalable_2018, kristiadi_being_2020, daxberger_laplace_2021}. An asymptotic discussion of evidence maximization in Bayesian ridge regression appeared in \cite{marion_hyperparameters_1994, Bruce_1994, Marion_1995}. \cite{bai_dont_2021} has shown that the logit model is overconfident in high-dimensions, and \cite{clarte_theoretical_2022} discussed how to mitigate it by properly regularizing. An exact asymptotic characterization of the empirical risk minimizer for random features model has been derived and discussed in \cite{mei_generalization_2022, gerace_generalisation_2020, goldt_gaussian_2021, Hu2020, Dhifallah2020,  loureiro_learning_2021}. Particularly relevant to our technical results is the recent progress in approximate message-passing schemes for structured matrices \cite{Gerbelot21, loureiro_fluctuations_2022}. Finally, exact asymptotics for Bayes-optimal estimation has been discussed in the context of generalized linear models in \cite{barbier_optimal_2019, gabrie2018entropy}.

\paragraph{Notation --} We denote vectors with bold letters, and matrices with capital letters. For $n\in\mathbb{N}$, we let $[n]\coloneqq \{1,\cdots, n\}$. $\mathcal{N}(\vec{\mu},\Sigma)$ denotes the Gaussian density, $\sigma(t):=(1+e^{-t})^{-1}$ denotes the sigmoid function. We define 
\begin{equation}
    \sigma_v(x) := \int \sigma(z) \mathcal{N}( z | x, v) {\rm d}z
\end{equation}
the averaged sigmoid with a Gaussian noise of variance $v$. 

\section{Setting}

\subsection{Probabilistic classifiers and uncertainty}
\label{sec:classifiers}
Consider a supervised binary classification task given by $n$ independent samples $\mathcal{D} = (\vec{x}^{\mu},y^{\mu})_{\mu\in [n]}\in\mathcal{X}\times \{-1,+1\}$ from a joint distribution $\nu$, and denote by $f_{\star}(\vec{x}) = \prob(y=1|\vec{x})$ the oracle class probability obtained by conditioning~$\prob$ over an input. In this work we are interested in studying the uncertainty associated to probabilistic classifiers $\hat{f}(\vec{x}) = \mathbb{P}(y=1|\vec{x}) \in [0,1]$ obtained by fitting the data\footnote{In the following, we consistently denote with a hat classifiers which are a function of the training data.}, and how they compare with the true class probability $f_{\star}$. A key motivation is the recent stream of works on uncertainty quantification for neural networks, and in particular the line of works proposing uncertainty measures based on classifiers defined by sampling over the last layer of neural networks \cite{brosse_last-layer_2020, kristiadi_being_2020}. To set notation, let $\vec{\varphi}:\mathcal{X}\to\mathbb{R}^{p}$ denote a \emph{feature map}, for instance the features learned by the first layers of a trained neural network.
We shall be interested in the following classifiers:
\paragraph{Empirical risk classifier --} The empirical risk classifier is the one obtained by naively interpreting the scores in the last layer as probability distributions. Mathematically, it is defined as $\hat{f}_{\erm}(\vec{x}) = \sigma(\hat{\vec{\theta}}_{\erm}^{\top}\varphi(\vec{x}))$, where $\sigma:\mathbb{R}\to(0,1)$ is a non-linearity. For concreteness, we will focus on the popular case where $\sigma(z) = (1+e^{-z})^{-1}$ is the sigmoid function, and $\hat{\vec{\theta}}_{\erm}\in\mathbb{R}^{p}$ is the minimizer of the associated (regularized) logistic or cross-entropy risk:
\begin{align}
    \hat{\mathcal{R}}_{n}(\vec{\theta}) = \frac{1}{n}\sum\limits_{\mu=1}^{n}\log\left(1+e^{-y^{\mu}\vec{\theta}^{\top}\vec{\varphi}(\vec{x}^{\mu})}\right)+\frac{\lambda}{2}||\vec{\theta}||_{2}^{2} \, . 
    \label{def:risk}
\end{align}
This is also commonly referred to as the \emph{logit classifier}.
\paragraph{Bayes-optimal classifier --} Denoting the training features $\mathcal{D}_{\varphi} := \{(\vec{\varphi}(\vec{x}^{\mu}), y^{\mu})\}_{\mu\in[n]}$, the optimal Bayesian classifier for the last layer is given by:
\begin{align}
\hat{f}_{\bo}( \vec x ) \!=\! \int\dd\vec{\theta} ~p\left(y=1|\vec{\theta},\{\vec{\varphi}\left(\vec{x}^{\mu}\right)\}_{\mu\in[n]}\right)p\left(\vec{\theta}|\mathcal{D}_{\varphi}\right)
\end{align}
\noindent where $p\left(y=1|\vec{\theta},\{\vec{\varphi}(\vec{x}^{\mu})\}_{\mu\in[n]}\right)$ is the likelihood over the labels and $p(\vec{\theta}|\mathcal{D}_{\varphi})$ is the posterior distribution over the weights given the training features and labels. In practice, the Bayes-optimal classifier $f_{\bo}$ is not accessible to the statistician, since she doesn't have access to the distribution~$\nu$ that has generated the data - and even if she had, sampling from the high-dimensional posterior distribution would be computationally cumbersome. However, as we will discuss in Sec.~\ref{sec:techres}, for the data generative model considered here, the Bayes-optimal classifier can be asymptotically characterized, and its marginals can be computed by a polynomial-time message passing algorithm. 

\paragraph{Bayesian classifiers --} Since the optimal Bayesian classifier is not accessible in practice, different classifiers inspired by Bayesian methods have been proposed in the literature. In this manuscript, we will consider two popular choices. 

The first is the \textit{empirical Bayes} classifier $\hat{f}_{\empbayes}$ \cite{marion_hyperparameters_1994, jospin_hands_2022}. In full generality, the empirical Bayes method consists of postulating a class of plausible likelihoods and priors and doing model selection from the training data via evidence maximization. In the context of Bayesian neural networks, the likelihood and priors are defined by the network architecture and regularization, which are normalized to define proper probability distributions. 

In our setting, the empirical Bayes classifier is explicitly given by:
\begin{align}
    \label{eq:def_empbayes}
    \hat{f}_{\empbayes}(\vec{x}) = \int_{\mathbb{R}^{p}} \dd\vec{\theta}~\sigma(\beta\vec{\theta}^{\top}\vec{\varphi}(\vec{x}))p_{\empbayes}(\vec{\theta}|\mathcal{D}, \beta,\lambda),\\
    p_{\empbayes}(\vec{\theta}|\mathcal{D},\beta,\lambda) = \frac{\prod\limits_{\mu} \sigma(\beta y^{\mu}\vec{\theta}^{\top}\vec{\varphi}(\vec{x}^{\mu})) \mathcal{N}(\vec{\theta} | \sfrac{\mat{I}_{p}}{\beta\lambda} )}{p(\mathcal{D} | \beta, \lambda)}\notag
\end{align}
The normalisation constant $p(\mathcal{D} | \beta, \lambda)$ is known as the \emph{marginal likelihood} or the \textit{evidence}. In the empirical Bayes method the evidence is maximized in order to select the most likely hyperparameters $(\beta,\lambda)$ explaining the training data \cite{MacKay1996}. In our specific model, we note that the evidence is actually only a function of the ratio $\sfrac{\lambda}{\beta}$ (this can be seen from the change of variables $\vec{\theta}\leftarrow\beta\vec{\theta}$). Therefore, without loss of generality we take $\beta=1$ and optimize only over $\lambda$. It is important to stress that the postulated prior and likelihood in $\hat{f}_{\empbayes}$ may not correspond to the ones that generated the data in general. \looseness=-1

Note that, differently from the Bayes-optimal estimator, the empirical Bayes classifier can be a priori computed using only the training data. However, it can be computationally demanding to sample from the posterior distribution above, specially in large dimensions $p, n \gg 1$. To avoid this computational bottleneck, a common approximation consists of expanding the posterior around the $\hat{\vec{\theta}}_{\erm}$ to second order, known as the \emph{Laplace approximation} \cite{kristiadi_being_2020, ritter_scalable_2018, daxberger_laplace_2021}:
\begin{align}
\label{eq:laplace}
    \hat{f}_{\lap}\left(\vec{x}\right) = \int\dd\vec{\theta}~\sigma\left(\hat{\vec{\theta}}_{\erm}^{\top}\vec{\varphi}(\vec{x})\right)\mathcal{N}(\vec{\theta}|\hat{\vec{\theta}}_{\erm}, \mathcal{H}^{-1})
\end{align}
\noindent where $\mathcal{H} \coloneqq \nabla^{2}_{\vec{\theta}}\hat{\mathcal{R}}_{n}(\hat{\vec{\theta}}_{\erm})$ is the Hessian of the empirical risk evaluated at the minimum. Therefore, in the Laplace approximation the posterior is effectively approximated by a Gaussian distribution centred at $\hat{\vec{\theta}}_{\erm}$ and with covariance given by the inverse curvature around the minimum. The "sharper" the minimum, the lower the variance and the more confident the Laplace classifier is. Note that the generalization errors associated to the Laplace classifier coincide exactly with the empirical risk classifier. Finally, in the model considered here, the Laplace approximation $\hat{f}_{\lap}$ will always be less confident than the ERM estimator using $\hat{\vec{\theta}}_{\erm}$. This is due to the concavity of the logit function $\sigma$ on $[0, \infty)$.

\paragraph{Performance and uncertainty -- } Given a probabilistic classifier $\hat{f}$, the most common measure for the generalization performance is the \emph{misclassification test error} (also known as \emph{0/1 error}) :
\begin{equation}
    \mathcal{E}_{\rm{gen.}}(\hat{f}) = \mathbb{E}_{(\vec{x}, y)\sim \prob} \mathbb{P}\left(\sign(\hat{f}(\vec{x}))\neq y\right)\, . \label{eq:generror}
\end{equation}
For $\hat{f}_{\erm}$, another commonly used metric is the \emph{test loss} : 
\begin{equation}
    \mathcal{L}_{\rm{gen.}}(\hat{f}) = - \mathbb{E}_{(\vec{x}, y)\sim \prob} \log(\sigma(y \hat{f}(\vec x)))\, .\label{eq:genloss}
\end{equation}
However, our key goal in this manuscript is to mathematically characterize the uncertainty associated to the prediction of the different classifiers above, and in particular how they correlate with the true class uncertainty as measured by $f_{\star}$. Mathematically, this can be measured by the following joint density:
\begin{align}
    \rho_{\star, t}(a,b) \coloneqq \mathbb{E}_{\mathcal{D}}\mathbb{P}_{\vec{x}}\left(f_{\star}(\vec{x}) = a, \hat{f}_{t}(\vec{x})=b\right)
    \label{eq:definition_rho}
\end{align}
\noindent where $(a,b)\in [0,1]^2$ and  $\hat{f}_{t}$, $t\in\{\bo,\erm,\lap, \empbayes\}$ can be any of the classifiers defined above, and the expectation is taken both over the training data $\mathcal{D} = \{(\vec{x}^{\mu}, y^{\mu})\}_{\mu\in[n]}$. In particular, this joint density gives access to different notions used in the literature to quantify uncertainty. For instance, a widely studied notion is the \emph{calibration at level $\ell \in [0,1]$} of a classifier $\hat{f}$:
\begin{align}
    \Delta_{\ell}(\hat{f}) \coloneqq \ell - \mathbb{E}_{\vec{x},\mathcal{D}}\left[f_{\star}(\vec{x})|\hat{f}(\vec{x}) = \ell \right] \, .
\end{align} 

A related metric is the \emph{Expected Calibration Error} (ECE):
\begin{equation}
\label{eq:ece}
    \ECE(\hat{f}) \coloneqq \mathbb{E}_{\mathbf{x}} \left[ | \Delta_{\hat{f}(\mathbf{x})}|  \right] \, .
\end{equation}

Note that in this work we focus on the calibration. Other uncertainty quantification metrics exist in the literature, e.g. the \textit{Brier Score} and the \textit{Maximum Calibration Error}, and although the theoretical methods presented here can be readily adapted to characterize their asymptotics, this is outside of the scope of this work.

\subsection{The random features model}
\label{sec:rf}
Following our aim to investigate the interplay between overparametrization and uncertainty, we will focus on one of the simplest settings of feature maps defined by two-layer neural networks $\vec{\varphi}: \vec{x}\in\mathcal{X}\subset\mathbb{R}^{p} \mapsto \phi(F\vec{x}) / \sqrt{p}$ with weights $F\in\mathbb{R}^{p\times d}$ and component-wise activation~$\phi$. We will consider \emph{random features} \cite{Rahimi2007}, where the first layer weights $F\in\mathbb{R}^{p\times d}$ are fixed at initialization, typically taken to be i.i.d. standard Gaussian. Random features have been widely studied as a convex proxy for investigating the impact of overparametrization in generalization, since they were shown to display the characteristic non-monotonic \emph{double descent} behaviour of the generalization error \cite{Belkin2019, Spigler_2019}, with optimal generalization achieved beyond interpolation of the data \cite{mei_generalization_2022, gerace_generalisation_2020, Ascoli2020}, also known as \emph{benign overfitting} \cite{Bartlett2020}. 

We will assume Gaussian input data $\vec{x}^{\mu}\sim\mathcal{N}(\vec{0},\sfrac{1}{d}\mat{I}_{d})$ with labels drawn from a logit model:
\begin{align}
\label{eq:def_datamodel}
    f_{\star}(\vec{x}) &= \int_{\mathbb{R}} \sigma\left(\vec{\theta}_{\star}^{\top}\vec{x} + \noisestr z \right) \mathcal{N}(z | 0, 1) \dd z
\end{align}
\noindent with random weights $\vec{\theta}_{\star} \sim \mathcal{N}(\vec{0},\mat{I}_{d})$ and $\noisestr\geq 0$ defines a tunable label noise level. This completely specifies the data distribution $\nu$. In the following, we will be interested in the \emph{proportional high-dimensional limit} defined by $n,p,d \to\infty$ with fixed ratios $\alpha \coloneqq \sfrac{n}{p}$ and $\gamma\coloneqq \sfrac{p}{d}$. While this \textit{teacher-student} setup is quite common in high-dimensional statistics, we could make it more realistic by assuming a general covariance $\Psi$ for the input $\vec{x} \sim \mathcal{N}(\vec{0}, \sfrac{1}{d} \Psi )$. 

An asymptotic characterization of the generalization and training errors of empirical risk minimization for the random features model in the proportional limit was derived for ridge regression in \cite{mei_generalization_2022} and generalized to convex losses in \cite{gerace_generalisation_2020}. A key ingredient in this analysis is a \emph{Gaussian equivalence principle} \cite{Goldt2020, goldt_gaussian_2021} proven in \cite{Hu2020,Montanari2022} stating that the statistics of the empirical risk minimizer is asymptotically equal to the one of an equivalent Gaussian problem with matching moments. More recently, Gaussian equivalence has been proven for two-layer neural tangent features in \cite{Montanari2022} and features coming from mixture models in \cite{gerace2022gaussian}, and was conjectured to hold for a broader class including features coming from trained neural networks \cite{loureiro_learning_2021}. Although the discussion in this manuscript focus in the random features case, our analysis can be readily extended to all cases in which Gaussian equivalence holds. We provide in Appendix~\ref{sec:app:derivation} an extension of our main theoretical result to a general Gaussian covariate model with convex loss encompassing all these cases.

\section{Results}

\subsection{Technical results}
\label{sec:techres}
Let $\hat{\mu}_{p}$ denote the empirical spectral distribution of the matrix $\mat{F}\mat{F}^{\top}\in\mathbb{R}^{p\times p}$. In the following, we assume that in the proportional high-dimensional limit defined above, $\hat{\mu}_{p}$ weakly converges to an asymptotic spectral distribution $\mu$ on $\mathbb{R}_{+}$ with normalized second moment $\int\mu(\dd x)x^2 = 1$. Further, assume $\kappa_{0} = \mathbb{E}[\phi(z)]$, $\kappa_{1} =\mathbb{E}[z\phi(z)]$ and $\kappa_{\star}^2 = \mathbb{E}[\phi(z)^2]-\kappa_{1}^2-\kappa_{0}^2$ are all finite for $z\sim\mathcal{N}(0,1)$. Note that this assumption simply implies some mild regularity in the activation (e.g. that is does not grow too fast). All the commonly considered activations functions, e.g. ReLU, tanh, sigmoid, etc., satisfy these assumptions. Also, for simplicity of exposition, in the following we assume $\kappa_{0}=0$, which can always be obtained by letting $\phi\to \phi-\kappa_{0}$. 

We also define the effective noise $\tau^2_{\rm add} = 1 - \mathbb{E}_{x\sim\mu}\left[\frac{\kappa_1^2 x}{\kappa_1^2 x + \kappa_{\star}^2}\right]$.

The first step is to characterize the density $\rho_{\star, t}$ with $t \in \{ \bo, \erm, \lap, \empbayes \}$ defined in eq.~\eqref{eq:definition_rho}. All relevant quantities depend on this density. In the asymptotic regime, the estimator $\hat{f}(\vec{x})$ is characterized by six quantities $(m, q, v, \hat{m}, \hat{q}, \hat{v})$ that are solutions of self-consistent equations.

\begin{table*}[ht!]
    \begin{center}
    \begin{tabular}{cccc}
    \textbf{Classifier}  & $g_{t}(y, \omega, v)$ & $\eigfunction_{t}(x)$ & $\hat{\tau}_t$ \\
    \hline \\
    $\hat{f}_{\erm}$        & $\prox_{\log \sigma(y \times \cdot)}(\omega)$ & $\lambda $  & 0 \\
    $\hat{f}_{\lap}$        & $\prox_{\log \sigma(y \times \cdot)}(\omega)$ & $\lambda$  & $\mathbb{E}_{x\sim\mu}\left[\frac{\kappa_1^2 x + \kappa_{\star}}{\lambda + \hat{v}_{\star}(\kappa_1^2 x + \kappa_{\star})}\right]$ \\
    $\hat{f}_{\empbayes}$   & $\partial \omega \log \int \sigma(\beta y \times z) \mathcal{N}(z | \omega, v) {\rm d}z$ & $\lambda$ & $v^{\star}$ \\
    $\hat{f}_{\bo}$         & $\partial \omega \log \int \sigma_{\noisevar + \addnoisestr^2}(y \times z)  \mathcal{N}(z | \omega, v) {\rm d}z$ & $\frac{\kappa_1^2 x}{(\kappa_1^2 x + \kappa_{\star}^2)^2}$ & $v^{\star} + \noisevar + \addnoisestr^2$ 
    \end{tabular}
    \end{center}
    \caption{Auxiliary functions and value of $\hat{\tau}_t$ for the different classifiers defined in Sec.~\ref{sec:classifiers}.}  
\label{table:channel_denoising_functions}
\end{table*}

\begin{theorem}[Joint density]
\label{thm:jointstats}
    Let $\mathcal{D}=\{(\vec{x}^{\mu},y^{\mu})\}_{\mu=1}^{n}$ denote data independently drawn from the model defined in Equation \eqref{eq:def_datamodel}. Consider $\hat{f}_t$, $t \in \{ \bo, \erm, \lap, \empbayes \}$ one of the classifiers defined in Sec.~\ref{sec:classifiers}. Then, in the proportional high-dimensional limit where $n, d, p \to \infty$ with fixed $\alpha = \sfrac{n}{p}, \gamma = \sfrac{p}{d}$, the asymptotic joint density $\rho_{\star, t}$ defined in Equation~\eqref{eq:definition_rho} is given by  $\rho^{\rm lim}_{\star, t}(a, b)=\lim\limits_{p\to\infty}\rho_{\star, t}(a, b)$:
    \begin{align}
        \rho^{\rm lim}_{\star, t}(a, b) = \frac{\mathcal{N}\left( \begin{bmatrix} \sigma_{\noisestr^2 + \addnoisestr^2}^{-1}(a) \\ \sigma_{\hat{\tau}^2_t}^{-1}(b) \end{bmatrix}  \Big| \mathbf{0}_2, \Sigma_{t} \right)}{ | \sigma_{\noisestr^2 + \addnoisestr^2}' ( \sigma_{\noisestr^2 + \addnoisestr^2}^{-1}(a)) | 
        | \sigma_{\hat{\tau}^2_t}' ( \sigma_{\hat{\tau}^2_t}^{-1}(b)) |  }  \label{eq:res:jointdensity}
    \end{align}
    \noindent where 
\begin{equation}
    \Sigma_{t} = \begin{bmatrix} 1 & m_{t}^{\star} \\ m_{t}^{\star} & q_{t}^{\star} \end{bmatrix}
    \label{eq:def_sigma}
\end{equation}
and the sufficient statistics $(m_{t}^{\star}, q_{t}^{\star}, v_{t}^{\star})\in\mathbb{R}^{3}$ are the unique fixed points of the following system of equations:
\begin{align}
    \begin{cases}
        v = 2 \times \partial_{\hat{q}} \Psi_w(\hat{m}, \hat{q}, \hat{v} ;  \eigfunction_{t}) \\
        q = 2 \times (\partial_{\hat{q}} \Psi_w - \partial_{\hat{v}} \Psi_w)(\hat{m}, \hat{q}, \hat{v} ;  \eigfunction_{t}) \\
        m = \sqrt{\gamma} \partial_{\hat{m}} \Psi_w(\hat{m}, \hat{q}, \hat{v} ;  \eigfunction_{t}) 
    \end{cases}
    \label{eq:state_evolution_erm}
\end{align}
\begin{align*}
    \begin{cases}
        \hat{v}\!\!\!\!&=\!\! - \alpha \mathbb{E}_{\xi \sim \mathcal{N}(0, q)}\!\left[ \sum_y \mathcal{Z}_0 \left( y, \sfrac{m}{q}\xi, v_{\star} \right)  \partial_{\omega} g_{t} \left( y, \xi, v \right) \right] \\
        \hat{q} \!\!\!\!&=\!\! \alpha \mathbb{E}_{\xi \sim \mathcal{N}(0, q)} \! \left[ \sum_y \mathcal{Z}_0 \left( y, \sfrac{m}{q}\xi, v_{\star} \right) g_{t}\left( y, \xi, v \right)^2 \right] \\
        \hat{m}\!\!\!\! &=\!\! \sqrt{\gamma} \alpha \mathbb{E}_{\xi \sim \mathcal{N}(0, q)} \! \left[ \sum_y \partial_{\omega} \mathcal{Z}_0 \left( y, \sfrac{m}{q}\xi, v_{\star} \right) g_{t}\left( y, \xi, v \right) \right]    
    \end{cases}
\end{align*}
where $\mathcal{Z}_0(y, \omega, v) = \sigma_{v + \noisestr^2 + \tau^2_{\rm add}}(y\omega)$, $v_{\star} = 1 - \sfrac{m^2}{q} - \tau^2_{\rm add}$. The functions $g_{t}$ and $\eigfunction_{t}$ and the scalar $\hat{\tau}_t$ depend on the estimator and the sufficient statistics, and are given in Table~\ref{table:channel_denoising_functions}. Also : 
\begin{align}
    \Psi_w(\hat{m}, \hat{q}, \hat{v}; \eigfunction) &= \frac{1}{2} \mathbb{E}_{x\sim\mu}\left[\frac{\hat{m} \kappa_1^2 x + \hat{q} (\kappa_1^2 x + \kappa_{\star}^2)}{\eigfunction(x) + \hat{v}(\kappa_1^2 x + \kappa_{\star}^2)}\right] \notag \\
    &\quad- \frac{1}{2}\log (\eigfunction(x) + \hat{v} (\kappa_1^2 x + \kappa_{\star}^2)) \, .
    \label{eq:def_psi_w}
\end{align}
\end{theorem}
\noindent \paragraph{Proof idea:} Let $\vec{x}\!\sim\!\mathcal{N}(\vec{0},\sfrac{1}{d}\mat{I}_{d})$. For any of the classifiers $t\! \in\! \{ \bo, \erm, \lap, \empbayes \}$ from Sec.~\ref{sec:classifiers}, the 2-dimensional vector $(f_{\star}(\vec{x}), \hat{f}_{t}(\vec{x}))$ is asymptotically distributed as $(\sigma(z), \sigma_{\Tilde{v}}(z_{t}'))$ for some $\Tilde{v}$ that depends on the estimator, where $(z, z_{t}') \sim \mathcal{N}(\mathbf{0}_2, \Sigma_{t})$, and

\begin{align*}
    \Sigma_{t} = \frac{1}{d}\begin{pmatrix}
        ||\vec{\theta}_{\star}||_{2}^{2} & \hat{\vec{\theta}}_{t}^{\top}\Phi\vec{\theta}_{\star} \\
        \vec{\theta}_{\star}^{\top}\Phi^{\top}\hat{\vec{\theta}}_{t} & \hat{\vec{\theta}}_{t}^{\top}\Omega\hat{\vec{\theta}}_{t}
    \end{pmatrix}
\end{align*}
\noindent where we defined the shorthand $\Phi = \kappa_{1}\mat{F}\in\mathbb{R}^{p\times d}$ and $\Omega = \kappa^2_{1}\mat{F}\mat{F}^{\top}+\kappa^{2}_{\star}\mat{I}_{p}$ and $\hat{\vec{\theta}}_{t}$ is either the unique minimizer of the empirical risk in eq.~\eqref{def:risk} for $t\in\{\erm, \lap\}$ or the mean over the respective posterior distribution for $t\in\{\bo,\empbayes\}$. The computation of $\rho_{\star, t}$ thus boils down to computing the sufficient statistics $(m^{\star}, q^{\star}):=(\hat{\vec{\theta}}_{t}^{\top}\Phi\vec{\theta}_{\star}, \hat{\vec{\theta}}_{t}^{\top}\Omega\hat{\vec{\theta}}_{t})$. For $\hat{f}_{\erm}$ on the random features model, the theorem can be proven using recent work in high-dimensional statistics \cite{mei_generalization_2022, Dhifallah2020,loureiro_learning_2021}, where $(m^{\star}, q^{\star})$ is proven to asymptotically obey a set of self-consistent "state-evolution" equations  \cite{Gerbelot21,bayati2011dynamics,donoho_high_2013}, mathematically equivalent to eqs.~\eqref{eq:state_evolution_erm}. A similar strategy was used in \cite{clarte_theoretical_2022} for the simpler vanilla logistic model. This is discussed in Appendix \ref{sec:app:replicas} where we show how to derive analogous results for $t \in \{ \bo, \lap, \empbayes \}$.\looseness=-1

\begin{corollary}[Test error and calibration] Under the conditions of Theorem \ref{thm:jointstats}, the asymptotic generalization error and calibration are given by: 
\begin{align}
\mathcal{E}^{\rm lim}_{{\rm gen.}} &= \iint_{b < 0.5, a} a \times \rho^{\rm lim}_{\star, t}(a, b) {\rm d} a {\rm d} b \notag \\
                         &+ \iint_{b > 0.5, a} (1 - a) \times \rho^{\rm lim}_{\star, t}(a, b) {\rm d} a {\rm d} b \\
\Delta^{\rm lim}_{p} &= p - \frac{\int  a \times \rho^{\rm lim}_{\star, t}(a, p) {\rm d} a}{\int \rho^{\rm lim}_{\star, t}(a, p){\rm d} a}\, .
\label{eq:asympcal}
\end{align}
\end{corollary}

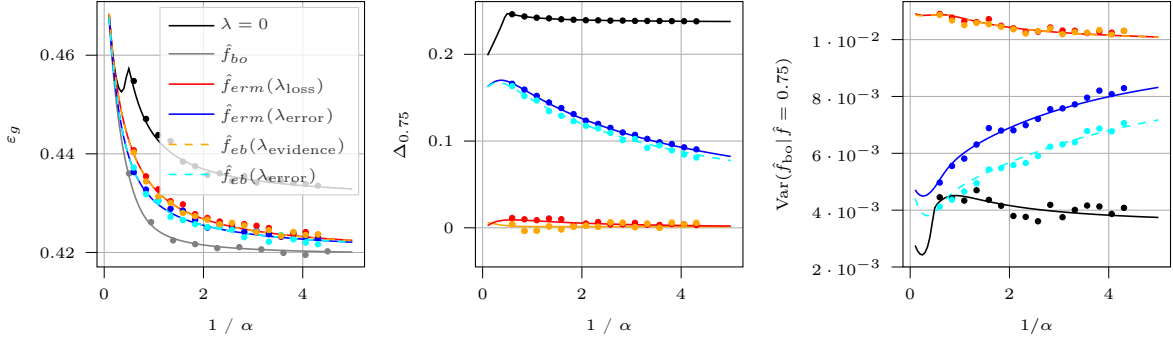
\begin{figure*}[t!]
    \centering
    \def\figwidth{0.3\columnwidth}
    \def\figheight{0.3\columnwidth}
  
    \input{Figures/exp_test_errors_n_over_p=2.0}
    \input{Figures/exp_calibration_p=0.75}
    \input{Figures/conditional_variance/exp_conditional_variance_bo_p=0.75_n_over_p=2.0.tex}

    \caption{(\textbf{Left}) Test errors of the different methods as a function of the number of parameter per sample $\sfrac{p}{n}$.  ERM, and Empirical Bayes (EB) are used with different penalizations. Here we use a logit teacher with $\sfrac{n}{d} = 2.0, \noisestr = \sfrac{1}{2}$ and \texttt{erf} activation.  The curves $\hat{f}_{\empbayes}(\lambdaerror)$ and $\hat{f}_{\erm}(\lambdaerror)$ are very close and indistinguishable on the plot, as well as the curves $\hat{f}_{\empbayes}(\lambdaevidence)$ and $\hat{f}_{\erm}(\lambdaloss)$. Due to the intrinsic noise in the model the oracle error is $\mathcal{E}^{\star}_{gen.} \simeq 0.332$. (\textbf{Center}) Calibration at a level $\ell = 0.75$. (\textbf{Right}) Variance of $\hat{f}_{\bo}$ conditioned on the different other estimators. Points are experimental values obtained on Gaussian data at $d = 200$, averaged over 30 trials.}
    \label{fig:test_errors_n_over_p=2.0}    
\end{figure*}

\paragraph{Intuition of the technical results --} The key intuition behind Theorem \ref{thm:jointstats} is the fact that for the models considered here all the statistics of interest depend only on low-dimensional projections of the estimators and the features, i.e. the joint distribution of $\hat{\vec{\theta}}_{t}^{\top} \varphi (\vec{x})$ and $\vec{\theta}_{\star}^{\top} \vec{x}$. Even if the input data $\vec{x}$ is assumed Gaussian, the distribution of the features $\varphi(\vec{x})$ can be complicated. However, thanks to recent universality results in the high-dimensional statistics literature \cite{goldt_gaussian_2021, Hu2020, Montanari2022, dandi_universality_2023}, in the high-dimensional limit of interest here the joint distribution of these projections are asymptotically captured by a Gaussian model with matching second moments $(m_{t}^{\star},q_{t}^{\star})$, see Appendix \ref{app:eq_model} for a detailed discussion. Moreover, these moments (which are the sufficient statistics for the quantities of interest) can be explicitly computed from the state evolution \eqref{eq:state_evolution_erm} of a tailored message passing scheme for each of the estimators $\hat{\vec{\theta}}_{t}$, see Appendix \ref{sec:app:derivation} for the technical details. This allow us to fully characterize all the quantities of interest asymptotically.

\subsection{Trade-off between performance and uncertainty}
\label{sec:tradeoff}
In sensitive applications of machine learning having a reliable estimation of the model's uncertainty can be as important as having accurate predictions. Therefore, a key question is \emph{"can my model achieve good generalization while being calibrated?"}.

\paragraph{Comparing the performances:} In Figure \ref{fig:test_errors_n_over_p=2.0} (left) we compare the misclassification test error eq.~\eqref{eq:generror} of the different classifiers defined in Sec.~\ref{sec:classifiers}\footnote{Note that by construction $\mathcal{E}_{\rm{gen.}}(\hat{f}_{\lap}) = \mathcal{E}_{\rm{gen.}}(\hat{f}_{\erm})$.} as a function of the overparametrization ratio $p/n$ at fixed sample complexity $n/d = 2$ for different choices of the hyperparameters $(\beta,\lambda)$. First, note the characteristic double descent behaviour of the empirical risk minimizer with $\lambda \to 0^{+}$, with the peak at the interpolating threshold corresponding in our setting to the existence of linear separator \cite{NIPS2003_0fe47339}. As discussed in e.g. \cite{nakkiran2021optimal} for neural networks and shown in e.g. \cite{gerace_generalisation_2020} for random features classification, this peak is mitigated by cross-validation on the $\ell_2$ regularization $\lambda>0$, which is shown in Fig.~\ref{fig:test_errors_n_over_p=2.0} with the blue and red full lines, corresponding to optimally tuning $\lambda$ to minimize the misclassification error eq.~\eqref{eq:generror} and the test loss respectively eq.~\eqref{eq:genloss}. 

It is interesting to contrast these ERM estimators to the empirical Bayes classifier, which averages over different classifiers. We see that, evaluating the empirical Bayes with a Gaussian prior of variance given by the cross-validated $\lambda_{\rm{error}}$ achieves almost identical performance to the ERM estimator, with a difference of the order of $10^{-5}$.

An often quoted strength of the Bayesian approach is that model selection can be performed directly on the training data by evidence maximization over the model hyperparameters \cite{MacKay1996}. Curiously, in our setting this yields a very close performance to ERM cross-validated with respect to the test loss, as shown in Fig.~\ref{fig:test_errors_n_over_p=2.0} (left) in dashed yellow line. Despite achieving similar performances in our setting, it is important to stress that these two classifiers are computationally radically different, as the empirical Bayes classifier requires sampling from a high-dimensional distribution which can be prohibitive in practice. These should be contrasted with the Bayes-optimal classifier, shown in solid grey, which by definition gives the best achievable performance at fixed data availability. 

To summarise, from the point-of-view of the performance we observe no significant difference between Bayesian and ERM estimators, with (not surprisingly) best performance achieved by cross-validating over the misclassification error. 

\paragraph{Calibration:} Despite the relatively small difference in performance, the discussed classifiers are rather different in terms of calibration. Figure~\ref{fig:test_errors_n_over_p=2.0} (center) shows the calibration at fixed level $\ell=0.75$ for the same classifiers. Note that the max-margin interpolator $\lambda\to 0^{+}$ produces consistently overconfident predictions. Indeed, we observe a maximum in the calibration curve around the interpolation threshold reminiscent of the double descent behaviour, with worst possible calibration $\Delta_{\ell} = \ell-\sfrac{1}{2}$ corresponding to a confidence completely uncorrelated with the true class probabilities achieved at the interpolation transition. 
As noted in \cite{bai_dont_2021}, overconfidence is inherent for unregularized logistic regression in high-dimensions, as it is present even when data is abundant with respect to the number of parameters. However, in their simpler setting of matched linear classifiers the number of parameters is equal to the input dimension $p=d$, and therefore overparametrization cannot be distinguished from high-dimensionality. Indeed, they observe an asymptotic scaling of the calibration $\Delta_{\ell} \sim d/n$, which suggests that overconfidence increases with the number of parameters. Our setting allow us to decouple the number of parameters $p$ from the data dimension $d$, suggesting instead that overparametrization can improve calibration at fixed number of samples.

More strikingly, we observe that optimal regularization does not mitigate this double descent-like behaviour in the calibration, which is in contrast with what happens with the error itself that becomes monotonic when optimally regularized. Indeed, while cross-validating with respect to the misclassification error achieves the best accuracy, it produces consistently overconfident predictions for both the empirical risk minimizer and the empirical Bayes classifiers. On the other hand, cross-validation with respect to the loss produces better calibrated estimates, with an interesting non-monotonic behaviour crossing from over- to underconfidence as a function of overparametrization. In contrast, maximising the evidence yields better calibrated estimation with a monotonic calibration curve very close to zero.

To summarise, we observe a fundamental trade-off between optimising the accuracy of classification and obtaining calibrated classifiers. A similar discussion holds for other calibration levels and for the expected calibration error eq.~\eqref{eq:ece}, as shown in Appendix~\ref{app:other_settings}. 

\paragraph{Conditional variance:} Theorem \ref{thm:jointstats} gives us access to a rich set of uncertainty measures, of which the calibration is a particular example. For instance, we have access to the full distribution of the Bayes-optimal classifier $\hat{f}_{\bo}$ conditioned on the predictors defined in Sec.~\ref{sec:classifiers}. Note that since $\mathbb{E}(\hat{f}_{\bo} | \hat{f} = \ell) = \mathbb{E}(f_{\star} | \hat{f}=\ell) = \ell-\Delta_{\ell}(\hat{f})$, the mean of this conditional distribution is equal to the calibration up to a constant. A natural measure of uncertainty beyond the calibration is the variance of this conditional distribution $\variance(\hat{f}_{\bo} | \hat{f}=\ell)$, which quantifies how much the prediction $\hat{f}(\vec{x})=\ell$ inform us on $\hat{f}_{\bo}(\vec{x})$, which is by definition the best achievable classifier at finite availability of data. An explicit expression for this variance can be derived from Theorem \ref{thm:jointstats} for any of the classifiers $t\in\{\erm,\bo,\lap,\empbayes\}$:
\begin{align*}
&\variance(\hat{f}_{\bo}(\vec{x}) | \hat{f}_{t}(\vec{x})=\ell) =  \int {\rm d}a~\sigma_{\hat{v}^{\star}_{\bo} + \noisestr^2 + \tau_{\rm add}^2}(a)^2 \times\notag\\
&\quad\times\mathcal{N} \left( a | \sfrac{m^{\star}_{t}}{q^{\star}_{t}} \sigma_{\hat{\tau}_t}^{-1}(\ell), q^{\star}_{\bo} - \sfrac{{m^{\star}_{t}}^2}{q^{\star}_{t}} \right)-\left(\ell-\Delta_{\ell}\right)^2
\end{align*}
\noindent where $(m^{\star}_{t}, q_{t}^{\star}, q^{\star}_{\bo})$ are solutions to the self-consistent equations eq.~\eqref{eq:state_evolution_erm}  and $\Delta_{\ell}$ is the asymptotic calibration eq.~\eqref{eq:asympcal}. The detailed derivations are shown in Appendix~\ref{app:cond_variance}.

Figure~\ref{fig:test_errors_n_over_p=2.0} (right) shows this conditional variance as a function of the overparametrization. Note that in this setting worse calibration is correlated with a lower conditional variance, and we observe a trade-off between these two metrics. We also observe a behaviour reminiscent of double descent in the value of the conditional variance that does not go away with optimal regularization.

\begin{figure*}[t!]
    \centering
    \def\figwidth{0.3\columnwidth}
    \def\figheight{0.3\columnwidth}
    
    \input{Figures/exp_calibration_p=0.75_laplace}
    \input{Figures/temperature_scaling/exp_calibration_temp_scaling_p=0.75_n_over_p=2.0}
    \input{Figures/conditional_variance/exp_conditional_variance_bo_temp_scaling_n_over_p=2.0}

    \caption{(\textbf{Left}) Calibration at level $\ell=0.75$ of $\hat{f}_{\erm}$ (solid lines, refer to Figure~\ref{fig:test_errors_n_over_p=2.0} for the legend) and $\hat{f}_{\lap}$ with the three different regularizations. (\textbf{Center}) Calibration at level $\ell=0.75$ of $\hat{f}_{\erm}$ after temperature scaling (TS), compared to $\hat{f}_{\empbayes}$ (dashed yellow) and $\hat{f}_{\rm loss}$ (full red) for reference. (\textbf{Right}) Variance of $\hat{f}_{\bo}$ conditioned on $\hat{f}_{\erm} = 0.75$ after temperature scaling, compared to variance at $\hat{f}_{\rm loss}$ (full red) and $\hat{f}_{\empbayes}$ (dashed yellow). Points are experimental values obtained on Gaussian data at $d = 200$, averaged over 30 trials.} 
    \label{fig:calibration_laplace_n_over_p=2.0}
\end{figure*}
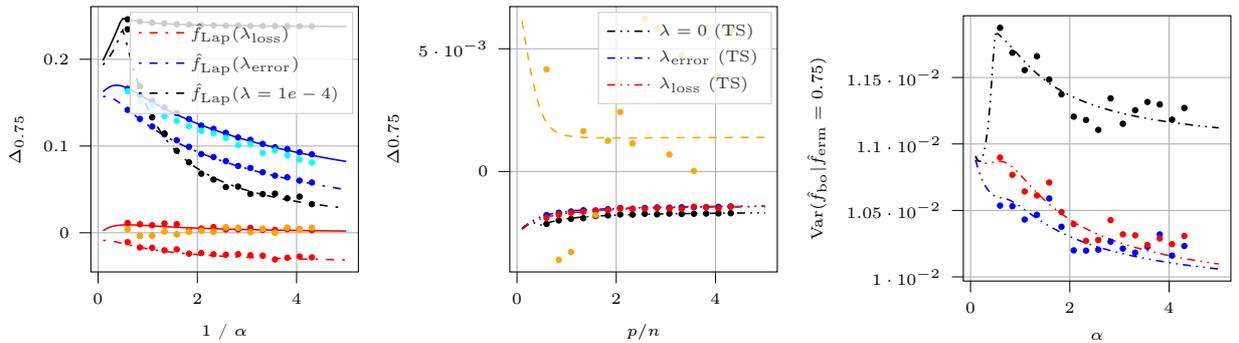
\subsection{Temperature scaling} 
Temperature scaling is a calibration method introduced in \cite{guo_calibration_2017} to mitigate overconfidence in trained neural networks. It is applied after training, and consists in introducing a "temperature" scaling parameter on the last layer pre-activations $\hat{f}_{\erm}(\vec{x}) \!=\! \sigma(\hat{\vec{\theta}}_{\erm}^{\top} \vec{\varphi}(\vec{x}) / T)$. It is then tuned to minimize the validation loss. In our analysis, this corresponds to simply re-scaling the predictor $\hat{\vec{\theta}}_{\erm}\!\to \!\hat{\vec{\theta}}_{\erm}/T$, \&  Thm. \ref{thm:jointstats} thus applies mutatis mutandis. 

Figure \ref{fig:calibration_laplace_n_over_p=2.0} (center) compares the calibration at level $\ell = 0.75$ of the regularized empirical risk minimizers with $\lambdaloss$ and $\lambdaerror$ after temperature scaling with the empirical Bayes classifier with $\lambdaevidence$ and ERM classifier at $\lambda_{\rm loss}$, the best calibrated in our setting so far. We observe that the temperature scaling yields very similar calibrations for $\lambdaloss$ and $\lambdaerror$. While empirical Bayes remains the best calibrated estimator, temperature scaling has a calibration around 0.1\%, which would be satisfying in most practical scenarios. We also observe that the maximum around the interpolation threshold is not present in the calibration curves after temperature scaling. 

Looking at the variance of $\hat{f}_{\bo}$ conditioned on $\hat{f}_{\erm}$ after temperature scaling we see that it is lower for $\lambdaerror$ than for $\lambdaloss$, see Fig.~ \ref{fig:calibration_laplace_n_over_p=2.0} (right). 
We see that again the variance has an increase in the vicinity of the interpolation threshold, reminiscent of the double descent behaviour. As discussed in the previous section, we aim to have the lowest variance possible to ensure that the uncertainty estimation is accurate not only on average but also point-wise.
It appears that cross-validating the empirical risk minimizer on the misclassification error and then applying temperature scaling gives an estimator that both has the best test error and is very well calibrated, both on average and point-wise.
\subsection{The calibration of the Laplace approximation}
Estimating any of the Bayesian classifiers in Sec.~\ref{sec:classifiers} is computationally demanding, since they involve a sampling over a high-dimensional  distribution. This has motivated practitioners to develop different approximations for making Bayesian methods more efficient. These include Bayesian dropout \cite{gal_dropout_2016}, deep ensembles \cite{NIPS2017_9ef2ed4b}, stochastic gradient Langevin dynamics \cite{welling_bayesian_2011} and the Laplace approximation \cite{ritter_scalable_2018, daxberger_laplace_2021}, among others. The Laplace approximation was introduced by \cite{10.1162/neco.1992.4.3.415} in the context of Gaussian processes, and consists of approximating the posterior distribution by a Gaussian density centred around the empirical risk minimizer - or equivalently to a low-temperature expansion of the posterior - see eq.~\eqref{eq:laplace}. By construction, the Laplace classifier has the same misclassification error as the empirical risk minimizer, and hence can be effectively seen as endowing this point-estimator with a covariance given by the inverse of the Hessian evaluated at the minimum. Although computing the Hessian of the empirical risk for a deep neural network can be costly, an approximate scheme has been recently proposed \cite{ritter_scalable_2018, daxberger_laplace_2021}, making Laplace a viable uncertainty estimation technique for deep learning.\looseness=-1

On the theory side, sharp results have been limited to the Gaussian process and ridge regression setting, where the Laplace approximation is exact \cite{NIPS1998_5cbdfd0d, NIPS2001_d68a1827}. While exact asymptotic results characterizing the statistics of the logit estimator in high-dimensions abound \cite{sur_modern_2018, gerace_generalisation_2020,aubin_generalization_2020,deng2022model}, to our best knowledge the asymptotic spectral distribution of the Hessian at the minimum is missing. Recently, \cite{NEURIPS2021_a7d8ae45} has computed the asymptotic spectral distribution of the Hessian in a matched logit model under the assumption that the weights are uncorrelated with the input data. Hence, their results do not apply for the empirical risk minimizer, and cannot be used to characterize the uncertainty of the Laplace classifier. Characterizing the Hessian of the logistic risk eq.~\eqref{def:risk} at the minimizer is a challenging technical result that we believe is of independent interest to the scope of the discussion in this manuscript.\looseness=-1

\begin{claim}[Hessian of logit, informal]
\label{thm:hessian}
Let 
\begin{align*}
\mathcal{H}(\vec{\theta}) := \sum_{\mu\in[n]} (\sigma'(y^{\mu}\vec{\theta}^{\top} \vec{\varphi}(\vec{x}^{\mu}))-1)\vec{\varphi}(\vec{x}^{\mu}) \vec{\varphi}(\vec{x}^{\mu})^{\top} + \lambda \mat{I}_p
\end{align*}
denote the Hessian of the logistic empirical risk eq.~\eqref{def:risk}, and denote $\hat{\vec{\theta}}_{\erm}\!=\!\argmin_{\vec{\theta}}\hat{\mathcal{R}}_{n}(\vec{\theta})$ its minimizer. Then, under the same conditions of Thm.~\ref{thm:jointstats} and additional technical assumptions, the following asymptotic characterization holds in the sense of deterministic equivalent:
\begin{align}
\mathcal{H}^{-1}(\hat{\vec{\theta}}_{\erm}) \underset{p\to\infty}{\asymp} (\hat{v}^{\star} \left(\kappa_{1}^2\mat{F}\mat{F}^{\top}+\kappa_{\star}^2\mat{I}_{p}\right) + \lambda I_{p})^{-1}
\end{align}
\noindent where $\hat{v}^{\star}\in\mathbb{R}$ is the solution of the self-consistent  eq.~\eqref{eq:state_evolution_erm}. \looseness=-1
\end{claim}
A heuristic derivation of this result is provided in App. \ref{app:laplace} in the general context of the Gaussian covariate model. With Claim.~\ref{thm:hessian} in hands, we can characterize the asymptotic calibration of the Laplace classifier for our model.
 
 Figure \ref{fig:calibration_laplace_n_over_p=2.0} (left) shows the calibration curve at level $\ell=0.75$, at sample complexity $\sfrac{n}{d}=2$ and noise variance $\tau_{0}=0.5$ as a function of the number of parameters. As mentioned in the introduction, we observe here that $\hat{f}_{\lap}$ is always less confident than $\hat{f}_{\erm}$, due to the concavity of $\sigma$. While this might seem desirable in the scenarios where ERM is very overconfident, e.g. for $\lambda\to 0^{+}$ or $\lambdaerror$, it hurts calibration when the classifier is well-calibrated as for $\lambdaloss$. Moreover, it highly depends on the sample complexity and noise variance, see Appendix~\ref{app:other_settings} in the supplementary material where we show a setting in which the Laplace approximation yields an underconfident classifier even in the $\lambda\to 0^{+}$ at mild overparametrization. Then, the Laplace approximation seems to be an unreliable way to control the calibration of the estimators, contrary to temperature scaling.


\section{Conclusion} 
In this paper, we studied the performance of different frequentist and Bayesian classifiers for random features classification. In the high-dimensional limit, the asymptotic behaviour of these algorithms can be precisely characterized. 
Our first contribution is the derivation of the Bayes-optimal estimator. By definition it is the estimator with the best possible performance, and although it is inaccessible in practice, it provides a baseline to compare the classifiers. Then, we compared the generalization error of frequentist and Bayesian approaches, showing they yield very similar test error.
We then focused on uncertainty quantification, and showed there is a trade-off between generalization and calibration in our model. Moreover, we observed a non-monotonic behaviour of the calibration curve for certain estimators, akin to the famous \textit{double-descent} phenomenon for the test error.
Finally, we compared two popular approaches for post-training calibration: temperature scaling and the Laplace approximation, benchmarking them against the baseline classifiers. In our model, we observe that temperature scaling on top of cross-validating the empirical risk classifier on the accuracy achieves the best result : it has both the lowest test error and best calibration. Moreover, despite requiring a validation set, in practice it is a computationally more efficient method than the Bayesian approach, which requires sampling from a high-dimensional distribution. The code used in this project will be made available at \url{github.com/SPOC-group/double_descent_uncertainty}.

\noindent \textbf{Limitations:} It is worth pointing some limitation of our results. The first resides in the (nevertheless classical) Gaussian assumption for the data. We note, however, that there are good reasons to believe that this can be a very good model in high-dimensions \cite{Hu2020, Montanari2022}.
A second limitation of is the lack of feature learning. While many of the uncertainty quantification methods discussed here apply directly to the last layer of trained neural networks, other methods considered in the literature apply to the full architecture \cite{ABDAR2021243}. Since the performance of deep neural networks can be largely attributed to feature learning, it shall be important to take it into account in theoretical studies of uncertainty. We hope that our work can offer a starting point towards this more ambitious goal.

\section*{Acknowledgements}
We thank Pierre-Alexandre Mattei, Yevgeny Seldin, Anshuk Uppal, Kristoffer Stensbo-Smidt, Simon Bartels and Melih Kandemir for valuable discussions.
We acknowledge funding from the ERC under the European Union’s Horizon 2020 Research and Innovation Program Grant Agreement 714608-SMiLe, the Swiss National Science Foundation grant SNFS OperaGOST, $200021\_200390$ and the \textit{Choose France - CNRS AI Rising Talents} program. This research was supported by the NCCR MARVEL, a National Centre of Competence in Research, funded by the Swiss National Science Foundation (grant number 205602).

\newpage

\bibliography{refs}

\vfill

\renewcommand\thesection{\Alph{section}}
\setcounter{section}{0}
\renewcommand*{\theHsection}{chX.\the\value{section}}

\appendix

\newpage
\section*{Appendix}
\section{Gaussian equivalence}
\label{app:eq_model}
As discussed in the main, our analysis of the random features model introduced in Sec.~\ref{sec:rf} relies on a recent progress in high-dimensional statistics known as the \emph{Gaussian equivalence theorem} (GET). In this Appendix, we recall the reader of the main results in this line of work.

\subsection{Informal discussion and key idea}
\label{sec:app:informal}
For convenience let's first recall the model of interest. Consider data $(\vec{x}^{\mu},y^{\mu})_{\mu\in[n]}\in\mathbb{R}^{d}\times \mathcal{Y}$ which we assume was independently drawn from the following model:
\begin{align}
y^{\mu} = f_{\star}(\vec{\theta}_{\star}^{\top}\vec{x}^{\mu}), && \vec{x}^{\mu}\sim\mathcal{N}(\vec{0},\sfrac{1}{d}\mat{I}_{d}), && \vec{\theta}_{\star}\sim\mathcal{N}(\vec{0},\mat{I}_{d})
\end{align}
\noindent where $f_{\star}:\mathbb{R}\to\mathcal{Y}$ is an activation function, which we assume can be potentially stochastic (as in the logit model studied in the main, Sec.~\ref{sec:rf}). For convenience, define the matrix $\mat{X}\in\mathbb{R}^{n\times d}$ and the vector $\vec{y}\in\mathcal{Y}^{n}$ obtained by stacking together $\vec{x}^{\mu}$ and $\vec{y}^{\mu}$ row-wise. We are interested in studying the following generalized linear predictor:
\begin{align}
\hat{y} = f(\hat{\vec{\theta}}^{\top}\vec{\varphi}(\vec{x}))
\end{align}
\noindent where $\vec{\varphi}:\mathbb{R}^{d}\to\mathbb{R}^{p}$ is a feature map, and $\hat{\vec{\theta}}\in\mathbb{R}^{p}$ are weights, which generally depend on the training data $\hat{\vec{\theta}} = \hat{\vec{\theta}}(\varphi(\mat{X}),\vec{\theta}_{\star})$, where for convenience we defined the feature matrix $\varphi(\mat{X})\in\mathbb{R}^{n\times p}$. The \emph{random features model} correspond to the specific feature map:
\begin{align}
\vec{\varphi}(\vec{x})  = \frac{1}{\sqrt{p}}\phi(\mat{F}\vec{x})
\end{align}
\noindent where $\mat{F}\in\mathbb{R}^{p\times d}$ is a random matrix and $\phi:\mathbb{R}\to\mathbb{R}$ is a component-wise activation function. Our key goal is to characterize the statistics of the model, i.e. to compute expectations over functions of the test and training predictions:
\begin{align}
\label{eq:app:expectations}
\mathbb{E}_{\mat{X},\vec{x},\vec{\theta}_{\star}}\left[\psi\left(f_{\star}(\vec{\theta}_{\star}^{\top}\vec{x}),f(\hat{\vec{\theta}}(\varphi(\mat{X}),\vec{\theta}_{\star})^{\top}\vec{\varphi}(\vec{x}))\right)\right], && \mathbb{E}_{\mat{X},\vec{\theta}_{\star}}\left[\tilde{\psi}\left(f_{\star}(\mat{X}\vec{\theta}_{\star}), f(\vec{\varphi}(\mat{X})\hat{\vec{\theta}}(\varphi(\mat{X}),\vec{\theta}_{\star})\right)\right]
\end{align}
\noindent where $\psi:\mathcal{Y}^{2}\to\mathbb{R}$ and $\tilde{\psi}:\mathcal{Y}^{2n}\to\mathbb{R}$ are test functions. Note in particular that the generalization errors eqs.~\eqref{eq:generror}, \eqref{eq:genloss} and the density eq.~\eqref{eq:definition_rho} are examples of the above. 

Different tools from high-dimensional statistics have been designed to compute such expectations in the limit where $n,p,d\to\infty$ at fixed rates $\alpha = \sfrac{n}{p}$ and $\gamma=\sfrac{d}{p}$, both rigorously and heuristically, e.g. the replica method \cite{mezard1987spin, zdeborova2016statistical}, CGMT \cite{Stojnic2013, Thrampoulidis15, 8365826}, approximate message passing \cite{bayati2011dynamics,donoho_high_2013, Krzakala_2012, Gerbelot21}, cavity / leave-one-out method \cite{mezard2009information, Karoui2013, 8683376}, tools from random matrix theory \cite{Karoui2009, Dobriban2018}, among others. A shortcoming of all the aforementioned methods is that they typically rely on the Gaussianity of the input data, and therefore are not directly applicable to the random features model (note that even if $\mat{F}\vec{x}\in\mathbb{R}^{p}$ is a Gaussian vector, the features $\phi(\mat{F}\vec{x})$ are \emph{not} Gaussian).

Gaussian equivalence provides a surprising answer to this hurdle. Assuming for simplicity that the features are centred $\mathbb{E}_{\vec{x}}[\vec{\varphi}(\vec{x})]=0$ and defining the covariances:
\begin{align}
\Phi = \mathbb{E}_{\vec{x}}[\vec{\varphi}(\vec{x})\vec{x}^{\top}]\in\mathbb{R}^{p\times d}, && \Omega = \mathbb{E}_{\vec{x}}[\vec{\varphi}(\vec{x})\vec{\varphi}(\vec{x})^{\top}]\in\mathbb{R}^{p\times p}
\end{align}
\noindent Gaussian equivalence states that in the high-dimensional limit, the expectations in eq.~\eqref{eq:app:expectations} can be computed for an \emph{equivalent Gaussian model} with matching second moments:
\begin{align}
\label{eq:app:ge}
\mathbb{E}_{\mat{X},\vec{x},\vec{\theta}_{\star}}\left[\psi\left(f_{\star}(\vec{\theta}_{\star}^{\top}\vec{x}),f(\hat{\vec{\theta}}(\varphi(\mat{X}),\vec{\theta}_{\star})^{\top}\vec{\varphi}(\vec{x}))\right)\right]&\xrightarrow[p\to\infty]{}\mathbb{E}_{\mat{X},\mat{V}, \vec{x}, \vec{v},\vec{\theta}_{\star}}\left[\psi\left(f_{\star}(\vec{\theta}_{\star}^{\top}\vec{x}),f(\hat{\vec{\theta}}(\mat{V},\vec{\theta}_{\star})^{\top}\vec{v})\right)\right]\notag\\
\mathbb{E}_{\mat{X},\vec{\theta}_{\star}}\left[\tilde{\psi}\left(f_{\star}(\mat{X}\vec{\theta}_{\star}), f(\vec{\varphi}(\mat{X})\hat{\vec{\theta}}(\varphi(\mat{X}),\vec{\theta}_{\star})\right)\right]&\xrightarrow[p\to\infty]{}\mathbb{E}_{\mat{X},\mat{V},\vec{\theta}_{\star}}\left[\tilde{\psi}\left(f_{\star}(\mat{X}\vec{\theta}_{\star}), f(\mat{V}\hat{\vec{\theta}}(\mat{V},\vec{\theta}_{\star})\right)\right]
\end{align}
\noindent where $(\vec{x}^{\mu},\vec{v}^{\mu})_{\mu\in[n]}$ are $n$ independent samples of jointly Gaussian random variables:
\begin{align}
(\vec{x},\vec{v}) \sim\mathcal{N}\left(\vec{0}_{d+p},\begin{bmatrix} \sfrac{\mat{I}_{d}}{d} & \sfrac{\Phi}{\sqrt{pd}}  \\ \sfrac{\Phi^{\top}}{\sqrt{pd}} & \sfrac{\Omega}{p} \end{bmatrix}\right)
\end{align}
\noindent and as before we defined the matrices $\mat{X}\in\mathbb{R}^{n\times d}$ and $\mat{V}\in\mathbb{R}^{n\times p}$ by stacking the samples row-wise. For the random features model $\vec{\varphi}(\vec{x}) = \sfrac{1}{\sqrt{p}}~\phi(\mat{F}\vec{x})$, the asymptotic covariances $\Phi, \Omega$ can be computed explicitly, and are given by:
\begin{align}
\Phi \underset{p\to\infty}{\asymp} \frac{\kappa_{1}}{\sqrt{p}}\mat{F}, && \Omega \underset{p\to\infty}{\asymp} \kappa_{0}^2 \mat{1}_{p}\mat{1}_{p}^{\top}+\frac{\kappa_{1}^2}{p}\mat{F}\mat{F}^{\top}+\kappa_{\star}^2\mat{I}_{p}
\end{align}
\noindent where the constants $(\kappa_{0},\kappa_{1},\kappa_{\star})\in\mathbb{R}^{3}$ are the Gaussian moments of the activation function $\phi$:
\begin{align}
\kappa_{0} = \mathbb{E}_{z\sim\mathcal{N}(0,1)}[\phi(z)], && \kappa_{1} = \mathbb{E}_{z\sim\mathcal{N}(0,1)}[\phi'(z)], && \kappa_{\star}=\sqrt{\mathbb{E}_{z\sim\mathcal{N}(0,1)}[\phi(z)^2]-\kappa_{1}^2-\kappa_{0}^2}.
\end{align}
Therefore, for the random features problem the Gaussian equivalent model can be written explicitly in terms of the input data $\vec{x}\sim\mathcal{N}(\vec{0},\sfrac{1}{d}\mat{I}_{d})$ and the weights $\mat{F}\in\mathbb{R}^{p\times d}$ as: 
\begin{align}
\vec{v} = \kappa_{0}\mat{1}_{p}+\frac{\kappa_{1}}{\sqrt{p}}\mat{F}\vec{x} + \kappa_{\star}\vec{z}
\label{app:get:eq}
\end{align}
\noindent where $\vec{z}\sim\mathcal{N}(\vec{0},\mat{I}_{p})$ is an effective noise vector which is independent from $\mat{F}$, $\vec{x}$ and $\vec{\theta}_{\star}$. In summary, in the high-dimensional limit the statistics of the random features model is equivalent to the statistics of a Gaussian equivalent model with noisy features. The later can be directly characterized by the methods mentioned in the last paragraph.

\subsection{Gaussian equivalence theorem}
\label{sec:app:get}
\paragraph{Related literature:} Gaussian universality has a long history, and appeared in many contexts ranging from random matrix theory \cite{Erdos2009, Erdos2010, Tao2012} to signal processing \cite{Donoho_2009}, statistical learning \cite{NEURIPS2019_dffbb6ef, NIPS2017_136f9513, Montanari2010} and physics \cite{Carmona2004, Chatterjee2005}. In the context of random features, a precursor of the result discussed here is the observation that for Gaussian data the high-dimensional limit of kernel spectra is linearly related the spectrum of the inputs \cite{10.1214/08-AOS648}. This result was generalized to random features kernels in \cite{NIPS2017_0f3d014e, Liao2018}, and leveraged by \cite{mei_generalization_2022} to derive exact asymptotic expressions for the generalization and training error of random features regression. For ridge regression, computing the performance boils down to computing traces of the feature matrix, and therefore Gaussian universality can be seen as an instance of spectral universality of random matrices \cite{Benigni2021}. In non-linear problems where a closed-form solution is not available (as in our classification setting), Gaussian universality for the random features model was shown to hold for the empirical risk minimizer in \cite{gerace_generalisation_2020, goldt_gaussian_2021}, and was later proven in \cite{Hu2020, Dhifallah2020}. More recently, \cite{Montanari2022} extended this proof to more general empirical risk minimization problems, and showed the universality of the free energy density associated to the empirical risk at finite temperature.
Finally, \cite{dandi_universality_2023} generalized this free energy universality result and proved universal weak convergence for both empirical risk minimizers and sampling from the empirical Bayes measure. In particular, this result covers the Bayesian classifier and the uncertainty quantification metrics (calibration, ECE, etc.) studied here. This version is better suited to our discussion, since it is closer to the Bayesian classifiers studied in the main.

\begin{theorem}[Lemma 1 from \cite{Montanari2022}]
Consider the random features model discussed in Sec.~\ref{sec:app:informal}. Assume that the activation $\phi$ is three times differentiable and has zero Gaussian mean $\kappa_{0}=\mathbb{E}_{z\sim\mathcal{N}(0,1)}[\phi(z)]=0$ and that the weight matrix $\mat{F}\in\mathbb{R}^{p\times d}$ has rows $\vec{f}_{i}\sim\mathcal{N}(\vec{0},\mat{I}_{d})$ for $i\in [p]$. Further, assume that the function $f_{\star}$ is Lipschitz with i.i.d. bounded sub-Gaussian noise.  Define the free energy density at inverse temperature $\beta>0$:
\begin{align}
f_{\beta}(\varphi({X})) = -\frac{1}{p\beta}\log\int_{\mathbb{R}^{p}}\dd\vec{\theta} ~\exp\left\{-\beta\left[-\sum\limits_{\mu=1}^{n}\log\sigma\left(f_{\star}(\vec{\theta}_{\star}^{\top}\vec{x}^{\mu})\times \vec{\theta}^{\top}\vec{\varphi}(\vec{x}^{\mu}))\right)+\frac{\lambda}{2}||\vec{\theta}||^{2}_{2}\right]\right\}.
\end{align}
Then for any bounded differentiable function $\psi$ with Lipschitz derivative we have:
\begin{align}
\lim\limits_{p\to\infty}\left|\mathbb{E}\left[\psi\left(f_{\beta}(\varphi({X})\right)\right] - \mathbb{E}\left[\psi\left(f_{\beta}(\mat{V})\right)\right]\right| = 0
\end{align}
\end{theorem}
We refer the reader to \cite{Montanari2022} for the technical details on the proof of this result.

\subsection{Beyond random features}
\label{sec:app:gcm}
The Gaussian equivalence theorem for the random features model motivates the study of generalized linear models with general Gaussian covariates. For instance, consider $n$ independently drawn Gaussian covariates $(\vec{u}^{\mu},\vec{v}^{\mu})\in\mathbb{R}^{d+p}$:
\begin{align}
\begin{pmatrix}
        \vec{u} \\ \vec{v}
    \end{pmatrix}
    \sim \mathcal{N}\left(\vec{0}_{d+p},\begin{bmatrix}
        \Psi & \Phi \\ \Phi^{\top} & \Omega
    \end{bmatrix}\right)
\end{align}
\noindent for positive definite matrices $\Psi\in\mathbb{R}^{d\times d}$, $\Omega\in\mathbb{R}^{p\times p}$ and $\Phi\in\mathbb{R}^{p\times d}$ such that $\Psi - \Phi\Omega^{-1}\Phi^{\top}$ is invertible. Labels $y^{\mu}\in\mathcal{Y}$ are generated from the covariate $\vec{u}\in\mathbb{R}^{p}$ from a generalized linear model:
\begin{align}
y^{\mu} = f_{\star}(\sfrac{\vec{\theta}_{\star}^{\top}\vec{u}^{\mu}}{\sqrt{d}}), && \vec{\theta}_{\star}\sim\mathcal{N}(0,\mat{I}_{d}).
\end{align}
However, the statistician only observes the pairs $(\vec{v}^{\mu}, y^{\mu})\in\mathbb{R}^{p}\times \mathcal{Y}$, from which she tries to learn:
\begin{align}
\hat{y} = f({\hat{\vec{\theta}}^{\top}\vec{v}}/{\sqrt{p}}).
\end{align}
The asymptotic statistics of this Gaussian covariate model has been derived and proven \cite{loureiro_learning_2021} for the particular case in which $\hat{\vec{\theta}}$ is the empirical risk minimizer. In Appendix \ref{sec:app:derivation}, we recover and generalize this result to the other estimators defined in Sec.~\ref{sec:classifiers}. 

Note that thanks to Gaussian equivalence, in the proportional high-dimensional limit, the random features model discussed in Sec.~\ref{sec:app:informal} is a particular case of this Gaussian covariate model where $\vec{u} = \vec{x}$ and $\vec{v} = \vec{\varphi}(\vec{x})$. However, the Gaussian covariate model can accommodate a richer class of models. For instance, one could consider the case in which the target covariates themselves come from a feature map: $\vec{u}=\vec{\varphi}_{\star}(\vec{x})$. Although Gaussian equivalence has only been established for a limited number of feature maps, \cite{loureiro_learning_2021} has empirically observed that the asymptotic formulas derived for Gaussian covariates are in good agreement with a rich class of feature maps, including case in which the fixed features are learned from neural networks. While the goal of this work \emph{is not} to investigate Gaussian equivalence, this line of work motivate us to derive our result for general Gaussian covariates, hence making them readily applicable to equivalences proven in the future.

\newpage
\section{Derivation of theorem \ref{thm:jointstats}}
\label{sec:app:derivation}
In this Appendix we provide a derivation of the self-consistent equations \eqref{eq:state_evolution_erm} characterizing the sufficient statistics $(m^{\star}_t, q^{\star}_t, v^{\star}_t, \hat{m}^{\star}_t, \hat{q}^{\star}_t, \hat{v}^{\star}_t)$ for $t\in\{\bo,\erm,\empbayes,\lap\}$. As motivated in Appendix \ref{app:eq_model}, our discussion will focus on the more general Gaussian covariate model, which contains the random features setting as a particular case. The key idea is to design an approximate message passing and show that the associated state evolution equations coincide exactly the with the self-consistent equations for the sufficient statistics in Theorem \ref{thm:jointstats}.

This Appendix is organized as follows. We start by a recap of the Gaussian covariate model for our specific setting in Sec.~\ref{sec:app:recap}, and introduce a convenient change of variables. Next, in Sec.~\ref{sec:app:gamp} we introduced a tailored message passing algorithm, and provide an informal derivation of the associated state evolution equations. In Sec.~\ref{sec:app:replicas} we provide a heuristic derivation of the self-consistent equations from the replica method, and show that it agrees with the state evolution equations for our algorithm. In Sec.~\ref{sec:app:rigor} we discuss how these equations are made rigorous from recent progress in the literature. Finally, in the three last subsections we discuss variations of this result to the context of the Laplace approximation and temperature scaling, and a simplification for the the random features case.

\subsection{Recap of the setting}
\label{sec:app:recap}
As motivated in Sec.~\ref{sec:app:gcm}, our goal is to derive the self-consistent equations in the more general setting of the Gaussian covariate model (GCM), which thanks to Gaussian universality contains the random features model as a special case. For the ease of reading, we first recall the reader of the model of interest, which specializes Sec.~\ref{sec:app:gcm} to binary classification.

\paragraph{Data model: } Let $(\vec{u},\vec{v})$ denote a pair of Gaussian covariates:
\begin{align}
\begin{pmatrix}
        \vec{u} \\ \vec{v}
    \end{pmatrix}
    \sim \mathcal{N}\left(\vec{0}_{d+p},\begin{bmatrix}
        \Psi & \Phi \\ \Phi^{\top} & \Omega
    \end{bmatrix}\right).
\end{align}
and define the oracle classifier as:
\begin{align}
    f_{\star}(\vec{u}) = \mathbb{P}(y=1|\vec{\theta}_{\star}^{\top}\vec{u}) = \sigma_{\noisevar}\left(\frac{\vec{\theta}_{\star}^{\top} \vec{u}}{\sqrt{d}}\right), && \vec{\theta}_{\star}\sim\mathcal{N}(\vec{0},\mat{I}_{d}),
\end{align}
\noindent where we remind the reader of the convenient notation:
\begin{equation}
    \sigma_\tau(x) := \int \sigma(z) \mathcal{N}( z | x, \tau) {\rm d}z
\end{equation}
\noindent with $\sigma(z)=(1+e^{-z})^{-1}$ the sigmoid function. 

\paragraph{Classifiers: } Given $n$ independent pairs $(\vec{v}^{\mu}, y^{\nu})_{\mu\in[n]} \in\mathbb{R}^{p}\times\{-1,1\}$ from the model above and defining the training data $\mathcal{D}=\{(\vec{v}^{\mu},y^{\mu})_{\mu\in[n]}\}$ we are interested in studying the family of probabilistic classifiers of the type:
\begin{align}
    \hat{f}_{t}(\vec{v}) = \mathbb{P}(y=1|\tau_{t}, \mathcal{D}) \int\dd\vec{\theta}~ \sigma_{\tau_{t}}\left(\frac{\vec{\theta}^{\top}\vec{v}}{\sqrt{p}}\right)p_{t}(\vec{\theta}|\mathcal{D})
\end{align}
\noindent where the "posterior" $p_{t}(\vec{\theta}|\mathcal{D})$ and the noise level $\tau_{t}$ depend on the specific classifier $t\in\{\bo, \erm, \empbayes,\lap\}$ of interest introduced in Sec.~\ref{sec:classifiers}. 

\paragraph{A convenient rewriting: } Since the covariate $\vec{u}\in\mathbb{R}^{d}$ is not observed by the statistician, it is useful to rewrite it explicitly as a function of $\vec{v}$ and an effective uncorrelated noise. Additionally, it is also convenient to write $\vec{v}$ in terms of an uncorrelated variable. Mathematically, this is given by a standard Gaussian partition:
\begin{equation}
    \vec{u} = \Phi \Omega^{-1} \vec{v} + \left( \Psi - \Phi \Omega^{-1} \Phi^{\top} \right)^{\sfrac{1}{2}} \vec{z}, 
\end{equation}
for $\vec{z}\sim\mathcal{N}(0, I_{d})$ uncorrelated with $\vec{v}$. This motivate us to define the \textit{projected oracle weights}: 
\begin{align}
\projectedteacher = \Omega^{-1} \Phi^{\top} \vec{\theta}_{\star} 
\end{align}
Then, the oracle classifier can be equivalently written as: 
\begin{align}
    \mathbb{P}(y = 1 | \projectedteacher^{\top}\vec{v}) &= \int \sigma_{\noisevar} \left( \frac{\projectedteacher^{\top} \vec{v}}{\sqrt{d}} + \frac{1}{\sqrt{d}}\vec{\theta}_{\star}^{\top} \left( \Psi - \Phi \Omega^{-1} \Phi^{\top}\right)^{\sfrac{1}{2}} \vec{z} \right) \mathcal{N}\left(\vec{z} | 0, \mat{I}_d\right) \dd \vec{z} \\
    &= \int \sigma_{\noisevar} \left( \frac{\projectedteacher^{\top} \vec{v}}{\sqrt{d}} + \xi \right) \mathcal{N}\left(\xi \Big| 0,  \frac{\vec{\theta}_{\star}^{\top} \left(  \Psi - \Phi \Omega^{-1} \Phi^{\top} \right) \vec{\theta}_{\star}}{d}\right)\dd \xi
\end{align}
which is a logit model on the observed features with an effective mismatch noise $\xi \sim \mathcal{N}(0, \frac{1}{d} \vec{\theta}_{\star}^{\top} \left( \Psi - \Phi \Omega^{-1} \Phi^{\top} \right) \vec{\theta}_{\star})$. Recalling that $\vec{\theta}_{\star}\sim\mathcal{N}(\vec{0}_{d},\mat{I}_{d})$, in the asymptotic limit the noise variance concentrates:
\begin{align}
\frac{1}{d} \vec{\theta}_{\star}^{\top} \left( \Psi - \Phi \Omega^{-1} \Phi^{\top} \right) \vec{\theta}_{\star} \to \frac{1}{d}\Tr\left( \Psi - \Phi \Omega^{-1} \Phi^{\top} \right) =: \addnoisestr^2.
\end{align}
Therefore, the oracle classifier is equivalent to: 
\begin{equation}
    f_{\star}(\vec{v}) = \mathbb{P}(y = 1 | \projectedteacher^{\top}\vec{v}) = \sigma_{\noisevar + \addnoisevar}\left(\frac{\projectedteacher^{\top} \vec{v}}{\sqrt{d}}\right)
\end{equation}
with 
\begin{equation}
    \projectedteacher \sim \mathcal{N}(0, \Sigma_{\star}), \quad \Sigma_{\star} = \Omega^{-1} \Phi^{\top} \Phi \Omega^{-1}
    \label{eq:definition_sigma_star}
\end{equation}
Finally, to further simplify the algebra it is convenient to consider the following change of variables: 
\begin{equation}
    \vec{v} \rightarrow \Omega^{-\sfrac{1}{2}} \vec{v}, \quad \projectedteacher \rightarrow \Omega^{\sfrac{1}{2}} \projectedteacher
\end{equation}
Such that $\vec{v}\sim\mathcal{N}(\vec{0}_{p},\mat{I}_p)$ and $\projectedteacher\sim\mathcal{N}(\vec{0}_{d}, \tilde{\Phi}^{\top}\tilde{\Phi})$ with $\tilde{\Phi} \equiv \Phi \Omega^{-\sfrac{1}{2}}$. Note that the labels are invariant under this change, and therefore we can assume input data with identity covariance.

\subsection{State evolution for GAMP}
\label{sec:app:gamp}
\begin{algorithm}[h!]
   \caption{GAMP for an estimator $t \in \{\erm, \bo, \empbayes\}$}
\begin{algorithmic}
   \STATE {\bfseries Input:} Data $\mat{V}\in\mathbb{R}^{n\times p}$, $\vec{y}\in\{-1,1\}^{n}$ 
   
   \STATE Define $\mat{V}^2 = \mat{V} \odot \mat{V} \in\mathbb{R}^{n\times p}$ and Initialize $\hat{\studentweight}^{T=0} = \mathcal{N}(\vec{0}, \sigma_{w}^2\mat{I}_{d})$, $\hat{\vec{c}}^{T=0} = \vec{1}_{d}$, $\vec{g}^{T=0} = \vec{0}_{n}$.
   \FOR{$T\leq T_{\text{max}}$}
   
   \STATE $\vec{V}^{T} = \mat{V}^{2} \hat{\vec{c}}^{T}$ ; $\vec{\omega}^{T} = \mat{V} \hat{\studentweight}^{T} - \vec{V}^{T}\odot \vec{g}^{t-1}$ ; \qquad\textit{/* Update channel mean and variance */}
    
   \STATE $\vec{g}^{T} = f_{\text{out}, t}(\vec{y}, \vec{\omega}^{T}, \vec{V}^{T})$ ; $\partial\vec{g}^{T} = \partial_{\omega}f_{\text{out}, t}(\vec{y}, \vec{\omega}^{T}, \vec{V}^{T})$ ; \qquad\textit{/* Update channel */}
   
   \STATE $\vec{A}^{T} = -{\mat{V}^{2}}^{\top} \partial \vec{g}^{T}$ ; $\vec{b}^{T} = \mat{V}^{\top} \vec{g}^{T} + \vec{A}^{T}\odot \hat{\studentweight}^{T}$ ; \qquad\textit{/* Update prior mean and variance } 
   \STATE \textit{/* Update marginals */}
   \STATE $\hat{\studentweight}^{T+1} = f_{w, t}(\vec{b}^{T}, \vec{A}^{T})$ ;\qquad $\hat{\vec{c}}^{T+1} = {\rm diag}(\partial_{\vec{b}}f_{w, t}(\vec{b}^{T}, \vec{A}^{T})$)
   
   \ENDFOR
   \STATE {\bfseries Return:} Estimators $(\hat{\studentweight}^{\amp}_{t}, \hat{\vec{c}}_{t}^{\amp}) :=(\hat{\studentweight}^{T_{\rm{max}}}_{t}, \hat{\vec{c}}^{T_{\rm{max}}}_{t})$.
\end{algorithmic}
\label{alg:gamp}
\end{algorithm}
With the model in hands, we now show discuss how the sufficient statistics $(v_{t}^{\star},q_{t}^{\star},m_{t}^{\star})$ needed to  characterize the asymptotic density defined in eq.~\eqref{eq:res:jointdensity} satisfy a set of self-consistent equations, which in the particular case of the random features model are explicitly written given in eq.~\eqref{eq:state_evolution_erm}. Our derivation follows from the analysis of an approximate message passing scheme, which provides a powerful tool to derive exact asymptotic results in an unified way, and has been employed in many works in the high-dimensional statistics literature, e.g. \cite{6069859, bayati2011dynamics, rangan2011generalized, Krzakala_2012, donoho_high_2013, sur_modern_2018, NEURIPS2021_543e8374, celentano2020estimation, Gerbelot21,  loureiro_fluctuations_2022, Cornacchia2022}.

Given the training data $\mathcal{D} = (\mat{V}, \vec{y})$, the initial step is to consider the following set of iterates known as Generalized Approximate Message Passing (GAMP) algorithm \ref{alg:gamp}, where the \emph{denoising functions} $(f_{\text{out}, t}, f_{w, t})$ depend on the specific classifier of interest $t \in \{ \bo, \erm, \empbayes \}$, and are summarized in table \ref{table:denoising}. The convenience of the GAMP is precisely to allow us to deal with classifiers of very different nature ($t\in\{\bo, \empbayes\}$ are defined by sampling, while $t=\erm$ is a point-estimator) in an unified framework. Note that the GAMP algorithm \ref{alg:gamp} is close to the one in \cite{rangan2011generalized}, with the important difference that the denoising functions $f_{w,t}$ is vector valued - a consequence of the fact that implicitly the classifiers of interest have non-separable priors.
\begin{table*}[ht]
    \begin{center}
    \begin{tabular}{ccc}
        Classifier & $f_{{\rm} out, t}(y, \omega, v) $ & $f_{w, t}(\vec{b}, \mat{A})$ \\
        \hline \\
        $\hat{f}_{\erm}$ & $\prox_{\log \sigma(y \times \cdot)}(\omega)$ & $(\lambda \mat{I}_p + \mat{A})\vec{b}$ \\
        $\hat{f}_{\bo}$ & $\partial_{\omega} \log \int \mathbb{P}(y = 1 | z)  \mathcal{N}(z | \omega, v) {\rm d}z$ & $(\Sigma_{\star}^{-1} + \mat{A})\vec{b}$ \\
        $\hat{f}_{\empbayes}$ & $\partial_{\omega} \log \int \sigma(\beta y \times z) \mathcal{N}(z | \omega, v) {\rm d}z$ & $(\lambda \mat{I}_p + \mat{A})\vec{b}$ \\
    \end{tabular}
    \end{center}
    \caption{GAMP denoising functions for the ERM, Bayes-optimal and empirical Bayes estimators. We recall that the covariance matrix is given by $\Sigma_{\star} = \Omega^{-1} \Phi^{\top} \Phi \Omega^{-1}$.}
\label{table:denoising}
\end{table*}
A second convenient property of GAMP is that in the high-dimensional limit of interest here, the statistics of the sequence of estimators $\hat{\vec{\theta}}^{T,\amp}_{t},\hat{\vec{c}}_{t}^{T,\amp}$ can be exactly tracked by a set of equations known as \emph{state evolution}. Therefore, the key idea in the proof strategy is to show that the statistics of the iterates $\hat{\vec{\theta}}^{T,\amp}_{t},\hat{\vec{c}}_{t}^{T, \amp}$ (given by the state evolution equations) coincide with the statistics of the classifiers defined in Sec.~\ref{sec:classifiers}. The state evolution for GAMP with non-separable priors was rigorously derived in \cite{Berthier2019,Gerbelot21}. Therefore, in the following we limit ourselves to an informal but intuitive derivation. In Sec.~\ref{sec:app:replicas} we  show that the state evolution for the GAMP estimators indeed coincides with the fixed-point equations describing the statistics of the classifiers of interest according to the replica method. The fact that GAMP (rigorous) state evolution equations corresponds to the replica saddle-point equations is a very general fact \cite{zdeborova2016statistical}, which is at the roots of many rigorous proofs to the replica predictions.

In the limit where $n,p\to\infty$ with fixed $\alpha=\sfrac{n}{p}$, it can be shown that the GAMP algorithm \ref{alg:gamp} is asymptotically equivalent to the following rBP equations (this is discussed in for instance \cite{NEURIPS2018_84f0f204, 9240945}):
\begin{align}
	&\begin{cases}
		\omega^{T}_{\mu\to i} = \sum\limits_{j\neq i} v^{\mu}_{j}\hat{\theta}_{j\to\mu}^{T}\\
		V^{T}_{\mu\to i} = \sum\limits_{j\neq i}(v^{\mu}_{j})^2 \hat{c}^{T}_{j\to \mu}
	\end{cases}, &&
	\begin{cases}
		g^{T}_{\mu\to i} = f_{\text{out}, t}(y^{\mu}, \omega_{\mu\to i}^{T}, V^{T}_{\mu\to i})\\
		\partial g^{T}_{\mu\to i} = \partial_{\omega}f_{\text{out}, t}(y^{\mu}, \omega_{\mu\to i}^{T}, V^{T}_{\mu\to i})\\
	\end{cases}\\
	&\begin{cases}
		b^{T}_{\mu\to i} = \sum\limits_{\nu\neq \mu} v^{\nu}_{i}g^{T}_{\nu\to i}\\
		A^{T}_{\mu\to i} = -\sum\limits_{\nu\neq \mu} (v^{\nu}_{i})^2 \partial g^{T}_{\nu\to i}\\
	\end{cases},	&&
	\begin{cases}
		\hat{\theta}^{T+1}_{i\to\mu} = f_{w, t}(b^{T}_{i\to\mu}, A^{T}_{i\to \mu})	\\
		\hat{c}^{T+1}_{i\to\mu} = \partial_{b}f_{w, t}(b^{T}_{\mu\to i}, A^{T}_{\mu\to i})	
	\end{cases}
\end{align}
\noindent where we recall the reader $i\in[p]$, $\mu\in[n]$, and to lighten notation we have dropped the indexes $t\in\{\bo,\erm,\empbayes\}$ for the classifier and $^{\amp}$ which stresses that the messages concern GAMP estimators. By construction, the rBP messages are independent, and are only coupled to each other through the data, which we recall is given by:
\begin{align}
    y^{\mu} \sim P_{0}(\cdot|\projectedteacher^{\top}\vec{v}^{\mu}), && \vec{v}^{\mu}\sim\mathcal{N}(0, \mat{I}_p),&& \projectedteacher \sim \mathcal{N}(\vec{0}, \Sigma_{\star})
\end{align}
For convenience, we define the so-called \emph{teacher local field}:
\begin{align}
z_{\mu} = \projectedteacher^{\top} \vec{v}^{\mu} / \sqrt{d}
\end{align}
Without loss of generality, we can write $y^{\mu} = f_0(z_{\mu}, \eta^{\mu})$ for $\eta^{\mu} \sim \mathcal{N}(0, 1)$. We now characterize the joint statistics of the rBP messages. 
\subsubsection*{Step 1: Asymptotic joint distribution of $(z_{\mu}, \omega_{\mu\to i}^{T})$}
Note that $(z_{\mu}, \omega_{\mu\to i}^{T} )$ are given by a sum of independent random variables with variance $p^{-1/2}$, and therefore by the Central Limit Theorem in the limit $p\to\infty$ they are asymptotically Gaussian. Therefore we only need to compute their means, variances and cross correlation. The means are straightforward, since $v^{\mu}_{i}$ have mean zero and therefore they will also have mean zero. The variances are given by:
\begin{align}
    \mathbb{E}\left[z_{\mu}^2\right] &= \frac{1}{d}\mathbb{E}\left[\sum\limits_{i=1}^{p}\sum\limits_{j=1}^{p}v^{\mu}_{i}v^{\mu}_{j}w_{\star i}w_{\star j}\right] = \frac{1}{d}\sum\limits_{i=1}^{p}\sum\limits_{j=1}^{p}\mathbb{E}\left[v^{\mu}_{i}v^{\mu}_{j}\right]
    w_{\star i}w_{\star j}\notag = \frac{1}{d}\sum\limits_{i=1}^{p}\sum\limits_{j=1}^{p} \delta_{ij} w_{\star i}w_{\star j}\notag\\
     &\underset{p\to\infty}{\rightarrow} \rho \\
    \mathbb{E}\left[\left(\omega^{T}_{\mu\to i}\right)^2\right] &= \frac{1}{p} \mathbb{E}\left[\sum\limits_{j\neq i}^{p}\sum\limits_{k\neq i}^{p}v^{\mu}_{j}v^{\mu}_{k}\hat{\theta}^{T}_{j\to\mu}\hat{\theta}^{T}_{k\to\mu}\right] = \frac{1}{p} \sum\limits_{j\neq i}^{p}\sum\limits_{k\neq i}^{p}\mathbb{E}\left[v^{\mu}_{j}v^{\mu}_{k}\right]
     \hat{\theta}^{T}_{j\to\mu}\hat{\theta}^{T}_{k\to\mu} \notag\\
    &= \frac{1}{p}\sum\limits_{j\neq i}^{p}\sum\limits_{k\neq i}^{p}\delta_{jk}\hat{\theta}^{T}_{j\to\mu}\hat{\theta}^{T}_{k\to\mu} \underset{p\to\infty}{\rightarrow} q^{T}\\
    \mathbb{E}\left[z_{\mu}\omega^{T}_{\mu\to i}\right] &= \frac{1}{\sqrt{dp}} \mathbb{E}\left[\sum\limits_{j\neq i}^{p}\sum\limits_{k=1}^{p}v^{\mu}_{j}v^{\mu}_{k}\hat{\theta}^{T}_{j\to\mu}w_{\star k}\right] = \frac{1}{\sqrt{dp}} \sum\limits_{j\neq i}^{p}\sum\limits_{k=1}^{p}\mathbb{E}\left[v^{\mu}_{j}v^{\mu}_{k}\right]
    \hat{\theta}^{T}_{j\to\mu}w_{\star k}\notag \\
    &= \frac{1}{\sqrt{dp}}\sum\limits_{j\neq i}^{p}\sum\limits_{k=1}^{p}\delta_{jk}\hat{\theta}^{T}_{j\to\mu}w_{\star k}
     \underset{p\to\infty}{\rightarrow} m^{T} \\
\end{align}

\noindent where we have used that $\hat{\theta}_{i\to\mu}^{T} = O(p^{-1/2})$ to simplify the sums at large $p$. Summarising our findings:
\begin{align}
\label{eq:joint:fields}
(z_{\mu}, \omega^{T}_{\mu\to i})\sim\mathcal{N}\left(\vec{0}_{3},\begin{bmatrix}
\rho & m^{T} \\ m^{T} & q^{T}	
\end{bmatrix}
\right)	
\end{align}
\noindent with:
\begin{align}
\rho \equiv \frac{1}{d} \projectedteacher^{\top} \projectedteacher	, && q^{T} \equiv \frac{1}{p} ({\hat{\studentweight}_{t}^T})^{\top} \hat{\studentweight}_{t}^{T}, && m^{T} \equiv \frac{1}{\sqrt{dp}} (\hat{\studentweight}_t^T)^{\top} \projectedteacher \notag\\ 
\end{align}

\subsection*{Step 2: Concentration of variances $V_{\mu\to i}^{T}$}

Since the variance $V_{\mu\to i}^{T}$ depends on $(\vec{v}_{i}^{\mu})^2$, in the asymptotic limit $p\to\infty$ it concentrates around its mean :
\begin{align}
\mathbb{E}\left[V^{T}_{\mu\to i}\right]  = \frac{1}{p}\sum\limits_{j\neq i}\mathbb{E}\left[\left(v^{\mu}_{i}\right)^2\right]\hat{c}^{T}_{j\to \mu}	= \frac{1}{p}\sum\limits_{j\neq i}\hat{c}^{T}_{j\to \mu} =  \frac{1}{p}\sum\limits_{j=1}^{p}\hat{c}^{T}_{j\to \mu} - \frac{1}{p}\hat{c}^{T}_{i\to \mu}\underset{p\to\infty}{\rightarrow} V^{T} \equiv \frac{1}{p}\sum\limits_{j=1}^{p}\hat{c}^{T}_{j}
\end{align}
\noindent where we have defined the variance overlap $V^{T}$. We thus have $V_{\mu\to i}^{T} \to V^{T}$. Note that $V^{T}$ corresponds to the divergence with respect to $\vec{b}$ of $\log \mathcal{Z}_w(\vec{b}, \mat{A})$.

\subsection*{Step 3: Distribution of $b_{\mu\to i}^{T}, \tilde{b}_{\mu\to i}^{T}$}
By definition, we have
\begin{align}
	b^{T}_{\mu\to i} &= \frac{1}{\sqrt{p}}\sum\limits_{\nu\neq \mu} v^{\nu}_{i}g^{T}_{\nu\to i} = \frac{1}{\sqrt{p}}\sum\limits_{\nu\neq \mu} v^{\nu}_{i}f_{\text{out}}(y^{\mu}, \omega_{\nu\to i}^{T}, V^{T}_{\nu\to i}) = \frac{1}{\sqrt{p}}\sum\limits_{\nu\neq \mu} v^{\nu}_{i}f_{\text{out}}(f_{0}(z_{\nu}, \eta^{\nu}), \omega_{\nu\to i}^{T}, V^{T}_{\nu\to i})\notag\\
\end{align}
Note that in the sum $z_{\nu} = \frac{1}{\sqrt{d}}\sum\limits_{j=1}^{p}v^{\nu}_{j}w_{\star j}$ there is a term $i=j$, and therefore $z_{\mu}$ is correlated with $v^{\nu}_{i}$. To make this explicit, we split the teacher local field: 
\begin{align}
z_{\mu} = \frac{1}{\sqrt{d}} \sum\limits_{j=1}^{p}v^{\mu}_{j}w_{\star j}	 = \underbrace{\frac{1}{\sqrt{d}} \sum\limits_{j\neq i}v^{\mu}_{j}w_{\star j}}_{z_{\mu\to i}}+\frac{1}{\sqrt{d}} v^{\mu}_{i}w_{\star i}
\end{align}
\noindent and note that $z_{\mu\to i} = O(1)$ is independent from $v^{\nu}_{i}$. Since $v^{\mu}_{i}w_{\star i} = O(p^{-1/2})$, to take the average at leading order, we can expand the denoising function:
\begin{align}
f_{\text{out}}(f_{0}(z_\mu, \eta^{\nu}), \omega^{T}_{\nu\to i}, V^{T}_{\nu\to i}) &= f_{\text{out}}(f_{0}(z_{\nu\to i}, \eta^{\nu}), \omega^{T}_{\nu\to i}, V^{T}_{\nu\to i}) \\ &+  \frac{1}{\sqrt{d}} \partial_{z}f_{\text{out}}(f_{0}(z_{\nu\to i}, \eta^{\nu}), \omega^{T}_{\nu\to i}, V^{T}_{\nu\to i})v^{\nu}_{i}w_{\star i} + O(p^{-1})\notag
\end{align}
Inserting in the expression for $b^{T}_{\mu\to i}$,
\begin{align}
	b^{T}_{\mu\to i} &= \frac{1}{\sqrt{p}}\sum\limits_{\nu\neq \mu} v^{\nu}_{i}f_{\text{out}}(f_{0}(z_{\nu\to i}, \eta^{\nu}), \omega_{\nu\to i}^{T}, V^{T}_{\nu\to i})
	\\& + \frac{1}{\sqrt{dp}}\sum\limits_{\nu \neq \mu} (v^{\nu}_{i})^2 \partial_{z}f_{\text{out}}(f_{0}(z_{\nu\to i}, \eta^{\nu}), \omega_{\nu\to i}^{T}, V^{T}_{\nu\to i})w_{\star i} + O(p^{-3/2})\notag
\end{align}
Therefore:
\begin{align}
\mathbb{E}\left[b^{T}_{\mu\to i}\right] &= \frac{w_{\star i}}{\sqrt{dp}}\sum\limits_{\nu\neq\mu} \partial_{z}f_{\text{out}}(f_{0}(z_{\nu\to i}, \eta^{\nu}), \omega_{\nu\to i}^{T}, V^{T}_{\nu\to i})+ O(p^{-3/2}) \notag\\
&= \frac{w_{\star i}}{\sqrt{dp}}\sum\limits_{\nu=1}^{n} \partial_{z}f_{\text{out}}(f_{0}(z_{\nu\to i}, \eta^{\nu}), \omega_{\nu\to i}^{T}, V^{T}_{\nu\to i})  + O(p^{-3/2})
\end{align}
Note that as $p\to\infty$, for fixed $t$ and for all $\nu$, the fields $(z_{\nu\to i}, \omega^{T}_{\nu\to i})$ are identically distributed according to average in eq.~\eqref{eq:joint:fields}. Therefore, 
\begin{align}
\frac{1}{\sqrt{dp}}\sum\limits_{\nu=1}^{n} \partial_{z}f_{\text{out}}(f_{0}(z_{\nu\to i}, \eta^{\nu}  ), \omega_{\nu\to i}^{T}, V^{T}_{\nu\to i}) \underset{p\to\infty}{\rightarrow} \alpha \sqrt{\gamma} ~\mathbb{E}_{(\omega, z), \eta}\left[\partial_{z}f_{\text{out}}(f_{0}(z, \eta), \omega, V^{T})\right]	\equiv \hat{m}^{T}
\end{align}
\noindent so:
\begin{align}
\mathbb{E}\left[b^{T}_{\mu\to i}\right]  \underset{p\to\infty}{\rightarrow} w_{\star i}\hat{m}^{T}.
\end{align}
Similarly, the variance is given by:
\begin{align}
&\text{Var}\left[b^{T}_{\mu\to i}\right]\\
&= \frac{1}{p}\sum\limits_{\nu\neq\mu}\sum\limits_{\kappa\neq\mu}\mathbb{E}\left[v^{\nu}_{i}v^{\kappa}_{i}\right] f_{\text{out}}(f_{0}(z_{\nu\to i}, \eta^{\nu}), \omega_{\nu\to i}^{T}, V^{T}_{\nu\to i}) f_{\text{out}}(f_{0}(z_{\kappa\to i}, \eta^{\kappa}), \omega_{\kappa\to i}^{T}, V^{T}_{\kappa\to i})+ O(d^{-2})\notag\\
&= \frac{1}{p}\sum\limits_{\nu\neq\mu}f_{\text{out}}(f_{0}(z_{\nu\to i}, \eta^{\nu}), \omega_{\nu\to i}^{T}, V^{T}_{\nu\to i})^2 + O(p^{-2})\notag\\
&= \frac{1}{p}\sum\limits_{\nu=1}^{n}f_{\text{out}}(f_{0}(z_{\nu\to i}, \eta^{\nu}), \omega_{\nu\to i}^{T}, V^{T}_{\nu\to i})^2 + O(p^{-2}) \notag\\
&\underset{d\to\infty}{\rightarrow} \alpha ~\mathbb{E}_{(z,\omega), \xi}\left[f_{\text{out}}(f_{0}(z, \eta), \omega, V^{T})^2\right]\equiv \hat{q}^{T}
\end{align}
To summarise, we have:
\begin{align}
b^{T}_{\mu\to i} \sim \mathcal{N}\left(w_{\star i} \hat{m}^{T}, \hat{q}^{T}\right)
\end{align}

\subsection*{Step 4: Concentration of $A_{\mu\to i}^{T}, \tilde{A}_{\mu\to i}^{T}$}
The only missing piece is to determine the distribution of the prior variances $A_{\mu\to i}^{T}, \tilde{A}_{\mu\to i}^{T}$. Similar to the previous variance, they concentrate:
\begin{align}
	A_{\mu\to i}^{T} &= - \frac{1}{p}\sum\limits_{\nu\neq \mu}(v^{\nu}_{i})^2 \partial_{\omega}f_{\text{out}, t}(y^{\nu}, \omega_{\nu\to i}^{T}, V^{T}_{\nu\to i}) \\
	&= -\frac{1}{p}\sum\limits_{\nu\neq \mu}(x^{\nu}_{i})^2 \partial_{\omega}f_{\text{out}, t}(f_{0}(z_{\nu\to i}, \eta^{\nu}), \omega_{\nu\to i}^{T}, V^{T}_{\nu\to i}) + O(p^{-3/2})\notag\\
	&= -\frac{1}{p}\sum\limits_{\nu = 1} \partial_{\omega} f_{\text{out}, t}(f_{0}(z_{\nu \to i}, \eta^{\nu}), \omega_{\nu \to i}^{T}, V^{T}_{\nu \to i}) + O(d^{-3/2}) \notag \\
	&\underset{p \to \infty}{\rightarrow} - \alpha ~\mathbb{E}_{(z,\omega), \xi } \left [\partial_{\omega}f_{\text{out}, t}(f_{0}(z, \eta), \omega, V^{T}) \right] \equiv \hat{v}^{T}
\end{align}
\subsection*{Summary}
We now have all the ingredients we need to characterize the asymptotic distribution of the GAMP iterates for any of the classifiers $t\in\{\bo,\erm,\empbayes\}$:
\begin{align}
\hat{\studentweight}^{T,\amp}_{t} &\sim f_{\text{out}, t}(\projectedteacher {\hat{m}}_{t}^{T, \amp}+\sqrt{\hat{q}_{t}^{T, \amp}}\vec{\xi}, \hat{v}_{t}^{T,\amp})
\end{align}
\noindent where $\vec{\xi} \sim\mathcal{N}(\vec{0},\mat{I}_{p})$ is an independent Gaussian vector. Therefore, we recover the GAMP state evolution equations \cite{rangan2011generalized, Berthier2019} for the overlaps:
\begin{align}
	&\begin{cases}
		V^{T+1, \amp} &= \mathbb{E}_{(\projectedteacher, \vec{\xi})}\left[\partial_{\vec{b}} \cdot f_{w, t}(\hat{m}^{T, \amp} \projectedteacher + \sqrt{\hat{q}^{T, \amp}}\vec{\xi}, \hat{v}^{T, \amp} \mat{I}_p)\right] \\
		q^{T+1, \amp} &= \mathbb{E}_{(\projectedteacher, \vec \xi)}\left[f_{w, t}(\hat{m}^{T, \amp} \projectedteacher + \sqrt{\hat{q}^{T, \amp}} \vec{\xi}, \hat{v}^{T} \mat{I}_p)^2\right] \\
		m^{T+1, \amp} &= \sqrt{\gamma} \mathbb{E}_{(\projectedteacher, \vec \xi)}\left[f_{w, t}(\hat{m}^{T, \amp} \projectedteacher + \sqrt{\hat{q}^{T, \amp}} \vec{\xi}, \hat{v}^{T, \amp} \mat{I}_p)^{\top} \projectedteacher \right]
	\end{cases}, \\
    &\begin{cases}
		\hat{v}^{T, \amp} &=-\alpha \mathbb{E}_{(z, \omega), \eta }\left[\partial_{\omega}f_{\text{out}, t}(f_{0}(z, \eta), \omega, V^{T, \amp})\right]\\
		\hat{q}^{T, \amp} &=\alpha \mathbb{E}_{(z, \omega), \eta}\left[f_{\text{out}, t}(f_{0}(z, \eta), \omega, V^{T, \amp})^2\right]\\
		\hat{m}^{T, \amp} &=\alpha \sqrt{\gamma} \mathbb{E}_{(z, \omega), \eta}\left[\partial_z f_{\text{out}, t}(f_{0}(z, \eta), \omega, V^{T, \amp})\right]
	\end{cases}
	\label{eq:gamp_state_evolution}
\end{align}
\noindent where $\projectedteacher \sim \mathcal{N}(\vec{0}, \Sigma_{\star}), \xi \sim \mathcal{N}(\vec{0}, I_p)$ and $(z, \omega) \sim \mathcal{N}(\vec{0}, \begin{bmatrix} \rho & m^T \\ m^T & q^T \end{bmatrix}), \eta \sim \mathcal{N}(0, 1)$

Interestingly, we can show that the equations~\eqref{eq:gamp_state_evolution} are strictly equivalent to the self-consistent equations~\eqref{eq:state_evolution_erm} of Theorem~\ref{thm:jointstats}. Consider first the update equations for $V_{t}^{T,\amp}, q_{t}^{T,\amp}, m_{t}^{T,\amp}$. First, note that for all the estimators considered here, the function 
$f_{w, t}$ has the form $(\Sigma_{t}^{-1} + \mat{A})\vec{b}$ for some matrix $\Sigma_{t}$. Now, let us introduce 
\begin{align}
    \tilde{\Psi}(\vec{b}, \mat{A}, \Sigma) &= \frac{1}{2p} \Tr \log (\Sigma^{-1} + \mat{A}) + \frac{1}{2p} \vec{b}^{\top} (\Sigma^{-1} + \mat{A})^{-1} \vec{b} \\
    \Psi_w(\hat{m}, \hat{q}, \hat{v})      &= \mathbb{E}_{\projectedteacher, \vec{\xi}}\tilde{\Psi}(\hat{m} \projectedteacher + \hat{q} \vec{\xi}, \hat{v} \mat{I}_p)
\end{align}
With some algebra, we can see that for any estimator described in Table~\ref{table:denoising} we have 
\begin{align}
    \partial_{\hat{m}} \Psi_w(\hat{m}, \hat{q}, \hat{v}) &= \mathbb{E}_{(\projectedteacher, \vec{\xi})}\left[f_{w, t}(\hat{m}^{T} \projectedteacher + \sqrt{\hat{q}^{T}} \vec{\xi}, \hat{v}^{T} \mat{I}_p)^{\top} \projectedteacher \right] \\
    \partial_{\hat{q}}\Psi_w - \partial_{\hat{v}} \Psi_w &= \frac{1}{2} \mathbb{E}_{(\projectedteacher, \vec{\xi})}\left[f_{w, t}(\hat{m}^{T} \projectedteacher + \sqrt{\hat{q}^{T}} \vec{\xi}, \hat{v}^{T} \mat{I}_p)^2\right] \\
    \partial_{\hat{q}} \Psi_w &= \frac{1}{2} \mathbb{E}_{(\projectedteacher, \xi)}\left[\partial_{\vec{b}} \cdot f_{w, t}(\hat{m}^{T} \projectedteacher + \sqrt{\hat{q}^{T}}\vec{\xi}, \hat{v}^{T} \mat{I}_p)\right]
\end{align}
Thus the update equations~\eqref{eq:gamp_state_evolution} for ${m}, q, v$ are equivalent to Equations~\eqref{eq:state_evolution_erm}. It is the same for $\hat m, \hat q, \hat v$ :  consider the update equation for $\hat{q}^{T, \amp}$. We can rewrite it with a Dirac delta
\begin{align}
    \hat{q}^{T, \amp} &= \alpha \sum_y \mathbb{E}_{(z, \omega), \eta}\left[f_{\text{out}, t}(y, \omega, V^{T, \amp})^2 \delta(y - f_0(z, \eta) \right] \\
    &= \alpha \sum_y \mathbb{E}_{\omega}\left[f_{\text{out}, t}(y, \omega, V^{T, \amp})^2 \mathbb{E}_{z | \omega, \eta}  ( \delta(y - f_0(z, \eta) ) \right] \\
\end{align}
The distribution of $z$ conditioned on $\omega$ is a Gaussian with mean $m^{T, \amp} / q^{T, \amp} \times \omega$ and variance $\rho - m^{T, \amp} \times m^{T, \amp} / q^{T, \amp}$. Then, $\mathbb{E}_{z | \omega, \eta}  ( \delta(y - f_0(z, \eta) )$ can be written 
\begin{align*}
    \mathbb{E}_{z | \omega, \eta}  ( \delta(y - f_0(z, \eta) ) &= \int \dd z \mathbb{E}_{\eta} \left( \delta( y - f_0(z, \eta)) \right) \mathcal{N}(z | m^{T, \amp} / q^{T, \amp} \times \omega, \rho - m^{T, \amp} \times m^{T, \amp} / q^{T, \amp}) \\
    &= \int \dd z \mathbb{P}(y = 1 | z) \mathcal{N}(z | m^{T, \amp} / q^{T, \amp} \times \omega, \rho - m^{T, \amp} \times m^{T, \amp} / q^{T, \amp}) \\
    &= \mathcal{Z}_0(y, m^{T, \amp} / q^{T, \amp} \times \omega, \rho - m^{T, \amp} \times m^{T, \amp} / q^{T, \amp})
\end{align*}
we thus recover the equation for $\hat{q}$ in \eqref{eq:state_evolution_erm} : 
\begin{equation}
    \hat{q}^{T+1} = \alpha \sum_y \mathbb{E}_{\omega}\left[f_{\text{out}, t}(y, \omega, V^{T, \amp})^2 \mathcal{Z}_0(y, m^T / q^T \omega, \rho - \sfrac{m^2}{q}) \right]
\end{equation}
Similar computations can be done for $\hat{m}$ and $\hat{v}$.
 
We have thus  eqs.~\eqref{eq:gamp_state_evolution} with the self-consistent equations \eqref{eq:state_evolution_erm} of Theorem~\ref{thm:jointstats}. It remains to show two points. First, that in the particular case of the random feature model the expression of $\Psi_w$ simplifies to Equation~\eqref{eq:def_psi_w} - this is discussed in Appendix~\ref{sec:app_random_features}. Second, to show that the fixed points of the state evolution equations $(m_{t}^{\amp}, q_{t}^{\amp}, v_{t}^{\amp})$ indeed corresponds to the sufficient statistics $(m^{\star}_{t}, q_{t}^{\star}, v^{\star}_{t})$ for the classifiers of interest. First, we provide a heuristic derivation of this fact, based on the replica method from statistical physics \cite{mezard1987spin}. We defer the discussion of the formal aspects to Sec.~\ref{sec:app:rigor}.

\subsection{Self-consistent equation from the replica method}
\label{sec:app:replicas}
As discussed above, the goal of this section is to provide a derivation of the self-consistent equations in eq.~\eqref{eq:state_evolution_erm} from the replica method. For the particular case of $\hat{f}_{\erm}$, this derivation appeared \cite{loureiro_learning_2021}, where it was also rigorously proven using CGMT. Here, we extend this analysis to $t\in\{\bo,\empbayes\}$. 

We can treat the different classifiers of interest in the replica analysis by defining the following Gibbs distribution: 
\begin{equation}
    \mu_{t}(\studentweight|\mathcal{D}) = \frac{1}{\mathcal{Z}_t} \prod_{\mu\in[n]} P^{t}_{\sigma}(y^{\mu} | \studentweight^{\top}\vec{v}^{\mu}) \times P^{t}_{\studentweight}(\studentweight)
\end{equation}
where $(P^{t}_{\sigma}, P^{t}_{\studentweight})$ are a likelihood and priors (not necessarily normalized) depending on the particular classifier, and are explicitly given in Table \ref{table:estimators_prior_likelihood}, and the normalization constant $\mathcal{Z}_t$ is the partition function.

\begin{table*}[ht]

    \begin{center}
    \begin{tabular}{ccc}
    \textbf{Classifier}  & $P_{\sigma}^t(y | z)$ & $P_{\theta}^t(\studentweight)$ \\
    \hline \\
    $\hat{f}_{\erm}$        & $\sigma(y \times z)^{\beta}$ & $e^{-\sfrac{\beta \lambda}{2} \| \studentweight \|^2}$  \\
    $\hat{f}_{\empbayes}$   & $\sigma(\beta y \times z) $ & $\mathcal{N}(\studentweight | \Vec{0}, \sfrac{1}{\lambda} I_p)$ \\
    $\hat{f}_{\bo}$         & $\sigma_{\noisevar + \addnoisevar}(y \times z)$  & $\mathcal{N}(\studentweight | \Vec{0}, \Sigma_{\star})$
    \end{tabular}
    \end{center}
\caption{Prior and likelihood for the different estimators. For $\hat{f}_{\erm}$, the temperature $\beta$ must be taken in the limit $\beta \to \infty$, and the Gibbs measure $\mu_{\erm}(\studentweight | \mathcal{D})$ is peaked around the minimizer of the empirical risk $\hat \studentweight_{\erm}$.}
\label{table:estimators_prior_likelihood}
\end{table*}
The aim of the replica method is to compute the free energy density defined as:
\begin{equation}
    \beta f_{\beta} = - \lim_{p \to \infty} \frac{1}{p}\mathbb{E}_{\mathcal{D}} \log \mathcal{Z}_{t}
\end{equation}
The free energy is the cumulant generating function of the Gibbs measure, and therefore computing it give us access to the statistics of the measure, which in particular allow us to compute the test error and calibration (among others quantities) of the classifiers defined in Sec.~\ref{sec:classifiers}. Since taking the expectation over the log is intractable, we resort to the \textit{replica method} \cite{mezard1987spin}, which consists of the following trick:
\begin{equation}
    \log \mathcal{Z}_t = \lim_{r \to 0^+} \frac{1}{r}\mathcal{Z}_t^r
\end{equation}
Swapping the limit and the expectation, what we need to compute is:
\begin{align*}
    \mathbb{E}_{\mathcal{D}}\mathcal{Z}_t^r &= \prod_{\mu = 1}^n \mathbb{E}_{\vec{v}^{\mu}, y^{\mu}} \prod_{a = 1}^r \int_{\mathbb{R}^d} P^{t}_{\studentweight}\left(\studentweight^a\right) P^{t}_{\sigma}\left(y^{\mu} | \frac{{\vec{v}^{\mu}}^{\top} \studentweight^a}{\sqrt{d}}\right) \\
    &= \prod_{\mu = 1}^n \sum_{y} \int P(\projectedteacher) \int \left( \prod_a P^{t}_{\studentweight}(\studentweight^a) \right) \mathbb{E}_{\vec{v}^{\mu}} \left[ P_0\left(y^{\mu} | \frac{{\vec{v}^{\mu}}^{\top} \projectedteacher}{\sqrt{d}}\right) \prod_a P_{\sigma}^{t}\left(y^{\mu} | \frac{{\vec{v}^{\mu}}^{\top} \studentweight^a}{\sqrt{p}}\right) \right]
\end{align*}
Next, we introduce the \textit{local fields} $\nu_{\star}^{\mu} = \frac{1}{\sqrt{d}}{\vec{v}^{\mu}}^{\top} \projectedteacher$ and $\nu_a^{\mu} = \frac{1}{\sqrt{d}}{\vec{v}^{\mu}}^{\top} \studentweight^a$. Then, the term between brackets in the above equation is equal to 
\begin{align}
    \int \dd \nu_{\star}^{\mu} P_0(y^{\mu} | \nu_{\star}^{\mu}) \int \prod_a \dd \nu_a^{\mu} P^{t}_{\sigma}(y^{\mu} | \nu_a^{\mu}) \mathbb{E}_{\vec{v}^{\mu}} \left[ \delta(\nu_{\star}^{\mu} - \frac{{\vec{v}^{\mu}}^{\top} \projectedteacher}{\sqrt{d}}) \prod_a \delta(\nu_{a}^{\mu} - \frac{{\vec{v}^{\mu}}^{\top} \studentweight^a}{\sqrt{p}})) \right]
\end{align}
Note that $\mathbb{E}_{\vec{v}^{\mu}} \left[ \delta(\nu_{\star}^{\mu} - {\vec{v}^{\mu}}^{\top} \projectedteacher) \prod_a \delta(\nu_{a}^{\mu} - {\vec{v}^{\mu}}^{\top} \studentweight^a)) \right]$ defines the joint distribution of the local fields. It is straightforward to show that this is a Gaussian distribution on zero mean and covariance $\Sigma_{\nu}$: 
\begin{equation}
    \mathbb{E}(\nu_{\star}, \nu_{\star}) = \frac{1}{d}\projectedteacher^{\top} \Omega \projectedteacher = \rho, \quad \mathbb{E}(\nu_{\star}, \nu_a) = \frac{1}{\sqrt{pd}}\projectedteacher^{\top} \Omega \studentweight^a = m^a, \quad \mathbb{E}(\nu_{a}, \nu_b) = \frac{1}{p}{\studentweight^a}^{\top} \Omega \studentweight^b = Q^{ab}
    \label{eq:app_def_overlaps}
\end{equation}
Then, we have 

\begin{equation}
    \mathbb{E} \mathcal{Z}_t^r = \prod_{\mu} \sum_{y^{\mu}} \int P_{\studentweight, 0}(\projectedteacher) \int \prod_a P^{t}_{\studentweight}(\studentweight^a) \int \dd \nu_{\star}^{\mu} \prod_a \dd \nu_a^{\mu} P_0(y^{\mu} | \nu_{*}^{\mu}) \prod_a P_{\sigma}^{t}(y^{\mu} | \nu_{a}^{\mu}) \times \mathcal{N}(\nu_*^{\mu}, \nu_a^{\mu} | \mathbf{0}, \Sigma_{\nu})
\end{equation}

The elements of the covariance matrix $\Sigma_{\nu}$ are fixed by eq.~\eqref{eq:app_def_overlaps}. We can free these overlaps by doing the Fourier transform of the Dirac delta. We get in the end 
\begin{align}
    \mathbb{E}_{\mathcal{D}} \mathcal{Z}_t^r \propto \int \dd \rho \dd \hat{\rho} \prod_a \dd m^a \dd \hat{m}^a \prod_{a,b} \dd Q^{ab} \dd \hat{Q}^{ab} e^{p\Phi(r)}
\end{align}
Where 
\begin{align}
    \Phi(r) &= - \frac 1 \gamma \rho \hat{\rho} - \frac{1}{\sqrt \gamma} \sum_a m^a \hat{m}^a - \sum_{a \leqslant b} Q^{ab} \hat{Q}^{ab} + \alpha \times \Psi_y^{(r)} + \Psi_w^{(r)} \\
    \Psi_y^{(r)} &= \frac{1}{p} \log \int P_{\studentweight, 0}(\projectedteacher) \int \prod_a P^{t}_{\studentweight}(\studentweight^a) e^{\hat{\rho} \| \projectedteacher \|^2 + \sum_a \hat{m}^a \projectedteacher \Omega \studentweight^a + \sum_{a \leqslant b} \hat{Q}^{a, b} \studentweight^{a} \Omega \studentweight^b} \\
    \Psi_w^{(r)} &= \frac{1}{p} \log \sum_y \int \dd \nu_{\star} P_0(y | \nu_{\star}) \int \prod_a \dd \nu_a P_{\sigma}^{t}(y | \nu_a) \mathcal{N}(\nu, \nu_a ; \Sigma_{\nu})
\end{align}

\subsubsection{Replica symmetric ansatz}

In the replica symmetric ansatz, we assume $m^a = m$, $Q^{ab} = q$ for $a \neq b$, $Q^{aa} = v + q$, $\hat{m}^a = \hat{m}$,  $\hat{Q}^{ab} = \hat{q} $ for $a \neq b$, $\hat{Q}^{aa} = - \frac{1}{2}(\hat{v} - \hat{q})$ where the quantities $m, q, v, \hat{m}, \hat{q}, \hat{v}$ are to be determined.

We refer to \cite{gerace_generalisation_2020,aubin_generalization_2020} for the detailed computation of $\lim_{r \to 0^+} \Psi_y^{(r)}$ and $\lim_{r \to 0^+} \Psi_w^{(r)}$. In the end, we obtain : 
\begin{empheq}[box=\fbox]{align}
    f_{\beta} &= {\rm extr}_{m, q, v, \hat{m}, \hat{q}, \hat{v}} \left\{ - \frac{1}{\sqrt{\gamma}} m \hat{m} + \frac{1}{2}(q \hat{v} - \hat{q} v + \hat{v} v) + \Psi_w + \alpha \times \Psi_y \right\} \\
    \Psi_w    &= \lim_{d \to \infty} \frac{1}{p} \mathbb{E}_{\xi, \projectedteacher} \log \int \dd \studentweight P^{t}_{\studentweight}(\studentweight)e^{-\hat{v}^2 \studentweight^{\top} \Omega \studentweight+ \studentweight^{\top}(\hat{m} \Omega \projectedteacher + \hat{q}\Omega^{-\sfrac{1}{2}}\vec{\xi})}  \\
    \Psi_y    &= \mathbb{E}_{\xi \sim \mathcal{N}(0,q)} \left[ \sum_y \mathcal{Z}_0(y, \sfrac{m}{q}\xi, \rho - \sfrac{m^2}{q}) \log \mathcal{Z}_g(y, \xi, v) \right]
    \label{eq:free_energy}
\end{empheq}    

where 
\begin{equation}
    \mathcal{Z}_{0/g}(y, \omega, v) = \int \dd z P_{0 / g}(y | z) \mathcal{N}(z | \omega, v)
\end{equation}

The self-consistent equations~\eqref{eq:state_evolution_erm} are obtained by cancelling the derivative of the free energy with respect to each of $(m, q, v, \hat{m}, \hat{q}, \hat{v})$. 

\noindent \paragraph*{$\Psi_w$ for Gaussian priors: }For all the estimators considered here, the prior distribution $P^{t}_{\studentweight}(\studentweight)$ is Gaussian $\mathcal{N}(0, \Sigma)$, where $\Sigma$ depends on the considered estimator.
Then, 

\begin{align}
    \Psi_w &= \int \dd \studentweight e^{-\frac{1}{2}\studentweight^{\top} \Sigma^{-1} \studentweight} e^{-\frac{v}{2} \studentweight^{\top} \Omega \studentweight + \studentweight^{\top} (\hat{m} \Omega \projectedteacher + \sqrt{\hat{q}} \Omega^{-1/2} \vec{\xi})} \\
           &= \frac{ \exp \left( \frac{1}{2}(\hat{m} \projectedteacher + \sqrt{\hat{q}} \Omega^{-1/2} \vec{\xi} )^{\top} (\Sigma + \hat{v} \Omega)^{-1} (\hat{m} \projectedteacher + \sqrt{\hat{q}} \Omega^{-1/2} \vec{\xi} ) \right) }{ \sqrt{{\rm det}(\Sigma + \hat{v}\Omega)}} \\
           &= \lim - \frac{1}{2p} \Tr \log (\Sigma + \hat{v}\Omega) + \frac{1}{2p} \Tr (\hat{m}^2  \Omega \projectedteacher \projectedteacher^{\top} \Omega + \hat{q} \Omega )(\Sigma + \hat{v}\Omega)
\end{align}

We get in the end the following expression for $\Psi_w$
\begin{align}
    \Psi_w &= -\frac{1}{2p} \Tr \log \left( \hat{v} \Omega + \Sigma \right) + \frac{1}{2p} \Tr \left( ( \hat{m}^2 \Omega \projectedteacher \projectedteacher^{\top} \Omega + \hat{q} \Omega ) ( \hat{v} \Omega + \Sigma )^{-1} \right) \\
    \label{eq:def_psi_w_matrices}
\end{align}

\paragraph{Saddle-point equations:} To compute the free energy, we cancel its derivative with respect to $m, q, v, \hat{m}, \hat{q}, \hat{v}$. We have : 
\begin{align}
    \begin{cases}
        \partial_{\hat{m}} f_{\beta} &= - \frac{1}{\sqrt{\gamma}} m + \partial_{\hat{m}} \Psi_w \\
        \partial_{\hat{q}} f_{\beta} &= - \frac{1}{2} v + \partial_{\hat{q}} \Psi_w \\
        \partial_{\hat{v}} f_{\beta} &= \frac{1}{2} (v + q) + \partial_{\hat{v}} \Psi_w \\
    \end{cases}
\end{align}
Cancelling the derivatives gives the condition:
\begin{empheq}[box=\fbox]{align}
    \begin{cases}
        m = \sqrt{\gamma} \partial_{\hat{m}} \Psi_w \\
        v = 2 \times \partial_{\hat{q}} \Psi_w \\
        q = - v - 2 \times \partial_{\hat{v}} \Psi_w = 2 \times (\partial_{\hat{q}} \Psi_w - \partial_{\hat{v}} \Psi_w)
    \end{cases}
\end{empheq}
Which are the first three equations of Theorem~\ref{thm:jointstats}. The derivative of the free energy with respect to $(m, q, v)$ is given by
\begin{align}
    \begin{cases}
        \partial_{m} f_{\beta} &= - \frac{1}{\sqrt{\gamma}}\hat{m} + \alpha  \partial_m \Psi_y \\
        \partial_{q} f_{\beta} &= \frac{1}{2} \hat{v} + \alpha \partial_q \Psi_y \\
        \partial_{v} f_{\beta} &= \frac{1}{2} (\hat{v} - \hat{q}) + \alpha \partial_v \Psi_y \\
    \end{cases}
\end{align}
Cancelling the derivatives, and computing the derivatives of $\Psi_y$  gives then 
\begin{empheq}[box=\fbox]{align}
\begin{cases}
    \hat{v} &= - \alpha \mathbb{E}_{\xi \sim \mathcal{N}(0, q)}\!\left[ \sum_y \mathcal{Z}_0 \left( y, \sfrac{m}{q}\xi, v_{\star} \right)  \partial_{\omega} g_{t} \left( y, \xi, v \right) \right] \\
    \hat{q} &= \alpha \mathbb{E}_{\xi \sim \mathcal{N}(0, q)} \! \left[ \sum_y \mathcal{Z}_0 \left( y, \sfrac{m}{q}\xi, v_{\star} \right) g_{t}\left( y, \xi, v \right)^2 \right] \\ 
    \hat{m} &= \alpha \sqrt{\gamma} \mathbb{E}_{\xi \sim \mathcal{N}(0, q)} \! \left[ \sum_y \partial_{\omega} \mathcal{Z}_0 \left( y, \sfrac{m}{q}\xi, v_{\star} \right) g_{t}\left( y, \xi, v \right) \right]
\end{cases}
\end{empheq}
which are the last three equations for $\hat{m}, \hat{q}, \hat{v}$ in Theorem~\ref{thm:jointstats}. 

Therefore, we have shown that the self-consistent equations characterizing the sufficient statistics in Theorem~\ref{thm:jointstats} can be obtained from the replica method by computing the asymptotic free energy density. Moreover, in Sec.~\ref{sec:app:gamp} we have shown that these equations exactly agree with the state evolution equations for a tailored GAMP algorithm \ref{alg:gamp}. 

\subsection{Rigorous version of replica and self-consistent equations}
\label{sec:app:rigor}
As discussed in the introduction of this Appendix, the derivation of Theorem~\ref{thm:jointstats} consists in two steps. First, one constructs a tailored GAMP algorithm \ref{alg:gamp} for which the estimates can be exactly tracked by a set of state evolution equations. Second, one shows that these equations actually agree with the self-consistent equations describing the sufficient statistics for the joint density of interest in Theorem \ref{thm:jointstats}. The first part was discussed in Sec.~\ref{sec:app:gamp}, and although we provided an informal derivation of the state evolution equations, they rigorously follow from the recent progress on state evolution proofs for structured message passing schemes with non-separable priors \cite{Berthier2019, Gerbelot21}. For the second part, in Sec.\ref{sec:app:replicas} we discussed a heuristic derivation of the self-consistent equations from the replica method, and showed it agrees with the state evolution equations from GAMP. Therefore, it remains to rigorously justify this last step. Thankfully, one can resort to a large number of recent progress on generic proofs of the replica predictions  \cite{6069859,barbier_optimal_2019,sur_modern_2018,candes_phase_2018,montanari2020generalization,Dhifallah2020,loureiro_learning_2021}, which we now discuss in detail.

First, let us recall the statement of the theorem in the more general context of the Gaussian covariate model. Let $(\vec{u},\vec{v})$ denote a pair of Gaussian covariates:
\begin{align}
\begin{pmatrix}
        \vec{u} \\ \vec{v}
    \end{pmatrix}
    \sim \mathcal{N}\left(\vec{0}_{d+p},\begin{bmatrix}
        \Psi & \Phi \\ \Phi^{\top} & \Omega
    \end{bmatrix}\right).
\end{align}
For any of the classifiers $t\! \in\! \{ \bo, \erm, \lap, \empbayes \}$ from Sec.~\ref{sec:classifiers}, the 2-dimensional vector $(f_{\star}(\vec{u}), \hat{f}_{t}(\vec{v}))$ is asymptotically distributed as $(\sigma(z), \sigma_{\Tilde{v}}(z_{t}'))$ for some $\Tilde{v}$ that depends on the estimator, where $(z, z_{t}') \sim \mathcal{N}(\mathbf{0}_2, \Sigma_{t})$, and
\begin{align*}
    \Sigma_{t} = 
    \begin{pmatrix}
        \vec{\theta}_{\star}^{\top}\Psi\vec{\theta}_{\star}^{\top} & \hat{\vec{\theta}}_{t}^{\top}\Phi\vec{\theta}_{\star} \\
        \vec{\theta}_{\star}^{\top}\Phi^{\top}\hat{\vec{\theta}}_{t} & \hat{\vec{\theta}}_{t}^{\top}\Omega\hat{\vec{\theta}}_{t}
    \end{pmatrix}
\end{align*}
 where $\hat{\vec{\theta}}_{t}$ is either the unique minimizer the empirical risk in eq.~\eqref{def:risk} for $t\in\{\erm, \lap\}$ or the mean over the respective posterior distribution for $t\in\{\bo,\empbayes\}$. The computation of $\rho_{\star, t}$ thus boils down to computing the sufficient statistics $(\hat{\vec{\theta}}_{t}^{\top}\Phi\vec{\theta}_{\star}, \hat{\vec{\theta}}_{t}^{\top}\Omega\hat{\vec{\theta}}_{t})$. A first important point is that, asymptotically in $p$, these quantities converge in probability to single, deterministic quantities. This was shown in general for sampling problems with log concave measure (such as the one we use in the Bayes-optimal and empirical Bayes method) in \cite{JeanOverlap}, and for empirical risk minimization with convex risks in \cite{loureiro_learning_2021}. We shall thus use the following lemma:
 
\begin{lemma}[Overlap Concentration, from \cite{JeanOverlap,loureiro_learning_2021}]

In the asymptotic limit $p \to \infty$, the random variables $(\hat{\vec{\theta}}_{t}^{\top}\Phi\vec{\theta}_{\star}, \hat{\vec{\theta}}_{t}^{\top}\Omega\hat{\vec{\theta}}_{t})$ converge in probability to some value $(m_t^{\star}, q_t^{\star})$ for $t \in \{ \bo, \erm, \lap, \empbayes \}$.
\end{lemma}
The problem is thus reduced to the computation of these statistics, as a function of the parameters of the problems ($\alpha$, $\gamma$, $\tau_{0}$, etc.) for each of the estimators of interest. In Theorem~\ref{thm:jointstats}, we claim that these are given by the replica equations derived in Appendix \ref{sec:app:replicas}. Thankfully, for different estimators these equations were proven in the literature in slightly different contexts, written as formal proofs of replica predictions.
\begin{itemize}[leftmargin=*]
\item For $\hat{f}_{\erm}$ on the random features model the self-consistent equations for $(m_{\erm}^{\star}, q_{\erm}^{\star})$ were heuristically derived in \cite{gerace_generalisation_2020} and rigorously proven in \cite{Dhifallah2020}. In the more general context of the Gaussian covariate model, analogous equations were proven in \cite{loureiro_learning_2021}. In both cases, they agree with our equations in eq.~\eqref{eq:state_evolution_erm}. While these works use the Gordon minimax approach to prove these equations, we note an independent GAMP-based proof for both the random features and Gaussian covariate models appeared in \cite{loureiro_fluctuations_2022}, leveraging recent progress on structured message passing schemes from \cite{Gerbelot21}.

\item As noted in Sec.~\ref{sec:classifiers}, the average over the Laplace posterior agrees exactly with the empirical risk minimizer: 
\begin{align}
p_{\lap}(\vec{\theta}|\mathcal{D}) = \mathcal{N}(\vec{\theta}|\hat{\vec{\theta}}_{\erm},\mathcal{H}^{-1})
\end{align}
Therefore, the self-consistent equations for $(m_{\lap}^{\star}, q_{\lap}^{\star})$ agree exactly with the ones for $(m_{\erm}^{\star}, q_{\erm}^{\star})$. Therefore, they are also rigorous. 

In both cases, our result follows from:
\begin{theorem}[ERM statistics, Thms. 4 \& 5 from \cite{loureiro_learning_2021}, Informal]
In the setting of Theorem~\ref{thm:jointstats}, the ERM predictions from the replica are correct: $(m,q, v)$ converges in probability to their replica fixed points $(m_{\erm}^{\star}, q_{\erm}^{\star}, v_{\erm}^{\star})$, while the minimum training error converges in probability to the replica free energy density.
\end{theorem}

\item The "finite temperature" sampling problems related to Bayesian estimation pose different challenges. We start by discussing the Bayes-optimal $\hat{f}_{\bo}$ classifier. For i.i.d. Gaussian data, the rigour of the replica prediction has been proven for Generalized linear models in \cite{barbier_optimal_2019}, together with the GAMP optimality. Thanks to Gaussian equivalence, our problem can be framed as a Bayesian generalized linear reconstruction problem, but with data matrix that are instead correlated. In the random features case, the data matrix is a product of two random matrix (see  eq.(\ref{app:get:eq})). Thus, in this case replica predictions were for the Bayes-optimal problem was rigorously proven in \cite{gabrie2018entropy,barbier2018mutual}. Note that while these works only prove the correctness of the  replica free energy density, the techniques in \cite{barbier_optimal_2019} can be readily applied to generalize the proof to overlaps:
\begin{theorem}[BO statistics, Th. 1 from \cite{barbier2018mutual} and Th. 1 from \cite{barbier2018mutual}, Informal]
In the setting of Theorem~\ref{thm:jointstats}, the BO prediction from the replica is correct: $(m, q, v)$ converges in probability to their replica fixed points $(m_{\bo}^{\star}, q_{\bo}^{\star}, v_{\bo}^{\star})$, while the minimum training error converges in probability to the replica free energy density.
\end{theorem}
Additionally, given that the performance of the GAMP algorithms follows the same self-consistent equations as the replica's \cite{Berthier2019,Gerbelot21}, it follows that GAMP performs Bayes-optimal estimation for this problem, a classical property in Bayesian estimation \cite{zdeborova2016statistical}. \footnote{Note that this crucially relies on the strong replica symmetry \cite{JeanOverlap} condition, which impose the existence of an unique fixed point in our problem. Without this property, one could generically have more than a fixed point, associated to a so-called "hard phase" where GAMP is not optimal, see \cite{zdeborova2016statistical}.}

\item The remaining case is the empirical Bayes (EM) classifier $\hat{f}_{\empbayes}$. In this case, where Bayesian estimation is performed {\it with mismatched noise}, the complete proof of the replica equation is not available in the literature. In principle, this can be done following the steps of \cite{barbier2021performance} for the square loss (recall we consider the logistic loss in this work). Indeed, \cite{barbier2021performance} shows how the concentration of $(m^{\star},q^{\star})$ (referred to as strong replica symmetry in \cite{JeanOverlap}) can be used together with rigorous control of the cavity method \cite{aizenman2006mean} to prove the cavity equations. While this is, we believe, a worthwhile direction of research, we instead shall redefine the empirical Bayes method performance as the one of the {\it best empirical Bayesian estimator} in linear time, that is, the estimator achieved by the GAMP algorithm \ref{alg:gamp} with the corresponding empirical Bayes denoiser. It can indeed be shown that GAMP is the best {\it first-order algorithm} for this class of Bayesian estimation problems \cite{celentano2020estimation}, and it is widely expected to perform an exact sampling for these problems \cite{zdeborova2016statistical} (as it was proven in for the Bayes-optimal case). With this definition, the performance of GAMP is by construction given by its rigorous state evolution \cite{Berthier2019,Gerbelot21}, which we recall the reader matches the replica prediction.
\end{itemize}

\subsection{Laplace approximation : computing the inverse Hessian}
\label{sec:app:hessian}

In this section, we show how to compute the prediction for the Laplace approximation
\begin{equation}
    \hat{f}_{\lap}(\vec{v}) = \int \dd z \sigma(z) \mathcal{N}(z | \hat \studentweight_{\erm}^{\top} \vec{v}, \vec{v}^{\top} \mathcal{H}^{-1} \vec{v})
 \end{equation}
 with $\mathcal{H}$ the Hessian of the empirical risk at $\hat \studentweight_{\erm}$. Note that in the high-dimensional limit, $\vec{v}^{\top} \mathcal{H}^{-1} \vec{v} \rightarrow_{p \to \infty} \Tr(\mathcal{H}^{-1} \Omega)$. As shown in Appendix~\ref{app:laplace}, to compute this quantity we can add the term $\vec{h}^{\top} \studentweight$ to the loss and compute the second derivative of the free energy density with respect to $\vec{h}$. The computations are the same as those done in Section~\ref{sec:app:replicas}, except that the Gibbs distribution $\mu_t(\studentweight)$ is replaced by 
\begin{equation}
    \mu_{t}(\studentweight) = \frac{1}{\mathcal{Z}_t(\vec{h})} \prod_i P^{t}_{\sigma}(y_i | \studentweight^{\top} \vec{\varphi}(\vec{x}_i)) \times P^{t}_{\studentweight}(\studentweight) \times e^{\beta \vec{h}^{\top}\studentweight}
\end{equation}
Adapting the derivation from Sec.~\ref{sec:app:replicas} for $\hat{f}_{\erm}$ and taking the temperature $\beta \to \infty$ and get as before
\begin{align}
    f_{0} := \lim\limits_{\beta\to\infty} f_{\beta} &= \underset{m,q,v,\hat{m},\hat{q},\hat{v}}{\extr}\left\{ - \frac{1}{\sqrt{\gamma}} m \hat{m} + \frac{1}{2}(q \hat{v} - \hat{q} v) + \Psi_w(\hat{m},\hat{q},\hat{v}, \vec{h}) + \alpha \Psi_y(m,q,v) \right\} \\
    \Psi_y  &= \mathbb{E}_{\xi \sim \mathcal{N}(0,q)} \left[ \sum_y \mathcal{Z}_0(y, \sfrac{m}{q}\xi, \rho - \sfrac{m^2}{q}) \log \mathcal{Z}_g(y, \xi, v) \right]
\end{align}
However, now $\Psi_w$ is
\begin{equation}
    \Psi_w = -\frac{1}{2p} \Tr \log \left( \hat{v} \Omega + \Sigma \right) + \frac{1}{2p} \Tr \left[ ( (\hat{m} \Omega \projectedteacher + \vec{h})( \hat{v} \Omega + \lambda I )^{-1}(\hat{m} \Omega \projectedteacher + \vec{h}) + \hat{q} \Omega ( \hat{v} \Omega + \Sigma )^{-1})  \right]
\end{equation}

The second derivative of $\Psi_w$ with respect to $\vec{h}$ is $(\lambda I_d + \hat{v} \Omega)^{-1}$. As a consequence, the second derivative of the free energy
$$
(\nabla^2_{\mathbf{h}} \log \mathcal{Z}_{\erm})_{| \vec{h} = \vec{0}} = \nabla^2_{\mathbf{h}} \Psi_w(m_{\erm}^{\star},q_{\erm}^{\star}, v^{\star}_{\erm}, \vec{h}) = (\lambda I_d + \hat{v}^{\star}_{\erm} \Omega)^{-1}
$$
and $\nabla^2_{\mathbf{h}} f_{0} = - (\lambda I_d + \hat{v}^{\star}_{\erm} \Omega)^{-1}$. We then deduce that the inverse Hessian is equal to
\begin{equation}
    \mathcal{H}^{-1} = \left( \lambda \mat{I}_d + \hat v^{\star}_{\erm} \Omega \right)^{-1} 
\end{equation}

\subsection{Simplification for random features}
\label{sec:app_random_features}
As discussed in Appendix \ref{sec:app:derivation}, the random features model $\vec{\varphi}(\vec{x}) = \phi(\mat{F}\vec{x})$ is asymptotically equivalent to the Gaussian covariate model up to an identification of the covariances:
\begin{align}
\Omega = \kappa_{1}^2 FF^{\top} + \kappa_{\star}^2 \mat{I}_p, && \Phi = \kappa_1 F, && \Psi = \mat{I}_{d} 
\end{align}
\noindent where:
\begin{align}
\kappa_{1} = \mathbb{E}_{z\sim\mathcal{N}(0,1)}\left[\phi'(z)\right], && \kappa_{\star} = \sqrt{\mathbb{E}_{z\sim\mathcal{N}(0,1)}\left[\phi(z)^2\right] - \kappa_{1}^{2}}
\end{align}
\noindent where for simplicity we assume $\kappa_{0} = \mathbb{E}_{z\sim\mathcal{N}(0,1)}[\phi(z)]$ = 0. Thus, in this case we can explicitly write: 
\begin{align}
    \Sigma_{\star} = \frac{\kappa_1^2 \mat{F}\mat{F}^{\top}}{\kappa_{\star}^2 \mat{I}_p + \kappa_1^2 \mat{F}\mat{F}^{\top}}.
\end{align}.
Note that the matrices $\Omega, \Sigma, \Sigma_{\star}$ are diagonalizable in the same basis, since $\Sigma$ is either a multiple of the identity, or a function of $\Omega, \Phi\Phi^{\top}$. Assuming that $\mat{F}\mat{F}^{\top}$ has an asymptotic spectral distribution $\mu$, we can write $\Psi_w$ directly in terms of an average over $\mu$: 
\begin{equation}
    \Psi_w = \frac{1}{2} \mathbb{E}_{x \sim \mu)} \left[ \log \left( \hat{v}(\kappa_1^2 x + \kappa_{\star}^2) + \pi(x)  \right) + \left( \frac{\hat{m}^2 \frac{\kappa_1^2 x}{\kappa_1^2 x + \kappa_{\star}^2} + \hat{q}(\kappa_1^2 x + \kappa_{\star}^2)}{\hat{v}(\kappa_1^2 x + \kappa_{\star}^2) + \pi(x)} \right) \right]
\end{equation}
where the function $\pi$ represents the eigenvalues of $\Sigma$ : since we can write $\Sigma = f(\Phi \Phi^{\top})$ here, we have, $\pi(x) = f(x)$. For $\hat{f}_{\erm}$ and $\hat{f}_{\empbayes}$, $\pi(x) = \lambda$. For $\hat{f}_{\bo}$, $\pi(x) = \frac{\kappa_1^2 x}{\kappa_1^2 x + \kappa_{\star}}$. This gives us the values of $\hat{\pi}_t$ in Table~\ref{table:channel_denoising_functions}. 

In particular, when $\mat{F}$ has Gaussian i.i.d. entries (as in all plots presented here), $\mu$ is simply the Marcenko-Pastur distribution with shape parameter $\gamma$.

\subsection{Temperature scaling}

In this section, we show how to compute the optimal temperature $T$ that minimizes the test loss for $\hat{f}_{\erm}$. Once the overlaps $m^{\star}, q^{\star}, v^{\star}, \hat{m}^{\star}, \hat{q}^{\star}, \hat{v}^{\star}$ are computed, we get the test loss with the expression
\begin{equation}
    \mathcal{L}_{\rm gen.}(m, q) = \sum_y \mathbb{E}_{\xi \sim \mathcal{N}(0, 1)}\left[ \mathcal{Z}_0(y, \sfrac{m^{\star}}{q^{\star}} \xi, \rho - m^2 / q) \times (- \log \sigma( y \times \sqrt{q} \xi)) \right]
\end{equation}

Given a temperature $T$, temperature scaling will divide the weights such that the prediction is now $\sigma(\studentweight^{\top} \vec{\varphi}(\vec{x} / T)$. It is easy to see that in this case, the overlaps $m^{\star}, q^{\star}$ now become $m^{\star} / T, q^{\star} / T^2$. Then, temperature scaling amounts to finding
\begin{equation}
    T^{\star} = \argmin_T \mathcal{L}_{\rm gen.}(m^{\star} / T, q^{\star} / T^2)
\end{equation}

\newpage
\section{Confidence function and Hessian of Laplace method}
\label{app:laplace}

\subsection{Computing the Hessian of the training loss}
\label{app:hessian}
In this section, we show how we can compute the (inverse of) the Hessian thanks to classical properties of Legendre transforms. We consider the ERM estimator $\hat{f}_{\erm}$ trained by minimizing the following loss : 
\begin{equation}
    \mathcal{L}(\vec w) = - \sum_i \log \sigma(\studentweight^{\top} \vec{\varphi}(\vec{x})_i \times y_i) + \sfrac{\lambda}{2} \| \studentweight \|^2
\end{equation}
whose Hessian at the minimum is given by 
\begin{equation}
    \mathcal{H} \coloneqq \nabla^2 \mathcal{L} = - \sum_i (1 - \sigma'(\studentweight^{\top} \vec{\varphi}(\vec{x})_i \times y_i))\vec{\varphi}(\vec{x})_i \vec{\varphi}(\vec{x})_i^{\top} + \lambda I_d 
    \big|_{\vec{\theta}=\hat{\vec{\theta}}_{\erm}}
    \label{eq:def_hessian}
\end{equation}

Our starting point to compute this Hessian is a very classical lemma in statistical mechanics, that uses the Legendre transform of the loss.
\begin{lemma}[Inverse Hessian from Legendre Transforms]
    We define the Legendre transform of the loss by adding a source term to the loss (an external field in the parlance of statistical mechanics) 
    \begin{equation}
    \mathcal{L}^L(\vec{h}) =     \min_{\studentweight} \left[-\sum_i \log\sigma(y_i \studentweight^{\top} \vec{\varphi}(\vec{x})_i) + \sfrac{\lambda}{2} \| \studentweight \|^2 + \vec{h}^{\top} \studentweight\right] =  \min_{\studentweight} \left[ \mathcal{L}(\studentweight) + \vec{h}^{\top} \studentweight\right]
    \end{equation}
    then the Inverse of the Hessian \eqref{eq:def_hessian} is the Hessian of the Legendre transform $ \mathcal{L}(\vec{h})$
    \begin{equation}
    \mathcal{H}^{-1}(\hat{\studentweight}_{\erm}) =  -{\frac{\partial^2   \mathcal{L}^L(\vec{h})}{\partial^2 \vec{h}}} \big|_{\vec{h} = 0}
\end{equation}
\end{lemma}

\paragraph{Proof}
This is a classical result from Legendre transform of strongly convex functions, which we informally recall.  First notice that at the minimum of $\mathcal{L}(\studentweight) + \vec{h}^{\top} \studentweight$ over $\studentweight$ is characterized by
\begin{equation}
\frac{\partial \mathcal{L}(\studentweight)}{\partial \theta_j} + h_j = 0 \,\, \forall j \label{app:eq1}
\end{equation}
so that 
\begin{eqnarray}
\frac    {\partial \mathcal{L}^L(\vec{h})}{\partial h_i} &=& \frac    {\partial \left[ \mathcal{L}(\studentweight) + \vec{h}^{\top} \studentweight\right]}{\partial h_i} \big|_{\hat{\studentweight}_{\erm}} = \sum_{j=1}^{p}  \left[  \frac{\partial \mathcal{L}(\studentweight)}{\partial \theta_j} \frac{\partial \theta_j}{\partial h_i} + h_j  \frac{\partial \theta_j}{\partial h_i} \right]\big|_{\hat{\studentweight}_{\erm}}   + \theta_i \big|_{\hat \studentweight_{\erm}}  \nonumber \\
&=& \sum_{j=1}^{p} \frac{\partial \theta_j}{\partial h_i} \left[  \frac{\partial \mathcal{L}(\studentweight)}{\partial \theta_j}  + h_j  \right]\big|_{\hat{\studentweight}_{\erm}}   + \theta_i \big|_{\hat \studentweight_{\erm}} = \theta_{\erm,i}
\end{eqnarray}
It thus follows that 
\begin{equation}
  \frac{\partial^2   \mathcal{L}^L(\vec{h})}{\partial h_i \partial h_j} = \frac{\partial \theta_i}{\partial h_j}\,. \label{app:res1}
\end{equation}
However, we have from eq.(\ref{app:eq1})
\begin{equation}
\frac {\partial^2  \mathcal{L}(\studentweight)  }{\partial \theta_i\partial \theta_j} = - \frac {\partial h_j}{ \partial \theta_i}\,.
\label{app:res2}
\end{equation}
Using both eqs (\ref{app:res1}) and (\ref{app:res2}) at $h=0$ concludes the proof.\hfill $\square$

Note that this relation is not asymptotic and is valid for a given instance of the problem. This lemma is however particularly practical in the large $n$ limit, since an asymptotic expression for the loss ${\mathcal L}^L$ is known so that we can use it to obtain the asymptotic expression.  Using the value of the minimal loss from \cite{loureiro_learning_2021,loureiro_fluctuations_2022}, we deduce, taking its second derivative, that for large $n$ we must have (See Section~\ref{sec:app:hessian} for the derivation)
\begin{equation}
    \mathcal{H}_{{\rm rep.}}^{-1} = (\lambda \mat{I}_p + \hat{v}^{\star}_{\erm}\Omega)^{-1}
    \label{eq:hessian}
\end{equation}
Where $\hat{v}_{\erm}^{\star}$ is the unique solution the following self-consistent equations:

\begin{align}
\label{eq:app:replicassp}
    \begin{cases}
        m &= \frac{\gamma \hat{m}}{p} \Tr \left( \Omega \projectedteacher \projectedteacher^{\top} \Omega (\lambda \mat{I}_p + \hat{v}\Omega)^{-1}) \right) \\
        q &= \frac{1}{p} \Tr \left( (\hat{q}\Omega + \hat{m}^2 \Omega \projectedteacher \projectedteacher^{\top} \Omega ) \Omega (\lambda \mat{I}_p + \hat{v}\Omega)^{-2} \right) \\
        v &= \frac{1}{p} \Tr (\lambda \mat{I}_p + \hat{v} \Omega)^{-1} \Omega)
    \end{cases}, && 
    \begin{cases}
    \hat{v} &= - \alpha \mathbb{E}_{\xi \sim \mathcal{N}(0, q)}\!\left[ \sum_y \mathcal{Z}_0 \left( y, \sfrac{m}{q}\xi, v_{\star} \right)  \partial_{\omega} f_{{\rm out}, \erm} \left( y, \xi, v \right) \right] \\
    \hat{q} &= \alpha \mathbb{E}_{\xi \sim \mathcal{N}(0, q)} \! \left[ \sum_y \mathcal{Z}_0 \left( y, \sfrac{m}{q}\xi, v_{\star} \right) f_{{\rm out}, \erm}\left( y, \xi, v \right)^2 \right] \\ 
    \hat{m} &= \alpha  \mathbb{E}_{\xi \sim \mathcal{N}(0, q)} \! \left[ \sum_y \partial_{\omega} \mathcal{Z}_0 \left( y, \sfrac{m}{q}\xi, v_{\star} \right) f_{{\rm out},\erm}\left( y, \xi, v \right) \right]
    \end{cases}
\end{align}
Note that the subscript emphasizes that this expression was obtained by differentiating the asymptotic free energy density. 

It would be tempting to assume that the convergence of the free energy to its asymptotic value in \cite{loureiro_learning_2021} would also be valid for the second derivative, so that the replica Hessian would be close, pointwise, to the actual Hessian when $p \to \infty$. This, however, turns out to be wrong, as one can easily check in the ridge regression case. However, we conjecture that the limit holds in the sense of deterministic equivalents \cite{hachem2007deterministic}. This leads us to the following conjecture:
\begin{conjecture}[Deterministic equivalent of the inverse Hessian]
\label{thm:app:hessian}
    For any deterministic matrix $A\in\mathbb{R}^{p\times p}$, in the asymptotic limit where $n,d,p\to\infty$ at fixed ratios $\alpha =\sfrac{n}{p}$ and $\gamma = \sfrac{d}{p}$, we have:
    \begin{equation}
    \lim\limits_{p\to\infty}\frac{1}{p}{\rm Tr}\left(A\mathcal{H}^{-1}\right) = \lim\limits_{p\to\infty}\frac{1}{p}{\rm Tr}\left[A(\lambda \mat{I}_p + \hat{v}^{\star}_{\erm}\Omega)^{-1}\right]
\end{equation}
\end{conjecture}
For the purpose of characterizing the Laplace approximation, we apply this formula using $A = \vec{v}\vec{v}^{\top}$. Proving rigorously the convergence in the sense of deterministic equivalent remains an open problem. It is, however, easy to prove that it is valid for the square loss, as we show in the next section. Fortunately, it can be checked numerically to great precision that the conjecture is empirically satisfied for the study of Laplace method, as is also shown in the next section.

Additionally, we note that the statement is made in term of an expression of the inverse of the Hessian (which, conveniently, is actually what we want to know). 

\subsection{An instructive example: the square loss}
Although in this work we only focus in the classification case, Conj. \ref{thm:app:hessian} actually applies in the more general context of a convex loss function $\ell$, for which the exact asymptotics was characterized in \cite{loureiro_learning_2021}. An instructive example is therefore given by looking at the square loss $\ell(y,x)=\sfrac{1}{2}(y-x)^2$, for which the Hessian is simply given by:
\begin{align}
    \mathcal{H} \coloneqq \nabla^{2}\mathcal{L}= V^{\top}V + \lambda I_{p}
\end{align}
\noindent where $V\in\mathbb{R}^{n\times p}$ is the feature matrix with rows given by $\varphi(\vec{x^{\mu}})\in\mathbb{R}^{p}$ for $\mu\in[n]$. Therefore, it is independent of the minimizer $\hat{\vec{\theta}}_{\erm}$. In this case, Conj. \ref{thm:app:hessian} boils down to the classical deterministic equivalents for the sample covariance matrix $\hat{\Omega}_{n}\coloneqq V^{\top}V$. Deterministic equivalents for sample covariance matrices have been characterized under different levels of generality for $V$ \cite{10.2307/24308489, Knowles2017, louart2018random, chouard2022quantitative, Schroder2023}, including in particular the random features case $V = \sigma(FX^{\top})$ with $X$ i.i.d. Gaussian considered here. They state precisely that, for any deterministic matrix $A\in\mathbb{R}^{p\times p}$ and in the asymptotic limit considered here:
\begin{align}
   \lim\limits_{p\to\infty} \frac{1}{p}{\rm Tr}\left[A\left(V^{\top}V + \lambda I_{p}\right)^{-1}\right] =\lim\limits_{p\to\infty} \frac{1}{p}{\rm Tr}\left[A\left(\hat{v}^{\star}\Omega+\lambda I_{p}\right)^{-1}\right]
\end{align}
\noindent where $\Omega = \mathbb{E}\left[\sigma(F\vec{x})\sigma(F\vec{x})^{\top}\right]$ is the population covariance of the features and $\hat{v}^{\star}$ is the solution of the following self-consistent equation:
\begin{align}
\label{eq:app:selfcons}
    \frac{\alpha}{\hat{v}} - 1 = 1-\lambda\int\frac{\mu_{\Omega}(\dd t)}{\lambda+\hat{v} t}
\end{align}
\noindent with $\mu_{\Omega}$ the asymptotic spectral density of $\Omega$. It is not hard to check that this self-consistent equation \eqref{eq:app:selfcons} coincides exactly with the self-consistent equations \eqref{eq:app:replicassp} from \cite{loureiro_learning_2021} when $\ell$ is the square loss.
\subsection{Comparison with numerics}

In this section, we apply the computations of the previous section and show that they gives extremly good prediction even at very moderate sizes. 
In Figure~\ref{fig:laplace_experimental}, we compare the theoretical value of $\vec{\varphi}(\vec{x})^{\top} \mathcal{H}^{-1} \vec{\varphi}(\vec{x})$ for $\vec{\varphi}(\vec{x})=\rm{erf}(\mat{F}\vec{x})$ from eq.~\eqref{eq:hessian} and the one observed experimentally. Experiments are done by training the logistic classifier $\hat{f}_{\erm}$ on training data $(\vec{x}^{\mu}, y^{\mu})_{\mu\in[n]}$ and computing the Hessian \eqref{app:hessian} at the minimizer $\hat{\studentweight}_{\erm}$. We observe a good fit between theory and experiment, validation our analysis.
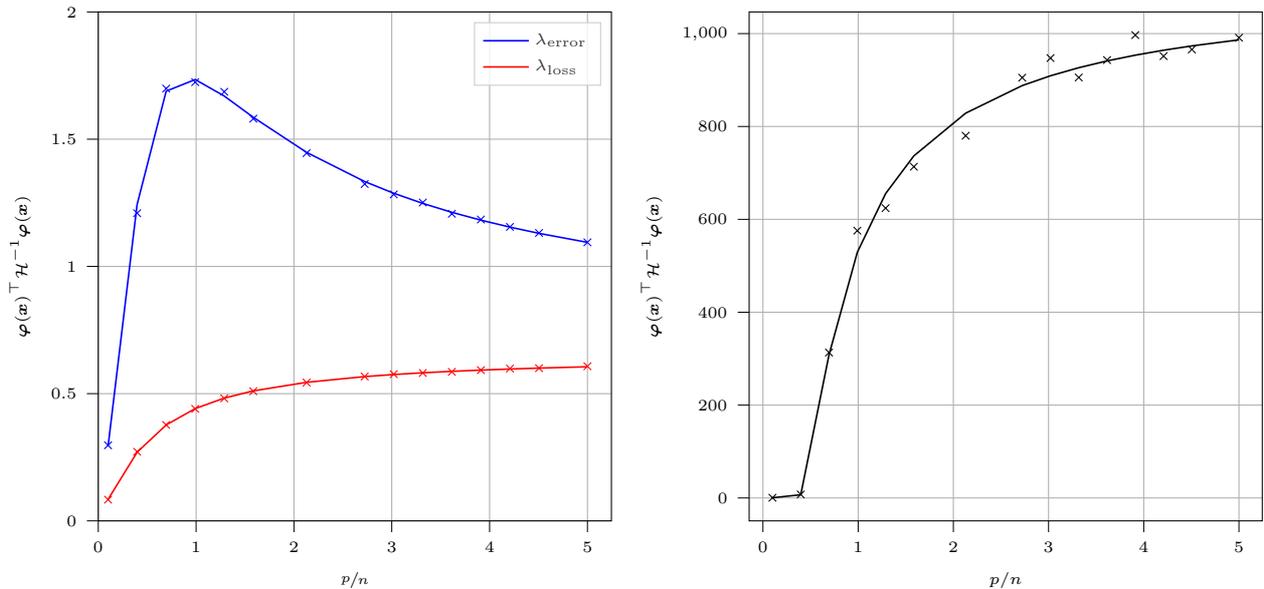
\begin{figure}
    \centering
    \def\figwidth{0.49\textwidth}
    \def\figheight{0.49\textwidth}
    
    \input{Figures/supplementary/laplace_supplementary/hessian_laplace}
    \input{Figures/supplementary/laplace_supplementary/hessian_trace_lambda=0.0001}
    
    \caption{ (\textbf{Left}) Theoretical predictions (lines) and experimental values (crosses) of $\vec{\varphi}(\vec{x})^{\top} \mathcal{H}^{-1} \vec{\varphi}(\vec{x})$ with $\sfrac{n}{d} = 2, \noisevar = 0.5$, $\vec{\varphi}(\vec{x})=\rm{erf}(\mat{F}\vec{x})$ and $\mat{F}$ Gaussian, as in Figure~\ref{fig:test_errors_n_over_p=2.0}, for $\lambdaerror$ and $\lambdaloss$. Experimental values are obtained by fixing $d = 256$. (\textbf{Right}) Theoretical and experimental values for $\lambda = 10^{-4}$.}
    \label{fig:laplace_experimental}
\end{figure}

\vfill

\section{Conditional variance of the Bayes-optimal estimator}
\label{app:cond_variance}

In this section, we prove the expression of the variance of $\hat{f}_{\bo}$ conditioned on the confidence of other estimators : 
\begin{equation}
    \variance(\hat{f}_{\bo}(\vec{x}) | \hat{f}_{t}(\vec{x})=\ell) =  \int {\rm d}a~\sigma_{\hat{v}^{\star}_{\bo} + \noisevar + \tau_{\rm add}^2}(a)^2 \times\mathcal{N} \left( a | \sfrac{m^{\star}_{t}}{q^{\star}_{t}} \sigma_{\hat{\tau}_t}^{-1}(\ell), q^{\star}_{\bo} - \sfrac{{m^{\star}_{t}}^2}{q^{\star}_{t}} \right)-\left(\ell-\Delta_{\ell}\right)^2
    \label{eq:conditional_variance_bo}
\end{equation}

The first step is to show that for any estimator $t \in \{ \erm, \empbayes, \lap \}$, the joint density of the confidence of $\hat{f}_{\bo}, \hat{f}$, defined as 
\begin{equation}
    \rho_{\bo, t}(a, b) = \mathbb{P}_{\vec{x}} (\hat{f}_{\bo}(\vec{x}) = a, \hat{f}_{t}(\vec{x}) = b)
\end{equation}
can be computed in the similar way as $\rho_{\star, t}$ in Theorem~\ref{thm:jointstats}. This was shown previously for a simpler model in \cite{clarte_theoretical_2022}, where the teacher and input data have identity covariance.

\begin{lemma}
\label{lemma:bo_t_density}
In the same setting as Theorem~\ref{thm:jointstats}, in the asymptotic limit, the density $\rho_{\bo, t}(a, b)$ converges to $\rho^{\lim}_{\bo, t}(a, b)$
\begin{equation}
    \rho^{\rm lim}_{\bo, t}(a, b) = \frac{\mathcal{N}\left( \begin{bmatrix} \sigma_{\noisestr^2 + \addnoisestr^2}^{-1}(a) \\ \sigma_{\hat{\tau}^2_t}^{-1}(b) \end{bmatrix}  \Big| \mathbf{0}_2, \Sigma_{bo, t} \right)}{ | \sigma_{\noisestr^2 + \addnoisestr^2}' ( \sigma_{\noisestr^2 + \addnoisestr^2}^{-1}(a)) | 
            | \sigma_{\hat{\tau}^2_t}' ( \sigma_{\hat{\tau}^2_t}^{-1}(b)) |  }  \label{eq:res:jointdensity_bo}
\end{equation}
where this time
\begin{equation}
    \label{eq:def_sigma_bo_t}
    \Sigma_{\bo, t} = \begin{bmatrix} q_{\bo}^{\star} & m_{t}^{\star} \\ m_{t}^{\star} & q_{t}^{\star} \end{bmatrix}
\end{equation}

\end{lemma}

To prove Lemma~\ref{lemma:bo_t_density}, the main idea is to observe that, as with $f_{\star}$, to compute the density we need the covariance matrix
\begin{equation}
    \frac{1}{d}\begin{pmatrix}
        \hat{\studentweight}_{\bo}^{\top} \Omega \hat{\studentweight}_{\bo} & \hat{\studentweight}_{\bo}^{\top} \Omega \studentweight_t \\
        \hat{\studentweight}_{\bo}^{\top} \Omega \hat{\studentweight}_t &  \hat{\studentweight}_{t}^{\top} \Omega \studentweight_{t} \\
    \end{pmatrix}
\end{equation}
The diagonal terms are $q_{\bo}^{\star}, q_{t}^{\star}$ respectively by definition. We then just need to compute the overlap $m_{\bo, t} = \frac{1}{d}\hat{\studentweight}_{\bo}^{\top} \Omega \hat{\studentweight}_t$. Our goal is to prove that $m_{\bo, t} = m^{\star}_t$, 

However, using the Nishimori identity from statistical physics, for any vector $\vec{z}(\mathcal{D})$ that can depend on the training data, we have
\begin{equation}
    \mathbb{E}_{\mathcal{D}} \left( \hat{\theta}_{\bo}^{\top} \vec{z}(\mathcal{D}) \right)=\mathbb{E}_{\projectedteacher, \mathcal{D}} \left( \projectedteacher^{\top} \vec{z}(\mathcal{D})\right)
    \label{eq:general_nishimori}
\end{equation}
Equation~\ref{eq:general_nishimori} is just an application of Bayes formula. In particular, if we take $\vec{z}(\mathcal{D}) = \hat{\studentweight}_t$, we obtain that 
\begin{equation}
    \mathbb{E}_{\mathcal{D}} ( \hat{\theta}_{\bo}^{\top} \hat{\studentweight}_t) =\mathbb{E}_{\projectedteacher, \mathcal{D}} ( \projectedteacher^{\top} \hat{\studentweight}_t) 
\end{equation}
and we see that in expectation, $\mathbb{E}_{\mathcal{D}}(m_{\bo, t}) = \mathbb{E}_{\projectedteacher, \mathcal{D}}(m^{\star}_t)$. We already know that the right-hand side of the equality self-averages, i.e $\lim_{d \to \infty} \mathbb{E}_{\projectedteacher, \mathcal{D}}(m^{\star}_t) = m^{\star}_t$. It remains to show that the left-hand side also self-averages.

\begin{lemma}[Concentration of the overlap $m_{\bo, t}$]
\begin{equation}
  \lim_{d \to \infty} {\mathbb E}\left[\left(\frac{\hat{\studentweight}_{\bo}^{\top} \studentweight_{t}}{d}\right)^2\right] =  \lim_{d \to \infty} {\mathbb E}\left[\frac{\hat{\studentweight}_{\bo}^{\top} \studentweight_{t}}{d} \right]^2
\end{equation}

\end{lemma}
\begin{proof}
The proof again uses Nishimori identity. 
\begin{eqnarray}
{\mathbb E}\left[\left(\frac{\hat{\studentweight}_{\bo}^{\top} \studentweight_{t}}{d}\right)^2\right] &=& {\mathbb E}\left[\left(\frac{\hat{\studentweight}_{\bo}^{\top} \studentweight_{t}}{d}\right)\left(\frac{\hat{\studentweight}_{\bo}^{\top} \studentweight_{t}}{d}\right)\right] \\ &=& 
{\mathbb E}_{\mathcal D}\left[\left(\frac{ {\mathbb E}_{\hat{\studentweight}|\mathcal D} \hat{\studentweight}^{\top} \studentweight_{t}}{d}\right)\left(\frac{ {\mathbb E}_{\hat{\studentweight}|\mathcal D} \hat{\studentweight}\cdot \studentweight_{t}}{d}\right)\right] \\
&=& 
{\mathbb E}_{\mathcal D}
{\mathbb E}_{\hat{\studentweight}_1,\hat{\studentweight}_2|\mathcal D}
\left[\left(\frac{\hat{\studentweight}_1^{\top} \studentweight_{t}}{d}\right)\left(\frac{ \hat{\studentweight}_2\cdot \studentweight_{t}}{d}\right)\right] \\
&=& 
{\mathbb E}_{\mathcal D,\projectedteacher}
\left[\left(\frac{\projectedteacher^{\top} \studentweight_{t}}{d}\right)\left(\frac{ {\mathbb E}_{\hat{\studentweight}|\mathcal D} \hat{\studentweight}\cdot \studentweight_{t}}{d}\right)\right] \\
&=& 
{\mathbb E}_{\mathcal D,\projectedteacher}
\left[\left(\frac{\projectedteacher^{\top} \studentweight_{t}}{d}\right)\left(\frac{ {\vec w}_{\bo}^{\top} \studentweight_{t}}{d}\right)\right] 
\end{eqnarray}
Then, from Cauchy-Schwartz we have
\begin{eqnarray}
{\mathbb E}\left[\left(\frac{\hat{\studentweight}_{\bo}^{\top} \studentweight_{t}}{d}\right)^2\right]^2 &\le& {\mathbb E}\left[\left(\frac{\hat{\studentweight}_{\bo}^{\top} \studentweight_{t}}{d}\right)^2\right]
{\mathbb E}\left[\left(\frac{\projectedteacher^{\top} \studentweight_{t}}{d}\right)^2\right] \\
{\mathbb E}\left[\left(\frac{\hat{\studentweight}_{\bo}^{\top} \studentweight_{t}}{d}\right)^2\right] &\le&
{\mathbb E}\left[\left(\frac{\projectedteacher^{\top} \studentweight_{t}}{d}\right)^2\right]
\end{eqnarray}
and as $d \to \infty$, we can use the concentration of the right hand side to $m^{\star}_t$ to obtain
\begin{eqnarray}
\lim_{d \to \infty} {\mathbb E}\left[\left(\frac{\hat{\studentweight}_{\bo}^{\top} \studentweight_{t}}{d}\right)^2\right] &\le&
(m^{\star}_t)^2
\end{eqnarray}
so that, given the second moment has to be larger or equal to its (squared) mean: 
\begin{eqnarray}
\lim_{d \to \infty} {\mathbb E}\left[\left(\frac{\hat{\studentweight}_{\bo}^{\top} \studentweight_{t}}{d}\right)^2\right] &=& (m^{\star}_t)^2
\end{eqnarray}
\end{proof}
We have thus shown that $m^{\star}_t = m_{\bo, t}$, proving Lemma~\ref{lemma:bo_t_density}.

\noindent \paragraph{Computing the conditional variance} Fix now the confidence $\hat{f}_t = \ell$, the local field of the estimator $t$ is $\nu_t := \sigma_{\hat{\tau}_t}^{-1}(\ell)$.
The conditional distribution of the Bayes-optimal local field $\lambda_{\bo}$ is a Gaussian $\mathcal{N}(\sfrac{m^{\star}_t}{q^{\star}_t} \nu_t, q_{\bo}^{\star} - \sfrac{{m^{\star}_t}^2}{q^{\star}_t})$. Thus,
\begin{equation}
    \mathbb{E} \left( \hat{f}_{\bo}^2 | \hat{f} = \ell \right) = \int \dd z \sigma_{v^{\star}_{\bo} + \noisevar + \addnoisevar}(z)^2 \mathcal{N}(z | \sfrac{m^{\star}_t}{q^{\star}_t} \nu_t, q_{\bo}^{\star} - \sfrac{{m^{\star}_t}^2}{q^{\star}_t})
\end{equation}
The last step to prove eq.~\eqref{eq:conditional_variance_bo} is to show that
$\mathbb{E} \left( \hat{f}_{\bo} | \hat{f}_{t} = \ell \right) = \ell - \Delta_{\ell}$ : 
\begin{align*}
    \mathbb{E} \left( \hat{f}_{\bo} | \hat{f}_{t} = \ell \right) &= \int \sigma_{\noisevar + \addnoisevar + \hat{v}_{\bo}^*}(z) \mathcal{N}( z | \sfrac{m^{\star}_t}{q^{\star}_t} \nu_t , q_{\bo}^{\star} - \sfrac{{m^{\star}_t}^2}{q^{\star}_t}) \dd z  \\
    &= \sigma_{\noisevar + \addnoisevar + \hat{v}_{\bo}^* + q_{\bo}^{\star} - \sfrac{{m^{\star}_t}^2}{q^{\star}_t}}( \sfrac{m^{\star}_t}{q^{\star}_t} \nu_t ) \\
    &= \sigma_{\noisevar + \addnoisevar + \rho - \sfrac{{m^{\star}_t}^2}{q^{\star}_t}}( \sfrac{m^{\star}_t}{q^{\star}_t} \nu_t ) = \mathbb{E} \left( f_{\star} | \hat{f}_t = \ell \right) = \ell - \Delta_{\ell}
\end{align*}
since, due to Bayes optimality, $\hat{v}_{\bo}^* = \rho - q_{\bo}^{\star}$.

\newpage
\section{Additional numerical evaluations}
\label{app:other_settings}

\subsection{Calibration at different levels}

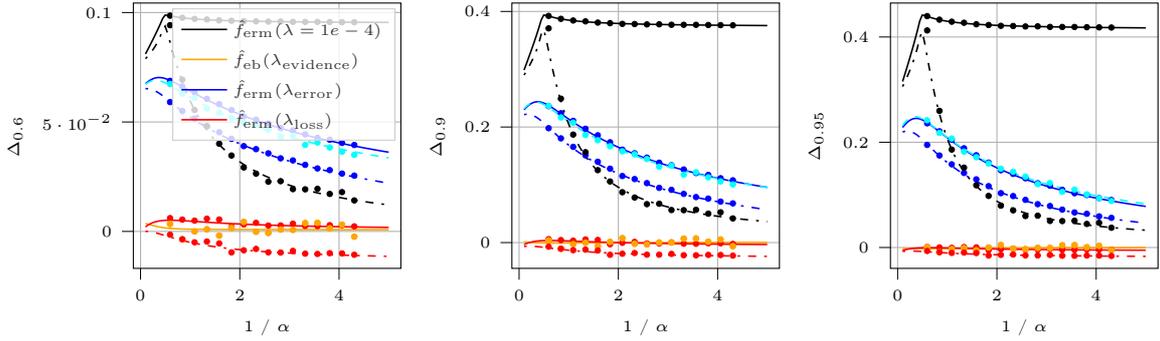
\begin{figure}[h!]
    \centering
    \def\figwidth{0.3\columnwidth}
    \def\figheight{0.3\columnwidth}
    
    \input{Figures/exp_calibration_p=0.6.tex}
    \input{Figures/exp_calibration_p=0.9.tex}
    \input{Figures/exp_calibration_p=0.95.tex}
    \caption{Calibration of several estimators in the same setting as Figure~\ref{fig:test_errors_n_over_p=2.0} at level $\ell = 0.6$ (Left), $\ell = 0.9$ (Middle), $\ell = 0.95$ (Right). Dashed lines correspond to $\hat{f}_{\lap}$. We observe the same phenomenology as in Figure~\ref{fig:test_errors_n_over_p=2.0}, as $\hat{f}_{\lap}$ tends to be underconfident, and $\hat{f}_{\empbayes}(\lambdaevidence)$ is the best calibrated estimator across all levels. Dots correspond to experiments at $d = 200$.}
    \label{fig:my_label}
\end{figure}

\subsection{Additional setting : $\noisevar = 0, \sfrac{n}{d} = 10.0$}
\label{app:add_setting_1}
\begin{figure*}[h!]
    \centering
    \def\figwidth{0.3\columnwidth}
    \def\figheight{0.3\columnwidth}

    \input{Figures/supplementary/n_over_p=10/test_errors_n_over_p=10.0.tex}
    \input{Figures/supplementary/n_over_p=10/calibration_p=0.75_n_over_p=10.0}
    \input{Figures/supplementary/n_over_p=10/conditional_variance_bo_p=0.75_n_over_p=10.0.tex}

    \caption{(\textbf{Left}) Test error of the estimators as a function of $\sfrac{1}{\alpha}$ in the setting of Section~\ref{app:add_setting_1} : $\| \vec{\theta}_{\star} \|^2 = 1, \noisevar = 0, \sfrac{n}{d} = 10$. (\textbf{Middle}) Calibration of the estimators. (\textbf{Right}) Variance of $\hat{f}_{\bo}$ conditioned on $\hat{f} = 0.75$ for the different estimators. }
    \label{fig:add_figures_n_over_p=10.0}
\end{figure*}
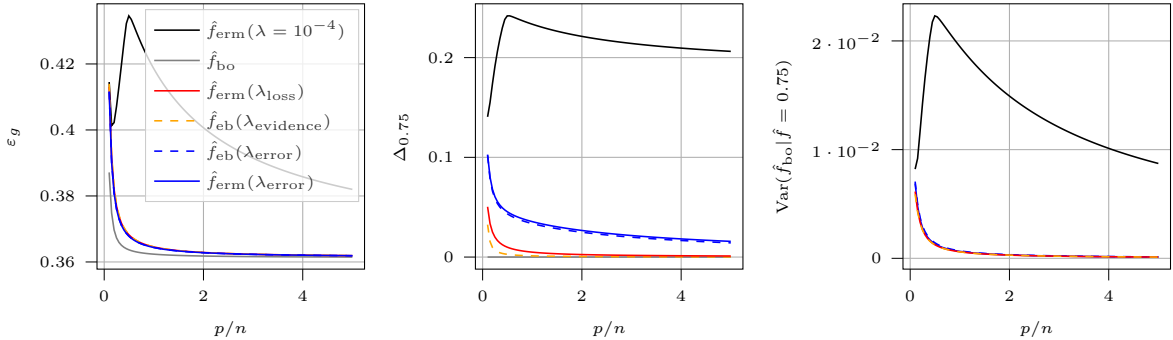

\begin{figure*}[h!]
    \centering
    \def\figwidth{0.3\columnwidth}
    \def\figheight{0.3\columnwidth}

    \input{Figures/supplementary/n_over_p=10/calibration_p=0.75_laplace_n_over_p=10.0.tex}
    \input{Figures/supplementary/n_over_p=10/calibration_temp_scaling_p=0.75_n_over_p=10.0.tex}
    \input{Figures/supplementary/n_over_p=10/conditional_variance_bo_temp_scaling_n_over_p=10.0.tex}

    \caption{(\textbf{Left}) Calibration of $\hat{f}_{\lap}$ and $\hat{f}_{\erm}$ in the setting of Section~\ref{app:add_setting_1}. (\textbf{Middle}) Calibration of $\hat{f}_{\erm}$ after temperature scaling. Curves for $\lambdaerror$ and $\lambdaloss$ are indistinguishable on the plot. (\textbf{Right}) Variance of $\hat{f}_{\bo}$ conditioned on the confidence of temperature scaling.}
    \label{fig:add_figures_n_over_p=10.0_laplace_ts}    
\end{figure*}
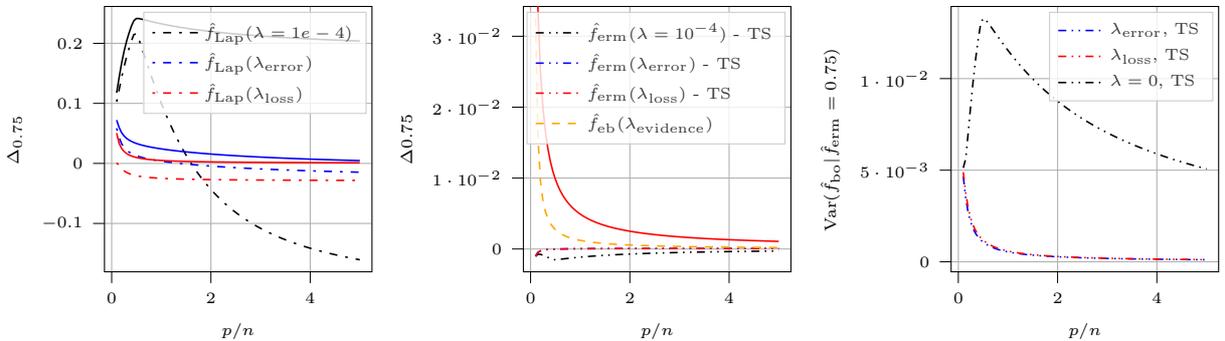

In this section, we consider a setting where $\noisevar = 0.0$. This allows us to consider a setting where the test error of our estimators will be lower, as we reduce the noise in the teacher and increase the amount of training data. This is confirmed by the first panel of Figure~\ref{fig:add_figures_n_over_p=10.0}, where the test error of the estimators is smaller than in Figure~\ref{fig:test_errors_n_over_p=2.0}. Moreover, compared to the setting of Figure~\ref{fig:test_errors_n_over_p=2.0}, the curves for $\hat{f}_{\erm}(\lambdaerror), \hat{f}_{\erm}(\lambdaloss)$ and $\hat{f}_{\empbayes}(\lambdaevidence)$ are much closer. 
Looking at the second panel, we note that as before, doing ERM with $\lambdaloss$ or empirical-Bayes with $\lambdaevidence$ yields the best calibration. However, the calibration curves $\Delta_{0.75}$ do not exhibit the \textit{double descent}-like behaviour shown in Figure~\ref{fig:test_errors_n_over_p=2.0}. 
On Figure~\ref{fig:add_figures_n_over_p=10.0_laplace_ts}, we see the calibration of $\hat{f}_{\lap}$ (left plot) and temperature scaling (center). We see that in this setting, $\hat{f}_{\lap}$ yields underconfident estimators for $\sfrac{p}{n}$ large enough. On the other hand, temperature scaling yields a well-calibrated estimator, whether we apply it on $\hat{f}_{\erm}(\lambda = 0)$ or $\hat{f}_{\erm}(\lambdaerror)$

\subsection{Additional setting 2 : $\noisevar = 0, \sfrac{n}{d} = 20, \| \mathbf{\theta}_* \|^2 = 50$}
\label{app:add_setting_2}

In the previous plots, we defined $\mathbf{\theta}_* = 1$. This is of course not a limitation of our model and we can assume any norm for the teacher. In this section, we will assume $\| \mathbf{\theta}_* \|^2 = 50$. This allows us to significantly reduce the noise in the data. Indeed, as $\| \mathbf{\theta}_* \|^2 \to \infty$, the label becomes deterministic in the input. As before, figures~\ref{fig:add_figures_n_over_p=20.0} and \ref{fig:add_figures_n_over_p=10.0_laplace_ts} show the test error, calibration and variance for the different estimators. In the left panel of Figure~\ref{fig:add_figures_n_over_p=20.0_laplace_ts}, we observe that $\hat{f}_{\lap}$ with $\lambda \in \{ \lambdaerror, \lambdaloss, 10 ^{-4} \}$ systematically under-confident for $p/n$ large enough. As with the previous settings, we also note that $\hat{f}_{\erm}(\lambdaerror)$ used in combination with temperature scaling is the most competitive estimator as it yields very good test error and calibration.  

\begin{figure*}[ht]
    \centering
    \def\figwidth{0.3\columnwidth}
    \def\figheight{0.3\columnwidth}

    \input{Figures/supplementary/n_over_p=20/test_errors_n_over_p=20.0}
    \input{Figures/supplementary/n_over_p=20/calibration_p=0.75_n_over_p=20.0}
    \input{Figures/supplementary/n_over_p=20/conditional_variance_bo_p=0.75_n_over_p=20.0}

    \caption{(\textbf{Left}) Test error of the estimators as a function of $\sfrac{1}{\alpha}$ in the setting described in section~\ref{app:add_setting_2}. (\textbf{Middle}) Calibration of the estimators. (\textbf{Right}) Variance of $\hat{f}_{\bo}$ conditioned on $\hat{f} = 0.75$ for the different estimators.}
    \label{fig:add_figures_n_over_p=20.0}
\end{figure*}
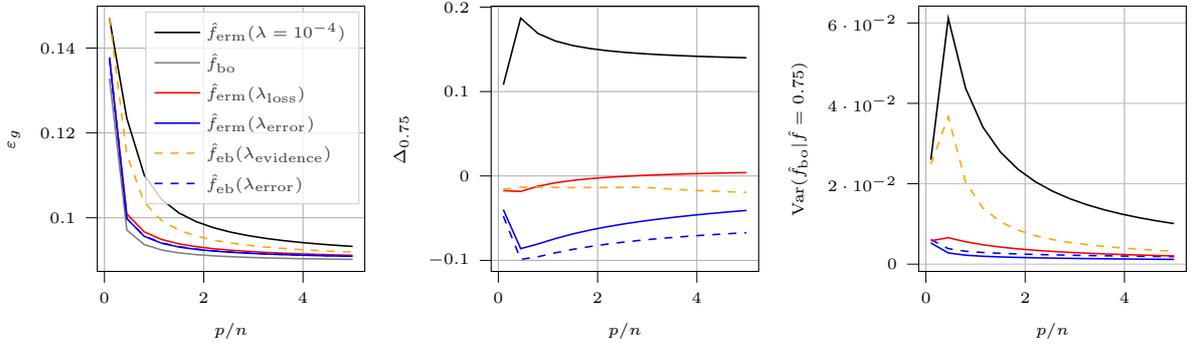

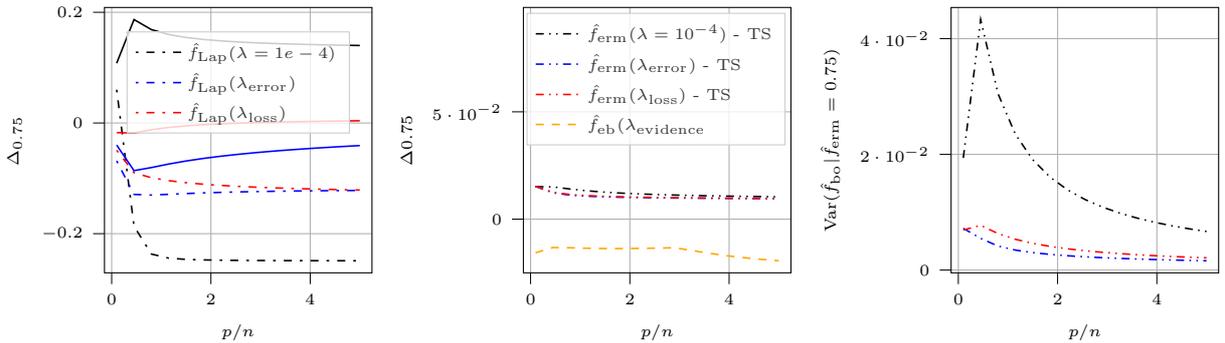
\begin{figure*}[ht]
    \centering
    \def\figwidth{0.3\columnwidth}
    \def\figheight{0.3\columnwidth}

    \input{Figures/supplementary/n_over_p=20/calibration_laplace_p=0.75_n_over_p=20.0}
    \input{Figures/supplementary/n_over_p=20/calibration_temp_scaling_p=0.75_n_over_p=20.0}
    \input{Figures/supplementary/n_over_p=20/conditional_variance_bo_temp_scaling_n_over_p=20.0}

    \caption{(\textbf{Left}) Calibration of $\hat{f}_{\lap}$ and $\hat{f}_{\erm}$ with the setting described in section~\ref{app:add_setting_2}. (\textbf{Middle}) Calibration of $\hat{f}_{\erm}$ after temperature scaling. Solid red line is $\hat{f}_{\erm}(\lambdaloss)$ before temperature scaling. (\textbf{Right}) Variance of $\hat{f}_{\bo}$ conditioned on the confidence of temperature scaling.}
    \label{fig:add_figures_n_over_p=20.0_laplace_ts}    
\end{figure*}

\end{document}

%% file: macros.tex
\newcommand{\extr}{\textrm{\textbf{extr}}}
\DeclareMathOperator{\variance}{Var}

\def\dd{\text{d}}

\def\bo{{\rm bo}}
\def\erm{{\rm erm}}
\def\amp{{\rm amp}}
\def\empbayes{{\rm eb}}
\def\lap{{\rm Lap}}

\def\prob{\nu}

\def\sign{{\rm sign}}

\def\noisestr{\tau_{0}}
\def\noisevar{\noisestr^2}
\def\addnoisestr{\tau_{{\rm add}}}
\def\addnoisevar{\addnoisestr^2}

\def\lambdaerror{\lambda_{{\rm error}}}
\def\lambdaloss{\lambda_{{\rm loss}}}
\def\lambdaevidence{\lambda_{{\rm evidence}}}

\def\mat#1{\text{#1}}
\renewcommand{\vec}[1]{\bm{#1}}

\DeclareMathOperator{\Tr}{Tr}

\DeclareMathOperator{\ECE}{ECE}
\DeclareMathOperator{\prox}{prox}

\DeclareMathOperator*{\argmin}{arg\,min}

\newcommand{\projectedteacher}{\vec{w}_{\star}}
\newcommand{\studentweight}{\vec{\theta}}

\newcommand{\eigfunction}{\hat{\pi}}

%% file: Figures/exp_test_errors_n_over_p=2.0.tex
\begin{tikzpicture}
\tikzstyle{every node}=[font=\tiny]

\definecolor{color0}{rgb}{1,0.647058823529412,0}
\definecolor{color1}{rgb}{0,1,1}

\begin{axis}[
height=\figheight,
legend cell align={left},
legend style={fill opacity=0.8, draw opacity=1, text opacity=1, draw=white!80!black},
tick align=outside,
tick pos=left,
width=\figwidth,
x grid style={white!69.0196078431373!black},
xlabel={1 / \(\displaystyle  \alpha \)},
xmajorgrids,
xmin=-0.145, xmax=5.245,
xtick style={color=black},
y grid style={white!69.0196078431373!black},
ylabel={\(\displaystyle \varepsilon_g\)},
ymajorgrids,
ymin=0.417067533578972, ymax=0.470870066584964,
ytick style={color=black}
]
\addplot [
  draw=white!50.1960784313725!black,
  fill=white!50.1960784313725!black,
  forget plot,
  mark=*,
  only marks,
  mark size=1pt
]
table{%
x  y
0.5 0.436012157071586
0.944444444444444 0.426153084448802
1.38888888888889 0.42244253785943
1.83333333333333 0.421683372817746
2.27777777777778 0.420876411149242
2.72222222222222 0.421079524071541
3.16666666666667 0.420629810956403
3.61111111111111 0.419907707311783
4.05555555555556 0.419513103261062
4.5 0.420233068156365
};
\addplot [draw=blue, fill=blue, forget plot, mark=*, only marks, mark size=1pt]
table{%
x  y
0.594949494949495 0.436284198671913
0.842424242424242 0.432758728700731
1.08989898989899 0.43051119577825
1.33737373737374 0.428864036007573
1.58484848484849 0.428494030996613
1.83232323232323 0.426649428483786
2.07979797979798 0.425025642597219
2.32727272727273 0.425544754211589
2.57474747474748 0.424992240864297
2.82222222222222 0.42415838059932
3.06969696969697 0.424361358933246
3.31717171717172 0.42368369658658
3.56464646464647 0.423178745312672
3.81212121212121 0.423935439481207
4.05959595959596 0.423298734172957
4.30707070707071 0.422887365620041
};
\addplot [draw=red, fill=red, forget plot, mark=*, only marks, mark size=1pt]
table{%
x  y
0.594949494949495 0.440721072422353
0.842424242424242 0.435434269767532
1.08989898989899 0.43281205582915
1.33737373737374 0.430393797819569
1.58484848484849 0.429744884163895
1.83232323232323 0.427771613527378
2.07979797979798 0.427007356282446
2.32727272727273 0.4262775915519
2.57474747474748 0.425719645686254
2.82222222222222 0.425759439152067
3.06969696969697 0.425416883952035
3.31717171717172 0.424955555201163
3.56464646464647 0.423299816734934
3.81212121212121 0.423639696842184
4.05959595959596 0.424131651076911
4.30707070707071 0.423613135968705
};
\addplot [draw=color1, fill=color1, forget plot, mark=*, only marks, mark size=1pt]
table{%
x  y
0.594949494949495 0.437217316803962
0.842424242424242 0.431819340814636
1.08989898989899 0.430428132730808
1.33737373737374 0.427211363465749
1.58484848484849 0.426560135808552
1.83232323232323 0.425709549552801
2.07979797979798 0.42512949146128
2.32727272727273 0.425171819295421
2.57474747474748 0.42458529617299
2.82222222222222 0.423081747983787
3.06969696969697 0.424175051060342
3.31717171717172 0.422174992758107
3.56464646464647 0.423537305871507
3.81212121212121 0.422655655220759
4.05959595959596 0.422009860050593
4.30707070707071 0.421659417872385
};
\addplot [draw=color0, fill=color0, forget plot, mark=*, only marks, mark size=1pt]
table{%
x  y
0.594949494949495 0.440293051065197
0.842424242424242 0.434412517024378
1.08989898989899 0.431376168296033
1.33737373737374 0.430065342178911
1.58484848484849 0.427996214621963
1.83232323232323 0.427367613642013
2.07979797979798 0.426659422508165
2.32727272727273 0.425724825514347
2.57474747474748 0.42601369240024
2.82222222222222 0.425462377875449
3.06969696969697 0.424125520339004
3.31717171717172 0.4245266809145
3.56464646464647 0.423353166615904
3.81212121212121 0.424160508615243
4.05959595959596 0.423530325895767
4.30707070707071 0.423667087244617
};
\addplot [draw=black, fill=black, forget plot, mark=*, only marks, mark size=1pt]
table{%
x  y
0.594949494949495 0.454762502008017
0.842424242424242 0.447084872148379
1.08989898989899 0.443798691145025
1.33737373737374 0.442601955626835
1.58484848484849 0.438432210975045
1.83232323232323 0.437485061285421
2.07979797979798 0.435704740542534
2.32727272727273 0.436034213391621
2.57474747474748 0.434471419657156
2.82222222222222 0.435494835487973
3.06969696969697 0.434159516222501
3.31717171717172 0.434860792426713
3.56464646464647 0.434460979776369
3.81212121212121 0.434089645209424
4.05959595959596 0.433853796085744
4.30707070707071 0.433547606104379
};
\addplot [semithick, black]
table {%
0.1 0.468221895846942
0.14949494949495 0.462312486122707
0.198989898989899 0.458037829032179
0.248484848484848 0.45504505817882
0.297979797979798 0.453217339519351
0.347474747474747 0.452569782118156
0.396969696969697 0.453252595371553
0.446464646464647 0.455621144836667
0.495959595959596 0.457266840458373
0.545454545454546 0.45535107739837
0.594949494949495 0.453344017754498
0.644444444444444 0.45157196650694
0.693939393939394 0.450028368330642
0.743434343434343 0.448681114768639
0.792929292929293 0.447498945524032
0.842424242424242 0.446455252425906
0.891919191919192 0.445528147734713
0.941414141414142 0.444699768965849
0.990909090909091 0.443955532891319
1.04040404040404 0.443283483848236
1.08989898989899 0.442673763406864
1.13939393939394 0.442118189737968
1.18888888888889 0.441609927374042
1.23838383838384 0.441143228267531
1.28787878787879 0.440713228216159
1.33737373737374 0.440315786077496
1.38686868686869 0.439947356053585
1.43636363636364 0.439604885638506
1.48585858585859 0.439285733533196
1.53535353535354 0.438987603222825
1.58484848484849 0.438708488884145
1.63434343434343 0.438446631076347
1.68383838383838 0.438200480232333
1.73333333333333 0.437968666409869
1.78282828282828 0.437749974098136
1.83232323232323 0.43754332112681
1.88181818181818 0.437347749551816
1.93131313131313 0.437162372549384
1.98080808080808 0.436986423896764
2.03030303030303 0.436819203361462
2.07979797979798 0.436660077429046
2.12929292929293 0.436508472549673
2.17878787878788 0.436363879359033
2.22828282828283 0.436225803966098
2.27777777777778 0.436093823030108
2.32727272727273 0.435967541873806
2.37676767676768 0.435846599160417
2.42626262626263 0.435730663448806
2.47575757575758 0.435619430166345
2.52525252525253 0.435512618942657
2.57474747474748 0.435409971005707
2.62424242424242 0.435311248335211
2.67373737373737 0.435216229331345
2.72323232323232 0.435124709651637
2.77272727272727 0.435036499493223
2.82222222222222 0.434951422534788
2.87171717171717 0.434869314759413
2.92121212121212 0.434790023400013
2.97070707070707 0.434713405991345
3.02020202020202 0.434639329516259
3.06969696969697 0.434567669635423
3.11919191919192 0.434498309991179
3.16868686868687 0.434431141577328
3.21818181818182 0.434366062167681
3.26767676767677 0.434302975797087
3.31717171717172 0.434241792289169
3.36666666666667 0.434182426827262
3.41616161616162 0.434124799560797
3.46565656565657 0.434068835247973
3.51515151515152 0.434014463345832
3.56464646464647 0.433961615973346
3.61414141414142 0.433910230340487
3.66363636363636 0.43386024662507
3.71313131313131 0.433811608225282
3.76262626262626 0.433764261545869
3.81212121212121 0.433718155801139
3.86161616161616 0.433673242833246
3.91111111111111 0.433629476944404
3.96060606060606 0.433586814741785
4.01010101010101 0.433545214994008
4.05959595959596 0.433504638498222
4.10909090909091 0.43346504795687
4.15858585858586 0.433426407863257
4.20808080808081 0.433388684395637
4.25757575757576 0.433351845318001
4.30707070707071 0.433315859888507
4.35656565656566 0.433280698773833
4.40606060606061 0.433246333969422
4.45555555555556 0.433212739189352
4.50505050505051 0.433179887969082
4.55454545454546 0.433147756297027
4.6040404040404 0.433116320787341
4.65353535353535 0.43308555905676
4.7030303030303 0.433055449671441
4.75252525252525 0.433025972097154
4.8020202020202 0.432997106652572
4.85151515151515 0.432968834465445
4.9010101010101 0.432941137431444
4.95050505050505 0.432913998175497
5 0.432887400015436
};
\addlegendentry{$\lambda = 0$}
\addplot [semithick, white!50.1960784313725!black]
table {%
0.1 0.467164521781342
0.14949494949495 0.460279117856605
0.198989898989899 0.454728735424094
0.248484848484848 0.450110516163127
0.297979797979798 0.446220395598436
0.347474747474747 0.44292798936881
0.396969696969697 0.440126669461257
0.446464646464647 0.437717161378481
0.495959595959596 0.435614855615464
0.545454545454546 0.433763999313151
0.594949494949495 0.432137717828623
0.644444444444444 0.430722798353965
0.693939393939394 0.429505236320775
0.743434343434343 0.42846546446595
0.792929292929293 0.427580247708084
0.842424242424242 0.426825999702309
0.891919191919192 0.426181074395026
0.941414141414142 0.425626766475683
0.990909090909091 0.425147458975594
1.04040404040404 0.424730352468091
1.08989898989899 0.424365047295632
1.13939393939394 0.424043115486328
1.18888888888889 0.423757719049987
1.23838383838384 0.423503290956276
1.28787878787879 0.423275277208947
1.33737373737374 0.423069931857001
1.38686868686869 0.422884155397109
1.43636363636364 0.422715367713258
1.48585858585859 0.422561408129304
1.53535353535354 0.422420456567485
1.58484848484849 0.422290971114301
1.63434343434343 0.422171638344207
1.68383838383838 0.422061333588777
1.73333333333333 0.421959088986226
1.78282828282828 0.4218640676534
1.83232323232323 0.421775542674255
1.88181818181818 0.421692879936062
1.93131313131313 0.421615524018517
1.98080808080808 0.421542986538772
2.03030303030303 0.421474836475925
2.07979797979798 0.421410692100003
2.12929292929293 0.421350214208106
2.17878787878788 0.421293100430657
2.22828282828283 0.421239080417897
2.27777777777778 0.421187911753746
2.32727272727273 0.421139376473411
2.37676767676768 0.421093278084245
2.42626262626263 0.421049439007862
2.47575757575758 0.42100769837624
2.52525252525253 0.420967910126468
2.57474747474748 0.42092994134833
2.62424242424242 0.420893670846768
2.67373737373737 0.420858987887517
2.72323232323232 0.420825791099481
2.77272727272727 0.420793987511581
2.82222222222222 0.420763491705394
2.87171717171717 0.420734225067734
2.92121212121212 0.420706115129786
2.97070707070707 0.420679094981374
3.02020202020202 0.420653102750645
3.06969696969697 0.420628081140829
3.11919191919192 0.420603977016962
3.16868686868687 0.420580741036398
3.21818181818182 0.420558327315328
3.26767676767677 0.420536693141779
3.31717171717172 0.420515798693405
3.36666666666667 0.420495606806986
3.41616161616162 0.420476082758793
3.46565656565657 0.420457194068358
3.51515151515152 0.420438910320818
3.56464646464647 0.420421203005824
3.61414141414142 0.420404045371242
3.66363636363636 0.420387412290096
3.71313131313131 0.420371280139376
3.76262626262626 0.4203556266895
3.81212121212121 0.420340431003354
3.86161616161616 0.420325673343964
3.91111111111111 0.420311335089946
3.96060606060606 0.420297398657988
4.01010101010101 0.420283847431691
4.05959595959596 0.420270665696166
4.10909090909091 0.42025783857786
4.15858585858586 0.420245351989128
4.20808080808081 0.420233192577114
4.25757575757576 0.420221347676569
4.30707070707071 0.420209805266246
4.35656565656566 0.420198553928564
4.40606060606061 0.420187582812268
4.45555555555556 0.420176881597808
4.50505050505051 0.420166440465232
4.55454545454546 0.420156250064376
4.6040404040404 0.420146301487144
4.65353535353535 0.420136586241742
4.7030303030303 0.420127096228681
4.75252525252525 0.420117823718418
4.8020202020202 0.420108761330508
4.85151515151515 0.420099902014149
4.9010101010101 0.420091239030011
4.95050505050505 0.42008276593325
5 0.42007447655762
};
\addlegendentry{$\hat f_{bo}$}
\addplot [semithick, red]
table {%
0.1 0.468424496902874
0.14949494949495 0.462451841436725
0.198989898989899 0.457850137918524
0.248484848484848 0.454147803011128
0.297979797979798 0.45108977582031
0.347474747474747 0.448516540778741
0.396969696969697 0.446319298251635
0.446464646464647 0.444419837287937
0.495959595959596 0.442760243656631
0.545454545454546 0.441296943303194
0.594949494949495 0.439996741011485
0.644444444444444 0.438833999124224
0.693939393939394 0.437788464029894
0.743434343434343 0.436843827497096
0.792929292929293 0.435986707841668
0.842424242424242 0.435205969720001
0.891919191919192 0.434492245902917
0.941414141414142 0.433837594249855
0.990909090909091 0.433235237984694
1.04040404040404 0.432679361871938
1.08989898989899 0.432164955694805
1.13939393939394 0.431687681454824
1.18888888888889 0.431243765048635
1.23838383838384 0.430829914739031
1.28787878787879 0.430443246648465
1.33737373737374 0.430081225775928
1.38686868686869 0.429741614641247
1.43636363636364 0.429422433738888
1.48585858585859 0.429121925297757
1.53535353535354 0.42883852426271
1.58484848484849 0.428570832018978
1.63434343434343 0.428317596498496
1.68383838383838 0.428077692981996
1.73333333333333 0.427850109133106
1.78282828282828 0.427633930470944
1.83232323232323 0.427428331526503
1.88181818181818 0.427232561157823
1.93131313131313 0.427045938876076
1.98080808080808 0.42686784383428
2.03030303030303 0.426697709482021
2.07979797979798 0.426535019919458
2.12929292929293 0.426379299180213
2.17878787878788 0.426230113201189
2.22828282828283 0.426087062344388
2.27777777777778 0.425949778205604
2.32727272727273 0.425817921712601
2.37676767676768 0.425691178403959
2.42626262626263 0.425569259185039
2.47575757575758 0.425451895141372
2.52525252525253 0.425338837317737
2.57474747474748 0.425229853534753
2.62424242424242 0.425124729661017
2.67373737373737 0.425023264675377
2.72323232323232 0.42492527194263
2.77272727272727 0.424830576803452
2.82222222222222 0.42473901686377
2.87171717171717 0.424650438883508
2.92121212121212 0.424564700266429
2.97070707070707 0.424481667924685
3.02020202020202 0.424401215244424
3.06969696969697 0.424323224925882
3.11919191919192 0.424247586054181
3.16868686868687 0.424174193767197
3.21818181818182 0.424102950067132
3.26767676767677 0.424033762152124
3.31717171717172 0.423966543054728
3.36666666666667 0.423901209807783
3.41616161616162 0.423837684364085
3.46565656565657 0.423775893143313
3.51515151515152 0.423715765614626
3.56464646464647 0.423657236352028
3.61414141414142 0.423600242588656
3.66363636363636 0.423544724555907
3.71313131313131 0.42349062606498
3.76262626262626 0.42343789359376
3.81212121212121 0.423386475392319
3.86161616161616 0.423336323894181
3.91111111111111 0.42328739278498
3.96060606060606 0.423239638105756
4.01010101010101 0.423193017733303
4.05959595959596 0.423147492021081
4.10909090909091 0.423103023211542
4.15858585858586 0.423059574305964
4.20808080808081 0.423017111374092
4.25757575757576 0.422975601022715
4.30707070707071 0.422935011654845
4.35656565656566 0.422895312711386
4.40606060606061 0.422856476157183
4.45555555555556 0.422818473117306
4.50505050505051 0.422781277682269
4.55454545454546 0.422744864536386
4.6040404040404 0.422709208722167
4.65353535353535 0.422674287451379
4.7030303030303 0.422640078463381
4.75252525252525 0.422606559236855
4.8020202020202 0.422573710035972
4.85151515151515 0.422541510602728
4.9010101010101 0.422509942083719
4.95050505050505 0.422478986188156
5 0.422448624850853
};
\addlegendentry{$\hat f_{erm}(\lambda_{\rm loss}$)}
\addplot [semithick, blue]
table {%
0.1 0.467655642950354
0.14949494949495 0.461165712311419
0.198989898989899 0.456059746127041
0.248484848484848 0.451904496516251
0.297979797979798 0.448468045604484
0.347474747474747 0.44560093663415
0.396969696969697 0.443193676300477
0.446464646464647 0.44116046329875
0.495959595959596 0.439431422292113
0.545454545454546 0.437950044509396
0.594949494949495 0.436671004715943
0.644444444444444 0.435558212540757
0.693939393939394 0.434582959049773
0.743434343434343 0.433722319704965
0.792929292929293 0.432957879069971
0.842424242424242 0.432274747184574
0.891919191919192 0.431660806392689
0.941414141414142 0.431106132204342
0.990909090909091 0.430602545040902
1.04040404040404 0.430143261241873
1.08989898989899 0.429722620002175
1.13939393939394 0.429335868757135
1.18888888888889 0.428978993748152
1.23838383838384 0.428648585632327
1.28787878787879 0.428341732367083
1.33737373737374 0.428055933395816
1.38686868686869 0.427789030537573
1.43636363636364 0.42753915203156
1.48585858585859 0.427304666990288
1.53535353535354 0.427084327012345
1.58484848484849 0.426876503218401
1.63434343434343 0.426680286113321
1.68383838383838 0.426494695337484
1.73333333333333 0.426318859480432
1.78282828282828 0.426152001327827
1.83232323232323 0.425993425430647
1.88181818181818 0.425842507583488
1.93131313131313 0.4256986858821
1.98080808080808 0.425561453092075
2.03030303030303 0.425430350111651
2.07979797979798 0.425304960349866
2.12929292929293 0.425184904876814
2.17878787878788 0.425069838222461
2.22828282828283 0.424959444728223
2.27777777777778 0.424853435365426
2.32727272727273 0.4247515449564
2.37676767676768 0.424653529735855
2.42626262626263 0.424559165208179
2.47575757575758 0.424468244257388
2.52525252525253 0.424380575477015
2.57474747474748 0.424295981690356
2.62424242424242 0.424214298634651
2.67373737373737 0.424135373791476
2.72323232323232 0.424059065341379
2.77272727272727 0.423985241229384
2.82222222222222 0.42391377832805
2.87171717171717 0.423844561685329
2.92121212121212 0.423777483847395
2.97070707070707 0.423712444249506
3.02020202020202 0.423649348664608
3.06969696969697 0.423588108705021
3.11919191919192 0.423528641370695
3.16868686868687 0.423470868639271
3.21818181818182 0.423414852557375
3.26767676767677 0.423360246318729
3.31717171717172 0.423307127322897
3.36666666666667 0.423255434017412
3.41616161616162 0.423205108267152
3.46565656565657 0.423156095116311
3.51515151515152 0.423108342571275
3.56464646464647 0.423061801399876
3.61414141414142 0.42301642494959
3.66363636363636 0.422972168978123
3.71313131313131 0.422928991496322
3.76262626262626 0.422886852624793
3.81212121212121 0.422845714463818
3.86161616161616 0.42280554096663
3.91111111111111 0.422766297829583
3.96060606060606 0.422727952383701
4.01010101010101 0.422690473498372
4.05959595959596 0.422653831491736
4.10909090909091 0.422617998043934
4.15858585858586 0.422582946121455
4.20808080808081 0.422548649900954
4.25757575757576 0.422515084704751
4.30707070707071 0.422482226934639
4.35656565656566 0.422450054014538
4.40606060606061 0.422418544333426
4.45555555555556 0.422387677194468
4.50505050505051 0.422357432766806
4.55454545454546 0.422327792039873
4.6040404040404 0.422298736780845
4.65353535353535 0.422270249495193
4.7030303030303 0.422242313388476
4.75252525252525 0.422214912332051
4.8020202020202 0.422188030829988
4.85151515151515 0.422161653986573
4.9010101010101 0.422135767479511
4.95050505050505 0.422110357530174
5 0.422085410879188
};
\addlegendentry{$\hat f_{erm}(\lambda_{\rm error}$)}
\addplot [semithick, color0, dashed]
table {%
0.1 0.468387617383065
0.14949494949495 0.462405140505159
0.198989898989899 0.45780132831107
0.248484848484848 0.454102295584566
0.297979797979798 0.45105060856387
0.347474747474747 0.448484852556621
0.396969696969697 0.446294888460702
0.446464646464647 0.444401669643489
0.495959595959596 0.442746870352802
0.545454545454546 0.441286802512906
0.594949494949495 0.439988383729213
0.644444444444444 0.438826153360308
0.693939393939394 0.437780104920662
0.743434343434343 0.436834174982442
0.792929292929293 0.435975186722447
0.842424242424242 0.435192179302958
0.891919191919192 0.434475926948952
0.941414141414142 0.433818597609293
0.990909090909091 0.43321349797653
1.04040404040404 0.432654878384954
1.08989898989899 0.43213777196381
1.13939393939394 0.43165787851978
1.18888888888889 0.431211446384927
1.23838383838384 0.430795189206074
1.28787878787879 0.430406245718552
1.33737373737374 0.430042080603682
1.38686868686869 0.429700458953623
1.43636363636364 0.429379402464779
1.48585858585859 0.429077139974829
1.53535353535354 0.428792119660767
1.58484848484849 0.428522925788144
1.63434343434343 0.4282683072111
1.68383838383838 0.428027121401907
1.73333333333333 0.427798373576225
1.78282828282828 0.427581121678068
1.83232323232323 0.427374540844611
1.88181818181818 0.427177875389938
1.93131313131313 0.426990438594471
1.98080808080808 0.426811604763254
2.03030303030303 0.426640806628764
2.07979797979798 0.426477515262282
2.12929292929293 0.426321260356252
2.17878787878788 0.426171593108088
2.22828282828283 0.426028116620477
2.27777777777778 0.425890457037077
2.32727272727273 0.42575827144052
2.37676767676768 0.425631243077932
2.42626262626263 0.425509078936552
2.47575757575758 0.425391504129187
2.52525252525253 0.425278271575226
2.57474747474748 0.425169146887728
2.62424242424242 0.425063913661368
2.67373737373737 0.424962365484607
2.72323232323232 0.424864315306201
2.77272727272727 0.424769586870297
2.82222222222222 0.42467801546696
2.87171717171717 0.42458944702602
2.92121212121212 0.424503737267882
2.97070707070707 0.424420750949525
3.02020202020202 0.424340361170441
3.06969696969697 0.424262448739489
3.11919191919192 0.424186901608598
3.16868686868687 0.4241136143424
3.21818181818182 0.42404248764571
3.26767676767677 0.423973427921964
3.31717171717172 0.423906346874051
3.36666666666667 0.423841161134022
3.41616161616162 0.423777791926524
3.46565656565657 0.423716164772915
3.51515151515152 0.423656209980117
3.56464646464647 0.423597859150214
3.61414141414142 0.423541049848588
3.66363636363636 0.423485722074891
3.71313131313131 0.423431818889005
3.76262626262626 0.423379286217734
3.81212121212121 0.423328072679621
3.86161616161616 0.423278129418332
3.91111111111111 0.423229409951065
3.96060606060606 0.42318187002569
4.01010101010101 0.423135467491003
4.05959595959596 0.423090162171498
4.10909090909091 0.423045915755211
4.15858585858586 0.423002691686746
4.20808080808081 0.422960455068153
4.25757575757576 0.42291917256597
4.30707070707071 0.422878812325203
4.35656565656566 0.422839343359152
4.40606060606061 0.422800737330136
4.45555555555556 0.42276296634967
4.50505050505051 0.422726003455675
4.55454545454546 0.422689823071963
4.6040404040404 0.422654400687229
4.65353535353535 0.422619712801339
4.7030303030303 0.422585737647078
4.75252525252525 0.422552452038617
4.8020202020202 0.422519835986957
4.85151515151515 0.422487869547992
4.9010101010101 0.422456533557721
4.95050505050505 0.422425809598141
5 0.422395711959596
};
\addlegendentry{$\hat f_{eb}(\lambda_{\rm evidence}$)}
\addplot [semithick, color1, dashed]
table {%
0.1 0.467655553280287
0.14949494949495 0.461165259452278
0.198989898989899 0.456058765741183
0.248484848484848 0.451902744510238
0.297979797979798 0.448465303015307
0.347474747474747 0.445597160894826
0.396969696969697 0.443189280369875
0.446464646464647 0.441155416697252
0.495959595959596 0.439425914982093
0.545454545454546 0.437944250014241
0.594949494949495 0.436665064571269
0.644444444444444 0.435552233643725
0.693939393939394 0.434577017703017
0.743434343434343 0.433716354264134
0.792929292929293 0.432952060073505
0.842424242424242 0.432269093046628
0.891919191919192 0.431655325950981
0.941414141414142 0.431100828150311
0.990909090909091 0.430597416064469
1.04040404040404 0.430138303611411
1.08989898989899 0.429717828568706
1.13939393939394 0.429331181977456
1.18888888888889 0.428974483315528
1.23838383838384 0.428644240545628
1.28787878787879 0.428337542499899
1.33737373737374 0.428051889468152
1.38686868686869 0.427785124074694
1.43636363636364 0.427535375292796
1.48585858585859 0.427301012901277
1.53535353535354 0.427080610208657
1.58484848484849 0.426872913400941
1.63434343434343 0.426676816197659
1.68383838383838 0.426491338794956
1.73333333333333 0.426315610284871
1.78282828282828 0.426148853907905
1.83232323232323 0.425990374626188
1.88181818181818 0.425839548606307
1.93131313131313 0.425695814280842
1.98080808080808 0.425558664720984
2.03030303030303 0.425427641102581
2.07979797979798 0.42530232708799
2.12929292929293 0.42518234397786
2.17878787878788 0.425067346513441
2.22828282828283 0.424957019229335
2.27777777777778 0.424851073275033
2.32727272727273 0.424749243636252
2.37676767676768 0.424651286698728
2.42626262626263 0.424556978106256
2.47575757575758 0.42446611087214
2.52525252525253 0.424378493709796
2.57474747474748 0.424293949553605
2.62424242424242 0.424212314244767
2.67373737373737 0.424133435361433
2.72323232323232 0.424057251055514
2.77272727272727 0.423983466660234
2.82222222222222 0.423912042020299
2.87171717171717 0.423842862262783
2.92121212121212 0.423775820007118
2.97070707070707 0.423710814755888
3.02020202020202 0.423647752344393
3.06969696969697 0.423586544442986
3.11919191919192 0.423527108105469
3.16868686868687 0.423469365359677
3.21818181818182 0.423413242834944
3.26767676767677 0.423358671422911
3.31717171717172 0.423305585968169
3.36666666666667 0.423253924985718
3.41616161616162 0.423203630402521
3.46565656565657 0.423154647320816
3.51515151515152 0.423106923801038
3.56464646464647 0.423060410662485
3.61414141414142 0.423015061300023
3.66363636363636 0.422970831515379
3.71313131313131 0.422927679361625
3.76262626262626 0.422885564999703
3.81212121212121 0.42284445056587
3.86161616161616 0.422804300049151
3.91111111111111 0.422765079177822
3.96060606060606 0.422726755314334
4.01010101010101 0.422689297357722
4.05959595959596 0.422652675653057
4.10909090909091 0.42261686190727
4.15858585858586 0.422581829110847
4.20808080808081 0.422547551464931
4.25757575757576 0.422514004313353
4.30707070707071 0.422481164079262
4.35656565656566 0.42244900820597
4.40606060606061 0.422417515101657
4.45555555555556 0.422386664087693
4.50505050505051 0.422356435350258
4.55454545454546 0.422326809895066
4.6040404040404 0.422297769504894
4.65353535353535 0.422269296699688
4.7030303030303 0.42224137469945
4.75252525252525 0.422213987388762
4.8020202020202 0.422187119284162
4.85151515151515 0.422160755502867
4.9010101010101 0.422134881733678
4.95050505050505 0.422109484209485
5 0.422084549681341
};
\addlegendentry{$\hat f_{eb}(\lambda_{\rm error}$)}
\end{axis}

\end{tikzpicture}

%% file: Figures/exp_calibration_p=0.75.tex
\begin{tikzpicture}
\tikzstyle{every node}=[font=\tiny]

\definecolor{color0}{rgb}{1,0.647058823529412,0}
\definecolor{color1}{rgb}{0,1,1}

\begin{axis}[
height=\figheight,
legend cell align={left},
legend style={fill opacity=0.8, draw opacity=1, text opacity=1, draw=white!80!black},
tick align=outside,
tick pos=left,
width=\figwidth,
x grid style={white!69.0196078431373!black},
xlabel={1 / \(\displaystyle \alpha\)},
xmajorgrids,
xmin=-0.145, xmax=5.245,
xtick style={color=black},
y grid style={white!69.0196078431373!black},
ylabel={\(\displaystyle \Delta_{0.75}\)},
ymajorgrids,
ymin=-0.045098495218578, ymax=0.260751796338073,
ytick style={color=black}
]
\addplot [draw=blue, fill=blue, forget plot, mark=*, only marks, mark size = 1pt ]
table{%
x  y
0.594949494949495 0.166031261190207
0.842424242424242 0.159338480342729
1.08989898989899 0.152107137029209
1.33737373737374 0.144459853386214
1.58484848484849 0.138055766005952
1.83232323232323 0.131075890488108
2.07979797979798 0.123533589285115
2.32727272727273 0.119802994451351
2.57474747474748 0.11523588968557
2.82222222222222 0.109880475434346
3.06969696969697 0.107175268049691
3.31717171717172 0.10254113376579
3.56464646464647 0.0986158378226369
3.81212121212121 0.096497630395501
4.05959595959596 0.0927748474989166
4.30707070707071 0.0904164455439479
};
\addplot [draw=red, fill=red, forget plot, mark=*, only marks, mark size = 1pt ]
table{%
x  y
0.594949494949495 0.0113148880285029
0.842424242424242 0.00980566759698742
1.08989898989899 0.00876221858253101
1.33737373737374 0.0107474240522567
1.58484848484849 0.00992145880158002
1.83232323232323 0.0025105818128216
2.07979797979798 0.00495293808500408
2.32727272727273 0.00657929806006385
2.57474747474748 0.00384041153067161
2.82222222222222 0.00396794902644682
3.06969696969697 0.00569852144230965
3.31717171717172 0.00345254428982944
3.56464646464647 0.00336130175551519
3.81212121212121 0.00422129283685324
4.05959595959596 0.00416585081752607
4.30707070707071 0.00306862360012861
};
\addplot [draw=color1, fill=color1, forget plot, mark=*, only marks, mark size = 1pt ]
table{%
x  y
0.594949494949495 0.163321789881347
0.842424242424242 0.151885384353769
1.08989898989899 0.146781664475271
1.33737373737374 0.134404241466594
1.58484848484849 0.129076515717045
1.83232323232323 0.12291383397579
2.07979797979798 0.117617885290558
2.32727272727273 0.114568387702064
2.57474747474748 0.109457344069415
2.82222222222222 0.101041116454195
3.06969696969697 0.102360128582519
3.31717171717172 0.0921739928156016
3.56464646464647 0.0949744250710625
3.81212121212121 0.089240639640712
4.05959595959596 0.0844813463843493
4.30707070707071 0.0809827915237162
};
\addplot [draw=color0, fill=color0, forget plot, mark=*, only marks, mark size = 1pt ]
table{%
x  y
0.594949494949495 0.00416408595941264
0.842424242424242 -0.00359957436758485
1.08989898989899 -0.00328392205262673
1.33737373737374 0.00164446887985759
1.58484848484849 -0.00177633013844802
1.83232323232323 0.00125172560196796
2.07979797979798 0.00242129642214561
2.32727272727273 0.00114273716867985
2.57474747474748 0.00622658310266
2.82222222222222 0.00582673056105742
3.06969696969697 0.000690445688329211
3.31717171717172 0.00475137215263466
3.56464646464647 2.46256934395284e-05
3.81212121212121 0.00588353822229393
4.05959595959596 0.00384919957827923
4.30707070707071 0.00568674782008216
};
\addplot [draw=color0, fill=color0, forget plot, mark=*, only marks, mark size = 1pt ]
table{%
x  y
0.594949494949495 0.00416408595941264
0.842424242424242 -0.00359957436758485
1.08989898989899 -0.00328392205262673
1.33737373737374 0.00164446887985759
1.58484848484849 -0.00177633013844802
1.83232323232323 0.00125172560196796
2.07979797979798 0.00242129642214561
2.32727272727273 0.00114273716867985
2.57474747474748 0.00622658310266
2.82222222222222 0.00582673056105742
3.06969696969697 0.000690445688329211
3.31717171717172 0.00475137215263466
3.56464646464647 2.46256934395284e-05
3.81212121212121 0.00588353822229393
4.05959595959596 0.00384919957827923
4.30707070707071 0.00568674782008216
};
\addplot [draw=black, fill=black, forget plot, mark=*, only marks, mark size = 1pt ]
table{%
x  y
0.594949494949495 0.246132445791867
0.842424242424242 0.243602225521843
1.08989898989899 0.242170753269526
1.33737373737374 0.241424139473292
1.58484848484849 0.240353729885288
1.83232323232323 0.239836652035998
2.07979797979798 0.239244002529685
2.32727272727273 0.239103972596387
2.57474747474748 0.238781766213196
2.82222222222222 0.238873879670612
3.06969696969697 0.238463998902553
3.31717171717172 0.238436570991951
3.56464646464647 0.238419196134758
3.81212121212121 0.238258274445229
4.05959595959596 0.238142202754361
4.30707070707071 0.238016089771082
};
\addplot [semithick, black]
table {%
0.1 0.198840805029523
0.14949494949495 0.204697478903802
0.198989898989899 0.210602672865687
0.248484848484848 0.216596729356316
0.297979797979798 0.222720046718845
0.347474747474747 0.229029244091433
0.396969696969697 0.235637544649401
0.446464646464647 0.24268754443515
0.495959595959596 0.246849510358225
0.545454545454546 0.246602475168014
0.594949494949495 0.245985578488217
0.644444444444444 0.245386146506138
0.693939393939394 0.244840307056351
0.743434343434343 0.244349517527843
0.792929292929293 0.243908914177583
0.842424242424242 0.243512600032378
0.891919191919192 0.24315497998123
0.941414141414142 0.242831093281393
0.990909090909091 0.242536652720573
1.04040404040404 0.242267991869407
1.08989898989899 0.242021982512184
1.13939393939394 0.241795954482495
1.18888888888889 0.241587623324375
1.23838383838384 0.241395028123
1.28787878787879 0.241216479284675
1.33737373737374 0.241050515157435
1.38686868686869 0.240895866144511
1.43636363636364 0.240751425051271
1.48585858585859 0.240616222524378
1.53535353535354 0.240489406674421
1.58484848484849 0.240370226105553
1.63434343434343 0.240258015737396
1.68383838383838 0.240152184913104
1.73333333333333 0.240052207386401
1.78282828282828 0.239957612859967
1.83232323232323 0.239867979806286
1.88181818181818 0.239782934055485
1.93131313131313 0.239702122666664
1.98080808080808 0.23962524400546
2.03030303030303 0.239552018763161
2.07979797979798 0.239482193275395
2.12929292929293 0.239415536764531
2.17878787878788 0.239351844902887
2.22828282828283 0.239290914577414
2.27777777777778 0.23923257443593
2.32727272727273 0.239176663071541
2.37676767676768 0.239123032183593
2.42626262626263 0.239071545277783
2.47575757575758 0.239022076516938
2.52525252525253 0.238974509703855
2.57474747474748 0.238928736876465
2.62424242424242 0.238884660007766
2.67373737373737 0.23884218527987
2.72323232323232 0.238801227456278
2.77272727272727 0.238761706795178
2.82222222222222 0.238723549041576
2.87171717171717 0.238686684962719
2.92121212121212 0.238651049931205
2.97070707070707 0.238616583548851
3.02020202020202 0.238583229306782
3.06969696969697 0.238550934277732
3.11919191919192 0.238519648837058
3.16868686868687 0.238489326409388
3.21818181818182 0.238459923238222
3.26767676767677 0.238431398176071
3.31717171717172 0.238403712492858
3.36666666666667 0.238376829701911
3.41616161616162 0.238350715398842
3.46565656565657 0.23832533711598
3.51515151515152 0.238300664233186
3.56464646464647 0.238276667666533
3.61414141414142 0.238253320047867
3.66363636363636 0.238230595418052
3.71313131313131 0.238208469183441
3.76262626262626 0.238186918027318
3.81212121212121 0.238165919828142
3.86161616161616 0.238145453583982
3.91111111111111 0.238125499342621
3.96060606060606 0.238106038136813
4.01010101010101 0.238087051924283
4.05959595959596 0.238068523532054
4.10909090909091 0.238050436604736
4.15858585858586 0.23803277555641
4.20808080808081 0.238015525526179
4.25757575757576 0.237998672336133
4.30707070707071 0.237982202452831
4.35656565656566 0.237966102951044
4.40606060606061 0.237950361479987
4.45555555555556 0.23793496628224
4.50505050505051 0.237919905965548
4.55454545454546 0.237905169768367
4.6040404040404 0.237890747341993
4.65353535353535 0.237876628773497
4.7030303030303 0.237862804563034
4.75252525252525 0.237849265602555
4.8020202020202 0.237836003155822
4.85151515151515 0.237823008839634
4.9010101010101 0.237810274606178
4.95050505050505 0.237797792726427
5 0.237785555774516
};
\addplot [semithick, color0]
table {%
0.1 0.00612983847076154
0.14949494949495 0.00552672959123013
0.198989898989899 0.00494555848895584
0.248484848484848 0.00439628757545607
0.297979797979798 0.00388821496055913
0.347474747474747 0.00342948815061628
0.396969696969697 0.00302672344930655
0.446464646464647 0.00268437035634395
0.495959595959596 0.00240204646992392
0.545454545454546 0.00217584868420728
0.594949494949495 0.00199807072542912
0.644444444444444 0.00186030872156306
0.693939393939394 0.0017544771010668
0.743434343434343 0.00167278939144833
0.792929292929293 0.00160973028431544
0.842424242424242 0.00156074884303004
0.891919191919192 0.00152240692807148
0.941414141414142 0.00149216816261499
0.990909090909091 0.00146818476936461
1.04040404040404 0.0014489877588626
1.08989898989899 0.00143363871212865
1.13939393939394 0.00142099238261184
1.18888888888889 0.00141061583507995
1.23838383838384 0.00140268336373928
1.28787878787879 0.00139595048465835
1.33737373737374 0.00139044923539955
1.38686868686869 0.0013859369367033
1.43636363636364 0.00138207065580387
1.48585858585859 0.00137947310568975
1.53535353535354 0.00137677146382431
1.58484848484849 0.0013747897203803
1.63434343434343 0.00137297259734115
1.68383838383838 0.00137218362997493
1.73333333333333 0.0013705673081742
1.78282828282828 0.00136979023927375
1.83232323232323 0.00136922480580548
1.88181818181818 0.00136886812820636
1.93131313131313 0.00136869461607525
1.98080808080808 0.00136868927038136
2.03030303030303 0.00136851081156319
2.07979797979798 0.00136870455372162
2.12929292929293 0.00136855109700396
2.17878787878788 0.0013688740371266
2.22828282828283 0.00136906231936362
2.27777777777778 0.00136928541308079
2.32727272727273 0.00136953441180265
2.37676767676768 0.00136980132144582
2.42626262626263 0.00137007891751473
2.47575757575758 0.00137069372458964
2.52525252525253 0.00137119275011666
2.57474747474748 0.00137161637938921
2.62424242424242 0.00137184555437675
2.67373737373737 0.00137226177124483
2.72323232323232 0.00137268526203882
2.77272727272727 0.00137311345458113
2.82222222222222 0.00137354528621647
2.87171717171717 0.00137397920235305
2.92121212121212 0.00137441425245532
2.97070707070707 0.00137484933498011
3.02020202020202 0.00137528344484961
3.06969696969697 0.0013757160815896
3.11919191919192 0.00137614642894968
3.16868686868687 0.00137657411449732
3.21818181818182 0.00137699854605877
3.26767676767677 0.00137741934033786
3.31717171717172 0.00137783611603493
3.36666666666667 0.00137824878899506
3.41616161616162 0.00137865744108123
3.46565656565657 0.00137906052571313
3.51515151515152 0.00137935828852576
3.56464646464647 0.00137975374774557
3.61414141414142 0.0013801442304977
3.66363636363636 0.00138052967068902
3.71313131313131 0.00138090996325269
3.76262626262626 0.00138128515727753
3.81212121212121 0.0013816551553737
3.86161616161616 0.00138201997005549
3.91111111111111 0.00138237958235266
3.96060606060606 0.00138273408023815
4.01010101010101 0.00138308338559101
4.05959595959596 0.00138342760865806
4.10909090909091 0.00138376677732321
4.15858585858586 0.00138410092181906
4.20808080808081 0.00138443007095534
4.25757575757576 0.0013847543055765
4.30707070707071 0.00138507370085106
4.35656565656566 0.00138546661378736
4.40606060606061 0.00138581625710676
4.45555555555556 0.00138612227760193
4.50505050505051 0.00138642367195307
4.55454545454546 0.00138672055667521
4.6040404040404 0.00138701302621247
4.65353535353535 0.00138730099943707
4.7030303030303 0.00138746239624121
4.75252525252525 0.00138774106316431
4.8020202020202 0.00138801557820556
4.85151515151515 0.00138828595015283
4.9010101010101 0.00138855252248227
4.95050505050505 0.00138881537002777
5 0.0013889804884788
};
\addplot [semithick, blue]
table {%
0.1 0.162594536037619
0.14949494949495 0.165037289064945
0.198989898989899 0.167045418613861
0.248484848484848 0.168567570891998
0.297979797979798 0.169573642391271
0.347474747474747 0.170063778055665
0.396969696969697 0.170071094293063
0.446464646464647 0.169657869195398
0.495959595959596 0.168897570781431
0.545454545454546 0.167865208258165
0.594949494949495 0.166628267606979
0.644444444444444 0.165243002569426
0.693939393939394 0.163754039033983
0.743434343434343 0.162195685457139
0.792929292929293 0.160593845221028
0.842424242424242 0.158967829624713
0.891919191919192 0.15733193676615
0.941414141414142 0.155696749980695
0.990909090909091 0.154070024879278
1.04040404040404 0.152457449535045
1.08989898989899 0.150863141686615
1.13939393939394 0.149290055677527
1.18888888888889 0.147740308473836
1.23838383838384 0.146215315466285
1.28787878787879 0.144716007563677
1.33737373737374 0.14324293431304
1.38686868686869 0.141796374716679
1.43636363636364 0.140376378597019
1.48585858585859 0.138982800837392
1.53535353535354 0.137616051692709
1.58484848484849 0.136274471867848
1.63434343434343 0.134958378572732
1.68383838383838 0.133667286947864
1.73333333333333 0.132400725153225
1.78282828282828 0.131158169963412
1.83232323232323 0.129939088934972
1.88181818181818 0.128742939167238
1.93131313131313 0.127569183134203
1.98080808080808 0.12641728051428
2.03030303030303 0.125286713089074
2.07979797979798 0.124176904717847
2.12929292929293 0.123087355469422
2.17878787878788 0.122017533695506
2.22828282828283 0.120966996133008
2.27777777777778 0.119935178628815
2.32727272727273 0.118921645064967
2.37676767676768 0.117925904021604
2.42626262626263 0.116947535595841
2.47575757575758 0.115986083128784
2.52525252525253 0.115041086593307
2.57474747474748 0.114112217477219
2.62424242424242 0.113198958294471
2.67373737373737 0.112300973839237
2.72323232323232 0.11141792074053
2.77272727272727 0.110549364443605
2.82222222222222 0.109694977955029
2.87171717171717 0.108854508941944
2.92121212121212 0.108027465177793
2.97070707070707 0.107213582297459
3.02020202020202 0.1064125597291
3.06969696969697 0.105624101881804
3.11919191919192 0.104847885914749
3.16868686868687 0.104083627669038
3.21818181818182 0.103331428466899
3.26767676767677 0.102590247520341
3.31717171717172 0.101860270461839
3.36666666666667 0.101141176107502
3.41616161616162 0.100432731261731
3.46565656565657 0.0997347033358215
3.51515151515152 0.099046893613254
3.56464646464647 0.0983690048336304
3.61414141414142 0.0977008436352714
3.66363636363636 0.0970422851315939
3.71313131313131 0.0963930922741533
3.76262626262626 0.0957529316447892
3.81212121212121 0.0951217652952697
3.86161616161616 0.0944992681570468
3.91111111111111 0.0938854065079102
3.96060606060606 0.093279940395996
4.01010101010101 0.0926826274226107
4.05959595959596 0.0920934316007861
4.10909090909091 0.0915120664409551
4.15858585858586 0.0909385212582634
4.20808080808081 0.09037241934553
4.25757575757576 0.0898137493924253
4.30707070707071 0.0892622900341929
4.35656565656566 0.0887180347251558
4.40606060606061 0.0881807900757978
4.45555555555556 0.0876503726063557
4.50505050505051 0.0871266680581823
4.55454545454546 0.0866095535087366
4.6040404040404 0.0860988770980736
4.65353535353535 0.0855945825219353
4.7030303030303 0.085096468664975
4.75252525252525 0.0846044490625696
4.8020202020202 0.0841185484768914
4.85151515151515 0.0836384048574688
4.9010101010101 0.083164107884965
4.95050505050505 0.0826954176649325
5 0.0822323194943213
};
\addplot [semithick, red]
table {%
0.1 0.00240832766970978
0.14949494949495 0.0041827673139615
0.198989898989899 0.00558427398060912
0.248484848484848 0.00666937565497572
0.297979797979798 0.00749090874492964
0.347474747474747 0.00809664330632498
0.396969696969697 0.00852830633636081
0.446464646464647 0.00882122443543565
0.495959595959596 0.00900478340614874
0.545454545454546 0.00910250384253064
0.594949494949495 0.00913349349743775
0.644444444444444 0.00911279487813088
0.693939393939394 0.00905233413605366
0.743434343434343 0.008961521497584
0.792929292929293 0.00884782628691383
0.842424242424242 0.00871710634991352
0.891919191919192 0.00857402952531805
0.941414141414142 0.00842230020625689
0.990909090909091 0.0082647779306072
1.04040404040404 0.0081037352073039
1.08989898989899 0.00794101003960701
1.13939393939394 0.00777791346533219
1.18888888888889 0.00761557267363777
1.23838383838384 0.00745478351885687
1.28787878787879 0.00729624316672728
1.33737373737374 0.0071403105282557
1.38686868686869 0.006987480694431
1.43636363636364 0.00683786462658031
1.48585858585859 0.00669170396088592
1.53535353535354 0.00654911414186121
1.58484848484849 0.00641013426541581
1.63434343434343 0.00627478405263171
1.68383838383838 0.00614304217570338
1.73333333333333 0.0060148994781809
1.78282828282828 0.00589028335797148
1.83232323232323 0.00576914267562956
1.88181818181818 0.00565140037273915
1.93131313131313 0.00553698637613875
1.98080808080808 0.00542578115881531
2.03030303030303 0.00531773029454707
2.07979797979798 0.00521268106999484
2.12929292929293 0.00511061471215912
2.17878787878788 0.00501141716744924
2.22828282828283 0.00491492968227514
2.27777777777778 0.0048211088779142
2.32727272727273 0.00472990315203881
2.37676767676768 0.00464117723290969
2.42626262626263 0.00455485312605741
2.47575757575758 0.00447089956184377
2.52525252525253 0.00438913793037876
2.57474747474748 0.0043095091267028
2.62424242424242 0.00423200867717977
2.67373737373737 0.00415654054157855
2.72323232323232 0.00408299921150745
2.77272727272727 0.0040113524398383
2.82222222222222 0.00394149860256054
2.87171717171717 0.00387339332360326
2.92121212121212 0.00380703412616323
2.97070707070707 0.00374226792911581
3.02020202020202 0.00367907463181116
3.06969696969697 0.00361743563286965
3.11919191919192 0.00355722625351518
3.16868686868687 0.00349845869914611
3.21818181818182 0.00344109229865774
3.26767676767677 0.00338504885586555
3.31717171717172 0.00333029770000703
3.36666666666667 0.00327675808513395
3.41616161616162 0.00322446410107569
3.46565656565657 0.00317331856149949
3.51515151515152 0.00312333161513079
3.56464646464647 0.00307442937023084
3.61414141414142 0.00302662558814848
3.66363636363636 0.00297977333232302
3.71313131313131 0.00293395858740764
3.76262626262626 0.00288910261382591
3.81212121212121 0.00284518032630798
3.86161616161616 0.00280222183172973
3.91111111111111 0.00276006857934963
3.96060606060606 0.00271883852299526
4.01010101010101 0.0026783528744504
4.05959595959596 0.00263872486267769
4.10909090909091 0.00259995103809496
4.15858585858586 0.0025618393586041
4.20808080808081 0.00252448101549352
4.25757575757576 0.00248789884935774
4.30707070707071 0.00245196493729782
4.35656565656566 0.00241670009789363
4.40606060606061 0.00238213636998108
4.45555555555556 0.00234820825118365
4.50505050505051 0.00231489905536786
4.55454545454546 0.0022822050157052
4.6040404040404 0.00225008855620423
4.65353535353535 0.00221854632735174
4.7030303030303 0.00218758097440563
4.75252525252525 0.00215717171839436
4.8020202020202 0.002127304916415
4.85151515151515 0.00209790390921805
4.9010101010101 0.00206907068044171
4.95050505050505 0.00204068644632882
5 0.00201279792171949
};
\addplot [semithick, color1, dashed]
table {%
0.1 0.162479124980031
0.14949494949495 0.164578640571414
0.198989898989899 0.166083532232804
0.248484848484848 0.16706304654941
0.297979797979798 0.167484719085566
0.347474747474747 0.167401069026763
0.396969696969697 0.166874243509245
0.446464646464647 0.165972598567817
0.495959595959596 0.164778512815149
0.545454545454546 0.163367941180861
0.594949494949495 0.161806605354973
0.644444444444444 0.160146255781075
0.693939393939394 0.158426332398037
0.743434343434343 0.156675268412963
0.792929292929293 0.154914646167312
0.842424242424242 0.153157834142395
0.891919191919192 0.151422601871993
0.941414141414142 0.149700220595403
0.990909090909091 0.148003354860064
1.04040404040404 0.146335478054216
1.08989898989899 0.144699149050193
1.13939393939394 0.143095708404425
1.18888888888889 0.141523474352981
1.23838383838384 0.139984703743665
1.28787878787879 0.138477672701654
1.33737373737374 0.137002793783342
1.38686868686869 0.135558354328998
1.43636363636364 0.134145343331051
1.48585858585859 0.132760995952233
1.53535353535354 0.131405970494725
1.58484848484849 0.130079405326749
1.63434343434343 0.128780529160441
1.68383838383838 0.127508540660631
1.73333333333333 0.12626239890517
1.78282828282828 0.125041773452639
1.83232323232323 0.123846061261184
1.88181818181818 0.122673853006108
1.93131313131313 0.121524780942847
1.98080808080808 0.120398161633211
2.03030303030303 0.119293334337162
2.07979797979798 0.118209298750987
2.12929292929293 0.117146748747314
2.17878787878788 0.116103302728173
2.22828282828283 0.115079467015378
2.27777777777778 0.114074719154153
2.32727272727273 0.113087394189817
2.37676767676768 0.112118949293664
2.42626262626263 0.111167016021255
2.47575757575758 0.110231734843028
2.52525252525253 0.10931416723956
2.57474747474748 0.108410393725412
2.62424242424242 0.107523620075925
2.67373737373737 0.106651608797103
2.72323232323232 0.105779870174513
2.77272727272727 0.104936635147557
2.82222222222222 0.104107626013999
2.87171717171717 0.10329252967681
2.92121212121212 0.102489395153857
2.97070707070707 0.101700015534221
3.02020202020202 0.100923780306162
3.06969696969697 0.100159406730045
3.11919191919192 0.0994069516566639
3.16868686868687 0.0986661871864203
3.21818181818182 0.097936833881318
3.26767676767677 0.0972186214349177
3.31717171717172 0.0965112883824062
3.36666666666667 0.0958145810116157
3.41616161616162 0.0951282538352987
3.46565656565657 0.0944520787379727
3.51515151515152 0.0937858001165465
3.56464646464647 0.0931292093905941
3.61414141414142 0.0924820904968907
3.66363636363636 0.0918442334586947
3.71313131313131 0.0912154352534231
3.76262626262626 0.090595498579366
3.81212121212121 0.0899842205810666
3.86161616161616 0.0893806688263329
3.91111111111111 0.0887862099705746
3.96060606060606 0.0881998793926569
4.01010101010101 0.0876215068699779
4.05959595959596 0.0870509271780351
4.10909090909091 0.0864879796492724
4.15858585858586 0.085932820610329
4.20808080808081 0.0853846203819026
4.25757575757576 0.0848436014056962
4.30707070707071 0.084309620559971
4.35656565656566 0.0837825385792391
4.40606060606061 0.0832622196640052
4.45555555555556 0.0827485332575084
4.50505050505051 0.0822413481576819
4.55454545454546 0.0817407761799508
4.6040404040404 0.0812454482170705
4.65353535353535 0.0807574905863607
4.7030303030303 0.0802756416584112
4.75252525252525 0.0797986813044554
4.8020202020202 0.079328214468414
4.85151515151515 0.0788633210979448
4.9010101010101 0.0784039639863146
4.95050505050505 0.0779500527699868
5 0.0775015009396961
};
\end{axis}

\end{tikzpicture}

%% file: Figures/conditional_variance/exp_conditional_variance_bo_p=0.75_n_over_p=2.0.tex
\begin{tikzpicture}
\tikzstyle{every node}=[font=\tiny]

\definecolor{color0}{rgb}{0,1,1}
\definecolor{color1}{rgb}{1,0.647058823529412,0}

\begin{axis}[
height=\figheight,
legend cell align={left},
legend style={
  fill opacity=0.8,
  draw opacity=1,
  text opacity=1,
  at={(0.97,0.03)},
  anchor=south east,
  draw=white!80!black
},
tick align=outside,
tick pos=left,
width=\figwidth,
x grid style={white!69.0196078431373!black},
xlabel={\(\displaystyle 1 / \alpha\)},
xmajorgrids,
xmin=-0.145, xmax=5.245,
xtick style={color=black},
y grid style={white!69.0196078431373!black},
ylabel={\(\displaystyle {\rm Var}(\hat{f}_{\rm bo} | \hat{f} = 0.75)\)},
ymajorgrids,
ymin=0.00200215473350759, ymax=0.0113369750037879,
ytick style={color=black}
]
\addplot [draw=black, fill=black, forget plot, mark=*, only marks, mark size=1pt]
table{%
x  y
0.594949494949495 0.0044538224349373
0.842424242424242 0.00442453936025972
1.08989898989899 0.00433481372351047
1.33737373737374 0.00470035949010633
1.58484848484849 0.00436328191559648
1.83232323232323 0.00415474114483877
2.07979797979798 0.00379514037173123
2.32727272727273 0.00376361309494577
2.57474747474748 0.00360805030217914
2.82222222222222 0.00418662815230214
3.06969696969697 0.00375198378245867
3.31717171717172 0.0040099594907545
3.56464646464647 0.00416196797246721
3.81212121212121 0.00412904887104215
4.05959595959596 0.00386603538535557
4.30707070707071 0.00408048612343287
};
\addplot [draw=blue, fill=blue, forget plot, mark=*, only marks, mark size=1pt]
table{%
x  y
0.594949494949495 0.00497046112286759
0.842424242424242 0.0055527249542513
1.08989898989899 0.00581329449747958
1.33737373737374 0.0063083213746723
1.58484848484849 0.00688543024824023
1.83232323232323 0.00679989254710989
2.07979797979798 0.00680828008142453
2.32727272727273 0.00698933103321886
2.57474747474748 0.00721056417811128
2.82222222222222 0.00755099416109428
3.06969696969697 0.00758046181374183
3.31717171717172 0.0077207629779652
3.56464646464647 0.00795683787224705
3.81212121212121 0.00820569878569477
4.05959595959596 0.00808345947501266
4.30707070707071 0.00829157499286304
};
\addplot [draw=color0, fill=color0, forget plot, mark=*, only marks, mark size=1pt]
table{%
x  y
0.594949494949495 0.00412083535895719
0.842424242424242 0.00436869327895423
1.08989898989899 0.00465654120769537
1.33737373737374 0.00494597123795382
1.58484848484849 0.00543083543695033
1.83232323232323 0.00548105692973833
2.07979797979798 0.00558590905789674
2.32727272727273 0.00565609802241607
2.57474747474748 0.00590256309207116
2.82222222222222 0.00623342672672988
3.06969696969697 0.00631182620070853
3.31717171717172 0.00639005311252938
3.56464646464647 0.0067844112404295
3.81212121212121 0.00687141541861841
4.05959595959596 0.00677003993085123
4.30707070707071 0.00706288935908239
};
\addplot [draw=red, fill=red, forget plot, mark=*, only marks, mark size=1pt]
table{%
x  y
0.594949494949495 0.0109070224539318
0.842424242424242 0.0107752556469065
1.08989898989899 0.0106505010175668
1.33737373737374 0.0106178421863729
1.58484848484849 0.0107207356893015
1.83232323232323 0.0104878459822878
2.07979797979798 0.0103991901225555
2.32727272727273 0.0102712450879344
2.57474747474748 0.0102778990538659
2.82222222222222 0.0104291099775546
3.06969696969697 0.010320083740987
3.31717171717172 0.0103122354459724
3.56464646464647 0.0102311167940857
3.81212121212121 0.0102850586803651
4.05959595959596 0.0102454331646035
4.30707070707071 0.0103052049137503
};
\addplot [draw=color1, fill=color1, forget plot, mark=*, only marks, mark size=1pt]
table{%
x  y
0.594949494949495 0.0108718052630242
0.842424242424242 0.0106724555640503
1.08989898989899 0.0105072434060036
1.33737373737374 0.0105874701176428
1.58484848484849 0.0105412563384926
1.83232323232323 0.010451590314969
2.07979797979798 0.0103646108311856
2.32727272727273 0.0102153619888491
2.57474747474748 0.0103001138221818
2.82222222222222 0.0103954065842687
3.06969696969697 0.0101881388183216
3.31717171717172 0.0102639747985516
3.56464646464647 0.0102452416743867
3.81212121212121 0.0103379769070443
4.05959595959596 0.0101785778771928
4.30707070707071 0.0103078282968218
};
\addplot [semithick, black]
table {%
0.1 0.00275728160438132
0.14949494949495 0.00256497375418746
0.198989898989899 0.00245153368288281
0.248484848484848 0.00242646474579306
0.297979797979798 0.00250283300592885
0.347474747474747 0.0026990939081179
0.396969696969697 0.00304322558133985
0.446464646464647 0.003579278054679
0.495959595959596 0.00408263835875067
0.545454545454546 0.00421792826399287
0.594949494949495 0.0043028807211371
0.644444444444444 0.00437137509851598
0.693939393939394 0.00442531666062329
0.743434343434343 0.00446500797020899
0.792929292929293 0.00449164055180862
0.842424242424242 0.00450708397132976
0.891919191919192 0.00451340347103929
0.941414141414142 0.00451253954886127
0.990909090909091 0.00450616635586326
1.04040404040404 0.00449565909553667
1.08989898989899 0.00448211478528338
1.13939393939394 0.00446639272771576
1.18888888888889 0.00444915820107672
1.23838383838384 0.00443092218602692
1.28787878787879 0.00441207479226341
1.33737373737374 0.00439291225982047
1.38686868686869 0.00437365827578778
1.43636363636364 0.00435448059427429
1.48585858585859 0.00433550389588649
1.53535353535354 0.00431681970919412
1.58484848484849 0.00429849405346716
1.63434343434343 0.00428057332847692
1.68383838383838 0.00426308885757581
1.73333333333333 0.00424606039754316
1.78282828282828 0.00422949885319501
1.83232323232323 0.00421340838477302
1.88181818181818 0.00419778964100526
1.93131313131313 0.00418263397075058
1.98080808080808 0.00416793590588826
2.03030303030303 0.00415368588905074
2.07979797979798 0.00413987266934218
2.12929292929293 0.00412648401758314
2.17878787878788 0.00411350918324688
2.22828282828283 0.0041009310937381
2.27777777777778 0.00408873794511788
2.32727272727273 0.00407691631178519
2.37676767676768 0.00406545289537447
2.42626262626263 0.00405433461853927
2.47575757575758 0.00404354869410251
2.52525252525253 0.00403308267500113
2.57474747474748 0.00402292445650704
2.62424242424242 0.00401306246087602
2.67373737373737 0.00400348532690714
2.72323232323232 0.00399418224094339
2.77272727272727 0.00398514277402567
2.82222222222222 0.00397635691187942
2.87171717171717 0.00396781504781213
2.92121212121212 0.00395950797352373
2.97070707070707 0.00395142686817584
3.02020202020202 0.00394356328623507
3.06969696969697 0.00393590914450387
3.11919191919192 0.00392845670865394
3.16868686868687 0.0039211985795195
3.21818181818182 0.00391412768000404
3.26767676767677 0.00390723723877306
3.31717171717172 0.00390052078077047
3.36666666666667 0.00389397211158493
3.41616161616162 0.00388758530509842
3.46565656565657 0.00388135469130307
3.51515151515152 0.00387527493437961
3.56464646464647 0.00386934064722888
3.61414141414142 0.00386354696273233
3.66363636363636 0.00385788912061896
3.71313131313131 0.00385236256174903
3.76262626262626 0.00384696291852493
3.81212121212121 0.00384168600575191
3.86161616161616 0.003836527811936
3.91111111111111 0.0038314844910049
3.96060606060606 0.00382655235443885
4.01010101010101 0.00382172786379009
4.05959595959596 0.00381700762358139
4.10909090909091 0.00381238837456044
4.15858585858586 0.00380786698728497
4.20808080808081 0.00380344045610276
4.25757575757576 0.00379910589329963
4.30707070707071 0.00379486052368849
4.35656565656566 0.00379070167940249
4.40606060606061 0.00378662679497005
4.45555555555556 0.00378263350396607
4.50505050505051 0.0037787192357237
4.55454545454546 0.00377488179944996
4.6040404040404 0.00377111899485011
4.65353535353535 0.0037674287029173
4.7030303030303 0.00376380888233063
4.75252525252525 0.00376025756603438
4.8020202020202 0.00375677285798864
4.85151515151515 0.00375335293008183
4.9010101010101 0.0037499960191984
4.95050505050505 0.0037467004244271
5 0.0037434645044126
};
\addplot [semithick, blue]
table {%
0.1 0.00471250070061874
0.14949494949495 0.00459555587854993
0.198989898989899 0.00452403430698489
0.248484848484848 0.00449523169767851
0.297979797979798 0.00450410254588202
0.347474747474747 0.00454449953799679
0.396969696969697 0.00461117529138277
0.446464646464647 0.00470084047342573
0.495959595959596 0.00481075951152604
0.545454545454546 0.00493589256643956
0.594949494949495 0.00506844454064381
0.644444444444444 0.0052003199140907
0.693939393939394 0.00532563811409831
0.743434343434343 0.00544143766012556
0.792929292929293 0.00554702358529985
0.842424242424242 0.00564300182095046
0.891919191919192 0.00573053431308995
0.941414141414142 0.00581090068219542
0.990909090909091 0.00588529213986666
1.04040404040404 0.00595473139742914
1.08989898989899 0.0060200620402262
1.13939393939394 0.00608196385773202
1.18888888888889 0.00614097646837219
1.23838383838384 0.00619752898819387
1.28787878787879 0.0062519603644754
1.33737373737374 0.00630453935872222
1.38686868686869 0.00635547901893324
1.43636363636364 0.00640495012700831
1.48585858585859 0.00645309152292639
1.53535353535354 0.006500020535881
1.58484848484849 0.00654581444323227
1.63434343434343 0.00659055481511178
1.68383838383838 0.00663430585779967
1.73333333333333 0.00667711965560186
1.78282828282828 0.00671904171369364
1.83232323232323 0.00676011102262036
1.88181818181818 0.00680036167989911
1.93131313131313 0.00683982339383765
1.98080808080808 0.00687852281557438
2.03030303030303 0.00691648316665244
2.07979797979798 0.00695372843363368
2.12929292929293 0.00699027778484962
2.17878787878788 0.00702615079634972
2.22828282828283 0.00706136236730265
2.27777777777778 0.00709593227801597
2.32727272727273 0.00712987427028272
2.37676767676768 0.00716320434373757
2.42626262626263 0.00719593524445317
2.47575757575758 0.00722808132926034
2.52525252525253 0.00725965713686244
2.57474747474748 0.00729067155564378
2.62424242424242 0.00732114184155913
2.67373737373737 0.00735107745434443
2.72323232323232 0.00738048840099603
2.77272727272727 0.00740938879932507
2.82222222222222 0.007437788305941
2.87171717171717 0.00746569359163951
2.92121212121212 0.00749312176319672
2.97070707070707 0.00752008036126789
3.02020202020202 0.00754657866545283
3.06969696969697 0.00757262592373015
3.11919191919192 0.00759823269670201
3.16868686868687 0.00762340806906209
3.21818181818182 0.00764817054431283
3.26767676767677 0.00767250910744233
3.31717171717172 0.00769644038402845
3.36666666666667 0.00771997555366549
3.41616161616162 0.00774312226053619
3.46565656565657 0.00776588824439917
3.51515151515152 0.00778827997107906
3.56464646464647 0.00781030803590621
3.61414141414142 0.00783197887189674
3.66363636363636 0.00785329624354258
3.71313131313131 0.00787426856399515
3.76262626262626 0.00789490817547506
3.81212121212121 0.00791521558949804
3.86161616161616 0.00793520292226751
3.91111111111111 0.007954870691528
3.96060606060606 0.00797422772368339
4.01010101010101 0.00799328296001173
4.05959595959596 0.00801203722664551
4.10909090909091 0.00803050129159288
4.15858585858586 0.00804867511305762
4.20808080808081 0.00806657302889169
4.25757575757576 0.00808419510136521
4.30707070707071 0.00810154962810788
4.35656565656566 0.00811863660695877
4.40606060606061 0.00813546331677312
4.45555555555556 0.0081520366514764
4.50505050505051 0.00816836085659323
4.55454545454546 0.00818444053508194
4.6040404040404 0.00820028140759582
4.65353535353535 0.00821588558146774
4.7030303030303 0.00823126066993679
4.75252525252525 0.00824640994823989
4.8020202020202 0.00826133254982397
4.85151515151515 0.00827604212655403
4.9010101010101 0.00829053540345576
4.95050505050505 0.00830482143304384
5 0.00831890088141396
};
\addplot [semithick, color0, dashed]
table {%
0.1 0.00441441721485469
0.14949494949495 0.00417825473588962
0.198989898989899 0.00400659463803299
0.248484848484848 0.00389078944213472
0.297979797979798 0.00382620741315282
0.347474747474747 0.00380374554499996
0.396969696969697 0.00381630064385424
0.446464646464647 0.00385967471607601
0.495959595959596 0.00393004098938649
0.545454545454546 0.00402140509206023
0.594949494949495 0.00412490390960663
0.644444444444444 0.00423143221007566
0.693939393939394 0.0043342508433683
0.743434343434343 0.0044297706922008
0.792929292929293 0.00451686373588034
0.842424242424242 0.00459588620858736
0.891919191919192 0.00466740970872515
0.941414141414142 0.00473338201450002
0.990909090909091 0.00479447349893747
1.04040404040404 0.00485159768978222
1.08989898989899 0.00490548898831705
1.13939393939394 0.00495675248098149
1.18888888888889 0.00500597549316262
1.23838383838384 0.0050534208633054
1.28787878787879 0.00509945588343946
1.33737373737374 0.0051442612748086
1.38686868686869 0.00518806799791038
1.43636363636364 0.00523092401935382
1.48585858585859 0.00527303595925077
1.53535353535354 0.00531440982689885
1.58484848484849 0.00535511057935328
1.63434343434343 0.00539518449547055
1.68383838383838 0.00543466903683054
1.73333333333333 0.00547360635031818
1.78282828282828 0.00551199691782595
1.83232323232323 0.00554985140611841
1.88181818181818 0.00558721901890513
1.93131313131313 0.00562409426251131
1.98080808080808 0.00566048684626086
2.03030303030303 0.00569640495689538
2.07979797979798 0.00573187402756214
2.12929292929293 0.00576683388036331
2.17878787878788 0.00580138081030718
2.22828282828283 0.00583546537186891
2.27777777777778 0.00586909108506806
2.32727272727273 0.00590231995543289
2.37676767676768 0.00593505726065185
2.42626262626263 0.00596740219244007
2.47575757575758 0.00599932825315713
2.52525252525253 0.00603076373190897
2.57474747474748 0.00606188794432189
2.62424242424242 0.0060925233110582
2.67373737373737 0.00612276625509361
2.72323232323232 0.00615335451929444
2.77272727272727 0.0061828034052106
2.82222222222222 0.00621184395795471
2.87171717171717 0.0062404790607048
2.92121212121212 0.00626879360893506
2.97070707070707 0.0062966869588677
3.02020202020202 0.00632417869446888
3.06969696969697 0.00635132211430572
3.11919191919192 0.00637810460653282
3.16868686868687 0.00640452819576198
3.21818181818182 0.00643059792818373
3.26767676767677 0.00645631883946929
3.31717171717172 0.00648169594800013
3.36666666666667 0.00650673428355092
3.41616161616162 0.00653143884291424
3.46565656565657 0.00655581412275907
3.51515151515152 0.00657986629228197
3.56464646464647 0.00660359953293754
3.61414141414142 0.00662701868696802
3.66363636363636 0.00665012857706448
3.71313131313131 0.00667293394896551
3.76262626262626 0.00669543951712459
3.81212121212121 0.00671765049586026
3.86161616161616 0.0067396075054636
3.91111111111111 0.00676124042225096
3.96060606060606 0.00678259188542341
4.01010101010101 0.00680366633953089
4.05959595959596 0.00682446816296844
4.10909090909091 0.00684500167985425
4.15858585858586 0.00686525622803835
4.20808080808081 0.00688526840764081
4.25757575757576 0.00690502465561471
4.30707070707071 0.00692452903018981
4.35656565656566 0.00694378552999569
4.40606060606061 0.00696279810609468
4.45555555555556 0.00698157057178833
4.50505050505051 0.00700010687553965
4.55454545454546 0.00701839960880823
4.6040404040404 0.00703651121868526
4.65353535353535 0.00705433963853386
4.7030303030303 0.00707194248497162
4.75252525252525 0.00708937515268776
4.8020202020202 0.00710656070715709
4.85151515151515 0.00712354065424586
4.9010101010101 0.00714031518226765
4.95050505050505 0.00715688703648948
5 0.00717325884701542
};
\addplot [semithick, red]
table {%
0.1 0.0109126649915025
0.14949494949495 0.0108783694780591
0.198989898989899 0.0108601079593055
0.248484848484848 0.0108535169771659
0.297979797979798 0.010853925241984
0.347474747474747 0.0108573152912909
0.396969696969697 0.0108612242849196
0.446464646464647 0.010864984082588
0.495959595959596 0.0108688412964234
0.545454545454546 0.010872584971689
0.594949494949495 0.0108750396826751
0.644444444444444 0.0108747132598581
0.693939393939394 0.0108706018560758
0.743434343434343 0.0108624627817332
0.792929292929293 0.0108506391274156
0.842424242424242 0.0108357710860737
0.891919191919192 0.0108185741525456
0.941414141414142 0.0107997154857661
0.990909090909091 0.0107797625045353
1.04040404040404 0.0107591724573216
1.08989898989899 0.0107383007953052
1.13939393939394 0.0107174168423668
1.18888888888889 0.0106967204120771
1.23838383838384 0.0106763563068716
1.28787878787879 0.0106564269055025
1.33737373737374 0.0106370021210273
1.38686868686869 0.0106181273836796
1.43636363636364 0.0105998295907481
1.48585858585859 0.010582121882098
1.53535353535354 0.0105650071838665
1.58484848484849 0.0105484810384041
1.63434343434343 0.0105325335681733
1.68383838383838 0.010517151131632
1.73333333333333 0.0105023174629409
1.78282828282828 0.0104880146709725
1.83232323232323 0.0104742236706399
1.88181818181818 0.0104609251174319
1.93131313131313 0.0104480993230954
1.98080808080808 0.0104357268554741
2.03030303030303 0.0104237886258907
2.07979797979798 0.0104122658098713
2.12929292929293 0.0104011405300529
2.17878787878788 0.0103903952589559
2.22828282828283 0.0103800132250274
2.27777777777778 0.0103699784112489
2.32727272727273 0.0103602754310358
2.37676767676768 0.0103508897552553
2.42626262626263 0.0103418073287352
2.47575757575758 0.0103330149078393
2.52525252525253 0.0103244998371369
2.57474747474748 0.0103162502088857
2.62424242424242 0.0103082544765575
2.67373737373737 0.010300501884467
2.72323232323232 0.0102929821155366
2.77272727272727 0.0102856854346899
2.82222222222222 0.0102786024822363
2.87171717171717 0.0102717245573527
2.92121212121212 0.0102650432669446
2.97070707070707 0.0102585505468473
3.02020202020202 0.0102522389933003
3.06969696969697 0.0102461013225805
3.11919191919192 0.0102401307041629
3.16868686868687 0.0102343207203339
3.21818181818182 0.0102286651426136
3.26767676767677 0.0102231581099801
3.31717171717172 0.0102177939430026
3.36666666666667 0.0102125673743397
3.41616161616162 0.0102074733255929
3.46565656565657 0.0102025069144067
3.51515151515152 0.0101976636447096
3.56464646464647 0.0101929389794598
3.61414141414142 0.0101883287229381
3.66363636363636 0.0101838289143945
3.71313131313131 0.0101794356692184
3.76262626262626 0.0101751453034237
3.81212121212121 0.0101709544701321
3.86161616161616 0.0101668596443245
3.91111111111111 0.0101628576599301
3.96060606060606 0.0101589454599056
4.01010101010101 0.0101551201690162
4.05959595959596 0.0101513789281104
4.10909090909091 0.0101477189830501
4.15858585858586 0.0101441379026695
4.20808080808081 0.0101406330530129
4.25757575757576 0.0101372021200347
4.30707070707071 0.0101338428130782
4.35656565656566 0.0101305530118536
4.40606060606061 0.010127330402663
4.45555555555556 0.0101241732351611
4.50505050505051 0.0101210794224544
4.55454545454546 0.0101180470708701
4.6040404040404 0.0101150745275456
4.65353535353535 0.0101121599248311
4.7030303030303 0.0101093015592151
4.75252525252525 0.0101064980823372
4.8020202020202 0.0101037477219176
4.85151515151515 0.010101049098464
4.9010101010101 0.0100984007261624
4.95050505050505 0.0100958012147804
5 0.010093249350823
};
\addplot [semithick, color1, dashed]
table {%
0.1 0.0108756513570374
0.14949494949495 0.0108489738185937
0.198989898989899 0.0108384497595233
0.248484848484848 0.0108385843994138
0.297979797979798 0.0108442446588125
0.347474747474747 0.0108513985002128
0.396969696969697 0.0108578107283207
0.446464646464647 0.0108631181310922
0.495959595959596 0.0108678494115793
0.545454545454546 0.0108720158805824
0.594949494949495 0.0108746152951811
0.644444444444444 0.0108742833941655
0.693939393939394 0.010870106544011
0.743434343434343 0.0108618946392967
0.792929292929293 0.0108500152915679
0.842424242424242 0.0108351157538419
0.891919191919192 0.0108179094335569
0.941414141414142 0.0107990577563915
0.990909090909091 0.0107791217618175
1.04040404040404 0.0107585529023247
1.08989898989899 0.0107377022446487
1.13939393939394 0.0107168358086608
1.18888888888889 0.0106961511864702
1.23838383838384 0.0106757919835165
1.28787878787879 0.0106558597702257
1.33737373737374 0.010636424379985
1.38686868686869 0.0106175313526892
1.43636363636364 0.0105992080155899
1.48585858585859 0.0105814678680118
1.53535353535354 0.0105643147476607
1.58484848484849 0.0105477445651002
1.63434343434343 0.0105317482633162
1.68383838383838 0.0105163123502122
1.73333333333333 0.0105014221159269
1.78282828282828 0.010487059229625
1.83232323232323 0.0104732055486646
1.88181818181818 0.0104598420324334
1.93131313131313 0.0104469494770467
1.98080808080808 0.0104345087701432
2.03030303030303 0.0104225014147287
2.07979797979798 0.0104109085292982
2.12929292929293 0.0103997130099743
2.17878787878788 0.0103888967881777
2.22828282828283 0.0103784439554158
2.27777777777778 0.0103683384520024
2.32727272727273 0.0103585650938112
2.37676767676768 0.0103491093783611
2.42626262626263 0.0103399574743401
2.47575757575758 0.010331095745994
2.52525252525253 0.0103225122467436
2.57474747474748 0.0103141949945812
2.62424242424242 0.0103061328261821
2.67373737373737 0.010298314391208
2.72323232323232 0.010290729699275
2.77272727272727 0.0102833690129046
2.82222222222222 0.0102762230778789
2.87171717171717 0.010269283101629
2.92121212121212 0.0102625407275446
2.97070707070707 0.0102559880131503
3.02020202020202 0.010249617407619
3.06969696969697 0.010243421729946
3.11919191919192 0.0102373941502827
3.16868686868687 0.0102315281695181
3.21818181818182 0.0102258176025718
3.26767676767677 0.0102202565597178
3.31717171717172 0.0102148394319601
3.36666666666667 0.0102095608743957
3.41616161616162 0.0102044157922303
3.46565656565657 0.0101993993307209
3.51515151515152 0.0101945070558276
3.56464646464647 0.0101897341416491
3.61414141414142 0.0101850765702637
3.66363636363636 0.0101805303106015
3.71313131313131 0.010176091511525
3.76262626262626 0.0101717564919388
3.81212121212121 0.0101675217324892
3.86161616161616 0.0101633838666166
3.91111111111111 0.0101593396729386
3.96060606060606 0.0101553860674873
4.01010101010101 0.0101515200972798
4.05959595959596 0.0101477389329229
4.10909090909091 0.0101440398629618
4.15858585858586 0.0101404202878127
4.20808080808081 0.0101368777142246
4.25757575757576 0.0101334097499218
4.30707070707071 0.0101300140987903
4.35656565656566 0.0101266883757148
4.40606060606061 0.0101234307308862
4.45555555555556 0.0101202391324113
4.50505050505051 0.0101171115349348
4.55454545454546 0.0101140460580237
4.6040404040404 0.0101110408930611
4.65353535353535 0.0101080943002908
4.7030303030303 0.0101052048982151
4.75252525252525 0.010102370490503
4.8020202020202 0.0100995898143312
4.85151515151515 0.0100968613736712
4.9010101010101 0.0100941837259733
4.95050505050505 0.0100915554812274
5 0.0100889788969916
};
\end{axis}

\end{tikzpicture}

%% file: Figures/exp_calibration_p=0.75_laplace.tex
\begin{tikzpicture}
\tikzstyle{every node}=[font=\tiny]

\definecolor{color0}{rgb}{1,0.647058823529412,0}
\definecolor{color1}{rgb}{0,1,1}

\begin{axis}[
height=\figheight,
legend cell align={left},
legend style={fill opacity=0.8, draw opacity=1, text opacity=1, draw=white!80!black},
tick align=outside,
tick pos=left,
width=\figwidth,
x grid style={white!69.0196078431373!black},
xlabel={1 / \(\displaystyle \alpha\)},
xmajorgrids,
xmin=-0.145, xmax=5.245,
xtick style={color=black},
y grid style={white!69.0196078431373!black},
ylabel={\(\displaystyle \Delta_{0.75}\)},
ymajorgrids,
ymin=-0.045098495218578, ymax=0.260751796338073,
ytick style={color=black}
]
\addplot [semithick, red, dash pattern=on 1pt off 3pt on 3pt off 3pt]
table {%
0.1 -0.00880749637945299
0.14949494949495 -0.00879658864487032
0.198989898989899 -0.00899109133838794
0.248484848484848 -0.00934514470753778
0.297979797979798 -0.00981830126216365
0.347474747474747 -0.0103762161877032
0.396969696969697 -0.010990928553531
0.446464646464647 -0.0116402763834296
0.495959595959596 -0.0123072121265411
0.545454545454546 -0.0129789829405206
0.594949494949495 -0.0136460688166542
0.644444444444444 -0.0143019624084051
0.693939393939394 -0.0149418844798617
0.743434343434343 -0.0155627085509575
0.792929292929293 -0.0161624166535467
0.842424242424242 -0.0167399559985332
0.891919191919192 -0.0172948338297904
0.941414141414142 -0.0178270315032902
0.990909090909091 -0.0183368768843378
1.04040404040404 -0.0188248391587661
1.08989898989899 -0.0192917073721787
1.13939393939394 -0.0197383482241169
1.18888888888889 -0.0201655003845367
1.23838383838384 -0.0205740769069231
1.28787878787879 -0.0209649769300977
1.33737373737374 -0.0213391129409315
1.38686868686869 -0.0216972952085533
1.43636363636364 -0.0220403654734479
1.48585858585859 -0.0223691164497606
1.53535353535354 -0.0226843302318612
1.58484848484849 -0.0229866954515052
1.63434343434343 -0.0232769083007678
1.68383838383838 -0.0235556050455161
1.73333333333333 -0.0238234091477582
1.78282828282828 -0.0240808499504065
1.83232323232323 -0.024328568302711
1.88181818181818 -0.0245669638393798
1.93131313131313 -0.0247965757864653
1.98080808080808 -0.0250178190688186
2.03030303030303 -0.0252310790559512
2.07979797979798 -0.025436886887195
2.12929292929293 -0.0256354586482664
2.17878787878788 -0.0258272193565031
2.22828282828283 -0.0260125139018781
2.27777777777778 -0.0261916268784244
2.32727272727273 -0.0263648831732832
2.37676767676768 -0.026532473326452
2.42626262626263 -0.0266947344111812
2.47575757575758 -0.0268518828509525
2.52525252525253 -0.0270041666291054
2.57474747474748 -0.0271517310050695
2.62424242424242 -0.0272948824393597
2.67373737373737 -0.0274337566175159
2.72323232323232 -0.0275685559663609
2.77272727272727 -0.0276994275683647
2.82222222222222 -0.0278265917771109
2.87171717171717 -0.0279501248762172
2.92121212121212 -0.0280701899375587
2.97070707070707 -0.0281870033731836
3.02020202020202 -0.0283005916196687
3.06969696969697 -0.0284111494200329
3.11919191919192 -0.0285187962140401
3.16868686868687 -0.0286235926536064
3.21818181818182 -0.0287256793946852
3.26767676767677 -0.0288251447515918
3.31717171717172 -0.0289221387489667
3.36666666666667 -0.0290167098194966
3.41616161616162 -0.0291089437977927
3.46565656565657 -0.0291989494790901
3.51515151515152 -0.0292867270973316
3.56464646464647 -0.0293724375283235
3.61414141414142 -0.0294561469282898
3.66363636363636 -0.029537889937916
3.71313131313131 -0.0296177644936083
3.76262626262626 -0.0296958460662619
3.81212121212121 -0.0297720972379204
3.86161616161616 -0.0298466861429698
3.91111111111111 -0.0299196437630274
3.96060606060606 -0.0299910168056908
4.01010101010101 -0.0300608183346012
4.05959595959596 -0.0301291259083377
4.10909090909091 -0.0301960222355775
4.15858585858586 -0.0302614592873309
4.20808080808081 -0.0303255754401337
4.25757575757576 -0.0303883739317875
4.30707070707071 -0.0304499004470192
4.35656565656566 -0.0305101517505179
4.40606060606061 -0.0305692904605677
4.45555555555556 -0.0306271909702481
4.50505050505051 -0.0306839739292241
4.55454545454546 -0.0307396878579856
4.6040404040404 -0.0307942805600891
4.65353535353535 -0.0308478651312583
4.7030303030303 -0.0309004946760355
4.75252525252525 -0.0309520508636034
4.8020202020202 -0.0310026956592667
4.85151515151515 -0.0310523944760107
4.9010101010101 -0.0311012046904854
4.95050505050505 -0.031149157196422
5 -0.0311962092387302
};
\addlegendentry{$\hat{f}_{\lap}(\lambda_{\rm loss})$}
\addplot [semithick, blue, dash pattern=on 1pt off 3pt on 3pt off 3pt]
table {%
0.1 0.157137756351061
0.14949494949495 0.157008449962112
0.198989898989899 0.156556949400714
0.248484848484848 0.155753462180326
0.297979797979798 0.154588963754304
0.347474747474747 0.153081193834834
0.396969696969697 0.151274402571477
0.446464646464647 0.149234165168349
0.495959595959596 0.147028905719947
0.545454545454546 0.144722375737803
0.594949494949495 0.14236696975578
0.644444444444444 0.140002252850606
0.693939393939394 0.13765617716001
0.743434343434343 0.135347370469751
0.792929292929293 0.133087539443274
0.842424242424242 0.130883457543292
0.891919191919192 0.12873852485993
0.941414141414142 0.126653942303679
0.990909090909091 0.124629439788715
1.04040404040404 0.122663857742795
1.08989898989899 0.120755479494648
1.13939393939394 0.118902285290997
1.18888888888889 0.117102142384185
1.23838383838384 0.115352833723249
1.28787878787879 0.113652173648061
1.33737373737374 0.111998031330115
1.38686868686869 0.110388371933939
1.43636363636364 0.108821243620954
1.48585858585859 0.107294765990681
1.53535353535354 0.105807863053115
1.58484848484849 0.104357485965412
1.63434343434343 0.102942830938418
1.68383838383838 0.101562396818911
1.73333333333333 0.100214811932942
1.78282828282828 0.0988987595728513
1.83232323232323 0.0976130033473203
1.88181818181818 0.0963563747502633
1.93131313131313 0.0951277780078222
1.98080808080808 0.093926174305258
2.03030303030303 0.0927505974961904
2.07979797979798 0.0916000745321706
2.12929292929293 0.090473743197573
2.17878787878788 0.0893707484798214
2.22828282828283 0.0882903488235881
2.27777777777778 0.0872317204574036
2.32727272727273 0.0861941829388705
2.37676767676768 0.0851770280889472
2.42626262626263 0.0841796341622881
2.47575757575758 0.083201364270829
2.52525252525253 0.0822415952151276
2.57474747474748 0.0812998399422165
2.62424242424242 0.0803754513227707
2.67373737373737 0.0794679623710295
2.72323232323232 0.0785769099560665
2.77272727272727 0.0777017573385166
2.82222222222222 0.0768420761891112
2.87171717171717 0.0759975155524607
2.92121212121212 0.0751675114292405
2.97070707070707 0.0743517168543296
3.02020202020202 0.0735497581563215
3.06969696969697 0.0727612723465474
3.11919191919192 0.0719858769302613
3.16868686868687 0.0712232302195385
3.21818181818182 0.0704733473748508
3.26767676767677 0.0697352014401971
3.31717171717172 0.0690088976509622
3.36666666666667 0.0682940772045411
3.41616161616162 0.0675904659303944
3.46565656565657 0.0668977934870939
3.51515151515152 0.0662158238150019
3.56464646464647 0.0655442327886824
3.61414141414142 0.0648827947465612
3.66363636363636 0.0642313496708624
3.71313131313131 0.0635896363752875
3.76262626262626 0.0629573068862356
3.81212121212121 0.0623342878709696
3.86161616161616 0.0617202429424428
3.91111111111111 0.0611151063415077
3.96060606060606 0.0605186236524083
4.01010101010101 0.059930539743676
4.05959595959596 0.0593507914715135
4.10909090909091 0.0587790858457167
4.15858585858586 0.0582153857698628
4.20808080808081 0.0576593174723868
4.25757575757576 0.0571108456584152
4.30707070707071 0.0565697421985263
4.35656565656566 0.0560359783333112
4.40606060606061 0.0555093537657517
4.45555555555556 0.0549896782604995
4.50505050505051 0.0544768263680709
4.55454545454546 0.0539706654642662
4.6040404040404 0.053471036982404
4.65353535353535 0.0529778713427256
4.7030303030303 0.0524909659329058
4.75252525252525 0.0520102246895487
4.8020202020202 0.0515356549163435
4.85151515151515 0.051066907259271
4.9010101010101 0.0506040501582863
4.95050505050505 0.0501468481960452
5 0.0496952744710113
};
\addlegendentry{$\hat{f}_{\lap}(\lambda_{\rm error})$}
\addplot [semithick, black, dash pattern=on 1pt off 3pt on 3pt off 3pt]
table {%
0.1 0.193397691943832
0.14949494949495 0.196804426194566
0.198989898989899 0.200487440330861
0.248484848484848 0.20455515897315
0.297979797979798 0.209143974561676
0.347474747474747 0.214460240922087
0.396969696969697 0.220901708625193
0.446464646464647 0.229316624870454
0.495959595959596 0.233513124516726
0.545454545454546 0.224675179632384
0.594949494949495 0.213963345242717
0.644444444444444 0.203300398028353
0.693939393939394 0.193101796690882
0.743434343434343 0.183510950384584
0.792929292929293 0.174568836213164
0.842424242424242 0.166268458501146
0.891919191919192 0.158579614264485
0.941414141414142 0.151461584304448
0.990909090909091 0.144870003950941
1.04040404040404 0.138760570218878
1.08989898989899 0.133090967737946
1.13939393939394 0.127821793636248
1.18888888888889 0.122916913358307
1.23838383838384 0.118343498916038
1.28787878787879 0.11407189637612
1.33737373737374 0.110075408469504
1.38686868686869 0.106330041984694
1.43636363636364 0.102814248472254
1.48585858585859 0.0995086734742437
1.53535353535354 0.0963959225008674
1.58484848484849 0.0934603470933151
1.63434343434343 0.0906878519788795
1.68383838383838 0.0880657227413216
1.73333333333333 0.08558247270764
1.78282828282828 0.083227707444425
1.83232323232323 0.0809920050881799
1.88181818181818 0.0788668745997442
1.93131313131313 0.0768443792497968
1.98080808080808 0.0749175222656875
2.03030303030303 0.0730798358110633
2.07979797979798 0.071325405452504
2.12929292929293 0.0696488160756891
2.17878787878788 0.0680451857009552
2.22828282828283 0.0665098027176799
2.27777777777778 0.0650385346889153
2.32727272727273 0.0636275154707552
2.37676767676768 0.0622731780014832
2.42626262626263 0.0609722265171325
2.47575757575758 0.0597216117316761
2.52525252525253 0.058518508637465
2.57474747474748 0.0573602924925825
2.62424242424242 0.0562445415365915
2.67373737373737 0.055168979784231
2.72323232323232 0.0541315037891301
2.77272727272727 0.0531301489281932
2.82222222222222 0.0521630818330859
2.87171717171717 0.0512285897299256
2.92121212121212 0.0503250708044012
2.97070707070707 0.0494510254639036
3.02020202020202 0.0486050484021159
3.06969696969697 0.0477858213824637
3.11919191919192 0.046992106667019
3.16868686868687 0.0462227410247319
3.21818181818182 0.0454766302616932
3.26767676767677 0.044752744221162
3.31717171717172 0.0440501122044396
3.36666666666667 0.0433678187917091
3.41616161616162 0.0427049999754741
3.46565656565657 0.042060839642389
3.51515151515152 0.0414345671848229
3.56464646464647 0.040825450929565
3.61414141414142 0.0402328009620104
3.66363636363636 0.0396559627289058
3.71313131313131 0.0390943156932663
3.76262626262626 0.0385472711677087
3.81212121212121 0.0380142703096591
3.86161616161616 0.0374947822657242
3.91111111111111 0.036988302451382
3.96060606060606 0.0364943509554471
4.01010101010101 0.0360124710589392
4.05959595959596 0.0355422278583534
4.10909090909091 0.0350832069855685
4.15858585858586 0.034635013414765
4.20808080808081 0.0341972703584834
4.25757575757576 0.0337696182236098
4.30707070707071 0.03335171365206
4.35656565656566 0.0329432286176865
4.40606060606061 0.0325438495838801
4.45555555555556 0.0321532776766358
4.50505050505051 0.0317712241652567
4.55454545454546 0.0313974153515733
4.6040404040404 0.0310315882679039
4.65353535353535 0.0306734909678661
4.7030303030303 0.0303228819570409
4.75252525252525 0.0299795296589053
4.8020202020202 0.0296432119126081
4.85151515151515 0.0293137155012251
4.9010101010101 0.0289908357075785
4.95050505050505 0.028674375896581
5 0.0283641471215949
};
\addlegendentry{$\hat{f}_{\lap}(\lambda = 1e-4)$}
\addplot [semithick, black]
table {%
0.1 0.198840805029523
0.14949494949495 0.204697478903802
0.198989898989899 0.210602672865687
0.248484848484848 0.216596729356316
0.297979797979798 0.222720046718845
0.347474747474747 0.229029244091433
0.396969696969697 0.235637544649401
0.446464646464647 0.24268754443515
0.495959595959596 0.246849510358225
0.545454545454546 0.246602475168014
0.594949494949495 0.245985578488217
0.644444444444444 0.245386146506138
0.693939393939394 0.244840307056351
0.743434343434343 0.244349517527843
0.792929292929293 0.243908914177583
0.842424242424242 0.243512600032378
0.891919191919192 0.24315497998123
0.941414141414142 0.242831093281393
0.990909090909091 0.242536652720573
1.04040404040404 0.242267991869407
1.08989898989899 0.242021982512184
1.13939393939394 0.241795954482495
1.18888888888889 0.241587623324375
1.23838383838384 0.241395028123
1.28787878787879 0.241216479284675
1.33737373737374 0.241050515157435
1.38686868686869 0.240895866144511
1.43636363636364 0.240751425051271
1.48585858585859 0.240616222524378
1.53535353535354 0.240489406674421
1.58484848484849 0.240370226105553
1.63434343434343 0.240258015737396
1.68383838383838 0.240152184913104
1.73333333333333 0.240052207386401
1.78282828282828 0.239957612859967
1.83232323232323 0.239867979806286
1.88181818181818 0.239782934055485
1.93131313131313 0.239702122666664
1.98080808080808 0.23962524400546
2.03030303030303 0.239552018763161
2.07979797979798 0.239482193275395
2.12929292929293 0.239415536764531
2.17878787878788 0.239351844902887
2.22828282828283 0.239290914577414
2.27777777777778 0.23923257443593
2.32727272727273 0.239176663071541
2.37676767676768 0.239123032183593
2.42626262626263 0.239071545277783
2.47575757575758 0.239022076516938
2.52525252525253 0.238974509703855
2.57474747474748 0.238928736876465
2.62424242424242 0.238884660007766
2.67373737373737 0.23884218527987
2.72323232323232 0.238801227456278
2.77272727272727 0.238761706795178
2.82222222222222 0.238723549041576
2.87171717171717 0.238686684962719
2.92121212121212 0.238651049931205
2.97070707070707 0.238616583548851
3.02020202020202 0.238583229306782
3.06969696969697 0.238550934277732
3.11919191919192 0.238519648837058
3.16868686868687 0.238489326409388
3.21818181818182 0.238459923238222
3.26767676767677 0.238431398176071
3.31717171717172 0.238403712492858
3.36666666666667 0.238376829701911
3.41616161616162 0.238350715398842
3.46565656565657 0.23832533711598
3.51515151515152 0.238300664233186
3.56464646464647 0.238276667666533
3.61414141414142 0.238253320047867
3.66363636363636 0.238230595418052
3.71313131313131 0.238208469183441
3.76262626262626 0.238186918027318
3.81212121212121 0.238165919828142
3.86161616161616 0.238145453583982
3.91111111111111 0.238125499342621
3.96060606060606 0.238106038136813
4.01010101010101 0.238087051924283
4.05959595959596 0.238068523532054
4.10909090909091 0.238050436604736
4.15858585858586 0.23803277555641
4.20808080808081 0.238015525526179
4.25757575757576 0.237998672336133
4.30707070707071 0.237982202452831
4.35656565656566 0.237966102951044
4.40606060606061 0.237950361479987
4.45555555555556 0.23793496628224
4.50505050505051 0.237919905965548
4.55454545454546 0.237905169768367
4.6040404040404 0.237890747341993
4.65353535353535 0.237876628773497
4.7030303030303 0.237862804563034
4.75252525252525 0.237849265602555
4.8020202020202 0.237836003155822
4.85151515151515 0.237823008839634
4.9010101010101 0.237810274606178
4.95050505050505 0.237797792726427
5 0.237785555774516
};
\addplot [semithick, blue]
table {%
0.1 0.162594536037619
0.14949494949495 0.165037289064945
0.198989898989899 0.167045418613861
0.248484848484848 0.168567570891998
0.297979797979798 0.169573642391271
0.347474747474747 0.170063778055665
0.396969696969697 0.170071094293063
0.446464646464647 0.169657869195398
0.495959595959596 0.168897570781431
0.545454545454546 0.167865208258165
0.594949494949495 0.166628267606979
0.644444444444444 0.165243002569426
0.693939393939394 0.163754039033983
0.743434343434343 0.162195685457139
0.792929292929293 0.160593845221028
0.842424242424242 0.158967829624713
0.891919191919192 0.15733193676615
0.941414141414142 0.155696749980695
0.990909090909091 0.154070024879278
1.04040404040404 0.152457449535045
1.08989898989899 0.150863141686615
1.13939393939394 0.149290055677527
1.18888888888889 0.147740308473836
1.23838383838384 0.146215315466285
1.28787878787879 0.144716007563677
1.33737373737374 0.14324293431304
1.38686868686869 0.141796374716679
1.43636363636364 0.140376378597019
1.48585858585859 0.138982800837392
1.53535353535354 0.137616051692709
1.58484848484849 0.136274471867848
1.63434343434343 0.134958378572732
1.68383838383838 0.133667286947864
1.73333333333333 0.132400725153225
1.78282828282828 0.131158169963412
1.83232323232323 0.129939088934972
1.88181818181818 0.128742939167238
1.93131313131313 0.127569183134203
1.98080808080808 0.12641728051428
2.03030303030303 0.125286713089074
2.07979797979798 0.124176904717847
2.12929292929293 0.123087355469422
2.17878787878788 0.122017533695506
2.22828282828283 0.120966996133008
2.27777777777778 0.119935178628815
2.32727272727273 0.118921645064967
2.37676767676768 0.117925904021604
2.42626262626263 0.116947535595841
2.47575757575758 0.115986083128784
2.52525252525253 0.115041086593307
2.57474747474748 0.114112217477219
2.62424242424242 0.113198958294471
2.67373737373737 0.112300973839237
2.72323232323232 0.11141792074053
2.77272727272727 0.110549364443605
2.82222222222222 0.109694977955029
2.87171717171717 0.108854508941944
2.92121212121212 0.108027465177793
2.97070707070707 0.107213582297459
3.02020202020202 0.1064125597291
3.06969696969697 0.105624101881804
3.11919191919192 0.104847885914749
3.16868686868687 0.104083627669038
3.21818181818182 0.103331428466899
3.26767676767677 0.102590247520341
3.31717171717172 0.101860270461839
3.36666666666667 0.101141176107502
3.41616161616162 0.100432731261731
3.46565656565657 0.0997347033358215
3.51515151515152 0.099046893613254
3.56464646464647 0.0983690048336304
3.61414141414142 0.0977008436352714
3.66363636363636 0.0970422851315939
3.71313131313131 0.0963930922741533
3.76262626262626 0.0957529316447892
3.81212121212121 0.0951217652952697
3.86161616161616 0.0944992681570468
3.91111111111111 0.0938854065079102
3.96060606060606 0.093279940395996
4.01010101010101 0.0926826274226107
4.05959595959596 0.0920934316007861
4.10909090909091 0.0915120664409551
4.15858585858586 0.0909385212582634
4.20808080808081 0.09037241934553
4.25757575757576 0.0898137493924253
4.30707070707071 0.0892622900341929
4.35656565656566 0.0887180347251558
4.40606060606061 0.0881807900757978
4.45555555555556 0.0876503726063557
4.50505050505051 0.0871266680581823
4.55454545454546 0.0866095535087366
4.6040404040404 0.0860988770980736
4.65353535353535 0.0855945825219353
4.7030303030303 0.085096468664975
4.75252525252525 0.0846044490625696
4.8020202020202 0.0841185484768914
4.85151515151515 0.0836384048574688
4.9010101010101 0.083164107884965
4.95050505050505 0.0826954176649325
5 0.0822323194943213
};
\addplot [semithick, red]
table {%
0.1 0.00240832766970978
0.14949494949495 0.0041827673139615
0.198989898989899 0.00558427398060912
0.248484848484848 0.00666937565497572
0.297979797979798 0.00749090874492964
0.347474747474747 0.00809664330632498
0.396969696969697 0.00852830633636081
0.446464646464647 0.00882122443543565
0.495959595959596 0.00900478340614874
0.545454545454546 0.00910250384253064
0.594949494949495 0.00913349349743775
0.644444444444444 0.00911279487813088
0.693939393939394 0.00905233413605366
0.743434343434343 0.008961521497584
0.792929292929293 0.00884782628691383
0.842424242424242 0.00871710634991352
0.891919191919192 0.00857402952531805
0.941414141414142 0.00842230020625689
0.990909090909091 0.0082647779306072
1.04040404040404 0.0081037352073039
1.08989898989899 0.00794101003960701
1.13939393939394 0.00777791346533219
1.18888888888889 0.00761557267363777
1.23838383838384 0.00745478351885687
1.28787878787879 0.00729624316672728
1.33737373737374 0.0071403105282557
1.38686868686869 0.006987480694431
1.43636363636364 0.00683786462658031
1.48585858585859 0.00669170396088592
1.53535353535354 0.00654911414186121
1.58484848484849 0.00641013426541581
1.63434343434343 0.00627478405263171
1.68383838383838 0.00614304217570338
1.73333333333333 0.0060148994781809
1.78282828282828 0.00589028335797148
1.83232323232323 0.00576914267562956
1.88181818181818 0.00565140037273915
1.93131313131313 0.00553698637613875
1.98080808080808 0.00542578115881531
2.03030303030303 0.00531773029454707
2.07979797979798 0.00521268106999484
2.12929292929293 0.00511061471215912
2.17878787878788 0.00501141716744924
2.22828282828283 0.00491492968227514
2.27777777777778 0.0048211088779142
2.32727272727273 0.00472990315203881
2.37676767676768 0.00464117723290969
2.42626262626263 0.00455485312605741
2.47575757575758 0.00447089956184377
2.52525252525253 0.00438913793037876
2.57474747474748 0.0043095091267028
2.62424242424242 0.00423200867717977
2.67373737373737 0.00415654054157855
2.72323232323232 0.00408299921150745
2.77272727272727 0.0040113524398383
2.82222222222222 0.00394149860256054
2.87171717171717 0.00387339332360326
2.92121212121212 0.00380703412616323
2.97070707070707 0.00374226792911581
3.02020202020202 0.00367907463181116
3.06969696969697 0.00361743563286965
3.11919191919192 0.00355722625351518
3.16868686868687 0.00349845869914611
3.21818181818182 0.00344109229865774
3.26767676767677 0.00338504885586555
3.31717171717172 0.00333029770000703
3.36666666666667 0.00327675808513395
3.41616161616162 0.00322446410107569
3.46565656565657 0.00317331856149949
3.51515151515152 0.00312333161513079
3.56464646464647 0.00307442937023084
3.61414141414142 0.00302662558814848
3.66363636363636 0.00297977333232302
3.71313131313131 0.00293395858740764
3.76262626262626 0.00288910261382591
3.81212121212121 0.00284518032630798
3.86161616161616 0.00280222183172973
3.91111111111111 0.00276006857934963
3.96060606060606 0.00271883852299526
4.01010101010101 0.0026783528744504
4.05959595959596 0.00263872486267769
4.10909090909091 0.00259995103809496
4.15858585858586 0.0025618393586041
4.20808080808081 0.00252448101549352
4.25757575757576 0.00248789884935774
4.30707070707071 0.00245196493729782
4.35656565656566 0.00241670009789363
4.40606060606061 0.00238213636998108
4.45555555555556 0.00234820825118365
4.50505050505051 0.00231489905536786
4.55454545454546 0.0022822050157052
4.6040404040404 0.00225008855620423
4.65353535353535 0.00221854632735174
4.7030303030303 0.00218758097440563
4.75252525252525 0.00215717171839436
4.8020202020202 0.002127304916415
4.85151515151515 0.00209790390921805
4.9010101010101 0.00206907068044171
4.95050505050505 0.00204068644632882
5 0.00201279792171949
};
\addplot [draw=blue, fill=blue, forget plot, mark=*, only marks, mark size=1pt]
table{%
x  y
0.594949494949495 0.166031261190207
0.842424242424242 0.159338480342729
1.08989898989899 0.152107137029209
1.33737373737374 0.144459853386214
1.58484848484849 0.138055766005952
1.83232323232323 0.131075890488108
2.07979797979798 0.123533589285115
2.32727272727273 0.119802994451351
2.57474747474748 0.11523588968557
2.82222222222222 0.109880475434346
3.06969696969697 0.107175268049691
3.31717171717172 0.10254113376579
3.56464646464647 0.0986158378226369
3.81212121212121 0.096497630395501
4.05959595959596 0.0927748474989166
4.30707070707071 0.0904164455439479
};
\addplot [draw=red, fill=red, forget plot, mark=*, only marks, mark size=1pt]
table{%
x  y
0.594949494949495 0.0113148880285029
0.842424242424242 0.00980566759698742
1.08989898989899 0.00876221858253101
1.33737373737374 0.0107474240522567
1.58484848484849 0.00992145880158002
1.83232323232323 0.0025105818128216
2.07979797979798 0.00495293808500408
2.32727272727273 0.00657929806006385
2.57474747474748 0.00384041153067161
2.82222222222222 0.00396794902644682
3.06969696969697 0.00569852144230965
3.31717171717172 0.00345254428982944
3.56464646464647 0.00336130175551519
3.81212121212121 0.00422129283685324
4.05959595959596 0.00416585081752607
4.30707070707071 0.00306862360012861
};
\addplot [draw=color1, fill=color1, forget plot, mark=*, only marks, mark size=1pt]
table{%
x  y
0.594949494949495 0.163321789881347
0.842424242424242 0.151885384353769
1.08989898989899 0.146781664475271
1.33737373737374 0.134404241466594
1.58484848484849 0.129076515717045
1.83232323232323 0.12291383397579
2.07979797979798 0.117617885290558
2.32727272727273 0.114568387702064
2.57474747474748 0.109457344069415
2.82222222222222 0.101041116454195
3.06969696969697 0.102360128582519
3.31717171717172 0.0921739928156016
3.56464646464647 0.0949744250710625
3.81212121212121 0.089240639640712
4.05959595959596 0.0844813463843493
4.30707070707071 0.0809827915237162
};
\addplot [draw=color0, fill=color0, forget plot, mark=*, only marks, mark size=1pt]
table{%
x  y
0.594949494949495 0.00416408595941264
0.842424242424242 -0.00359957436758485
1.08989898989899 -0.00328392205262673
1.33737373737374 0.00164446887985759
1.58484848484849 -0.00177633013844802
1.83232323232323 0.00125172560196796
2.07979797979798 0.00242129642214561
2.32727272727273 0.00114273716867985
2.57474747474748 0.00622658310266
2.82222222222222 0.00582673056105742
3.06969696969697 0.000690445688329211
3.31717171717172 0.00475137215263466
3.56464646464647 2.46256934395284e-05
3.81212121212121 0.00588353822229393
4.05959595959596 0.00384919957827923
4.30707070707071 0.00568674782008216
};
\addplot [draw=color0, fill=color0, forget plot, mark=*, only marks, mark size=1pt]
table{%
x  y
0.594949494949495 0.00416408595941264
0.842424242424242 -0.00359957436758485
1.08989898989899 -0.00328392205262673
1.33737373737374 0.00164446887985759
1.58484848484849 -0.00177633013844802
1.83232323232323 0.00125172560196796
2.07979797979798 0.00242129642214561
2.32727272727273 0.00114273716867985
2.57474747474748 0.00622658310266
2.82222222222222 0.00582673056105742
3.06969696969697 0.000690445688329211
3.31717171717172 0.00475137215263466
3.56464646464647 2.46256934395284e-05
3.81212121212121 0.00588353822229393
4.05959595959596 0.00384919957827923
4.30707070707071 0.00568674782008216
};
\addplot [draw=black, fill=black, forget plot, mark=*, only marks, mark size=1pt]
table{%
x  y
0.594949494949495 0.246132445791867
0.842424242424242 0.243602225521843
1.08989898989899 0.242170753269526
1.33737373737374 0.241424139473292
1.58484848484849 0.240353729885288
1.83232323232323 0.239836652035998
2.07979797979798 0.239244002529685
2.32727272727273 0.239103972596387
2.57474747474748 0.238781766213196
2.82222222222222 0.238873879670612
3.06969696969697 0.238463998902553
3.31717171717172 0.238436570991951
3.56464646464647 0.238419196134758
3.81212121212121 0.238258274445229
4.05959595959596 0.238142202754361
4.30707070707071 0.238016089771082
};
\addplot [draw=red, fill=red, mark=*, only marks, mark size=1pt]
table{%
x  y
0.594949494949495 -0.0108037627229278
0.842424242424242 -0.0168701665749232
1.08989898989899 -0.0172287771035917
1.33737373737374 -0.020236620183972
1.58484848484849 -0.019082162292591
1.83232323232323 -0.0240427342642506
2.07979797979798 -0.0233207458857357
2.32727272727273 -0.0247474383064219
2.57474747474748 -0.0252544758028908
2.82222222222222 -0.0246005685138463
3.06969696969697 -0.0257997144180683
3.31717171717172 -0.0263206682519683
3.56464646464647 -0.0306936174737298
3.81212121212121 -0.0287983604073552
4.05959595959596 -0.027327370751333
4.30707070707071 -0.0283378331184817
};
\addplot [draw=blue, fill=blue, mark=*, only marks, mark size=1pt]
table{%
x  y
0.594949494949495 0.141706892253187
0.842424242424242 0.131289432529014
1.08989898989899 0.122472846372213
1.33737373737374 0.113475231435162
1.58484848484849 0.10664324330884
1.83232323232323 0.0990360078803166
2.07979797979798 0.0907720286273456
2.32727272727273 0.0873295203482386
2.57474747474748 0.0826172375297698
2.82222222222222 0.0770316348137063
3.06969696969697 0.0746919467196149
3.31717171717172 0.069773201366625
3.56464646464647 0.0658333744871783
3.81212121212121 0.0639547970463675
4.05959595959596 0.0601782171661175
4.30707070707071 0.0579436852361823
};
\addplot [draw=black, fill=black, mark=*, only marks, mark size=1pt]
table{%
x  y
0.594949494949495 0.234599947506524
0.842424242424242 0.168942707418913
1.08989898989899 0.133103667325056
1.33737373737374 0.114466339257562
1.58484848484849 0.0944864896679161
1.83232323232323 0.0812448480937005
2.07979797979798 0.0681024125760908
2.32727272727273 0.0613909478795109
2.57474747474748 0.0527369572643311
2.82222222222222 0.0532581384266565
3.06969696969697 0.0447983537408829
3.31717171717172 0.0446029112093695
3.56464646464647 0.0451213643721775
3.81212121212121 0.0397442602449585
4.05959595959596 0.0416542191751762
4.30707070707071 0.0329822505663605
};
\end{axis}

\end{tikzpicture}

%% file: Figures/temperature_scaling/exp_calibration_temp_scaling_p=0.75_n_over_p=2.0.tex
\begin{tikzpicture}
\tikzstyle{every node}=[font=\tiny]

\definecolor{color0}{rgb}{1,0.647058823529412,0}

\begin{axis}[
height=\figheight,
legend cell align={left},
legend style={fill opacity=0.8, draw opacity=1, text opacity=1, draw=white!80!black},
tick align=outside,
tick pos=left,
width=\figwidth,
x grid style={white!69.0196078431373!black},
xlabel={\(\displaystyle p/n\)},
xmajorgrids,
xmin=-0.145, xmax=5.245,
xtick style={color=black},
y grid style={white!69.0196078431373!black},
ylabel={\(\displaystyle \Delta\)0.75},
ymajorgrids,
ymin=-0.00409088224109709, ymax=0.00671789097617224,
ytick style={color=black}
]
\addplot [draw=black, fill=black, forget plot, mark=*, only marks, mark size=1pt]
table{%
x  y
0.594949494949495 -0.00213733239779079
0.842424242424242 -0.00199407930720419
1.08989898989899 -0.00192730719011769
1.33737373737374 -0.00190222320816735
1.58484848484849 -0.00181179883144922
1.83232323232323 -0.00179063096236287
2.07979797979798 -0.00175022623129539
2.32727272727273 -0.0017577617862321
2.57474747474748 -0.00172178887760099
2.82222222222222 -0.00174541184476962
3.06969696969697 -0.00171450999689593
3.31717171717172 -0.00173080747668153
3.56464646464647 -0.00172153122920271
3.81212121212121 -0.00171290807375823
4.05959595959596 -0.00170743059916889
4.30707070707071 -0.00170026455051842
};
\addplot [draw=blue, fill=blue, forget plot, mark=*, only marks, mark size=1pt]
table{%
x  y
0.594949494949495 -0.0017740519521775
0.842424242424242 -0.00167973984346381
1.08989898989899 -0.00162772925435761
1.33737373737374 -0.00158003790917038
1.58484848484849 -0.00158936167774815
1.83232323232323 -0.00153594954287428
2.07979797979798 -0.0014867617567208
2.32727272727273 -0.00152050533492065
2.57474747474748 -0.00148446436526251
2.82222222222222 -0.00146585461632387
3.06969696969697 -0.0014731435239953
3.31717171717172 -0.00145835746411338
3.56464646464647 -0.00145421022305414
3.81212121212121 -0.00145789330651724
4.05959595959596 -0.00143903686634606
4.30707070707071 -0.00142504757300543
};
\addplot [draw=red, fill=red, forget plot, mark=*, only marks, mark size=1pt]
table{%
x  y
0.594949494949495 -0.00188990215956675
0.842424242424242 -0.00176026032614318
1.08989898989899 -0.0016780280112535
1.33737373737374 -0.00163671699943524
1.58484848484849 -0.0016046518317786
1.83232323232323 -0.00157856129778322
2.07979797979798 -0.00150969016656244
2.32727272727273 -0.0015239700902886
2.57474747474748 -0.00151545487988303
2.82222222222222 -0.00148683682413819
3.06969696969697 -0.00151015747439853
3.31717171717172 -0.00151955742037746
3.56464646464647 -0.00143548509099767
3.81212121212121 -0.00144779249112426
4.05959595959596 -0.00150000953926255
4.30707070707071 -0.00143571922430208
};
\addplot [draw=color0, fill=color0, forget plot, mark=*, only marks, mark size=1pt]
table{%
x  y
0.594949494949495 0.00416408595941264
0.842424242424242 -0.00359957436758485
1.08989898989899 -0.00328392205262673
1.33737373737374 0.00164446887985759
1.58484848484849 -0.00177633013844802
1.83232323232323 0.00125172560196796
2.07979797979798 0.00242129642214561
2.32727272727273 0.00114273716867985
2.57474747474748 0.00622658310266
2.82222222222222 0.00582673056105742
3.06969696969697 0.000690445688329211
3.31717171717172 0.00475137215263466
3.56464646464647 2.46256934395284e-05
3.81212121212121 0.00588353822229393
4.05959595959596 0.00384919957827923
4.30707070707071 0.00568674782008216
};
\addplot [semithick, black, dashdotdotted]
table {%
0.1 -0.00234088000605759
0.14949494949495 -0.00225932426279563
0.198989898989899 -0.00219264388948759
0.248484848484848 -0.00214224668620133
0.297979797979798 -0.00210996045414413
0.347474747474747 -0.00209831658736914
0.396969696969697 -0.00211064639327141
0.446464646464647 -0.00215217235475451
0.495959595959596 -0.00217994490070306
0.545454545454546 -0.00214757246827602
0.594949494949495 -0.00211228140427955
0.644444444444444 -0.00208007638847196
0.693939393939394 -0.00205115571341197
0.743434343434343 -0.00202536421611765
0.792929292929293 -0.00200226183167385
0.842424242424242 -0.00198151763643806
0.891919191919192 -0.00196284941077796
0.941414141414142 -0.00194593402463306
0.990909090909091 -0.00193058416512004
1.04040404040404 -0.00191654810963249
1.08989898989899 -0.00190374067381704
1.13939393939394 -0.00189197337519698
1.18888888888889 -0.00188114527039074
1.23838383838384 -0.00187111806897988
1.28787878787879 -0.0018618316156287
1.33737373737374 -0.00185321722034892
1.38686868686869 -0.00184517262034745
1.43636363636364 -0.00183769662275868
1.48585858585859 -0.00183069409343517
1.53535353535354 -0.00182410561030311
1.58484848484849 -0.0018179340742378
1.63434343434343 -0.00181212117697749
1.68383838383838 -0.00180663698951911
1.73333333333333 -0.00180146450172269
1.78282828282828 -0.00179657764680718
1.83232323232323 -0.00179194004460992
1.88181818181818 -0.00178754247190105
1.93131313131313 -0.00178333570460143
1.98080808080808 -0.00177939269732863
2.03030303030303 -0.0017756123745396
2.07979797979798 -0.0017720038843243
2.12929292929293 -0.00176855811685261
2.17878787878788 -0.00176527040058905
2.22828282828283 -0.00176213923832025
2.27777777777778 -0.00175912314413496
2.32727272727273 -0.00175623974091332
2.37676767676768 -0.00175348114718588
2.42626262626263 -0.00175081982406744
2.47575757575758 -0.00174827026031565
2.52525252525253 -0.00174583514264159
2.57474747474748 -0.00174344515190783
2.62424242424242 -0.00174119227575908
2.67373737373737 -0.00173900611335631
2.72323232323232 -0.00173689827990264
2.77272727272727 -0.00173488139176614
2.82222222222222 -0.00173290333582199
2.87171717171717 -0.00173099841438906
2.92121212121212 -0.00172917203996292
2.97070707070707 -0.0017274019714526
3.02020202020202 -0.00172565575631567
3.06969696969697 -0.0017240285048139
3.11919191919192 -0.00172241669645601
3.16868686868687 -0.00172085750411977
3.21818181818182 -0.00171934556403486
3.26767676767677 -0.00171789551749124
3.31717171717172 -0.00171643079549666
3.36666666666667 -0.00171507372614454
3.41616161616162 -0.00171373147661802
3.46565656565657 -0.00171242705504293
3.51515151515152 -0.00171115878873684
3.56464646464647 -0.00170992513933144
3.61414141414142 -0.00170872479897732
3.66363636363636 -0.00170755639319409
3.71313131313131 -0.00170645211833409
3.76262626262626 -0.00170531039902266
3.81212121212121 -0.00170426411867464
3.86161616161616 -0.00170321159758335
3.91111111111111 -0.0017021854184226
3.96060606060606 -0.00170118453957357
4.01010101010101 -0.00170020987842445
4.05959595959596 -0.0016992589715521
4.10909090909091 -0.00169833088670401
4.15858585858586 -0.00169742477963342
4.20808080808081 -0.00169653995406238
4.25757575757576 -0.00169567559059436
4.30707070707071 -0.00169483104261769
4.35656565656566 -0.00169400562015376
4.40606060606061 -0.00169319866716844
4.45555555555556 -0.00169240958870165
4.50505050505051 -0.00169163779744641
4.55454545454546 -0.001690882710221
4.6040404040404 -0.00169010950105664
4.65353535353535 -0.00168942055412835
4.7030303030303 -0.00168871244501689
4.75252525252525 -0.00168801907844507
4.8020202020202 -0.00168733992892423
4.85151515151515 -0.00168664011858799
4.9010101010101 -0.00168602262638895
4.95050505050505 -0.00168534909125273
5 -0.00168475725914208
};
\addlegendentry{$\lambda = 0$ (TS)}
\addplot [semithick, blue, dashdotdotted]
table {%
0.1 -0.00232672815156709
0.14949494949495 -0.00223451549660636
0.198989898989899 -0.00216428262656132
0.248484848484848 -0.0020877545597493
0.297979797979798 -0.0020229079774291
0.347474747474747 -0.00197182804842189
0.396969696969697 -0.00192215705568211
0.446464646464647 -0.00187387239347869
0.495959595959596 -0.00183528558867962
0.545454545454546 -0.00180324745896066
0.594949494949495 -0.00178000295741798
0.644444444444444 -0.00173941140858025
0.693939393939394 -0.00172439231407151
0.743434343434343 -0.00169883673076421
0.792929292929293 -0.00168306296268605
0.842424242424242 -0.00166862914181376
0.891919191919192 -0.00165525742456074
0.941414141414142 -0.00164263598702463
0.990909090909091 -0.00163627379875675
1.04040404040404 -0.00161890842553614
1.08989898989899 -0.00160849225870519
1.13939393939394 -0.00159902303627679
1.18888888888889 -0.0015840391805918
1.23838383838384 -0.00158891754978785
1.28787878787879 -0.00157062451819678
1.33737373737374 -0.00157324743311449
1.38686868686869 -0.00155832430811231
1.43636363636364 -0.00155290979075817
1.48585858585859 -0.00156105817079766
1.53535353535354 -0.00155544021562848
1.58484848484849 -0.00152505478703491
1.63434343434343 -0.00153427583383325
1.68383838383838 -0.00153740877760389
1.73333333333333 -0.00152694078769744
1.78282828282828 -0.00152349336799917
1.83232323232323 -0.00150282119543632
1.88181818181818 -0.00151361411221107
1.93131313131313 -0.00151585980537794
1.98080808080808 -0.00149772814093596
2.03030303030303 -0.00149642911886128
2.07979797979798 -0.00149814755813971
2.12929292929293 -0.00148756313054699
2.17878787878788 -0.00151010316193922
2.22828282828283 -0.00150337709314552
2.27777777777778 -0.00150288026464784
2.32727272727273 -0.00150286401552113
2.37676767676768 -0.00146846572323545
2.42626262626263 -0.00146869361292767
2.47575757575758 -0.00146908413113578
2.52525252525253 -0.00146966044920394
2.57474747474748 -0.00146653768968885
2.62424242424242 -0.00147003644276988
2.67373737373737 -0.00147329559060794
2.72323232323232 -0.00146192951281976
2.77272727272727 -0.00146086801647838
2.82222222222222 -0.00145981792858552
2.87171717171717 -0.0014587709092807
2.92121212121212 -0.00145772108205977
2.97070707070707 -0.00145666425375102
3.02020202020202 -0.00145559779551352
3.06969696969697 -0.00145452032824667
3.11919191919192 -0.00145343150068178
3.16868686868687 -0.00145233175318393
3.21818181818182 -0.0014512254796174
3.26767676767677 -0.00145010754108976
3.31717171717172 -0.00143825659692776
3.36666666666667 -0.0014343636294224
3.41616161616162 -0.00143700523143475
3.46565656565657 -0.0014203713673836
3.51515151515152 -0.00143903689042224
3.56464646464647 -0.00146072482223947
3.61414141414142 -0.00144411311944226
3.66363636363636 -0.00142817058921851
3.71313131313131 -0.00143246952836407
3.76262626262626 -0.00143104566226915
3.81212121212121 -0.00142962045579409
3.86161616161616 -0.00142816734703011
3.91111111111111 -0.00142666610171405
3.96060606060606 -0.00142510271372953
4.01010101010101 -0.00142346967636653
4.05959595959596 -0.00142176634531133
4.10909090909091 -0.00141999835502815
4.15858585858586 -0.00141817812317646
4.20808080808081 -0.00141632349228527
4.25757575757576 -0.00141445831224363
4.30707070707071 -0.00142321269349466
4.35656565656566 -0.00142583195108059
4.40606060606061 -0.00144385275490955
4.45555555555556 -0.00143746320787541
4.50505050505051 -0.00143218314170235
4.55454545454546 -0.00140314010778209
4.6040404040404 -0.00142543844716736
4.65353535353535 -0.00144464768350727
4.7030303030303 -0.00141761780014138
4.75252525252525 -0.0014190792622748
4.8020202020202 -0.00141875908853051
4.85151515151515 -0.00141504306435969
4.9010101010101 -0.00140993735784023
4.95050505050505 -0.00140307811622842
5 -0.00143874849357573
};
\addlegendentry{$\lambda_{\rm error}$ (TS)}
\addplot [semithick, red, dashdotdotted]
table {%
0.1 -0.00235016279307732
0.14949494949495 -0.00229111402001225
0.198989898989899 -0.00217830057144841
0.248484848484848 -0.00214936964736989
0.297979797979798 -0.0020834746968581
0.347474747474747 -0.00203860894940311
0.396969696969697 -0.00199828377237954
0.446464646464647 -0.00196166260631592
0.495959595959596 -0.00188781809022798
0.545454545454546 -0.00185723271271276
0.594949494949495 -0.00182830273600987
0.644444444444444 -0.00180061336728587
0.693939393939394 -0.00181837500007964
0.743434343434343 -0.00179587674795478
0.792929292929293 -0.00177503154557257
0.842424242424242 -0.00175570439316619
0.891919191919192 -0.00173777121827046
0.941414141414142 -0.00172111702750311
0.990909090909091 -0.00170563537797108
1.04040404040404 -0.00169122879761019
1.08989898989899 -0.00167780656154781
1.13939393939394 -0.00166722168150124
1.18888888888889 -0.00166352587360608
1.23838383838384 -0.00166027556233317
1.28787878787879 -0.0016573914027308
1.33737373737374 -0.00158485097604599
1.38686868686869 -0.00158246587641497
1.43636363636364 -0.00158028535622212
1.48585858585859 -0.00157827376956465
1.53535353535354 -0.00157640302339856
1.58484848484849 -0.00157464844044175
1.63434343434343 -0.00157299146939793
1.68383838383838 -0.00157141642717928
1.73333333333333 -0.00157951395443046
1.78282828282828 -0.00154296338777393
1.83232323232323 -0.00157577898223216
1.88181818181818 -0.001537175086474
1.93131313131313 -0.00156441748396496
1.98080808080808 -0.00156314655629453
2.03030303030303 -0.00156190764810615
2.07979797979798 -0.00156070220456928
2.12929292929293 -0.00155952391254
2.17878787878788 -0.00155837345919807
2.22828282828283 -0.00148655143851795
2.27777777777778 -0.00148542282095487
2.32727272727273 -0.00148432086158468
2.37676767676768 -0.00148324294228952
2.42626262626263 -0.00148219064850452
2.47575757575758 -0.00148116299376544
2.52525252525253 -0.00148016016682639
2.57474747474748 -0.00147918028386718
2.62424242424242 -0.00147822555639654
2.67373737373737 -0.00147729447221112
2.72323232323232 -0.001476387322213
2.77272727272727 -0.00147550336583169
2.82222222222222 -0.00147464369767825
2.87171717171717 -0.00147380639957118
2.92121212121212 -0.00147252612814108
2.97070707070707 -0.001469279758326
3.02020202020202 -0.00147559336285541
3.06969696969697 -0.00147333437787445
3.11919191919192 -0.00147119509351956
3.16868686868687 -0.00146910788105181
3.21818181818182 -0.00146707043209293
3.26767676767677 -0.00146508075690988
3.31717171717172 -0.00146313668756337
3.36666666666667 -0.00146123658773223
3.41616161616162 -0.00145937875041025
3.46565656565657 -0.00145756143404396
3.51515151515152 -0.00145578350652797
3.56464646464647 -0.00145404309782948
3.61414141414142 -0.00145233886176943
3.66363636363636 -0.00145066968685914
3.71313131313131 -0.00144903416546827
3.76262626262626 -0.00144743104965017
3.81212121212121 -0.00144585982079082
3.86161616161616 -0.00144431876109563
3.91111111111111 -0.00144280700913813
3.96060606060606 -0.00144132364381122
4.01010101010101 -0.00143986801230911
4.05959595959596 -0.00143843906078722
4.10909090909091 -0.00143703570355591
4.15858585858586 -0.00143565781802801
4.20808080808081 -0.00143430396080235
4.25757575757576 -0.00143297368652029
4.30707070707071 -0.00143166625801261
4.35656565656566 -0.00143038134524442
4.40606060606061 -0.00142911730233275
4.45555555555556 -0.00142787482235718
4.50505050505051 -0.0014266526209914
4.55454545454546 -0.00142544999563399
4.6040404040404 -0.00142426713364252
4.65353535353535 -0.00142310280145408
4.7030303030303 -0.00142195627027242
4.75252525252525 -0.00142082840126367
4.8020202020202 -0.00141971750424763
4.85151515151515 -0.00141862370183954
4.9010101010101 -0.00141754626339952
4.95050505050505 -0.00141648470787192
5 -0.0014154392981186
};
\addlegendentry{$\lambda_{\rm loss}$ (TS)}
\addplot [semithick, color0, dashed]
table {%
0.1 0.00612983847076154
0.14949494949495 0.00552672959123013
0.198989898989899 0.00494555848895584
0.248484848484848 0.00439628757545607
0.297979797979798 0.00388821496055913
0.347474747474747 0.00342948815061628
0.396969696969697 0.00302672344930655
0.446464646464647 0.00268437035634395
0.495959595959596 0.00240204646992392
0.545454545454546 0.00217584868420728
0.594949494949495 0.00199807072542912
0.644444444444444 0.00186030872156306
0.693939393939394 0.0017544771010668
0.743434343434343 0.00167278939144833
0.792929292929293 0.00160973028431544
0.842424242424242 0.00156074884303004
0.891919191919192 0.00152240692807148
0.941414141414142 0.00149216816261499
0.990909090909091 0.00146818476936461
1.04040404040404 0.0014489877588626
1.08989898989899 0.00143363871212865
1.13939393939394 0.00142099238261184
1.18888888888889 0.00141061583507995
1.23838383838384 0.00140268336373928
1.28787878787879 0.00139595048465835
1.33737373737374 0.00139044923539955
1.38686868686869 0.0013859369367033
1.43636363636364 0.00138207065580387
1.48585858585859 0.00137947310568975
1.53535353535354 0.00137677146382431
1.58484848484849 0.0013747897203803
1.63434343434343 0.00137297259734115
1.68383838383838 0.00137218362997493
1.73333333333333 0.0013705673081742
1.78282828282828 0.00136979023927375
1.83232323232323 0.00136922480580548
1.88181818181818 0.00136886812820636
1.93131313131313 0.00136869461607525
1.98080808080808 0.00136868927038136
2.03030303030303 0.00136851081156319
2.07979797979798 0.00136870455372162
2.12929292929293 0.00136855109700396
2.17878787878788 0.0013688740371266
2.22828282828283 0.00136906231936362
2.27777777777778 0.00136928541308079
2.32727272727273 0.00136953441180265
2.37676767676768 0.00136980132144582
2.42626262626263 0.00137007891751473
2.47575757575758 0.00137069372458964
2.52525252525253 0.00137119275011666
2.57474747474748 0.00137161637938921
2.62424242424242 0.00137184555437675
2.67373737373737 0.00137226177124483
2.72323232323232 0.00137268526203882
2.77272727272727 0.00137311345458113
2.82222222222222 0.00137354528621647
2.87171717171717 0.00137397920235305
2.92121212121212 0.00137441425245532
2.97070707070707 0.00137484933498011
3.02020202020202 0.00137528344484961
3.06969696969697 0.0013757160815896
3.11919191919192 0.00137614642894968
3.16868686868687 0.00137657411449732
3.21818181818182 0.00137699854605877
3.26767676767677 0.00137741934033786
3.31717171717172 0.00137783611603493
3.36666666666667 0.00137824878899506
3.41616161616162 0.00137865744108123
3.46565656565657 0.00137906052571313
3.51515151515152 0.00137935828852576
3.56464646464647 0.00137975374774557
3.61414141414142 0.0013801442304977
3.66363636363636 0.00138052967068902
3.71313131313131 0.00138090996325269
3.76262626262626 0.00138128515727753
3.81212121212121 0.0013816551553737
3.86161616161616 0.00138201997005549
3.91111111111111 0.00138237958235266
3.96060606060606 0.00138273408023815
4.01010101010101 0.00138308338559101
4.05959595959596 0.00138342760865806
4.10909090909091 0.00138376677732321
4.15858585858586 0.00138410092181906
4.20808080808081 0.00138443007095534
4.25757575757576 0.0013847543055765
4.30707070707071 0.00138507370085106
4.35656565656566 0.00138546661378736
4.40606060606061 0.00138581625710676
4.45555555555556 0.00138612227760193
4.50505050505051 0.00138642367195307
4.55454545454546 0.00138672055667521
4.6040404040404 0.00138701302621247
4.65353535353535 0.00138730099943707
4.7030303030303 0.00138746239624121
4.75252525252525 0.00138774106316431
4.8020202020202 0.00138801557820556
4.85151515151515 0.00138828595015283
4.9010101010101 0.00138855252248227
4.95050505050505 0.00138881537002777
5 0.0013889804884788
};
\end{axis}

\end{tikzpicture}

%% file: Figures/conditional_variance/exp_conditional_variance_bo_temp_scaling_n_over_p=2.0.tex
\begin{tikzpicture}
\tikzstyle{every node}=[font=\tiny]

\begin{axis}[
height=\figheight,
legend cell align={left},
legend style={fill opacity=0.8, draw opacity=1, text opacity=1, draw=white!80!black},
tick align=outside,
tick pos=left,
width=\figwidth,
x grid style={white!69.0196078431373!black},
xlabel={\(\displaystyle \alpha\)},
xmajorgrids,
xmin=-0.145, xmax=5.245,
xtick style={color=black},
y grid style={white!69.0196078431373!black},
ylabel={Var(\(\displaystyle \hat{f}_{\rm bo} | \hat{f}_{\rm erm} = 0.75\))},
ymajorgrids,
ymin=0.00996848119531142, ymax=0.0119647744098775,
ytick style={color=black}
]
\addplot [draw=black, fill=black, forget plot, mark=*, only marks, mark size=1pt]
table{%
x  y
0.594949494949495 0.0118740338092154
0.842424242424242 0.0116860922140029
1.08989898989899 0.0115557732472997
1.33737373737374 0.0116597748789844
1.58484848484849 0.011484722384922
1.83232323232323 0.0113733317432226
2.07979797979798 0.0112053445039801
2.32727272727273 0.0111800096879022
2.57474747474748 0.0111060750384711
2.82222222222222 0.0113430934170851
3.06969696969697 0.0111517918418017
3.31717171717172 0.0112538909227066
3.56464646464647 0.0113159843430527
3.81212121212121 0.0112975564827165
4.05959595959596 0.0111825046874913
4.30707070707071 0.011270714027812
};
\addplot [draw=blue, fill=blue, forget plot, mark=*, only marks, mark size=1pt]
table{%
x  y
0.594949494949495 0.010536085892629
0.842424242424242 0.0105318461943459
1.08989898989899 0.010431359604583
1.33737373737374 0.0104672684533257
1.58484848484849 0.0105899909108893
1.83232323232323 0.0103772764506643
2.07979797979798 0.0101993824300587
2.32727272727273 0.010197608362982
2.57474747474748 0.0102043723689084
2.82222222222222 0.0102647092131448
3.06969696969697 0.0102126738659766
3.31717171717172 0.01018227479103
3.56464646464647 0.0102266533701664
3.81212121212121 0.0103194146320063
4.05959595959596 0.0101602736589135
4.30707070707071 0.0102323730410303
};
\addplot [draw=red, fill=red, forget plot, mark=*, only marks, mark size=1pt]
table{%
x  y
0.594949494949495 0.0108976341923498
0.842424242424242 0.0107676640025064
1.08989898989899 0.0106425674076613
1.33737373737374 0.0106122286262698
1.58484848484849 0.010709770757983
1.83232323232323 0.0104875052010707
2.07979797979798 0.010397312426094
2.32727272727273 0.0102708909834166
2.57474747474748 0.010277686263767
2.82222222222222 0.0104272184397962
3.06969696969697 0.0103198250018299
3.31717171717172 0.0103123683699881
3.56464646464647 0.0102392366393492
3.81212121212121 0.0102889139037664
4.05959595959596 0.0102461671146432
4.30707070707071 0.0103077749873455
};
\addplot [semithick, blue, dashdotdotted]
table {%
0.1 0.0108803209595668
0.14949494949495 0.010815284995607
0.198989898989899 0.0107616696300629
0.248484848484848 0.0107182767687962
0.297979797979798 0.0106834773020176
0.347474747474747 0.0106555705263076
0.396969696969697 0.0106334379217781
0.446464646464647 0.0106168688021121
0.495959595959596 0.0106058505876981
0.545454545454546 0.0105995695176183
0.594949494949495 0.0105961437786917
0.644444444444444 0.0105934130153724
0.693939393939394 0.0105896680773238
0.743434343434343 0.0105841125873755
0.792929292929293 0.0105765330589566
0.842424242424242 0.0105671385959563
0.891919191919192 0.0105562835024861
0.941414141414142 0.0105443483866822
0.990909090909091 0.0105316666938127
1.04040404040404 0.0105185581820272
1.08989898989899 0.0105052148445556
1.13939393939394 0.010491822037337
1.18888888888889 0.0104785214198956
1.23838383838384 0.0104653614091571
1.28787878787879 0.0104524874030931
1.33737373737374 0.0104398676389174
1.38686868686869 0.010427599514481
1.43636363636364 0.0104156559080947
1.48585858585859 0.0104040419815511
1.53535353535354 0.0103928219490964
1.58484848484849 0.0103819714790729
1.63434343434343 0.0103714115876997
1.68383838383838 0.010361202448995
1.73333333333333 0.0103513440678076
1.78282828282828 0.0103417986466465
1.83232323232323 0.0103325855033294
1.88181818181818 0.0103236352434908
1.93131313131313 0.0103149850437195
1.98080808080808 0.0103066350801291
2.03030303030303 0.010298533835567
2.07979797979798 0.0102906861296459
2.12929292929293 0.0102830960203451
2.17878787878788 0.0102757115131833
2.22828282828283 0.0102685798046543
2.27777777777778 0.0102616591724072
2.32727272727273 0.010254945809998
2.37676767676768 0.0102484566959556
2.42626262626263 0.0102421318585594
2.47575757575758 0.0102359899549922
2.52525252525253 0.0102300235074257
2.57474747474748 0.0102242276849595
2.62424242424242 0.0102185896314009
2.67373737373737 0.0102131070964202
2.72323232323232 0.0102077809851346
2.77272727272727 0.0102025926142175
2.82222222222222 0.0101975417944342
2.87171717171717 0.0101926232320323
2.92121212121212 0.0101878318853786
2.97070707070707 0.0101831629525134
3.02020202020202 0.0101786118587495
3.06969696969697 0.0101741742449265
3.11919191919192 0.0101698459561401
3.16868686868687 0.0101656230310213
3.21818181818182 0.0101615157086896
3.26767676767677 0.0101574916640891
3.31717171717172 0.0101535639046689
3.36666666666667 0.0101497258460059
3.41616161616162 0.0101459751122828
3.46565656565657 0.0101423113098162
3.51515151515152 0.0101387269682174
3.56464646464647 0.0101352226436421
3.61414141414142 0.0101317974256613
3.66363636363636 0.0101284458266985
3.71313131313131 0.0101251660016265
3.76262626262626 0.0101219558780155
3.81212121212121 0.0101188132491686
3.86161616161616 0.0101157359987609
3.91111111111111 0.010112722095995
3.96060606060606 0.0101097695916303
4.01010101010101 0.0101068766147254
4.05959595959596 0.0101040413695989
4.10909090909091 0.0101012621324807
4.15858585858586 0.0100985372490281
4.20808080808081 0.0100958651310543
4.25757575757576 0.0100932442544049
4.30707070707071 0.0100906756852392
4.35656565656566 0.0100881542443603
4.40606060606061 0.0100856839205634
4.45555555555556 0.0100832532007153
4.50505050505051 0.0100808670658007
4.55454545454546 0.0100785165220824
4.6040404040404 0.0100762241336667
4.65353535353535 0.0100739720371713
4.7030303030303 0.0100717442799472
4.75252525252525 0.0100695649174131
4.8020202020202 0.0100674227121096
4.85151515151515 0.0100653159698082
4.9010101010101 0.0100632444209726
4.95050505050505 0.0100612069310779
5 0.0100592217959735
};
\addplot [semithick, red, dashdotdotted]
table {%
0.1 0.0109114272462751
0.14949494949495 0.0108778618392494
0.198989898989899 0.0108603876173443
0.248484848484848 0.0108540515498159
0.297979797979798 0.0108540315267684
0.347474747474747 0.0108564783217772
0.396969696969697 0.0108591589412427
0.446464646464647 0.010861577243083
0.495959595959596 0.0108641806213365
0.545454545454546 0.0108666618989104
0.594949494949495 0.0108680214131862
0.644444444444444 0.0108668477789821
0.693939393939394 0.0108620221619987
0.743434343434343 0.0108535914935561
0.792929292929293 0.0108417258903885
0.842424242424242 0.0108270185191343
0.891919191919192 0.0108101364065637
0.941414141414142 0.0107917032836118
0.990909090909091 0.0107722505458123
1.04040404040404 0.0107522067283399
1.08989898989899 0.0107319056692771
1.13939393939394 0.0107115938250695
1.18888888888889 0.0106914452932394
1.23838383838384 0.0106716194264845
1.28787878787879 0.0106522129001405
1.33737373737374 0.0106335108943149
1.38686868686869 0.0106151120653105
1.43636363636364 0.0105972672524428
1.48585858585859 0.0105799890418122
1.53535353535354 0.0105632804119955
1.58484848484849 0.0105471372074127
1.63434343434343 0.0105315502288329
1.68383838383838 0.0105165066359364
1.73333333333333 0.0105019679589982
1.78282828282828 0.0104880459758093
1.83232323232323 0.0104744559440735
1.88181818181818 0.0104615010954052
1.93131313131313 0.0104488588876278
1.98080808080808 0.0104367162726042
2.03030303030303 0.0104249930772626
2.07979797979798 0.0104136717122705
2.12929292929293 0.0104027347843562
2.17878787878788 0.010392165764741
2.22828282828283 0.0103820678876765
2.27777777777778 0.0103721832233747
2.32727272727273 0.0103626206426379
2.37676767676768 0.0103533660472066
2.42626262626263 0.010344406165405
2.47575757575758 0.0103357282420659
2.52525252525253 0.0103273201836185
2.57474747474748 0.0103191703925622
2.62424242424242 0.0103112680195854
2.67373737373737 0.0103036025902058
2.72323232323232 0.0102961642273377
2.77272727272727 0.01028894349965
2.82222222222222 0.0102819315261544
2.87171717171717 0.0102751197345736
2.92121212121212 0.0102685005171449
2.97070707070707 0.0102620676226225
3.02020202020202 0.0102558032302521
3.06969696969697 0.0102497169091302
3.11919191919192 0.0102437940390272
3.16868686868687 0.0102380283532133
3.21818181818182 0.0102324139101627
3.26767676767677 0.0102269450075078
3.31717171717172 0.0102216162742069
3.36666666666667 0.010216422503601
3.41616161616162 0.010211358768899
3.46565656565657 0.0102064203951988
3.51515151515152 0.0102016028298703
3.56464646464647 0.0101969018703834
3.61414141414142 0.0101923134254062
3.66363636363636 0.0101878335623795
3.71313131313131 0.0101834585793481
3.76262626262626 0.0101791849209228
3.81212121212121 0.01017500909497
3.86161616161616 0.0101709279317541
3.91111111111111 0.0101669382877205
3.96060606060606 0.0101630371673738
4.01010101010101 0.0101592216760347
4.05959595959596 0.0101554890927706
4.10909090909091 0.0101518368150513
4.15858585858586 0.010148262246284
4.20808080808081 0.0101447630412563
4.25757575757576 0.0101413368469463
4.30707070707071 0.0101379814378417
4.35656565656566 0.0101346946407167
4.40606060606061 0.0101314744932477
4.45555555555556 0.0101283188962963
4.50505050505051 0.0101252260140634
4.55454545454546 0.0101221940294296
4.6040404040404 0.010119221121903
4.65353535353535 0.0101163056592507
4.7030303030303 0.0101134460283103
4.75252525252525 0.0101106405546513
4.8020202020202 0.0101078878229571
4.85151515151515 0.0101051863331276
4.9010101010101 0.0101025347034106
4.95050505050505 0.010099931583977
5 0.0100973756231538
};
\addplot [semithick, black, dashdotdotted]
table {%
0.1 0.0109032984072911
0.14949494949495 0.0108711583818849
0.198989898989899 0.0108705097333308
0.248484848484848 0.010906588789563
0.297979797979798 0.010986072496615
0.347474747474747 0.0111179087007427
0.396969696969697 0.0113154889757261
0.446464646464647 0.011598849464991
0.495959595959596 0.0118229334315568
0.545454545454546 0.0118299227908522
0.594949494949495 0.0118147224183452
0.644444444444444 0.0117974553543088
0.693939393939394 0.0117794597067761
0.743434343434343 0.0117606349645528
0.792929292929293 0.0117409756646767
0.842424242424242 0.0117206676606518
0.891919191919192 0.0116999884985662
0.941414141414142 0.0116792220061364
0.990909090909091 0.0116586129202011
1.04040404040404 0.0116383545621843
1.08989898989899 0.0116185873877205
1.13939393939394 0.0115994088694064
1.18888888888889 0.0115808803269286
1.23838383838384 0.0115630364119809
1.28787878787879 0.0115458909301656
1.33737373737374 0.0115294432594333
1.38686868686869 0.0115136828739979
1.43636363636364 0.0114985912983133
1.48585858585859 0.0114841463851187
1.53535353535354 0.0114703228417528
1.58484848484849 0.0114570932804803
1.63434343434343 0.0114444302951094
1.68383838383838 0.0114323062724828
1.73333333333333 0.0114206939917807
1.78282828282828 0.011409567128857
1.83232323232323 0.0113989004608636
1.88181818181818 0.0113886703824254
1.93131313131313 0.0113788523225751
1.98080808080808 0.011369424624406
2.03030303030303 0.0113603674370444
2.07979797979798 0.0113516610251976
2.12929292929293 0.0113432869715684
2.17878787878788 0.0113352288907081
2.22828282828283 0.0113274686887229
2.27777777777778 0.0113199920844038
2.32727272727273 0.0113127844759745
2.37676767676768 0.0113058323545674
2.42626262626263 0.0112991231962257
2.47575757575758 0.0112926448915832
2.52525252525253 0.0112863861518304
2.57474747474748 0.0112803369359971
2.62424242424242 0.011274486580486
2.67373737373737 0.0112688262240078
2.72323232323232 0.0112633469616823
2.77272727272727 0.0112580403722111
2.82222222222222 0.0112528990397103
2.87171717171717 0.011247915290245
2.92121212121212 0.0112430821041379
2.97070707070707 0.0112383930743392
3.02020202020202 0.0112338422376154
3.06969696969697 0.0112294229509498
3.11919191919192 0.0112251304791936
3.16868686868687 0.0112209592602296
3.21818181818182 0.0112169043515331
3.26767676767677 0.0112129609138656
3.31717171717172 0.0112091250718237
3.36666666666667 0.0112053916156972
3.41616161616162 0.0112017572430653
3.46565656565657 0.0111982179442442
3.51515151515152 0.0111947701315651
3.56464646464647 0.0111914102770526
3.61414141414142 0.0111881351439692
3.66363636363636 0.0111849416112478
3.71313131313131 0.0111818264491333
3.76262626262626 0.0111787875972476
3.81212121212121 0.0111758213264893
3.86161616161616 0.0111729258334665
3.91111111111111 0.0111700984031765
3.96060606060606 0.0111673366901928
4.01010101010101 0.0111646384397494
4.05959595959596 0.011162001522552
4.10909090909091 0.0111594238920202
4.15858585858586 0.0111569035902259
4.20808080808081 0.0111544387427427
4.25757575757576 0.0111520275562027
4.30707070707071 0.0111496683119265
4.35656565656566 0.0111473593636287
4.40606060606061 0.0111450991331505
4.45555555555556 0.0111428861490346
4.50505050505051 0.0111407188774185
4.55454545454546 0.0111385959642215
4.6040404040404 0.0111365163254692
4.65353535353535 0.0111344779131383
4.7030303030303 0.0111324802551551
4.75252525252525 0.0111305219095347
4.8020202020202 0.0111286017352275
4.85151515151515 0.0111267188873166
4.9010101010101 0.0111248715488966
4.95050505050505 0.0111230597162444
5 0.011121281396599
};
\end{axis}

\end{tikzpicture}

%% file: Figures/supplementary/laplace_supplementary/hessian_laplace.tex
\begin{tikzpicture}
\tikzstyle{every node}=[font=\tiny]

\definecolor{darkgray176}{RGB}{176,176,176}
\definecolor{steelblue31119180}{RGB}{31,119,180}
\definecolor{darkgray176}{RGB}{176,176,176}
\definecolor{gray}{RGB}{128,128,128}
\definecolor{lightgray204}{RGB}{204,204,204}
\definecolor{orange}{RGB}{255,165,0}

\begin{axis}[
height=\figheight,
legend cell align={left},
legend style={fill opacity=0.8, draw opacity=1, text opacity=1, draw=lightgray204},
tick align=outside,
tick pos=left,
width=\figwidth,
x grid style={darkgray176},
xlabel={$\sfrac{p}{n}$},
xmajorgrids,
xmin=0.0, xmax=5.245,
xtick style={color=black},
y grid style={darkgray176},
ylabel={\(\displaystyle \vec{\varphi}(\vec{x})^{\top} \mathcal{H}^{-1} \vec{\varphi}(\vec{x})\)},
ymajorgrids,
ymin=0.0, ymax=2.0,
ytick style={color=black}
]

\addplot [semithick, blue]
table {%
0.1 0.292224013983568
0.396969696969697 1.24328694544314
0.693939393939394 1.6887333110844
0.990909090909091 1.73441082680363
1.28787878787879 1.6695619300946
1.58484848484849 1.58596228460025
2.12929292929293 1.44819736417767
2.72323232323232 1.33345691450121
3.02020202020202 1.28753983677137
3.31717171717172 1.2475993228652
3.61414141414142 1.21261814999065
3.91111111111111 1.18177140095727
4.20808080808081 1.15439165370321
4.50505050505051 1.12993858578463
5 1.0945404279005
};
\addlegendentry{$\lambdaerror$}
\addplot [semithick, red]
table {%
0.1 0.0843345987786777
0.396969696969697 0.269670432903222
0.693939393939394 0.37699811260765
0.990909090909091 0.441256886206015
1.28787878787879 0.482539279411304
1.58484848484849 0.510859989299307
2.12929292929293 0.54449851958477
2.72323232323232 0.567023487503385
3.02020202020202 0.575168723202892
3.31717171717172 0.581932240893446
3.61414141414142 0.58763575466045
3.91111111111111 0.592509117096252
4.20808080808081 0.596720495570624
4.50505050505051 0.600395555884386
5 0.605579431983388
};
\addlegendentry{$\lambdaloss$}

\addplot [draw=blue, fill=blue, mark=x, only marks]
table{%
x  y
0.1 0.297031499093225
0.396969696969697 1.20913930402583
0.693939393939394 1.69886986883514
0.990909090909091 1.72407213218436
1.28787878787879 1.68626881743107
1.58484848484849 1.58052554908723
2.12929292929293 1.44518442139184
2.72323232323232 1.32372821769383
3.02020202020202 1.28286623381724
3.31717171717172 1.25115015352724
3.61414141414142 1.20675052573089
3.91111111111111 1.18431060819794
4.20808080808081 1.1551863136304
4.50505050505051 1.13143860562099
5 1.09462641115563
};
\addplot [draw=red, fill=red, mark=x, only marks]
table{%
x  y
0.1 0.0833648806584973
0.396969696969697 0.271716843025604
0.693939393939394 0.377173151075736
0.990909090909091 0.440243870268626
1.28787878787879 0.48147528012177
1.58484848484849 0.509530694626967
2.12929292929293 0.543445819506409
2.72323232323232 0.56723533925817
3.02020202020202 0.576575994356744
3.31717171717172 0.581227122312366
3.61414141414142 0.585548380888579
3.91111111111111 0.592479937275558
4.20808080808081 0.598631055279036
4.50505050505051 0.600144961832693
5 0.607643988732089
};
\end{axis}

\end{tikzpicture}

%% file: Figures/supplementary/laplace_supplementary/hessian_trace_lambda=0.0001.tex
\begin{tikzpicture}

\definecolor{darkgray176}{RGB}{176,176,176}
\definecolor{steelblue31119180}{RGB}{31,119,180}

\tikzstyle{every node}=[font=\tiny]

\begin{axis}[
height=\figheight,
tick align=outside,
tick pos=left,
width=\figwidth,
x grid style={darkgray176},
xlabel={\(\displaystyle p / n\)},
xmin=-0.145, xmax=5.245,
xtick style={color=black},
y grid style={darkgray176},
xmajorgrids,ymajorgrids,
ylabel={\(\displaystyle \vec{\varphi}(\vec{x})^{\top} \mathcal{H}^{-1} \vec{\varphi}(\vec{x})\)},
ymin=-49.3237886505974, ymax=1046.61806094321,
ytick style={color=black}
]

\addplot [draw=black, fill=black, mark=x, only marks]
table{%
x  y
0.1 0.491749967303101
0.396969696969697 7.60653593295313
0.693939393939394 313.139589389814
0.990909090909091 575.645181151354
1.28787878787879 624.156168596964
1.58484848484849 713.34583389563
2.12929292929293 780.081716740521
2.72323232323232 905.102585699005
3.02020202020202 947.343404648192
3.31717171717172 905.496999453987
3.61414141414142 943.039107932363
3.91111111111111 996.802522325313
4.20808080808081 951.553569736106
4.50505050505051 965.82959172063
5 991.158167375027
};
\addplot [semithick, black]
table {%
0.1 0.497839629547203
0.396969696969697 6.85495080229122
0.693939393939394 308.376656061804
0.990909090909091 529.039324618467
1.28787878787879 655.345317232258
1.58484848484849 736.63185587716
2.12929292929293 828.669586295836
2.72323232323232 888.069490583772
3.02020202020202 909.1782339018
3.31717171717172 926.571684127183
3.61414141414142 941.150447885667
3.91111111111111 953.546068416218
4.20808080808081 964.214435786362
4.50505050505051 973.492846167123
5 986.533276224766
};

\end{axis}

\end{tikzpicture}

%% file: Figures/exp_calibration_p=0.6.tex
\begin{tikzpicture}
\tikzstyle{every node}=[font=\tiny]

\definecolor{color0}{rgb}{1,0.647058823529412,0}
\definecolor{color1}{rgb}{0,1,1}

\begin{axis}[
height=\figheight,
legend cell align={left},
legend style={fill opacity=0.8, draw opacity=1, text opacity=1, draw=white!80!black},
tick align=outside,
tick pos=left,
width=\figwidth,
x grid style={white!69.0196078431373!black},
xlabel={1 / \(\displaystyle \alpha\)},
xmajorgrids,
xmin=-0.145, xmax=5.245,
xtick style={color=black},
y grid style={white!69.0196078431373!black},
ylabel={\(\displaystyle \Delta_{0.6}\)},
ymajorgrids,
ymin=-0.0168848939868578, ymax=0.104347814016057,
ytick style={color=black}
]
\addplot [draw=blue, fill=blue, forget plot, mark=*, only marks, mark size=1pt]
table{%
x  y
0.594949494949495 0.068784917130323
0.842424242424242 0.0662556529476368
1.08989898989899 0.0635119536551561
1.33737373737374 0.060596961004131
1.58484848484849 0.058144407816701
1.83232323232323 0.0554580807906128
2.07979797979798 0.0525389082240254
2.32727272727273 0.0510886960657553
2.57474747474748 0.0493066266625503
2.82222222222222 0.0472075813488766
3.06969696969697 0.0461435522405944
3.31717171717172 0.0443139725808951
3.56464646464647 0.042757671788191
3.81212121212121 0.0419157802294247
4.05959595959596 0.0404306152433619
4.30707070707071 0.0394866705750407
};
\addplot [draw=red, fill=red, forget plot, mark=*, only marks, mark size=1pt]
table{%
x  y
0.594949494949495 0.00609327241176605
0.842424242424242 0.00540546058596214
1.08989898989899 0.0049275459480741
1.33737373737374 0.00582154145992897
1.58484848484849 0.00544747286230618
1.83232323232323 0.00207227156182799
2.07979797979798 0.00318704378182599
2.32727272727273 0.00392694289321616
2.57474747474748 0.00267809078070325
2.82222222222222 0.00273532540795574
3.06969696969697 0.00352330155071956
3.31717171717172 0.00249869398083857
3.56464646464647 0.00245499989230691
3.81212121212121 0.00284760023405328
4.05959595959596 0.00282241165773356
4.30707070707071 0.00231971439468659
};
\addplot [draw=color1, fill=color1, forget plot, mark=*, only marks, mark size=1pt]
table{%
x  y
0.594949494949495 0.0673562857511558
0.842424242424242 0.0629542416855129
1.08989898989899 0.060987272265446
1.33737373737374 0.0561914057098832
1.58484848484849 0.0541203340320573
1.83232323232323 0.0517132974737329
2.07979797979798 0.0496367053891841
2.32727272727273 0.0484407423176195
2.57474747474748 0.0464221915487979
2.82222222222222 0.0430711547610342
3.06969696969697 0.0436079699583898
3.31717171717172 0.0395177865865477
3.56464646464647 0.0406563351301552
3.81212121212121 0.0383445894024158
4.05959595959596 0.0364157157197372
4.30707070707071 0.0349927239345031
};
\addplot [draw=color0, fill=color0, forget plot, mark=*, only marks, mark size=1pt ]
table{%
x  y
0.594949494949495 0.00338811636299297
0.842424242424242 -0.00360486804458993
1.08989898989899 -1.10173165592808e-05
1.33737373737374 0.000970863302625524
1.58484848484849 -0.00190803066715128
1.83232323232323 0.00152778340942283
2.07979797979798 0.00436149778927075
2.32727272727273 -0.000148973817740594
2.57474747474748 -0.000698186121329991
2.82222222222222 0.00084224839000302
3.06969696969697 0.000699189272059808
3.31717171717172 0.00381555221676122
3.56464646464647 0.00393864998346416
3.81212121212121 0.00139519204759897
4.05959595959596 0.0030875235319322
4.30707070707071 -0.00239672194688523
};
\addplot [draw=color0, fill=color0, forget plot, mark=*, only marks, mark size=1pt]
table{%
x  y
0.594949494949495 0.00338811636299297
0.842424242424242 -0.00360486804458993
1.08989898989899 -1.10173165592808e-05
1.33737373737374 0.000970863302625524
1.58484848484849 -0.00190803066715128
1.83232323232323 0.00152778340942283
2.07979797979798 0.00436149778927075
2.32727272727273 -0.000148973817740594
2.57474747474748 -0.000698186121329991
2.82222222222222 0.00084224839000302
3.06969696969697 0.000699189272059808
3.31717171717172 0.00381555221676122
3.56464646464647 0.00393864998346416
3.81212121212121 0.00139519204759897
4.05959595959596 0.0030875235319322
4.30707070707071 -0.00239672194688523
};
\addplot [draw=black, fill=black, forget plot, mark=*, only marks, mark size=1pt]
table{%
x  y
0.594949494949495 0.0985725792858396
0.842424242424242 0.0976386741787345
1.08989898989899 0.0971102789367183
1.33737373737374 0.0968346696616613
1.58484848484849 0.0964395133565795
1.83232323232323 0.0962486183254743
2.07979797979798 0.0960298160666024
2.32727272727273 0.095978116816459
2.57474747474748 0.0958591555925761
2.82222222222222 0.0958931649862651
3.06969696969697 0.0957418309114524
3.31717171717172 0.095731704097261
3.56464646464647 0.0957252888754505
3.81212121212121 0.0956658728393353
4.05959595959596 0.0956230160462175
4.30707070707071 0.09557645131801
};
\addplot [draw=red, fill=red, forget plot, mark=*, only marks, mark size=1pt]
table{%
x  y
0.594949494949495 -0.00149782768604678
0.842424242424242 -0.00358248403475203
1.08989898989899 -0.00498593704104222
1.33737373737374 -0.00461794038175822
1.58484848484849 -0.00547758120367148
1.83232323232323 -0.00957664555624071
2.07979797979798 -0.00861566210329012
2.32727272727273 -0.0079985081349907
2.57474747474748 -0.00956867632419867
2.82222222222222 -0.00965539559493978
3.06969696969697 -0.00889988894238081
3.31717171717172 -0.010155152392811
3.56464646464647 -0.0102987725983217
3.81212121212121 -0.00994427004940457
4.05959595959596 -0.0100449458985228
4.30707070707071 -0.0106712460316927
};
\addplot [draw=blue, fill=blue, forget plot, mark=*, only marks, mark size=1pt ]
table{%
x  y
0.594949494949495 0.0590817308740273
0.842424242424242 0.0550166228438104
1.08989898989899 0.0515676084961092
1.33737373737374 0.0480243031086042
1.58484848484849 0.0453215771212728
1.83232323232323 0.0422890242793389
2.07979797979798 0.03896741925214
2.32727272727273 0.0375829290069575
2.57474747474748 0.0356718337857845
2.82222222222222 0.0333921371818912
3.06969696969697 0.0324391675820436
3.31717171717172 0.0304120693181158
3.56464646464647 0.0287821154133433
3.81212121212121 0.0280063459310607
4.05959595959596 0.0264314805788749
4.30707070707071 0.0254979992566572
};
\addplot [draw=black, fill=black, forget plot, mark=*, only marks, mark size=1pt ]
table{%
x  y
0.594949494949495 0.094214775377926
0.842424242424242 0.0693500819028747
1.08989898989899 0.0554708274045794
1.33737373737374 0.0481125796164814
1.58484848484849 0.0400861743722548
1.83232323232323 0.0346754683880166
2.07979797979798 0.0292219065287193
2.32727272727273 0.0264028213101566
2.57474747474748 0.0227283101668031
2.82222222222222 0.0229517402535705
3.06969696969697 0.0193184636091884
3.31717171717172 0.0192347803530082
3.56464646464647 0.0194583000417353
3.81212121212121 0.0171265543250062
4.05959595959596 0.0179565798917112
4.30707070707071 0.0141661286737484
};
\addplot [semithick, black]
table {%
0.1 0.0810686586940352
0.14949494949495 0.0832454785567415
0.198989898989899 0.0854367991182468
0.248484848484848 0.0876579438954945
0.297979797979798 0.0899242113800947
0.347474747474747 0.092256881174969
0.396969696969697 0.0946981417054584
0.446464646464647 0.0973010447169239
0.495959595959596 0.0988372363795613
0.545454545454546 0.0987460600794607
0.594949494949495 0.0985183723802114
0.644444444444444 0.0982971273780945
0.693939393939394 0.0980956596831791
0.743434343434343 0.0979145078156279
0.792929292929293 0.0977518772213297
0.842424242424242 0.0976055919020716
0.891919191919192 0.0974735871754973
0.941414141414142 0.0973540324098096
0.990909090909091 0.097245345492881
1.04040404040404 0.0971461733781624
1.08989898989899 0.097055361621562
1.13939393939394 0.0969719248038056
1.18888888888889 0.0968950198390401
1.23838383838384 0.0968239230351774
1.28787878787879 0.0967580108250077
1.33737373737374 0.0966967437591886
1.38686868686869 0.0966396532639006
1.43636363636364 0.0965863306985432
1.48585858585859 0.0965364182923922
1.53535353535354 0.0964896016249304
1.58484848484849 0.0964456033633536
1.63434343434343 0.0964041780304356
1.68383838383838 0.096365107615995
1.73333333333333 0.0963281978817325
1.78282828282828 0.0962932752385346
1.83232323232323 0.0962601840970363
1.88181818181818 0.0962287863471073
1.93131313131313 0.0961989517121751
1.98080808080808 0.0961705688540486
2.03030303030303 0.0961435346740535
2.07979797979798 0.0961177555386463
2.12929292929293 0.0960931462616893
2.17878787878788 0.0960696314202949
2.22828282828283 0.0960471360382865
2.27777777777778 0.0960255968639967
2.32727272727273 0.0960049543162274
2.37676767676768 0.0959851536514299
2.42626262626263 0.0959661444839304
2.47575757575758 0.0959478803617656
2.52525252525253 0.0959303183912587
2.57474747474748 0.0959134187184413
2.62424242424242 0.0958971451567762
2.67373737373737 0.0958814630731629
2.72323232323232 0.095866341002287
2.77272727272727 0.0958517495070436
2.82222222222222 0.0958376611756826
2.87171717171717 0.0958240504503326
2.92121212121212 0.0958108934731168
2.97070707070707 0.0957981679473218
3.02020202020202 0.0957858530119297
3.06969696969697 0.0957739291280405
3.11919191919192 0.0957623779758925
3.16868686868687 0.0957511823613464
3.21818181818182 0.0957403261308369
3.26767676767677 0.0957297940939101
3.31717171717172 0.0957195719524978
3.36666666666667 0.0957096462366862
3.41616161616162 0.0957000042452422
3.46565656565657 0.0956906339918792
3.51515151515152 0.095681524172339
3.56464646464647 0.0956726640494584
3.61414141414142 0.0956640435194739
3.66363636363636 0.0956556529987732
3.71313131313131 0.0956474834078263
3.76262626262626 0.0956395261385001
3.81212121212121 0.0956317730238748
3.86161616161616 0.0956242163103523
3.91111111111111 0.095616848631852
3.96060606060606 0.0956096629859124
4.01010101010101 0.0956026527115427
4.05959595959596 0.0955958114686684
4.10909090909091 0.0955891332190459
4.15858585858586 0.0955826122084972
4.20808080808081 0.095576242950503
4.25757575757576 0.0955700202106852
4.30707070707071 0.0955639389925862
4.35656565656566 0.0955579945242866
4.40606060606061 0.0955521822459373
4.45555555555556 0.095546497816751
4.50505050505051 0.0955409370307405
4.55454545454546 0.0955354959147744
4.6040404040404 0.0955301706481306
4.65353535353535 0.0955249575709667
4.7030303030303 0.095519853175943
4.75252525252525 0.0955148541003656
4.8020202020202 0.0955099571188064
4.85151515151515 0.0955051591361734
4.9010101010101 0.0955004571811947
4.95050505050505 0.0954958484002904
5 0.0954913300518039
};
\addlegendentry{$\hat{f}_{\rm erm}( \lambda = 1e-4)$}
\addplot [semithick, color0]
table {%
0.1 0.00363614179402072
0.14949494949495 0.00329838530733118
0.198989898989899 0.0029787140649169
0.248484848484848 0.0026804154589839
0.297979797979798 0.00240667944520867
0.347474747474747 0.00216035005704962
0.396969696969697 0.00194372614004001
0.446464646464647 0.0017582460385912
0.495959595959596 0.00160324626295349
0.545454545454546 0.00147653781907509
0.594949494949495 0.00137425828152904
0.644444444444444 0.00129227767103091
0.693939393939394 0.00122665310551939
0.743434343434343 0.00117361219130441
0.792929292929293 0.00113045082275509
0.842424242424242 0.00109493053736898
0.891919191919192 0.00106534306816974
0.941414141414142 0.00104041080469763
0.990909090909091 0.00101918666870959
1.04040404040404 0.00100090998658586
1.08989898989899 0.000985074300291156
1.13939393939394 0.000971087672399396
1.18888888888889 0.000958692899792113
1.23838383838384 0.000947918296100858
1.28787878787879 0.000938147381642773
1.33737373737374 0.000929354398578108
1.38686868686869 0.000921392084680117
1.43636363636364 0.000914071091301705
1.48585858585859 0.000907648552333273
1.53535353535354 0.00090146896402965
1.58484848484849 0.000895887425409025
1.63434343434343 0.000890628452937081
1.68383838383838 0.000886069680682944
1.73333333333333 0.00088134271612561
1.78282828282828 0.000877197130878304
1.83232323232323 0.000873331028731439
1.88181818181818 0.000869730474139963
1.93131313131313 0.000866372225009182
1.98080808080808 0.000863238898336061
2.03030303030303 0.00086016423400459
2.07979797979798 0.000857389737643044
2.12929292929293 0.000854577097725406
2.17878787878788 0.000852096947872694
2.22828282828283 0.000849661939638646
2.27777777777778 0.000847343553348923
2.32727272727273 0.00084513186716717
2.37676767676768 0.000843017814753266
2.42626262626263 0.000840993080086339
2.47575757575758 0.000839202939087569
2.52525252525253 0.000837435111468166
2.57474747474748 0.000835704130105031
2.62424242424242 0.000833951597864258
2.67373737373737 0.000832349196761228
2.72323232323232 0.000830811143252341
2.77272727272727 0.000829333223594908
2.82222222222222 0.000827912113111062
2.87171717171717 0.000826544441767441
2.92121212121212 0.000825227283402619
2.97070707070707 0.000823957796282193
3.02020202020202 0.000822733325133607
3.06969696969697 0.000821551577992019
3.11919191919192 0.000820410239836478
3.16868686868687 0.000819307312578621
3.21818181818182 0.000818240802282766
3.26767676767677 0.000817208908596712
3.31717171717172 0.000816209922970046
3.36666666666667 0.000815242357965373
3.41616161616162 0.000814304881321726
3.46565656565657 0.000813395486582436
3.51515151515152 0.000812467133795347
3.56464646464647 0.00081161189626433
3.61414141414142 0.000810781518717829
3.66363636363636 0.00080997492428303
3.71313131313131 0.000809191071017845
3.76262626262626 0.000808429037331249
3.81212121212121 0.000807687880879038
3.86161616161616 0.000806966753494121
3.91111111111111 0.000806264833636994
3.96060606060606 0.00080558138754272
4.01010101010101 0.000804915641557602
4.05959595959596 0.000804266942836529
4.10909090909091 0.000803634633035455
4.15858585858586 0.000803018085392626
4.20808080808081 0.000802416701288999
4.25757575757576 0.000801829933208431
4.30707070707071 0.000801257256741628
4.35656565656566 0.000800734128210911
4.40606060606061 0.000800206393668179
4.45555555555556 0.000799673400704082
4.50505050505051 0.000799152592121044
4.55454545454546 0.000798643571868429
4.6040404040404 0.000798145952329632
4.65353535353535 0.000797659282867569
4.7030303030303 0.000797127090452321
4.75252525252525 0.000796661002711141
4.8020202020202 0.00079620488026988
4.85151515151515 0.000795758375919831
4.9010101010101 0.000795321309869035
4.95050505050505 0.000794893391945051
5 0.00079443081498809
};
\addlegendentry{$\hat{f}_{\rm eb}(\lambda_{\rm evidence})$}
\addplot [semithick, blue]
table {%
0.1 0.0674891342328545
0.14949494949495 0.068411173129304
0.198989898989899 0.0691682929692187
0.248484848484848 0.0697416738465902
0.297979797979798 0.0701203937881387
0.347474747474747 0.0703047308317204
0.396969696969697 0.0703074109391028
0.446464646464647 0.0701517422075912
0.495959595959596 0.0698653554959807
0.545454545454546 0.0694763532319875
0.594949494949495 0.0690100163778798
0.644444444444444 0.068487421738156
0.693939393939394 0.0679252820073803
0.743434343434343 0.0673364717666459
0.792929292929293 0.0667306993014167
0.842424242424242 0.0661152288627079
0.891919191919192 0.0654954442482454
0.941414141414142 0.064875328009376
0.990909090909091 0.0642578181788062
1.04040404040404 0.0636450783101674
1.08989898989899 0.0630474448860908
1.13939393939394 0.0624396509045908
1.18888888888889 0.0618490315149369
1.23838383838384 0.0612672575211691
1.28787878787879 0.0606947029163236
1.33737373737374 0.0601316057710037
1.38686868686869 0.0595780875210383
1.43636363636364 0.0590341852346224
1.48585858585859 0.0584998833570727
1.53535353535354 0.0579753393960138
1.58484848484849 0.0574599558294459
1.63434343434343 0.056953858999771
1.68383838383838 0.0564567598397608
1.73333333333333 0.0559687899226524
1.78282828282828 0.0554896113823053
1.83232323232323 0.055019038859123
1.88181818181818 0.0545568797226416
1.93131313131313 0.0541029453550683
1.98080808080808 0.0536570449765746
2.03030303030303 0.0532190081102936
2.07979797979798 0.0527886114413932
2.12929292929293 0.0523656819915006
2.17878787878788 0.0519500536508943
2.22828282828283 0.0515415324784293
2.27777777777778 0.0511399494317075
2.32727272727273 0.0507451326240643
2.37676767676768 0.0503569177962057
2.42626262626263 0.0499751398451183
2.47575757575758 0.0495996499387722
2.52525252525253 0.0492302615383765
2.57474747474748 0.0488668844743914
2.62424242424242 0.0485093081497895
2.67373737373737 0.048157445222774
2.72323232323232 0.0478111259489815
2.77272727272727 0.0474702504628747
2.82222222222222 0.0471346578460486
2.87171717171717 0.0468042194995733
2.92121212121212 0.0464788420112907
2.97070707070707 0.0461583954062323
3.02020202020202 0.0458427704647076
3.06969696969697 0.0455317678785845
3.11919191919192 0.045225457351003
3.16868686868687 0.0449236415505733
3.21818181818182 0.0446263631689361
3.26767676767677 0.0443332318203146
3.31717171717172 0.044044299354763
3.36666666666667 0.0437594754045877
3.41616161616162 0.0434786787525872
3.46565656565657 0.0432018147637747
3.51515151515152 0.0429287942220238
3.56464646464647 0.0426595546286459
3.61414141414142 0.042393975907812
3.66363636363636 0.0421320414036961
3.71313131313131 0.041873645818666
3.76262626262626 0.0416186573863925
3.81212121212121 0.0413671146529164
3.86161616161616 0.0411188840078601
3.91111111111111 0.0408738981763281
3.96060606060606 0.0406321127920916
4.01010101010101 0.0403934491098212
4.05959595959596 0.0401578536566203
4.10909090909091 0.0399252590517792
4.15858585858586 0.0396956119227114
4.20808080808081 0.039468833774561
4.25757575757576 0.039244882742796
4.30707070707071 0.0390237509088779
4.35656565656566 0.0388053314589407
4.40606060606061 0.0385895717459377
4.45555555555556 0.0383764459377617
4.50505050505051 0.0381658938122378
4.55454545454546 0.0379578585487574
4.6040404040404 0.0377523287254447
4.65353535353535 0.0375492086905812
4.7030303030303 0.0373485099300033
4.75252525252525 0.0371501331616393
4.8020202020202 0.036954099565312
4.85151515151515 0.0367603061807865
4.9010101010101 0.0365687184618029
4.95050505050505 0.0363793434527282
5 0.0361921266228319
};
\addlegendentry{$\hat{f}_{\rm erm}(\lambda_{\rm error})$}
\addplot [semithick, red]
table {%
0.1 0.0020906700197294
0.14949494949495 0.00289432671671941
0.198989898989899 0.00352666144471492
0.248484848484848 0.00401456256297739
0.297979797979798 0.00438273659856081
0.347474747474747 0.00465325051330623
0.396969696969697 0.00484521353681122
0.446464646464647 0.0049747040053203
0.495959595959596 0.00505503666210361
0.545454545454546 0.0050968349924142
0.594949494949495 0.00510870962303467
0.644444444444444 0.00509742787794254
0.693939393939394 0.00506834860647676
0.743434343434343 0.0050256963120312
0.792929292929293 0.00497282203847627
0.842424242424242 0.00491235244846389
0.891919191919192 0.00484638003281856
0.941414141414142 0.00477656510311064
0.990909090909091 0.00470418888991397
1.04040404040404 0.00463026959812229
1.08989898989899 0.00455563095725275
1.13939393939394 0.00448085995139247
1.18888888888889 0.00440646173902037
1.23838383838384 0.00433279260196295
1.28787878787879 0.0042601650584283
1.33737373737374 0.00418873909784967
1.38686868686869 0.00411873704312371
1.43636363636364 0.00405020753810514
1.48585858585859 0.00398325885751327
1.53535353535354 0.00391794235232179
1.58484848484849 0.00385427496270663
1.63434343434343 0.00379226501835006
1.68383838383838 0.00373190235534859
1.73333333333333 0.00367318245791559
1.78282828282828 0.00361607207485348
1.83232323232323 0.00356054777606585
1.88181818181818 0.00350657440390256
1.93131313131313 0.00345412001775525
1.98080808080808 0.00340313020842575
2.03030303030303 0.00335358020281895
2.07979797979798 0.00330540051273698
2.12929292929293 0.00325858262081569
2.17878787878788 0.00321307464933684
2.22828282828283 0.00316880435126232
2.27777777777778 0.0031257520186303
2.32727272727273 0.00308389420941468
2.37676767676768 0.0030431693545615
2.42626262626263 0.00300354197257191
2.47575757575758 0.00296499788088533
2.52525252525253 0.00292745571596931
2.57474747474748 0.0028908885811646
2.62424242424242 0.00285529449893573
2.67373737373737 0.00282062972832708
2.72323232323232 0.0027868461930215
2.77272727272727 0.00275392923385476
2.82222222222222 0.00272183252705294
2.87171717171717 0.00269053586608503
2.92121212121212 0.00266003817234106
2.97070707070707 0.0026302696035031
3.02020202020202 0.00260122101894744
3.06969696969697 0.00257288396548805
3.11919191919192 0.00254520153460902
3.16868686868687 0.00251817933373966
3.21818181818182 0.00249179881838268
3.26767676767677 0.00246602429220921
3.31717171717172 0.00244084176777082
3.36666666666667 0.00241621436741157
3.41616161616162 0.0023921576967787
3.46565656565657 0.00236862735525112
3.51515151515152 0.00234562800392291
3.56464646464647 0.00232312587957206
3.61414141414142 0.00230112729439302
3.66363636363636 0.00227956506791682
3.71313131313131 0.00225847853779571
3.76262626262626 0.00223783170550551
3.81212121212121 0.00221761310837876
3.86161616161616 0.00219783653475869
3.91111111111111 0.00217842945397217
3.96060606060606 0.0021594458464419
4.01010101010101 0.00214080390634575
4.05959595959596 0.00212255545807449
4.10909090909091 0.00210469893727305
4.15858585858586 0.00208714644801944
4.20808080808081 0.00206993973493996
4.25757575757576 0.00205308926359649
4.30707070707071 0.00203653647266744
4.35656565656566 0.00202029089536659
4.40606060606061 0.00200436720755859
4.45555555555556 0.00198873541874367
4.50505050505051 0.00197338788951928
4.55454545454546 0.00195832289720199
4.6040404040404 0.00194352323266822
4.65353535353535 0.00192898736261504
4.7030303030303 0.00191471650084707
4.75252525252525 0.00190070113031782
4.8020202020202 0.0018869350021814
4.85151515151515 0.00187338299140005
4.9010101010101 0.00186009182842939
4.95050505050505 0.00184700708304264
5 0.00183415016167177
};
\addlegendentry{$\hat{f}_{\rm erm}(\lambda_{\rm loss})$}
\addplot [semithick, color1, dashed]
table {%
0.1 0.0672999569955225
0.14949494949495 0.0680401686343289
0.198989898989899 0.0685650626579585
0.248484848484848 0.0688992712383004
0.297979797979798 0.0690290833705501
0.347474747474747 0.0689729987775123
0.396969696969697 0.068753294815364
0.446464646464647 0.0683947643494281
0.495959595959596 0.0679276995757876
0.545454545454546 0.0673801596295355
0.594949494949495 0.0667765088614557
0.644444444444444 0.0661359672459202
0.693939393939394 0.065473204594963
0.743434343434343 0.0647988031778605
0.792929292929293 0.0641208092390847
0.842424242424242 0.0634441954615901
0.891919191919192 0.0627756593079849
0.941414141414142 0.0621118007801579
0.990909090909091 0.0614574161452013
1.04040404040404 0.0608138071092159
1.08989898989899 0.060181938245715
1.13939393939394 0.0595623199275837
1.18888888888889 0.0589542976654568
1.23838383838384 0.0583587457721398
1.28787878787879 0.0577750077614105
1.33737373737374 0.0572032522674046
1.38686868686869 0.0566428301194687
1.43636363636364 0.0560941366866026
1.48585858585859 0.0555561174756244
1.53535353535354 0.0550290420121234
1.58484848484849 0.0545125926392012
1.63434343434343 0.0540064866940249
1.68383838383838 0.0535104297286082
1.73333333333333 0.0530240337995412
1.78282828282828 0.0525471868941387
1.83232323232323 0.0520796709522393
1.88181818181818 0.0516209530606255
1.93131313131313 0.0511709055459523
1.98080808080808 0.0507292772805448
2.03030303030303 0.0502958254524004
2.07979797979798 0.0498701737095671
2.12929292929293 0.0494526086147297
2.17878787878788 0.0490422115427419
2.22828282828283 0.0486391946771515
2.27777777777778 0.0482433668535589
2.32727272727273 0.0478540870461449
2.37676767676768 0.0474719411307264
2.42626262626263 0.0470960099267099
2.47575757575758 0.0467263604943975
2.52525252525253 0.0463634236535263
2.57474747474748 0.0460056645055199
2.62424242424242 0.0456543595435503
2.67373737373737 0.0453086361955225
2.72323232323232 0.0449627713553502
2.77272727272727 0.0446279414684366
2.82222222222222 0.0442985116331104
2.87171717171717 0.043974367733006
2.92121212121212 0.0436547448403607
2.97070707070707 0.0433403637357542
3.02020202020202 0.0430309913921172
3.06969696969697 0.0427261264976341
3.11919191919192 0.0424257998390795
3.16868686868687 0.042129928784295
3.21818181818182 0.0418384098917746
3.26767676767677 0.0415511430544756
3.31717171717172 0.0412680313971527
3.36666666666667 0.0409889808538147
3.41616161616162 0.040713900367949
3.46565656565657 0.0404427055661662
3.51515151515152 0.0401753008123396
3.56464646464647 0.0399116090393868
3.61414141414142 0.0396515499500503
3.66363636363636 0.0393950454542003
3.71313131313131 0.0391420200263668
3.76262626262626 0.0388924002201784
3.81212121212121 0.0386461101422286
3.86161616161616 0.0384027800597442
3.91111111111111 0.0381629651418079
3.96060606060606 0.0379262825257528
4.01010101010101 0.0376926686708248
4.05959595959596 0.0374620618723897
4.10909090909091 0.0372344020911108
4.15858585858586 0.0370097571730122
4.20808080808081 0.0367877965543892
4.25757575757576 0.0365686143835271
4.30707070707071 0.0363521570756195
4.35656565656566 0.0361383724595932
4.40606060606061 0.0359272096255642
4.45555555555556 0.0357186196501972
4.50505050505051 0.0355125532154148
4.55454545454546 0.0353090595025717
4.6040404040404 0.0351075863915171
4.65353535353535 0.0349090007659874
4.7030303030303 0.0347127940707558
4.75252525252525 0.0345184734795909
4.8020202020202 0.0343266948498812
4.85151515151515 0.0341370872455435
4.9010101010101 0.0339496386057206
4.95050505050505 0.0337643151584169
5 0.0335810846156067
};
\addplot [semithick, red, dash pattern=on 1pt off 3pt on 3pt off 3pt]
table {%
0.1 6.80689689457337e-05
0.14949494949495 1.09465497012406e-05
0.198989898989899 -0.000133905519280164
0.248484848484848 -0.000346772001646678
0.297979797979798 -0.000610164007824432
0.347474747474747 -0.000909103353028629
0.396969696969697 -0.00123126123041051
0.446464646464647 -0.00156687029259339
0.495959595959596 -0.00190833068323382
0.545454545454546 -0.00225009101376539
0.594949494949495 -0.0025879507829325
0.644444444444444 -0.00291896534671177
0.693939393939394 -0.00324109092652958
0.743434343434343 -0.00355298993520059
0.792929292929293 -0.00385383666178163
0.842424242424242 -0.00414321978111465
0.891919191919192 -0.00442099519540451
0.941414141414142 -0.00468721878684819
0.990909090909091 -0.00494212231823477
1.04040404040404 -0.00518602417487712
1.08989898989899 -0.00541928358856203
1.13939393939394 -0.005642360533746
1.18888888888889 -0.00585568055912922
1.23838383838384 -0.00605971449699205
1.28787878787879 -0.00625488811869745
1.33737373737374 -0.00644169908919834
1.38686868686869 -0.00662051647681572
1.43636363636364 -0.00679181930514805
1.48585858585859 -0.00695597580417384
1.53535353535354 -0.00711336333437185
1.58484848484849 -0.00726435173561479
1.63434343434343 -0.0074092827958766
1.68383838383838 -0.00754848395779284
1.73333333333333 -0.0076822483118405
1.78282828282828 -0.00781087125809554
1.83232323232323 -0.00793461561332232
1.88181818181818 -0.00805373503741258
1.93131313131313 -0.00816846156608941
1.98080808080808 -0.00827903253626694
2.03030303030303 -0.00838564140156883
2.07979797979798 -0.00848851173128728
2.12929292929293 -0.00858779540183074
2.17878787878788 -0.00868367553402405
2.22828282828283 -0.00877634509004488
2.27777777777778 -0.00886593666488344
2.32727272727273 -0.00895257820425444
2.37676767676768 -0.0090364272906629
2.42626262626263 -0.00911760915788873
2.47575757575758 -0.00919622195449699
2.52525252525253 -0.00927242359376979
2.57474747474748 -0.00934631356046867
2.62424242424242 -0.0094179621693089
2.67373737373737 -0.00948747601058531
2.72323232323232 -0.00955496179271664
2.77272727272727 -0.00962048983360386
2.82222222222222 -0.00968415784630861
2.87171717171717 -0.0097460348338807
2.92121212121212 -0.00980616814964619
2.97070707070707 -0.00986466946810316
3.02020202020202 -0.00992158867756721
3.06969696969697 -0.00997697261528918
3.11919191919192 -0.0100309131839936
3.16868686868687 -0.010083439236383
3.21818181818182 -0.0101346012755759
3.26767676767677 -0.0101844647571979
3.31717171717172 -0.0102330722932413
3.36666666666667 -0.0102804872784918
3.41616161616162 -0.0103267204480887
3.46565656565657 -0.0103718397317595
3.51515151515152 -0.0104158639183918
3.56464646464647 -0.0104588481014108
3.61414141414142 -0.0105008071551851
3.66363636363636 -0.0105418266231888
3.71313131313131 -0.0105818871263761
3.76262626262626 -0.0106210418771339
3.81212121212121 -0.0106593192783186
3.86161616161616 -0.0106967222749076
3.91111111111111 -0.0107333372794668
3.96060606060606 -0.0107691265216803
4.01010101010101 -0.0108041840602182
4.05959595959596 -0.0108384729002031
4.10909090909091 -0.0108720075391521
4.15858585858586 -0.0109048861545454
4.20808080808081 -0.0109370798601774
4.25757575757576 -0.0109685897807283
4.30707070707071 -0.0109994839226192
4.35656565656566 -0.0110297633736951
4.40606060606061 -0.0110594237648955
4.45555555555556 -0.0110885039073753
4.50505050505051 -0.011117020405064
4.55454545454546 -0.011144983714014
4.6040404040404 -0.0111724190515594
4.65353535353535 -0.0111993359983239
4.7030303030303 -0.0112257411333762
4.75252525252525 -0.0112516512054901
4.8020202020202 -0.0112770794839423
4.85151515151515 -0.0113020671488914
4.9010101010101 -0.0113265752178011
4.95050505050505 -0.0113506632024732
5 -0.0113743163503617
};
\addplot [semithick, blue, dash pattern=on 1pt off 3pt on 3pt off 3pt]
table {%
0.1 0.0652771328871488
0.14949494949495 0.0651725181475258
0.198989898989899 0.0649542505133633
0.248484848484848 0.0646092398417171
0.297979797979798 0.0641325297989375
0.347474747474747 0.0635294644486336
0.396969696969697 0.0628158223259563
0.446464646464647 0.062015439807272
0.495959595959596 0.061153635279919
0.545454545454546 0.0602540850143182
0.594949494949495 0.0593363279585354
0.644444444444444 0.0584151591564862
0.693939393939394 0.0575010375868129
0.743434343434343 0.0566009580772823
0.792929292929293 0.0557193020495195
0.842424242424242 0.0548586208645984
0.891919191919192 0.0540201945786857
0.941414141414142 0.0532044670087498
0.990909090909091 0.0524113424558417
1.04040404040404 0.0516403913075149
1.08989898989899 0.050899187924008
1.13939393939394 0.0501622275273239
1.18888888888889 0.0494535678091883
1.23838383838384 0.0487640502834727
1.28787878787879 0.0480928581978145
1.33737373737374 0.0474392031892608
1.38686868686869 0.0468023199960121
1.43636363636364 0.0461814770040228
1.48585858585859 0.045575989461028
1.53535353535354 0.0449854518926227
1.58484848484849 0.0444087109143113
1.63434343434343 0.0438454695768064
1.68383838383838 0.0432950511478468
1.73333333333333 0.042757221133519
1.78282828282828 0.0422313412840037
1.83232323232323 0.0417169512936353
1.88181818181818 0.041213613706473
1.93131313131313 0.040720920858705
1.98080808080808 0.0402384858440974
2.03030303030303 0.0397659610607148
2.07979797979798 0.0393029664949339
2.12929292929293 0.0388491854730832
2.17878787878788 0.0384043217272355
2.22828282828283 0.0379680648794172
2.27777777777778 0.0375401391604024
2.32727272727273 0.0371202758936607
2.37676767676768 0.0367082224186926
2.42626262626263 0.0363037331291727
2.47575757575758 0.0359065849755239
2.52525252525253 0.0355165253339365
2.57474747474748 0.0351333991710761
2.62424242424242 0.0347569418322401
2.67373737373737 0.0343870109079432
2.72323232323232 0.0340233896467463
2.77272727272727 0.0336659318805086
2.82222222222222 0.0333144366706767
2.87171717171717 0.0329687371515172
2.92121212121212 0.0326287030500804
2.97070707070707 0.0322941718422445
3.02020202020202 0.031965003320626
3.06969696969697 0.0316409732929943
3.11919191919192 0.0313221171929806
3.16868686868687 0.0310082167500632
3.21818181818182 0.0306992840514999
3.26767676767677 0.0303949202750411
3.31717171717172 0.0300951510406448
3.36666666666667 0.0297998667702423
3.41616161616162 0.0295089679624039
3.46565656565657 0.0292223434465693
3.51515151515152 0.0289398883611958
3.56464646464647 0.0286615245336215
3.61414141414142 0.0283871194685729
3.66363636363636 0.0281166408966564
3.71313131313131 0.0278499722561703
3.76262626262626 0.0275869722955947
3.81212121212121 0.0273276641962379
3.86161616161616 0.0270719063559756
3.91111111111111 0.0268196215660471
3.96060606060606 0.0265707552591579
4.01010101010101 0.0263252203639168
4.05959595959596 0.0260829546504732
4.10909090909091 0.025843883002497
4.15858585858586 0.0256079442534964
4.20808080808081 0.0253750534847226
4.25757575757576 0.0251451614609114
4.30707070707071 0.0249182520525683
4.35656565656566 0.0246942142307195
4.40606060606061 0.0244729894821228
4.45555555555556 0.0242545454696714
4.50505050505051 0.0240388170253626
4.55454545454546 0.0238257425468142
4.6040404040404 0.0236153044881831
4.65353535353535 0.0234074043391556
4.7030303030303 0.0232020471692864
4.75252525252525 0.022999131436906
4.8020202020202 0.0227986720427871
4.85151515151515 0.0226005643178365
4.9010101010101 0.0224047698018308
4.95050505050505 0.0222112903839966
5 0.0220200687042905
};
\addplot [semithick, black, dash pattern=on 1pt off 3pt on 3pt off 3pt]
table {%
0.1 0.0789095115208539
0.14949494949495 0.0801392940348223
0.198989898989899 0.0814840530858353
0.248484848484848 0.0829823981268109
0.297979797979798 0.0846840163710515
0.347474747474747 0.0866652622591984
0.396969696969697 0.0890747135601627
0.446464646464647 0.0922306549135427
0.495959595959596 0.0938058576096841
0.545454545454546 0.0904816331028733
0.594949494949495 0.0864465466186951
0.644444444444444 0.0824205115744855
0.693939393939394 0.0785584128964997
0.743434343434343 0.0749139859966081
0.792929292929293 0.0715032180379551
0.842424242424242 0.0683245052100947
0.891919191919192 0.0653677276400082
0.941414141414142 0.0626188625334982
0.990909090909091 0.0600624710625183
1.04040404040404 0.0576830394524619
1.08989898989899 0.0554656788476575
1.13939393939394 0.0533964665184118
1.18888888888889 0.0514625834041257
1.23838383838384 0.0496523377310785
1.28787878787879 0.0479551269348297
1.33737373737374 0.046361368337961
1.38686868686869 0.0448624162583979
1.43636363636364 0.0434504757279831
1.48585858585859 0.0421185182926463
1.53535353535354 0.04086020292443
1.58484848484849 0.0396698033291834
1.63434343434343 0.0385421421148185
1.68383838383838 0.037472531709466
1.73333333333333 0.0364567216545861
1.78282828282828 0.0354908517804671
1.83232323232323 0.0345714106980758
1.88181818181818 0.0336952253957482
1.93131313131313 0.0328593117422964
1.98080808080808 0.0320610386405111
2.03030303030303 0.0312979629886746
2.07979797979798 0.0305678435028636
2.12929292929293 0.0298686217047887
2.17878787878788 0.0291984392216431
2.22828282828283 0.0285554888402899
2.27777777777778 0.0279381876976965
2.32727272727273 0.0273450486627942
2.37676767676768 0.0267746958193283
2.42626262626263 0.0262258544331
2.47575757575758 0.0256973419623336
2.52525252525253 0.0251880599933517
2.57474747474748 0.0246969852448755
2.62424242424242 0.0242231717722006
2.67373737373737 0.0237657276230687
2.72323232323232 0.023323827015173
2.77272727272727 0.0228966967577994
2.82222222222222 0.022483613685757
2.87171717171717 0.0220839007092014
2.92121212121212 0.0216969232351952
2.97070707070707 0.0213220859141197
3.02020202020202 0.0209588296774683
3.06969696969697 0.0206066290373237
3.11919191919192 0.0202649896214318
3.16868686868687 0.0199334459202202
3.21818181818182 0.0196115592253204
3.26767676767677 0.0192989157408004
3.31717171717172 0.0189951248493011
3.36666666666667 0.0186998175269908
3.41616161616162 0.0184126448721603
3.46565656565657 0.0181332767652242
3.51515151515152 0.0178614010020701
3.56464646464647 0.017596720593022
3.61414141414142 0.0173389551311219
3.66363636363636 0.0170878381402982
3.71313131313131 0.0168431166069554
3.76262626262626 0.0166045501443182
3.81212121212121 0.0163719102177202
3.86161616161616 0.0161449794262362
3.91111111111111 0.0159235508354599
3.96060606060606 0.0157074273575797
4.01010101010101 0.0154964211749002
4.05959595959596 0.0152903532030526
4.10909090909091 0.0150890525910508
4.15858585858586 0.0148923562544713
4.20808080808081 0.0147001084431552
4.25757575757576 0.0145121603310194
4.30707070707071 0.0143283696392218
4.35656565656566 0.0141486002805183
4.40606060606061 0.0139727220270806
4.45555555555556 0.0138006106234302
4.50505050505051 0.0136321458259524
4.55454545454546 0.0134672135911116
4.6040404040404 0.0133057042064798
4.65353535353535 0.0131475124455784
4.7030303030303 0.012992537341395
4.75252525252525 0.0128406819736607
4.8020202020202 0.0126918532685155
4.85151515151515 0.0125459618101021
4.9010101010101 0.0124029216629142
4.95050505050505 0.0122626502045632
5 0.0121250679679354
};
\end{axis}

\end{tikzpicture}

%% file: Figures/exp_calibration_p=0.9.tex
\begin{tikzpicture}
\tikzstyle{every node}=[font=\tiny]

\definecolor{color0}{rgb}{1,0.647058823529412,0}
\definecolor{color1}{rgb}{0,1,1}

\begin{axis}[
height=\figheight,
legend cell align={left},
legend style={fill opacity=0.8, draw opacity=1, text opacity=1, draw=white!80!black},
tick align=outside,
tick pos=left,
width=\figwidth,
x grid style={white!69.0196078431373!black},
xlabel={1 / \(\displaystyle \alpha\)},
xmajorgrids,
xmin=-0.145, xmax=5.245,
xtick style={color=black},
y grid style={white!69.0196078431373!black},
ylabel={\(\displaystyle \Delta_{0.9}\)},
ymajorgrids,
ymin=-0.0446293193742793, ymax=0.414572028152282,
ytick style={color=black}
]
\addplot [draw=blue, fill=blue, forget plot, mark=*, only marks, mark size=1pt ]
table{%
x  y
0.594949494949495 0.236155899776161
0.842424242424242 0.22381941497795
1.08989898989899 0.210670463433382
1.33737373737374 0.196983152961955
1.58484848484849 0.18570178606308
1.83232323232323 0.17361072560783
2.07979797979798 0.160791367092618
2.32727272727273 0.154543279629415
2.57474747474748 0.146990103921344
2.82222222222222 0.138265068255302
3.06969696969697 0.133909129963128
3.31717171717172 0.126540054596927
3.56464646464647 0.120385403720744
3.81212121212121 0.117091027337219
4.05959595959596 0.111372446687016
4.30707070707071 0.107788683153055
};
\addplot [draw=red, fill=red, forget plot, mark=*, only marks, mark size=1pt ]
table{%
x  y
0.594949494949495 0.00589267482405464
0.842424242424242 0.00438102716368394
1.08989898989899 0.00335208346411031
1.33737373737374 0.00542351249344919
1.58484848484849 0.00457757583165175
1.83232323232323 -0.00289085318921667
2.07979797979798 -0.000444159823274837
2.32727272727273 0.00120117808835818
2.57474747474748 -0.00155023177737046
2.82222222222222 -0.00141555613977107
3.06969696969697 0.000331331950596647
3.31717171717172 -0.00192223658812596
3.56464646464647 -0.00199929711382485
3.81212121212121 -0.00113798421683498
4.05959595959596 -0.00119436329887546
4.30707070707071 -0.00228077085617073
};
\addplot [draw=color1, fill=color1, forget plot, mark=*, only marks, mark size=1pt ]
table{%
x  y
0.594949494949495 0.236613348257737
0.842424242424242 0.216488690209541
1.08989898989899 0.207540358840293
1.33737373737374 0.186197667953641
1.58484848484849 0.177105023799159
1.83232323232323 0.166742996502564
2.07979797979798 0.157947426783361
2.32727272727273 0.152887899658813
2.57474747474748 0.144590184612094
2.82222222222222 0.131267194267821
3.06969696969697 0.133214249476934
3.31717171717172 0.117510326618445
3.56464646464647 0.121662956043571
3.81212121212121 0.112937193558056
4.05959595959596 0.105810504702891
4.30707070707071 0.100629727343983
};
\addplot [draw=color0, fill=color0, forget plot, mark=*, only marks, mark size=1pt ]
table{%
x  y
0.594949494949495 0.00357552042610032
0.842424242424242 -0.0105295224608911
1.08989898989899 -0.00277390543678924
1.33737373737374 -0.000362857765449309
1.58484848484849 -0.00616505324420991
1.83232323232323 0.00125489824157965
2.07979797979798 0.00770833864374754
2.32727272727273 -0.00210281346952812
2.57474747474748 -0.00320196260910921
2.82222222222222 0.000187183680795622
3.06969696969697 -5.80941879233521e-05
3.31717171717172 0.00685126793835078
3.56464646464647 0.00717252240510224
3.81212121212121 0.00154472515261239
4.05959595959596 0.00531938016002997
4.30707070707071 -0.00635929702056337
};
\addplot [draw=color0, fill=color0, forget plot, mark=*, only marks, mark size=1pt ]
table{%
x  y
0.594949494949495 0.00357552042610032
0.842424242424242 -0.0105295224608911
1.08989898989899 -0.00277390543678924
1.33737373737374 -0.000362857765449309
1.58484848484849 -0.00616505324420991
1.83232323232323 0.00125489824157965
2.07979797979798 0.00770833864374754
2.32727272727273 -0.00210281346952812
2.57474747474748 -0.00320196260910921
2.82222222222222 0.000187183680795622
3.06969696969697 -5.80941879233521e-05
3.31717171717172 0.00685126793835078
3.56464646464647 0.00717252240510224
3.81212121212121 0.00154472515261239
4.05959595959596 0.00531938016002997
4.30707070707071 -0.00635929702056337
};
\addplot [draw=black, fill=black, forget plot, mark=*, only marks, mark size=1pt ]
table{%
x  y
0.594949494949495 0.392265296755795
0.842424242424242 0.38720628793689
1.08989898989899 0.38434487530448
1.33737373737374 0.382852707469771
1.58484848484849 0.380713768677207
1.83232323232323 0.379680684605893
2.07979797979798 0.378496757582467
2.32727272727273 0.378217043106218
2.57474747474748 0.377573465839896
2.82222222222222 0.377757447246546
3.06969696969697 0.376938800270777
3.31717171717172 0.376884019440189
3.56464646464647 0.376849320093031
3.81212121212121 0.376527939753048
4.05959595959596 0.37629613826852
4.30707070707071 0.376044291656546
};
\addplot [draw=red, fill=red, forget plot, mark=*, only marks, mark size=1pt ]
table{%
x  y
0.594949494949495 -0.00706549097561449
0.842424242424242 -0.0107125278431743
1.08989898989899 -0.0131022153097111
1.33737373737374 -0.012160392665118
1.58484848484849 -0.0136697263545038
1.83232323232323 -0.0211133467777563
2.07979797979798 -0.0192868948261937
2.32727272727273 -0.0180939776086747
2.57474747474748 -0.0208843978595559
2.82222222222222 -0.0209869752449616
3.06969696969697 -0.0195803967172653
3.31717171717172 -0.0218001907799363
3.56464646464647 -0.0220086105105294
3.81212121212121 -0.0213529635234592
4.05959595959596 -0.0215133964519588
4.30707070707071 -0.0225908953154278
};
\addplot [draw=blue, fill=blue, forget plot, mark=*, only marks, mark size=1pt ]
table{%
x  y
0.594949494949495 0.197944877461513
0.842424242424242 0.180395233832411
1.08989898989899 0.165654325423645
1.33737373737374 0.150919547662525
1.58484848484849 0.139890853568038
1.83232323232323 0.12790093305678
2.07979797979798 0.115205277241493
2.32727272727273 0.109929679695537
2.57474747474748 0.10289184669855
2.82222222222222 0.0947121681523189
3.06969696969697 0.091267307914208
3.31717171717172 0.0842783137758778
3.56464646464647 0.0787472666121827
3.81212121212121 0.0760969745171238
4.05959595959596 0.0709240036420808
4.30707070707071 0.0678807779604442
};
\addplot [draw=black, fill=black, forget plot, mark=*, only marks, mark size=1pt ]
table{%
x  y
0.594949494949495 0.370751013889031
0.842424242424242 0.249036562102987
1.08989898989899 0.186915882934984
1.33737373737374 0.156418445576124
1.58484848484849 0.125351161330115
1.83232323232323 0.105759012980724
2.07979797979798 0.0871604851946168
2.32727272727273 0.0779893111112268
2.57474747474748 0.0665223084613495
2.82222222222222 0.0671939711986839
3.06969696969697 0.0563407860270199
3.31717171717172 0.0560877853831125
3.56464646464647 0.0567458466406425
3.81212121212121 0.0500326031108794
4.05959595959596 0.0524028877602863
4.30707070707071 0.0418127634508273
};
\addplot [semithick, black]
table {%
0.1 0.298611567177367
0.14949494949495 0.310042108993852
0.198989898989899 0.321631790815396
0.248484848484848 0.333453786247308
0.297979797979798 0.345582072923911
0.347474747474747 0.358123040552921
0.396969696969697 0.371295835682514
0.446464646464647 0.38537782671945
0.495959595959596 0.393699239628348
0.545454545454546 0.393205224982197
0.594949494949495 0.391971610204764
0.644444444444444 0.390772981329521
0.693939393939394 0.389681577061341
0.743434343434343 0.388700300108758
0.792929292929293 0.387819413411609
0.842424242424242 0.387027115512696
0.891919191919192 0.386312210420559
0.941414141414142 0.385664772308677
0.990909090909091 0.385076223547862
1.04040404040404 0.384539228957958
1.08989898989899 0.384047530467078
1.13939393939394 0.383595786620296
1.18888888888889 0.383179427775381
1.23838383838384 0.382794531666663
1.28787878787879 0.382437718890317
1.33737373737374 0.382106066088532
1.38686868686869 0.381797034132992
1.43636363636364 0.381508408785693
1.48585858585859 0.381238251552188
1.53535353535354 0.380984858907827
1.58484848484849 0.380746728342687
1.63434343434343 0.380522529994572
1.68383838383838 0.380311082857036
1.73333333333333 0.380111334747442
1.78282828282828 0.379922345379292
1.83232323232323 0.379743272000701
1.88181818181818 0.379573366557394
1.93131313131313 0.379411923479773
1.98080808080808 0.379258339761204
2.03030303030303 0.379112057050388
2.07979797979798 0.378972568273997
2.12929292929293 0.378839412121069
2.17878787878788 0.378712180154923
2.22828282828283 0.378590466397729
2.27777777777778 0.378473928346018
2.32727272727273 0.378362243387625
2.37676767676768 0.378255115112999
2.42626262626263 0.378152270715814
2.47575757575758 0.378053458694945
2.52525252525253 0.377958446820572
2.57474747474748 0.377867019326206
2.62424242424242 0.377778980302509
2.67373737373737 0.377694142258723
2.72323232323232 0.377612334854655
2.77272727272727 0.37753339873414
2.82222222222222 0.377457185508342
2.87171717171717 0.377383556826959
2.92121212121212 0.377312383544747
2.97070707070707 0.377243544969582
3.02020202020202 0.377176928182923
3.06969696969697 0.377112427424686
3.11919191919192 0.377049943535532
3.16868686868687 0.376989383450422
3.21818181818182 0.37693065973803
3.26767676767677 0.376873690181258
3.31717171717172 0.376818397394237
3.36666666666667 0.376764708474503
3.41616161616162 0.376712554680944
3.46565656565657 0.376661871142837
3.51515151515152 0.376612596681501
3.56464646464647 0.37656467318837
3.61414141414142 0.376518045983393
3.66363636363636 0.376472663202257
3.71313131313131 0.3764284757093
3.76262626262626 0.376385436920513
3.81212121212121 0.376343502640143
3.86161616161616 0.376302630909676
3.91111111111111 0.376262781868126
3.96060606060606 0.376223917622668
4.01010101010101 0.37618600212871
4.05959595959596 0.37614900107863
4.10909090909091 0.376112881798449
4.15858585858586 0.37607761315166
4.20808080808081 0.376043165450401
4.25757575757576 0.37600951037146
4.30707070707071 0.375976620879294
4.35656565656566 0.375944471153579
4.40606060606061 0.375913036521728
4.45555555555556 0.375882293496636
4.50505050505051 0.375852219320931
4.55454545454546 0.375822792497226
4.6040404040404 0.375793992353002
4.65353535353535 0.37576579908638
4.7030303030303 0.375738193720777
4.75252525252525 0.375711158062375
4.8020202020202 0.37568467466018
4.85151515151515 0.375658726768512
4.9010101010101 0.375633298311741
4.95050505050505 0.375608373851112
5 0.375583938553529
};
\addplot [semithick, color0]
table {%
0.1 0.00111027419907272
0.14949494949495 0.000865368212272899
0.198989898989899 0.000597883596756721
0.248484848484848 0.000323891287253408
0.297979797979798 5.82153002683938e-05
0.347474747474747 -0.000186132218288115
0.396969696969697 -0.000398343006858926
0.446464646464647 -0.000570416715402033
0.495959595959596 -0.000699805319294811
0.545454545454546 -0.000788003752359434
0.594949494949495 -0.000840761402906232
0.644444444444444 -0.000864860765120157
0.693939393939394 -0.000867028211024801
0.743434343434343 -0.000853905457474768
0.792929292929293 -0.000830011446962375
0.842424242424242 -0.00079903780978996
0.891919191919192 -0.000763674110134893
0.941414141414142 -0.000725802102348738
0.990909090909091 -0.000686695589690589
1.04040404040404 -0.000647321238465559
1.08989898989899 -0.000608173026071746
1.13939393939394 -0.000570011543304827
1.18888888888889 -0.000532922449318018
1.23838383838384 -0.000496415967651931
1.28787878787879 -0.000461479982668878
1.33737373737374 -0.000427836521882607
1.38686868686869 -0.000395513799918201
1.43636363636364 -0.000364665564834832
1.48585858585859 -0.000334484184254524
1.53535353535354 -0.000306209340851793
1.58484848484849 -0.000278862478263475
1.63434343434343 -0.000252881719726705
1.68383838383838 -0.000227274081652018
1.73333333333333 -0.000203823900276401
1.78282828282828 -0.000180741486886959
1.83232323232323 -0.000158580044495094
1.88181818181818 -0.000137264733967024
1.93131313131313 -0.000116750636720786
1.98080808080808 -9.69881476143852e-05
2.03030303030303 -7.82653217840101e-05
2.07979797979798 -5.99724911091171e-05
2.12929292929293 -4.27920711121299e-05
2.17878787878788 -2.58393620440334e-05
2.22828282828283 -9.69410899998024e-06
2.27777777777778 5.85546498144129e-06
2.32727272727273 2.08365394775978e-05
2.37676767676768 3.52745328325854e-05
2.42626262626263 4.91931986984495e-05
2.47575757575758 6.29540532908202e-05
2.52525252525253 7.61224998363508e-05
2.57474747474748 8.87646451248614e-05
2.62424242424242 0.000100782449106496
2.67373737373737 0.000112586308211204
2.72323232323232 0.000124013349307228
2.77272727272727 0.000135079908842317
2.82222222222222 0.000145802645655446
2.87171717171717 0.00015619659665056
2.92121212121212 0.000166276379957697
2.97070707070707 0.000176055504047423
3.02020202020202 0.000185546691464866
3.06969696969697 0.000194762360091749
3.11919191919192 0.000203713846660847
3.16868686868687 0.000212412235469195
3.21818181818182 0.00022086773198271
3.26767676767677 0.000229090145487443
3.31717171717172 0.000237088720237755
3.36666666666667 0.000244872474932434
3.41616161616162 0.000252450104598512
3.46565656565657 0.000259828183892807
3.51515151515152 0.000266912733761115
3.56464646464647 0.000273918049061295
3.61414141414142 0.000280746399326959
3.66363636363636 0.00028740430685914
3.71313131313131 0.000293897923453024
3.76262626262626 0.00030023324770001
3.81212121212121 0.000306415837360463
3.86161616161616 0.000312451088327337
3.91111111111111 0.000318344106756463
3.96060606060606 0.000324099865380845
4.01010101010101 0.000329722939194421
4.05959595959596 0.000335217880037209
4.10909090909091 0.000340588952769449
4.15858585858586 0.000345840233075601
4.20808080808081 0.000350975614273219
4.25757575757576 0.00035599887179405
4.30707070707071 0.000360913613645875
4.35656565656566 0.000365803131990416
4.40606060606061 0.000370551554959575
4.45555555555556 0.000375161616819719
4.50505050505051 0.000379676055579203
4.55454545454546 0.000384097833389552
4.6040404040404 0.000388429772965049
4.65353535353535 0.000392674406945792
4.7030303030303 0.000396709726354083
4.75252525252525 0.000400786857111868
4.8020202020202 0.000404784267840408
4.85151515151515 0.000408704191407816
4.9010101010101 0.000412549115840766
4.95050505050505 0.00041632117304613
5 0.000419928420740789
};
\addplot [semithick, blue]
table {%
0.1 0.229771989133188
0.14949494949495 0.234291614503589
0.198989898989899 0.238022152130586
0.248484848484848 0.24085904410343
0.297979797979798 0.242738796346448
0.347474747474747 0.243656454114854
0.396969696969697 0.243672243145837
0.446464646464647 0.242901887296691
0.495959595959596 0.241484487366531
0.545454545454546 0.239562276774936
0.594949494949495 0.23726344827944
0.644444444444444 0.234694855934365
0.693939393939394 0.231941086840849
0.743434343434343 0.229067101790675
0.792929292929293 0.22612167054146
0.842424242424242 0.223141064872365
0.891919191919192 0.220151956413929
0.941414141414142 0.217173872950352
0.990909090909091 0.214221028740265
1.04040404040404 0.21130370818607
1.08989898989899 0.208470679462762
1.13939393939394 0.205602118533262
1.18888888888889 0.202827003846432
1.23838383838384 0.20010550597406
1.28787878787879 0.197438939370737
1.33737373737374 0.194827953812925
1.38686868686869 0.192272638385355
1.43636363636364 0.18977268242323
1.48585858585859 0.187327529919573
1.53535353535354 0.184937411345998
1.58484848484849 0.182599127096031
1.63434343434343 0.180312786202527
1.68383838383838 0.17807662370042
1.73333333333333 0.175890778379013
1.78282828282828 0.173753290078853
1.83232323232323 0.171662901733453
1.88181818181818 0.169618339606178
1.93131313131313 0.167618363603742
1.98080808080808 0.165661739851364
2.03030303030303 0.163747338857382
2.07979797979798 0.161873815822289
2.12929292929293 0.160040061710678
2.17878787878788 0.158245011383867
2.22828282828283 0.156487493874462
2.27777777777778 0.154766460187005
2.32727272727273 0.153080862834031
2.37676767676768 0.151429700385266
2.42626262626263 0.149811980858045
2.47575757575758 0.1482267937647
2.52525252525253 0.146673081982531
2.57474747474748 0.145150205544611
2.62424242424242 0.143657030187635
2.67373737373737 0.142192948712361
2.72323232323232 0.140757020186953
2.77272727272727 0.139348604140828
2.82222222222222 0.137966817009654
2.87171717171717 0.136610918945588
2.92121212121212 0.13528032307981
2.97070707070707 0.133974300923705
3.02020202020202 0.13269221726964
3.06969696969697 0.131433080456121
3.11919191919192 0.130196995448
3.16868686868687 0.128982992624134
3.21818181818182 0.127791074502724
3.26767676767677 0.126619517692699
3.31717171717172 0.125468376687597
3.36666666666667 0.124337141882902
3.41616161616162 0.123225345883019
3.46565656565657 0.12213247386591
3.51515151515152 0.121058038336045
3.56464646464647 0.120001662169331
3.61414141414142 0.11896274757524
3.66363636363636 0.117941106595062
3.71313131313131 0.116936208983559
3.76262626262626 0.115947427996517
3.81212121212121 0.114974801816768
3.86161616161616 0.11401770502313
3.91111111111111 0.113075773963766
3.96060606060606 0.112148736731569
4.01010101010101 0.111236193062885
4.05959595959596 0.110337843318343
4.10909090909091 0.10945333825088
4.15858585858586 0.108582385431744
4.20808080808081 0.107724600607256
4.25757575757576 0.106879741639968
4.30707070707071 0.106047696543103
4.35656565656566 0.105227983701414
4.40606060606061 0.104420329055335
4.45555555555556 0.103624561391798
4.50505050505051 0.102840383285382
4.55454545454546 0.102067512948209
4.6040404040404 0.10130583952547
4.65353535353535 0.100554942172342
4.7030303030303 0.0998147989702597
4.75252525252525 0.0990849810805455
4.8020202020202 0.0983655055831092
4.85151515151515 0.097655935270755
4.9010101010101 0.0969560868177404
4.95050505050505 0.0962659301151745
5 0.09558521167413
};
\addplot [semithick, red]
table {%
0.1 -0.00345168687104846
0.14949494949495 -0.0016358919143834
0.198989898989899 -0.000184242354408726
0.248484848484848 0.000952008389216941
0.297979797979798 0.00182132265209722
0.347474747474747 0.00246940341077262
0.396969696969697 0.00293735504003401
0.446464646464647 0.0032606990223949
0.495959595959596 0.00346942202026268
0.545454545454546 0.00358777598825721
0.594949494949495 0.00363561195684292
0.644444444444444 0.00362864450814171
0.693939393939394 0.00357938351574116
0.743434343434343 0.00349773077301829
0.792929292929293 0.00339156379736949
0.842424242424242 0.00326707276740901
0.891919191919192 0.00312919615774676
0.941414141414142 0.0029818566093549
0.990909090909091 0.0028280863227591
1.04040404040404 0.00267029601878854
1.08989898989899 0.00251043480177604
1.13939393939394 0.0023498980421045
1.18888888888889 0.0021898823626455
1.23838383838384 0.00203123620711032
1.28787878787879 0.00187470009032509
1.33737373737374 0.0017206628340749
1.38686868686869 0.00156964806895687
1.43636363636364 0.00142178329828568
1.48585858585859 0.00127732668424219
1.53535353535354 0.0011364052145354
1.58484848484849 0.000999066240862656
1.63434343434343 0.000865335939607914
1.68383838383838 0.000735197562655499
1.73333333333333 0.000608645768258809
1.78282828282828 0.00048561002445735
1.83232323232323 0.000366041001014228
1.88181818181818 0.000249862602335615
1.93131313131313 0.000137005434014847
1.98080808080808 2.7349775936103e-05
2.03030303030303 -7.91584412634938e-05
2.07979797979798 -0.000182672798797023
2.12929292929293 -0.000283211675651707
2.17878787878788 -0.000380889820354668
2.22828282828283 -0.000475867040797051
2.27777777777778 -0.00056818680121018
2.32727272727273 -0.000657900978230153
2.37676767676768 -0.000745145729082908
2.42626262626263 -0.000829999496687495
2.47575757575758 -0.000912493761111244
2.52525252525253 -0.000992808116234478
2.57474747474748 -0.00107100193138809
2.62424242424242 -0.00114707975734374
2.67373737373737 -0.00122113811261015
2.72323232323232 -0.00129328293434783
2.77272727272727 -0.00136354663452587
2.82222222222222 -0.00143203116394675
2.87171717171717 -0.00149878104890244
2.92121212121212 -0.00156379887441149
2.97070707070707 -0.00162723798897202
3.02020202020202 -0.00168911854727949
3.06969696969697 -0.00174945923437775
3.11919191919192 -0.00180838478380085
3.16868686868687 -0.00186588303179758
3.21818181818182 -0.00192199469152388
3.26767676767677 -0.00197679792705796
3.31717171717172 -0.00203032340124787
3.36666666666667 -0.00208265173593414
3.41616161616162 -0.00213374894411655
3.46565656565657 -0.00218371201120948
3.51515151515152 -0.00223253084092401
3.56464646464647 -0.00228027912352491
3.61414141414142 -0.00232694316876836
3.66363636363636 -0.002372669365999
3.71313131313131 -0.00241737206158299
3.76262626262626 -0.00246112967776468
3.81212121212121 -0.00250396718151602
3.86161616161616 -0.0025458546286613
3.91111111111111 -0.00258694975065676
3.96060606060606 -0.00262713521189772
4.01010101010101 -0.00266658878206494
4.05959595959596 -0.00270519786959122
4.10909090909091 -0.00274296593373013
4.15858585858586 -0.00278008377159034
4.20808080808081 -0.00281646074464581
4.25757575757576 -0.00285207417396605
4.30707070707071 -0.00288705106285314
4.35656565656566 -0.00292137070235421
4.40606060606061 -0.00295500128113468
4.45555555555556 -0.00298800778358754
4.50505050505051 -0.00302040673476667
4.55454545454546 -0.00305220184555854
4.6040404040404 -0.00308343035016057
4.65353535353535 -0.00311409553771025
4.7030303030303 -0.00314419476264882
4.75252525252525 -0.00317374859770181
4.8020202020202 -0.00320277053838147
4.85151515151515 -0.00323133648978624
4.9010101010101 -0.00325934540961437
4.95050505050505 -0.00328691490854605
5 -0.00331399869479021
};
\addplot [semithick, color1, dashed]
table {%
0.1 0.231364274602517
0.14949494949495 0.235962994043346
0.198989898989899 0.239360681782523
0.248484848484848 0.241689820539026
0.297979797979798 0.242899851954818
0.347474747474747 0.243101295739452
0.396969696969697 0.242423693063023
0.446464646464647 0.241007559857644
0.495959595959596 0.239017663917555
0.545454545454546 0.236604544119012
0.594949494949495 0.23389743364446
0.644444444444444 0.230997678265585
0.693939393939394 0.227982238082879
0.743434343434343 0.224906521766857
0.792929292929293 0.221812391033333
0.842424242424242 0.218725949563172
0.891919191919192 0.215680579020326
0.941414141414142 0.212661531095775
0.990909090909091 0.20969224933328
1.04040404040404 0.206779365528594
1.08989898989899 0.20392767949993
1.13939393939394 0.201139629979912
1.18888888888889 0.198412351495055
1.23838383838384 0.195749726589948
1.28787878787879 0.193148608864727
1.33737373737374 0.190609570867722
1.38686868686869 0.188129441227311
1.43636363636364 0.185709736878902
1.48585858585859 0.183345439215446
1.53535353535354 0.181037463219552
1.58484848484849 0.178784080034854
1.63434343434343 0.176583717420885
1.68383838383838 0.174434754190642
1.73333333333333 0.172335166343402
1.78282828282828 0.170284146581385
1.83232323232323 0.16828042469148
1.88181818181818 0.166321375560329
1.93131313131313 0.164406147438751
1.98080808080808 0.162533361135999
2.03030303030303 0.160701684203448
2.07979797979798 0.15890922691637
2.12929292929293 0.157156931983141
2.17878787878788 0.155440627476069
2.22828282828283 0.153760958137781
2.27777777777778 0.152116861165061
2.32727272727273 0.150505402899542
2.37676767676768 0.148928798429106
2.42626262626263 0.14738297865797
2.47575757575758 0.145868007933301
2.52525252525253 0.144385446089375
2.57474747474748 0.142928750688696
2.62424242424242 0.141502979095483
2.67373737373737 0.140104345903328
2.72323232323232 0.138709332597478
2.77272727272727 0.137363398764923
2.82222222222222 0.136043314103589
2.87171717171717 0.134748439789156
2.92121212121212 0.133475529829794
2.97070707070707 0.132227325634974
3.02020202020202 0.131002726216041
3.06969696969697 0.12979957568952
3.11919191919192 0.12861785407727
3.16868686868687 0.127457092016573
3.21818181818182 0.126316742774816
3.26767676767677 0.125196278358886
3.31717171717172 0.124095188866078
3.36666666666667 0.123012980593981
3.41616161616162 0.121949176607871
3.46565656565657 0.120903330792246
3.51515151515152 0.119874957898542
3.56464646464647 0.118863649366478
3.61414141414142 0.117868988271453
3.66363636363636 0.116890570342884
3.71313131313131 0.115928005167327
3.76262626262626 0.114980914149709
3.81212121212121 0.114048913108516
3.86161616161616 0.113130504794641
3.91111111111111 0.112227713676308
3.96060606060606 0.111338996019017
4.01010101010101 0.110464028196993
4.05959595959596 0.109602496608753
4.10909090909091 0.108754096906064
4.15858585858586 0.107919006261365
4.20808080808081 0.107095912576048
4.25757575757576 0.106285098269122
4.30707070707071 0.105486294009937
4.35656565656566 0.104699238218586
4.40606060606061 0.103923676403432
4.45555555555556 0.103159363746268
4.50505050505051 0.102406056220448
4.55454545454546 0.10166387359513
4.6040404040404 0.100930732354336
4.65353535353535 0.100209752942176
4.7030303030303 0.0994990130090556
4.75252525252525 0.098796664939644
4.8020202020202 0.0981050463427471
4.85151515151515 0.0974227558597444
4.9010101010101 0.0967497011875137
4.95050505050505 0.0960857125343612
5 0.0954306269073667
};
\addplot [semithick, red, dash pattern=on 1pt off 3pt on 3pt off 3pt]
table {%
0.1 -0.00669864638167428
0.14949494949495 -0.00632691312079525
0.198989898989899 -0.00620069040814708
0.248484848484848 -0.00627066351806693
0.297979797979798 -0.00649222591639509
0.347474747474747 -0.00682662097885511
0.396969696969697 -0.00724152949571943
0.446464646464647 -0.00771099383280816
0.495959595959596 -0.00821457956650162
0.545454545454546 -0.00873702704895851
0.594949494949495 -0.00926666995659275
0.644444444444444 -0.00979508045841426
0.693939393939394 -0.0103161903621627
0.743434343434343 -0.0108257693135614
0.792929292929293 -0.0113209321633946
0.842424242424242 -0.011799868368054
0.891919191919192 -0.0122614891146597
0.941414141414142 -0.012705251053353
0.990909090909091 -0.0131310759340012
1.04040404040404 -0.0135391472602434
1.08989898989899 -0.0139298010429053
1.13939393939394 -0.0143036295554468
1.18888888888889 -0.0146612024022262
1.23838383838384 -0.0150032151723221
1.28787878787879 -0.0153303050925991
1.33737373737374 -0.0156432739133393
1.38686868686869 -0.0159426929765902
1.43636363636364 -0.0162293635763169
1.48585858585859 -0.0165038857700844
1.53535353535354 -0.0167668889406929
1.58484848484849 -0.0170189974028612
1.63434343434343 -0.0172607895042606
1.68383838383838 -0.0174928228285735
1.73333333333333 -0.0177155951528849
1.78282828282828 -0.0179296135821164
1.83232323232323 -0.0181353274847428
1.88181818181818 -0.0183331725598748
1.93131313131313 -0.0185235467907751
1.98080808080808 -0.0187068608735961
2.03030303030303 -0.018883445236117
2.07979797979798 -0.0190536890235858
2.12929292929293 -0.0192178487063104
2.17878787878788 -0.0193762411504617
2.22828282828283 -0.0195292028838715
2.27777777777778 -0.0196769589750247
2.32727272727273 -0.0198197272558223
2.37676767676768 -0.0199577817918893
2.42626262626263 -0.0200913373920423
2.47575757575758 -0.0202205596179839
2.52525252525253 -0.0203457255050868
2.57474747474748 -0.0204670045007903
2.62424242424242 -0.0205845130058805
2.67373737373737 -0.0206984351015813
2.72323232323232 -0.0208089555297986
2.77272727272727 -0.020916192544419
2.82222222222222 -0.0210203152631557
2.87171717171717 -0.0211214404734655
2.92121212121212 -0.0212196457863588
2.97070707070707 -0.0213151270216765
3.02020202020202 -0.0214079671245399
3.06969696969697 -0.0214982439093412
3.11919191919192 -0.0215861177779213
3.16868686868687 -0.0216716343132115
3.21818181818182 -0.0217548789543173
3.26767676767677 -0.0218359645930386
3.31717171717172 -0.02191496274799
3.36666666666667 -0.0219919829186714
3.41616161616162 -0.0220670391813852
3.46565656565657 -0.0221402496629044
3.51515151515152 -0.0222116435209501
3.56464646464647 -0.0222813158989646
3.61414141414142 -0.0223492892785632
3.66363636363636 -0.0224157140993613
3.71313131313131 -0.0224805512275397
3.76262626262626 -0.0225438929161527
3.81212121212121 -0.0226057867020268
3.86161616161616 -0.0226662343815444
3.91111111111111 -0.022725388286842
3.96060606060606 -0.0227831771154708
4.01010101010101 -0.0228397670421832
4.05959595959596 -0.0228950884952044
4.10909090909091 -0.0229491647577946
4.15858585858586 -0.0230021694489826
4.20808080808081 -0.0230540479952233
4.25757575757576 -0.0231047999259136
4.30707070707071 -0.0231545447772561
4.35656565656566 -0.0232032821139018
4.40606060606061 -0.0232510018744428
4.45555555555556 -0.0232977714047474
4.50505050505051 -0.0233436183234472
4.55454545454546 -0.0233885593238768
4.6040404040404 -0.0234326375482102
4.65353535353535 -0.0234758682413366
4.7030303030303 -0.0235182613644415
4.75252525252525 -0.0235598450976181
4.8020202020202 -0.0236006414563696
4.85151515151515 -0.0236407222457364
4.9010101010101 -0.0236800164252665
4.95050505050505 -0.0237186285400274
5 -0.0237565308503447
};
\addplot [semithick, blue, dash pattern=on 1pt off 3pt on 3pt off 3pt]
table {%
0.1 0.221555536550185
0.14949494949495 0.222046834604675
0.198989898989899 0.221853654476409
0.248484848484848 0.220934297539702
0.297979797979798 0.219286066477992
0.347474747474747 0.216955453905694
0.396969696969697 0.214037829879336
0.446464646464647 0.210664610260675
0.495959595959596 0.206971147914451
0.545454545454546 0.203081407804476
0.594949494949495 0.19909641356889
0.644444444444444 0.19509199557564
0.693939393939394 0.191121427834904
0.743434343434343 0.187220164245231
0.792929292929293 0.183410336036607
0.842424242424242 0.179704780974709
0.891919191919192 0.176109912653813
0.941414141414142 0.17262792176151
0.990909090909091 0.169258274322104
1.04040404040404 0.165998744864646
1.08989898989899 0.162881516477815
1.13939393939394 0.159795847712551
1.18888888888889 0.156844960825804
1.23838383838384 0.153988633143427
1.28787878787879 0.151222662840611
1.33737373737374 0.148542982021462
1.38686868686869 0.145945642727355
1.43636363636364 0.143426866846366
1.48585858585859 0.140983103544649
1.53535353535354 0.138611961287208
1.58484848484849 0.136308080523219
1.63434343434343 0.134069589327119
1.68383838383838 0.131893122050922
1.73333333333333 0.129777128117565
1.78282828282828 0.127718462810862
1.83232323232323 0.125714732288721
1.88181818181818 0.123763665509031
1.93131313131313 0.121863139795312
1.98080808080808 0.120011143216047
2.03030303030303 0.11820584474989
2.07979797979798 0.116445314434966
2.12929292929293 0.114727893768039
2.17878787878788 0.113052025948236
2.22828282828283 0.111416123019833
2.27777777777778 0.109818751242579
2.32727272727273 0.108258524418854
2.37676767676768 0.106734136585746
2.42626262626263 0.105244327116019
2.47575757575758 0.103787937475317
2.52525252525253 0.10236371670671
2.57474747474748 0.100970791117143
2.62424242424242 0.099607897860398
2.67373737373737 0.0982742384300497
2.72323232323232 0.0969687555599414
2.77272727272727 0.0956906624345842
2.82222222222222 0.0944389861414008
2.87171717171717 0.0932128899257729
2.92121212121212 0.0920116794308171
2.97070707070707 0.0908345552481624
3.02020202020202 0.0896808082223399
3.06969696969697 0.0885494429792267
3.11919191919192 0.0874403900628521
3.16868686868687 0.086352692518365
3.21818181818182 0.08528620866055
3.26767676767677 0.0842393754608372
3.31717171717172 0.0832121156651996
3.36666666666667 0.0822038883172846
3.41616161616162 0.0812141942460223
3.46565656565657 0.0802425009432244
3.51515151515152 0.0792883044374012
3.56464646464647 0.0783511979838832
3.61414141414142 0.0774305962319775
3.66363636363636 0.0765262594719132
3.71313131313131 0.0756376676083791
3.76262626262626 0.0747642245801902
3.81212121212121 0.0739058908248509
3.86161616161616 0.0730620788205339
3.91111111111111 0.0722324220366882
3.96060606060606 0.0714166332913689
4.01010101010101 0.0706143218123384
4.05959595959596 0.069825183189866
4.10909090909091 0.0690488745039118
4.15858585858586 0.0682851024835536
4.20808080808081 0.0675335000913457
4.25757575757576 0.0667938203760221
4.30707070707071 0.0660659262861499
4.35656565656566 0.0653493764565992
4.40606060606061 0.0646439034647062
4.45555555555556 0.0639493268259964
4.50505050505051 0.0632653632151472
4.55454545454546 0.0625917438912182
4.6040404040404 0.0619283430715596
4.65353535353535 0.0612747799406149
4.7030303030303 0.0606310050591777
4.75252525252525 0.0599966337526633
4.8020202020202 0.0593716512811274
4.85151515151515 0.058755668550017
4.9010101010101 0.0581485074876341
4.95050505050505 0.0575501176122576
5 0.0569602648954953
};
\addplot [semithick, black, dash pattern=on 1pt off 3pt on 3pt off 3pt]
table {%
0.1 0.289837381260699
0.14949494949495 0.297035888348364
0.198989898989899 0.304618911996061
0.248484848484848 0.31280839487807
0.297979797979798 0.321878039988887
0.347474747474747 0.332232459218274
0.396969696969697 0.34463689956085
0.446464646464647 0.360705886601408
0.495959595959596 0.368696776172665
0.545454545454546 0.351974587008047
0.594949494949495 0.331798622088165
0.644444444444444 0.311856021862942
0.693939393939394 0.292951654881518
0.743434343434343 0.275357206481684
0.792929292929293 0.259138673326601
0.842424242424242 0.244264760821328
0.891919191919192 0.230657606736733
0.941414141414142 0.218219297201871
0.990909090909091 0.206846314569639
1.04040404040404 0.196437276127582
1.08989898989899 0.186896845478011
1.13939393939394 0.178137478466299
1.18888888888889 0.170079945459098
1.23838383838384 0.162653183828456
1.28787878787879 0.155793805060168
1.33737373737374 0.149445445301724
1.38686868686869 0.143558066997448
1.43636363636364 0.138087271539235
1.48585858585859 0.132993653412832
1.53535353535354 0.128242210482989
1.58484848484849 0.123801814737356
1.63434343434343 0.119644742816311
1.68383838383838 0.115746262617721
1.73333333333333 0.112084271017417
1.78282828282828 0.108638977378934
1.83232323232323 0.105392627501699
1.88181818181818 0.102329354656835
1.93131313131313 0.0994345625135913
1.98080808080808 0.0966954155822144
2.03030303030303 0.0941001911095627
2.07979797979798 0.0916382676991888
2.12929292929293 0.0893000068297645
2.17878787878788 0.0870767618785558
2.22828282828283 0.0849603434948016
2.27777777777778 0.0829435587361843
2.32727272727273 0.081019752544635
2.37676767676768 0.0791828326194551
2.42626262626263 0.0774272128435552
2.47575757575758 0.0757477631961931
2.52525252525253 0.0741397653388918
2.57474747474748 0.0725988676851554
2.62424242424242 0.0711210775877609
2.67373737373737 0.0697026753837307
2.72323232323232 0.0683402411411171
2.77272727272727 0.0670306021745199
2.82222222222222 0.0657708161939633
2.87171717171717 0.0645581511520663
2.92121212121212 0.0633900671520107
2.97070707070707 0.0622642001573307
3.02020202020202 0.0611783473015074
3.06969696969697 0.0601304536201844
3.11919191919192 0.0591186000512026
3.16868686868687 0.0581409925648928
3.21818181818182 0.057195952305285
3.26767676767677 0.0562819066351853
3.31717171717172 0.0553973809875543
3.36666666666667 0.0545409914653737
3.41616161616162 0.0537114380524334
3.46565656565657 0.0529074984558868
3.51515151515152 0.0521280235258029
3.56464646464647 0.0513719276476137
3.61414141414142 0.0506381909283892
3.66363636363636 0.049925850065762
3.71313131313131 0.0492339956025439
3.76262626262626 0.0485617682611823
3.81212121212121 0.0479083555696292
3.86161616161616 0.0472729887536322
3.91111111111111 0.0466549398700796
3.96060606060606 0.0460535191610023
4.01010101010101 0.045468072608863
4.05959595959596 0.0448979796749339
4.10909090909091 0.0443426512059201
4.15858585858586 0.0438015274923169
4.20808080808081 0.0432740764766655
4.25757575757576 0.0427597920725261
4.30707070707071 0.0422581926205209
4.35656565656566 0.0417688194439021
4.40606060606061 0.0412912355061689
4.45555555555556 0.0408250253026637
4.50505050505051 0.0403697891985786
4.55454545454546 0.0399251489788595
4.6040404040404 0.0394907424838619
4.65353535353535 0.0390662237299169
4.7030303030303 0.0386512620231072
4.75252525252525 0.0382455411308064
4.8020202020202 0.0378487585055816
4.85151515151515 0.0374606245586506
4.9010101010101 0.0370808619783171
4.95050505050505 0.0367092050911784
5 0.0363453992621762
};
\end{axis}

\end{tikzpicture}

%% file: Figures/exp_calibration_p=0.95.tex
\begin{tikzpicture}
\tikzstyle{every node}=[font=\tiny]

\definecolor{color0}{rgb}{1,0.647058823529412,0}
\definecolor{color1}{rgb}{0,1,1}

\begin{axis}[
height=\figheight,
legend cell align={left},
legend style={
  fill opacity=0.8,
  draw opacity=1,
  text opacity=1,
  at={(0.91,0.5)},
  anchor=east,
  draw=white!80!black
},
tick align=outside,
tick pos=left,
width=\figwidth,
x grid style={white!69.0196078431373!black},
xlabel={1 / \(\displaystyle \alpha\)},
xmajorgrids,
xmin=-0.145, xmax=5.245,
xtick style={color=black},
y grid style={white!69.0196078431373!black},
ylabel={\(\displaystyle \Delta_{0.95}\)},
ymajorgrids,
ymin=-0.0396562379136329, ymax=0.464471747568322,
ytick style={color=black}
]
\addplot [draw=blue, fill=blue, forget plot, mark=*, only marks, mark size=1pt ]
table{%
x  y
0.594949494949495 0.235963652171894
0.842424242424242 0.220788393721892
1.08989898989899 0.204839820404281
1.33737373737374 0.188506141011637
1.58484848484849 0.175261375047419
1.83232323232323 0.161306957756463
2.07979797979798 0.146794354783304
2.32727272727273 0.139825410025908
2.57474747474748 0.131507145128125
2.82222222222222 0.122041450523026
3.06969696969697 0.117370658633361
3.31717171717172 0.10956665536473
3.56464646464647 0.103138857558911
3.81212121212121 0.0997255486707761
4.05959595959596 0.0938726774102095
4.30707070707071 0.0902434080952089
};
\addplot [draw=red, fill=red, forget plot, mark=*, only marks, mark size=1pt ]
table{%
x  y
0.594949494949495 0.00113492337551035
0.842424242424242 9.71975054422369e-05
1.08989898989899 -0.000599379991077842
1.33737373737374 0.000882313561618275
1.58484848484849 0.000287803834885358
1.83232323232323 -0.00489501172284124
2.07979797979798 -0.00320435249632423
2.32727272727273 -0.00205796678429393
2.57474747474748 -0.00396312692952339
2.82222222222222 -0.00386567496966916
3.06969696969697 -0.00265031052158593
3.31717171717172 -0.00420979314366221
3.56464646464647 -0.0042540446563325
3.81212121212121 -0.00365873475969258
4.05959595959596 -0.00369823751810483
4.30707070707071 -0.00444132689070897
};
\addplot [draw=color1, fill=color1, forget plot, mark=*, only marks, mark size=1pt ]
table{%
x  y
0.594949494949495 0.241781934849432
0.842424242424242 0.217747407127547
1.08989898989899 0.207094596745187
1.33737373737374 0.182076852318609
1.58484848484849 0.171523377091069
1.83232323232323 0.159664008473385
2.07979797979798 0.149714807724988
2.32727272727273 0.144001079770788
2.57474747474748 0.134815115036687
2.82222222222222 0.120404605761078
3.06969696969697 0.122375641989062
3.31717171717172 0.105803211751311
3.56464646464647 0.110030791582464
3.81212121212121 0.100936765692129
4.05959595959596 0.0936190488754578
4.30707070707071 0.0883542743882443
};
\addplot [draw=color0, fill=color0, forget plot, mark=*, only marks, mark size=1pt ]
table{%
x  y
0.594949494949495 0.000874007269778376
0.842424242424242 -0.00864107568149031
1.08989898989899 -0.00318083847537354
1.33737373737374 -0.00135210137961062
1.58484848484849 -0.00529411501568444
1.83232323232323 -2.25405043636773e-05
2.07979797979798 0.00471868395495778
2.32727272727273 -0.0022975257536304
2.57474747474748 -0.00304473045786291
2.82222222222222 -0.000599859337859998
3.06969696969697 -0.00074673182856777
3.31717171717172 0.00426402065189369
3.56464646464647 0.00451935790819014
3.81212121212121 0.000440716147581455
4.05959595959596 0.00319066326537143
4.30707070707071 -0.00508630323999049
};
\addplot [draw=color0, fill=color0, forget plot, mark=*, only marks, mark size=1pt ]
table{%
x  y
0.594949494949495 0.000874007269778376
0.842424242424242 -0.00864107568149031
1.08989898989899 -0.00318083847537354
1.33737373737374 -0.00135210137961062
1.58484848484849 -0.00529411501568444
1.83232323232323 -2.25405043636773e-05
2.07979797979798 0.00471868395495778
2.32727272727273 -0.0022975257536304
2.57474747474748 -0.00304473045786291
2.82222222222222 -0.000599859337859998
3.06969696969697 -0.00074673182856777
3.31717171717172 0.00426402065189369
3.56464646464647 0.00451935790819014
3.81212121212121 0.000440716147581455
4.05959595959596 0.00319066326537143
4.30707070707071 -0.00508630323999049
};
\addplot [draw=black, fill=black, forget plot, mark=*, only marks, mark size=1pt ]
table{%
x  y
0.594949494949495 0.439635517346646
0.842424242424242 0.432858116668963
1.08989898989899 0.429025794705424
1.33737373737374 0.427027688089756
1.58484848484849 0.424164027211193
1.83232323232323 0.422781142143613
2.07979797979798 0.421196542959333
2.32727272727273 0.420822195842775
2.57474747474748 0.419960942208991
2.82222222222222 0.420207141983184
3.06969696969697 0.419111678422116
3.31717171717172 0.41903837462679
3.56464646464647 0.418991946501823
3.81212121212121 0.418561931016359
4.05959595959596 0.418251785659528
4.30707070707071 0.417914831539253
};
\addplot [draw=red, fill=red, forget plot, mark=*, only marks, mark size=1pt ]
table{%
x  y
0.594949494949495 -0.00661910590858072
0.842424242424242 -0.00885542192190747
1.08989898989899 -0.010300235663951
1.33737373737374 -0.00958731867211238
1.58484848484849 -0.0105254884291578
1.83232323232323 -0.0152762625318941
2.07979797979798 -0.0140773782052189
2.32727272727273 -0.013286172857501
2.57474747474748 -0.015049349077601
2.82222222222222 -0.0150942661964698
3.06969696969697 -0.0141791563500029
3.31717171717172 -0.0155803236845465
3.56464646464647 -0.0156942327014379
3.81212121212121 -0.015268879892248
4.05959595959596 -0.0153635414057406
4.30707070707071 -0.016033144951344
};
\addplot [draw=blue, fill=blue, forget plot, mark=*, only marks, mark size=1pt ]
table{%
x  y
0.594949494949495 0.194938908957872
0.842424242424242 0.17483654183697
1.08989898989899 0.158068935017193
1.33737373737374 0.141629266666496
1.58484848484849 0.129492208217261
1.83232323232323 0.116587351095151
2.07979797979798 0.103242436944406
2.32727272727273 0.0977154370000014
2.57474747474748 0.0905119037622166
2.82222222222222 0.0822893179441471
3.06969696969697 0.0788141507094366
3.31717171717172 0.0719810448331506
3.56464646464647 0.0666336095884597
3.81212121212121 0.0640626289770236
4.05959595959596 0.0591728300637808
4.30707070707071 0.0563123133199644
};
\addplot [draw=black, fill=black, forget plot, mark=*, only marks, mark size=1pt ]
table{%
x  y
0.594949494949495 0.412472375666995
0.842424242424242 0.259610726702452
1.08989898989899 0.186079258771523
1.33737373737374 0.151746776938165
1.58484848484849 0.118264957896592
1.83232323232323 0.0980188229020963
2.07979797979798 0.0794988263906397
2.32727272727273 0.0706264848035585
2.57474747474748 0.0598083850976298
2.82222222222222 0.0604275459588584
3.06969696969697 0.0504525549104677
3.31717171717172 0.0502182465717388
3.56464646464647 0.0508179755246365
3.81212121212121 0.044781105258296
4.05959595959596 0.0469024758826683
4.30707070707071 0.037544499420482
};
\addplot [semithick, black]
table {%
0.1 0.315429578954366
0.14949494949495 0.330356476002391
0.198989898989899 0.345580490146914
0.248484848484848 0.361190437728555
0.297979797979798 0.37727684665995
0.347474747474747 0.393973353651877
0.396969696969697 0.411563809757537
0.446464646464647 0.430409127442337
0.495959595959596 0.441556839137324
0.545454545454546 0.440894903214961
0.594949494949495 0.439242024561383
0.644444444444444 0.437636109846256
0.693939393939394 0.436173940004402
0.743434343434343 0.434859387711126
0.792929292929293 0.433679390993162
0.842424242424242 0.43261812437225
0.891919191919192 0.431660576214752
0.941414141414142 0.430793439030921
0.990909090909091 0.430005213628334
1.04040404040404 0.429286067041241
1.08989898989899 0.428627610674372
1.13939393939394 0.428022685001806
1.18888888888889 0.427465165338848
1.23838383838384 0.426949794950017
1.28787878787879 0.426472044885623
1.33737373737374 0.426027997561238
1.38686868686869 0.425614250456656
1.43636363636364 0.425227836550141
1.48585858585859 0.424866158422607
1.53535353535354 0.424526933590856
1.58484848484849 0.424208148984951
1.63434343434343 0.423908022919051
1.68383838383838 0.423624973196958
1.73333333333333 0.423357590259334
1.78282828282828 0.423104614493102
1.83232323232323 0.422864916981398
1.88181818181818 0.422637495683795
1.93131313131313 0.422421405530282
1.98080808080808 0.422215838821361
2.03030303030303 0.422020047704572
2.07979797979798 0.421833353028363
2.12929292929293 0.421655136950212
2.17878787878788 0.421484852446521
2.22828282828283 0.421321955830687
2.27777777777778 0.421165988409028
2.32727272727273 0.42101651814464
2.37676767676768 0.420873148099959
2.42626262626263 0.420735512953964
2.47575757575758 0.420603275923252
2.52525252525253 0.42047612603706
2.57474747474748 0.420353774376822
2.62424242424242 0.42023595861551
2.67373737373737 0.4201224277081
2.72323232323232 0.420012953575195
2.77272727272727 0.419907322849179
2.82222222222222 0.419805336850486
2.87171717171717 0.41970681034344
2.92121212121212 0.419611570419797
2.97070707070707 0.419519455491532
3.02020202020202 0.419430314380623
3.06969696969697 0.419344005495131
3.11919191919192 0.419260396082201
3.16868686868687 0.419179361549731
3.21818181818182 0.419100784849491
3.26767676767677 0.419024555915286
3.31717171717172 0.418950571150007
3.36666666666667 0.418878732959797
3.41616161616162 0.418808949322736
3.46565656565657 0.418741133399174
3.51515151515152 0.418675203292779
3.56464646464647 0.4186110812181
3.61414141414142 0.418548693979843
3.66363636363636 0.418487972152739
3.71313131313131 0.418428849964798
3.76262626262626 0.418371265060297
3.81212121212121 0.418315158280971
3.86161616161616 0.41826047346379
3.91111111111111 0.418207157253884
3.96060606060606 0.41815515893131
4.01010101010101 0.41810443025047
4.05959595959596 0.41805492529113
4.10909090909091 0.418006600320065
4.15858585858586 0.417959413662289
4.20808080808081 0.417913325582125
4.25757575757576 0.417868298170736
4.30707070707071 0.417824295243054
4.35656565656566 0.417781282240782
4.40606060606061 0.417739226142037
4.45555555555556 0.417698095511788
4.50505050505051 0.417657859892066
4.55454545454546 0.417618490511316
4.6040404040404 0.417579959702218
4.65353535353535 0.417542240962879
4.7030303030303 0.417505308896129
4.75252525252525 0.417469139152576
4.8020202020202 0.417433708377137
4.85151515151515 0.417398994158824
4.9010101010101 0.417364974983524
4.95050505050505 0.4173316301896
5 0.417298939926092
};
\addplot [semithick, color0]
table {%
0.1 -0.00203778086190487
0.14949494949495 -0.00202543764195684
0.198989898989899 -0.00204958515584985
0.248484848484848 -0.00209682530479904
0.297979797979798 -0.00215456690955562
0.347474747474747 -0.00221156152088831
0.396969696969697 -0.00225826627092784
0.446464646464647 -0.00228731550014072
0.495959595959596 -0.00229541814288592
0.545454545454546 -0.00228236160294149
0.594949494949495 -0.00225115255526709
0.644444444444444 -0.00220570688495969
0.693939393939394 -0.00215004801255803
0.743434343434343 -0.00208826190388922
0.792929292929293 -0.00202302121587283
0.842424242424242 -0.00195648969109818
0.891919191919192 -0.00189018270548247
0.941414141414142 -0.00182509395539032
0.990909090909091 -0.00176182872508901
1.04040404040404 -0.00170081167390856
1.08989898989899 -0.0016421619418423
1.13939393939394 -0.00158621852749008
1.18888888888889 -0.00153286051430324
1.23838383838384 -0.00148157397128534
1.28787878787879 -0.00143292513072746
1.33737373737374 -0.00138658625660404
1.38686868686869 -0.00134246332115517
1.43636363636364 -0.00130056471799522
1.48585858585859 -0.00126021666349085
1.53535353535354 -0.00122222972844521
1.58484848484849 -0.00118582413482027
1.63434343434343 -0.00115124838077274
1.68383838383838 -0.00111772490005058
1.73333333333333 -0.00108648602693706
1.78282828282828 -0.00105618960528842
1.83232323232323 -0.00102718893564913
1.88181818181818 -0.000999387996127199
1.93131313131313 -0.000972715942133329
1.98080808080808 -0.000947101748027035
2.03030303030303 -0.00092272145287764
2.07979797979798 -0.000899104560513342
2.12929292929293 -0.000876717358183399
2.17878787878788 -0.000854894663425809
2.22828282828283 -0.000834032660106732
2.27777777777778 -0.000813956622857215
2.32727272727273 -0.000794626968773415
2.37676767676768 -0.000776006834250276
2.42626262626263 -0.000758061878219585
2.47575757575758 -0.000740514266472481
2.52525252525253 -0.000723663951752829
2.57474747474748 -0.000707449254003789
2.62424242424242 -0.000691928367964167
2.67373737373737 -0.000676791815155453
2.72323232323232 -0.000662146359655336
2.77272727272727 -0.000647969613217891
2.82222222222222 -0.000634239616941756
2.87171717171717 -0.000620936200057876
2.92121212121212 -0.000608040061272352
2.97070707070707 -0.000595533227652156
3.02020202020202 -0.000583398781419664
3.06969696969697 -0.000571620475608436
3.11919191919192 -0.000560183265265413
3.16868686868687 -0.000549072676699347
3.21818181818182 -0.000538275235020369
3.26767676767677 -0.000527778089080888
3.31717171717172 -0.000517569111225136
3.36666666666667 -0.000507636629849761
3.41616161616162 -0.000497969481248584
3.46565656565657 -0.000488558296537933
3.51515151515152 -0.00047946697476986
3.56464646464647 -0.000470535754737433
3.61414141414142 -0.000461831686076941
3.66363636363636 -0.000453346305810065
3.71313131313131 -0.000445071604335046
3.76262626262626 -0.000436999856374243
3.81212121212121 -0.000429123819551958
3.86161616161616 -0.000421436522076379
3.91111111111111 -0.00041393134666845
3.96060606060606 -0.000406601908621074
4.01010101010101 -0.000399442240107573
4.05959595959596 -0.000392446510702027
4.10909090909091 -0.000385609212348537
4.15858585858586 -0.00037892508196935
4.20808080808081 -0.000372389090617542
4.25757575757576 -0.000365996391075152
4.30707070707071 -0.000359742348703063
4.35656565656566 -0.000353564474939283
4.40606060606061 -0.000347545173971398
4.45555555555556 -0.000341680696099855
4.50505050505051 -0.000335938272448155
4.55454545454546 -0.000330314125460585
4.6040404040404 -0.000324804644954435
4.65353535353535 -0.000319406494384333
4.7030303030303 -0.000314206986370524
4.75252525252525 -0.000309022082078947
4.8020202020202 -0.000303938852813657
4.85151515151515 -0.000298954401387608
4.9010101010101 -0.000294065693794088
4.95050505050505 -0.000289270001579967
5 -0.000284631962184534
};
\addplot [semithick, blue]
table {%
0.1 0.228047217530655
0.14949494949495 0.233628737479054
0.198989898989899 0.238255100147413
0.248484848484848 0.241784974959858
0.297979797979798 0.244129938041595
0.347474747474747 0.245277454327893
0.396969696969697 0.245299683547491
0.446464646464647 0.244341011855958
0.495959595959596 0.242576922559083
0.545454545454546 0.240187592045816
0.594949494949495 0.237335621717915
0.644444444444444 0.23415649333366
0.693939393939394 0.230757279915539
0.743434343434343 0.227219977019305
0.792929292929293 0.223605896673131
0.842424242424242 0.219960389215967
0.891919191919192 0.216316563781874
0.941414141414142 0.212698433751264
0.990909090909091 0.209123261992847
1.04040404040404 0.205603331641856
1.08989898989899 0.202196948564155
1.13939393939394 0.198759882358481
1.18888888888889 0.195446535039721
1.23838383838384 0.192208586721881
1.28787878787879 0.189047087460672
1.33737373737374 0.185962276035786
1.38686868686869 0.182953717109146
1.43636363636364 0.180020502293143
1.48585858585859 0.177161439403723
1.53535353535354 0.174376258119738
1.58484848484849 0.1716607078005
1.63434343434343 0.169014409799273
1.68383838383838 0.166434825699807
1.73333333333333 0.163921630472978
1.78282828282828 0.161472095879731
1.83232323232323 0.159084326364868
1.88181818181818 0.156756427466777
1.93131313131313 0.154486563106017
1.98080808080808 0.152272923911418
2.03030303030303 0.15011383757166
2.07979797979798 0.148007407908947
2.12929292929293 0.145952023406972
2.17878787878788 0.1439461388414
2.22828282828283 0.14198810778448
2.27777777777778 0.140076436917484
2.32727272727273 0.138209650566647
2.37676767676768 0.136386339626091
2.42626262626263 0.134605120362276
2.47575757575758 0.132864713081485
2.52525252525253 0.131163691191106
2.57474747474748 0.129501098082054
2.62424242424242 0.127875449385907
2.67373737373737 0.126285847938658
2.72323232323232 0.124731044711512
2.77272727272727 0.123210127668602
2.82222222222222 0.121721932452028
2.87171717171717 0.120265460759848
2.92121212121212 0.118839888050836
2.97070707070707 0.117444246783519
3.02020202020202 0.116077678354053
3.06969696969697 0.114738953985595
3.11919191919192 0.113428018892888
3.16868686868687 0.112143682911036
3.21818181818182 0.110885791525816
3.26767676767677 0.109652378991392
3.31717171717172 0.108443361185791
3.36666666666667 0.107258065092561
3.41616161616162 0.106095868069675
3.46565656565657 0.104956103834462
3.51515151515152 0.1038381406499
3.56464646464647 0.102741467035786
3.61414141414142 0.101665347776518
3.66363636363636 0.100609477475444
3.71313131313131 0.0995732007981783
3.76262626262626 0.0985557689926448
3.81212121212121 0.0975571220320396
3.86161616161616 0.0965765211045355
3.91111111111111 0.0956135011844556
3.96060606060606 0.0946676947694943
4.01010101010101 0.0937386069195106
4.05959595959596 0.0928258490317564
4.10909090909091 0.0919289855070456
4.15858585858586 0.0910476418215135
4.20808080808081 0.090181353878949
4.25757575757576 0.0893298045044666
4.30707070707071 0.0884928099519843
4.35656565656566 0.0876698171560758
4.40606060606061 0.086860485267549
4.45555555555556 0.0860645789374043
4.50505050505051 0.0852817387849032
4.55454545454546 0.0845116237973337
4.6040404040404 0.0837540654992803
4.65353535353535 0.0830085891963617
4.7030303030303 0.0822751188431128
4.75252525252525 0.0815531767582486
4.8020202020202 0.0808427288401504
4.85151515151515 0.0801432937937927
4.9010101010101 0.0794546434733865
4.95050505050505 0.0787767018351557
5 0.0781091745307874
};
\addplot [semithick, red]
table {%
0.1 -0.00557064472908275
0.14949494949495 -0.0042974550317153
0.198989898989899 -0.00326827806561636
0.248484848484848 -0.00245464341234425
0.297979797979798 -0.00182619410984253
0.347474747474747 -0.00135300138021677
0.396969696969697 -0.00100734413131975
0.446464646464647 -0.000764778505833008
0.495959595959596 -0.000604372469698355
0.545454545454546 -0.000509027413523544
0.594949494949495 -0.000464644849691553
0.644444444444444 -0.000460003663463993
0.693939393939394 -0.000486134764227208
0.743434343434343 -0.00053591621527016
0.792929292929293 -0.000603668162372006
0.842424242424242 -0.000684917702669874
0.891919191919192 -0.000776091117776878
0.941414141414142 -0.000874346296163342
0.990909090909091 -0.000977482567324883
1.04040404040404 -0.0010837487957519
1.08989898989899 -0.00119172884635055
1.13939393939394 -0.00130040492608041
1.18888888888889 -0.00140890570167318
1.23838383838384 -0.00151661047508811
1.28787878787879 -0.00162297892777941
1.33737373737374 -0.00172772169742608
1.38686868686869 -0.00183045761681433
1.43636363636364 -0.00193108618974025
1.48585858585859 -0.00202941709276117
1.53535353535354 -0.00212535368997668
1.58484848484849 -0.00221885645090958
1.63434343434343 -0.00230990171705647
1.68383838383838 -0.00239849672683023
1.73333333333333 -0.00248464145758953
1.78282828282828 -0.0025683820293958
1.83232323232323 -0.00264975021269465
1.88181818181818 -0.00272879685277139
1.93131313131313 -0.00280556843698876
1.98080808080808 -0.00288014677665849
2.03030303030303 -0.00295256814157374
2.07979797979798 -0.00302293825398248
2.12929292929293 -0.00309126886975386
2.17878787878788 -0.00315763893613852
2.22828282828283 -0.00322215875638932
2.27777777777778 -0.00328485783253374
2.32727272727273 -0.00334577162333938
2.37676767676768 -0.00340499414221895
2.42626262626263 -0.00346257935622041
2.47575757575758 -0.00351854871660751
2.52525252525253 -0.0035730263672763
2.57474747474748 -0.00362605309384012
2.62424242424242 -0.00367763176551139
2.67373737373737 -0.00372782897495005
2.72323232323232 -0.00377671780727995
2.77272727272727 -0.00382432047117753
2.82222222222222 -0.00387070727801264
2.87171717171717 -0.00391590882169257
2.92121212121212 -0.00395992673984291
2.97070707070707 -0.00400286682665296
3.02020202020202 -0.00404474284062839
3.06969696969697 -0.00408556755396994
3.11919191919192 -0.00412542692508988
3.16868686868687 -0.00416431243139337
3.21818181818182 -0.00420225205382241
3.26767676767677 -0.00423929958747726
3.31717171717172 -0.00427547606461698
3.36666666666667 -0.00431083693963674
3.41616161616162 -0.00434535870132313
3.46565656565657 -0.00437910806741681
3.51515151515152 -0.00441207799907617
3.56464646464647 -0.00444431914724375
3.61414141414142 -0.00447582201357266
3.66363636363636 -0.0045066872393249
3.71313131313131 -0.00453685581429231
3.76262626262626 -0.00456638160855882
3.81212121212121 -0.00459528171464796
3.86161616161616 -0.00462353548843608
3.91111111111111 -0.00465125126609245
3.96060606060606 -0.00467834835086878
4.01010101010101 -0.004704948803263
4.05959595959596 -0.00473097521726784
4.10909090909091 -0.00475642993199898
4.15858585858586 -0.00478144388036783
4.20808080808081 -0.00480595476769596
4.25757575757576 -0.00482994698571915
4.30707070707071 -0.00485350762826375
4.35656565656566 -0.00487662242172748
4.40606060606061 -0.00489926949784258
4.45555555555556 -0.00492149337782166
4.50505050505051 -0.00494330533946896
4.55454545454546 -0.00496470787714354
4.6040404040404 -0.0049857264633949
4.65353535353535 -0.00500636330200888
4.7030303030303 -0.00502661653580139
4.75252525252525 -0.00504650021812103
4.8020202020202 -0.00506602355203956
4.85151515151515 -0.0050852384754736
4.9010101010101 -0.00510407575256711
4.95050505050505 -0.0051226158729738
5 -0.00514082709996611
};
\addplot [semithick, color1, dashed]
table {%
0.1 0.231538808076355
0.14949494949495 0.23790330712947
0.198989898989899 0.24272485354283
0.248484848484848 0.246149200013419
0.297979797979798 0.248106436152874
0.347474747474747 0.24873708112511
0.396969696969697 0.248211496605534
0.446464646464647 0.246716476939981
0.495959595959596 0.244469324889088
0.545454545454546 0.241668195146296
0.594949494949495 0.238482634131688
0.644444444444444 0.235045385663891
0.693939393939394 0.231457062688785
0.743434343434343 0.227790069946635
0.792929292929293 0.224098918796141
0.842424242424242 0.220417793255883
0.891919191919192 0.216789008177389
0.941414141414142 0.213195684792283
0.990909090909091 0.209667145956231
1.04040404040404 0.206211902866507
1.08989898989899 0.202836016432473
1.13939393939394 0.199542478015267
1.18888888888889 0.19632794947106
1.23838383838384 0.193196946787906
1.28787878787879 0.190145562742395
1.33737373737374 0.18717428371178
1.38686868686869 0.184279121173296
1.43636363636364 0.181461611122093
1.48585858585859 0.17871556529646
1.53535353535354 0.176041780791557
1.58484848484849 0.17343793385889
1.63434343434343 0.170901886746336
1.68383838383838 0.168431449318647
1.73333333333333 0.16602397281749
1.78282828282828 0.163678226347353
1.83232323232323 0.16139244667562
1.88181818181818 0.159163324043367
1.93131313131313 0.156989603058581
1.98080808080808 0.15486943083483
2.03030303030303 0.152801019302569
2.07979797979798 0.150781959761772
2.12929292929293 0.148813067681847
2.17878787878788 0.146889371241343
2.22828282828283 0.145011367486357
2.27777777777778 0.143177626122027
2.32727272727273 0.1413846319253
2.37676767676768 0.139634653273642
2.42626262626263 0.13792292947941
2.47575757575758 0.136249338812573
2.52525252525253 0.134615416062049
2.57474747474748 0.133013716131404
2.62424242424242 0.131449663934874
2.67373737373737 0.129918895482353
2.72323232323232 0.12839538066554
2.77272727272727 0.126929012398303
2.82222222222222 0.125494024446651
2.87171717171717 0.12408956280782
2.92121212121212 0.12271194304698
2.97070707070707 0.121364013326277
3.02020202020202 0.12004443279122
3.06969696969697 0.118750731385456
3.11919191919192 0.117482763542173
3.16868686868687 0.116239900491776
3.21818181818182 0.115021435063531
3.26767676767677 0.113826685544488
3.31717171717172 0.112654994731046
3.36666666666667 0.111505727666756
3.41616161616162 0.110378272024456
3.46565656565657 0.10927205281421
3.51515151515152 0.108186457811246
3.56464646464647 0.107120960534985
3.61414141414142 0.106075029537227
3.66363636363636 0.105048150613322
3.71313131313131 0.104039827896885
3.76262626262626 0.10304958153693
3.81212121212121 0.102076929293054
3.86161616161616 0.101120227705968
3.91111111111111 0.100181523611869
3.96060606060606 0.0992591286515996
4.01010101010101 0.0983526364573836
4.05959595959596 0.0974616539273284
4.10909090909091 0.0965858002976732
4.15858585858586 0.0957251958196035
4.20808080808081 0.0948784211642126
4.25757575757576 0.0940457114447495
4.30707070707071 0.0932267319916588
4.35656565656566 0.0924211583677828
4.40606060606061 0.0916286755920561
4.45555555555556 0.0908489806984443
4.50505050505051 0.0900817735430942
4.55454545454546 0.0893271273867418
4.6040404040404 0.0885828692153751
4.65353535353535 0.0878521359164692
4.7030303030303 0.0871329206043515
4.75252525252525 0.0864233051199662
4.8020202020202 0.0857256222165232
4.85151515151515 0.0850384090843108
4.9010101010101 0.0843615337601632
4.95050505050505 0.0836947869917954
5 0.083037967669694
};
\addplot [semithick, red, dash pattern=on 1pt off 3pt on 3pt off 3pt]
table {%
0.1 -0.00743007111600047
0.14949494949495 -0.00700840076791154
0.198989898989899 -0.0067698809089054
0.248484848484848 -0.00668097770614851
0.297979797979798 -0.00671032939048055
0.347474747474747 -0.00683002386907938
0.396969696969697 -0.0070163488316406
0.446464646464647 -0.00724996125913047
0.495959595959596 -0.00751543878243266
0.545454545454546 -0.00780109235080551
0.594949494949495 -0.00809789616248791
0.644444444444444 -0.00839922232649437
0.693939393939394 -0.00870021057681314
0.743434343434343 -0.00899738167942532
0.792929292929293 -0.00928827586083436
0.842424242424242 -0.00957124744979154
0.891919191919192 -0.0098452087294405
0.941414141414142 -0.0101094972038275
0.990909090909091 -0.0103638101551798
1.04040404040404 -0.0106080588614754
1.08989898989899 -0.0108422883211117
1.13939393939394 -0.0110667407595556
1.18888888888889 -0.0112816644297449
1.23838383838384 -0.0114874092080913
1.28787878787879 -0.0116843007420049
1.33737373737374 -0.0118727861221326
1.38686868686869 -0.0120531697089267
1.43636363636364 -0.0122259168464169
1.48585858585859 -0.0123913667916419
1.53535353535354 -0.012549884428841
1.58484848484849 -0.0127018364927404
1.63434343434343 -0.0128475638484361
1.68383838383838 -0.0129873975109035
1.73333333333333 -0.0131216324920856
1.78282828282828 -0.0132505728067047
1.83232323232323 -0.0133744870740883
1.88181818181818 -0.0134936371509883
1.93131313131313 -0.0136082621238294
1.98080808080808 -0.0137186110593986
2.03030303030303 -0.0138248820618072
2.07979797979798 -0.0139273128689167
2.12929292929293 -0.0140260558616501
2.17878787878788 -0.0141213038798581
2.22828282828283 -0.0142132634250747
2.27777777777778 -0.0143020693034778
2.32727272727273 -0.0143878524451772
2.37676767676768 -0.0144707809848529
2.42626262626263 -0.0145509851747549
2.47575757575758 -0.014628563924301
2.52525252525253 -0.0147036888766603
2.57474747474748 -0.0147764623870708
2.62424242424242 -0.0148469529550076
2.67373737373737 -0.0149152732821736
2.72323232323232 -0.0149815368113861
2.77272727272727 -0.0150458143476693
2.82222222222222 -0.0151082098398998
2.87171717171717 -0.0151687937261755
2.92121212121212 -0.0152276114702694
2.97070707070707 -0.0152847849570139
3.02020202020202 -0.0153403635230641
3.06969696969697 -0.0153943934386114
3.11919191919192 -0.0154469743019338
3.16868686868687 -0.0154981320044084
3.21818181818182 -0.0155479182623699
3.26767676767677 -0.0155964025751759
3.31717171717172 -0.0156436279900258
3.36666666666667 -0.015689662065125
3.41616161616162 -0.015734511239851
3.46565656565657 -0.0157782488408349
3.51515151515152 -0.0158208912304736
3.56464646464647 -0.0158624971819994
3.61414141414142 -0.0159030789965907
3.66363636363636 -0.0159427310161057
3.71313131313131 -0.0159814262595938
3.76262626262626 -0.0160192219627961
3.81212121212121 -0.0160561468019839
3.86161616161616 -0.0160922002825143
3.91111111111111 -0.016127478220804
3.96060606060606 -0.016161933527287
4.01010101010101 -0.0161956709400146
4.05959595959596 -0.0162286446288942
4.10909090909091 -0.0162608681718144
4.15858585858586 -0.0162924511390896
4.20808080808081 -0.0163233574727831
4.25757575757576 -0.0163535857303271
4.30707070707071 -0.0163832110495543
4.35656565656566 -0.0164122320757669
4.40606060606061 -0.0164406413802145
4.45555555555556 -0.0164684809112212
4.50505050505051 -0.0164957673227426
4.55454545454546 -0.0165225103446709
4.6040404040404 -0.0165487365814559
4.65353535353535 -0.016574454901083
4.7030303030303 -0.0165996708435079
4.75252525252525 -0.0166244015735276
4.8020202020202 -0.016648660357493
4.85151515151515 -0.0166724921450964
4.9010101010101 -0.016695851071354
4.95050505050505 -0.0167188029432397
5 -0.016741329482635
};
\addplot [semithick, blue, dash pattern=on 1pt off 3pt on 3pt off 3pt]
table {%
0.1 0.219479724283505
0.14949494949495 0.220744123712703
0.198989898989899 0.221106423863597
0.248484848484848 0.220513843184225
0.297979797979798 0.218966710149682
0.347474747474747 0.21652996445833
0.396969696969697 0.213331064842447
0.446464646464647 0.209542123442002
0.495959595959596 0.205339114131649
0.545454545454546 0.200882088178841
0.594949494949495 0.196300857250078
0.644444444444444 0.191692495585994
0.693939393939394 0.187124998528814
0.743434343434343 0.182643493518413
0.792929292929293 0.178276083832637
0.842424242424242 0.174039017088295
0.891919191919192 0.169940367344958
0.941414141414142 0.165982837796381
0.990909090909091 0.162165667214271
1.04040404040404 0.158485941367779
1.08989898989899 0.154979996212526
1.13939393939394 0.151520515176447
1.18888888888889 0.1482250286713
1.23838383838384 0.1450468823744
1.28787878787879 0.141980654123248
1.33737373737374 0.139021089916325
1.38686868686869 0.136163095389987
1.43636363636364 0.133401796551161
1.48585858585859 0.130732605774133
1.53535353535354 0.128152234940769
1.58484848484849 0.125654201998104
1.63434343434343 0.123235865642259
1.68383838383838 0.120892991505707
1.73333333333333 0.118623356208338
1.78282828282828 0.116423031161536
1.83232323232323 0.11428893859838
1.88181818181818 0.112218164181765
1.93131313131313 0.110207982567511
1.98080808080808 0.108255815034148
2.03030303030303 0.106359302103
2.07979797979798 0.104516004934059
2.12929292929293 0.102723797085974
2.17878787878788 0.100980681617881
2.22828282828283 0.0992846505166106
2.27777777777778 0.0976338783242843
2.32727272727273 0.0960266086768594
2.37676767676768 0.0944611872598468
2.42626262626263 0.0929360243827189
2.47575757575758 0.0914496524687516
2.52525252525253 0.0900005249588919
2.57474747474748 0.0885874974435285
2.62424242424242 0.0872090410749093
2.67373737373737 0.085864115040713
2.72323232323232 0.084551427327036
2.77272727272727 0.0832699726107856
2.82222222222222 0.0820185686796718
2.87171717171717 0.0807961823689118
2.92121212121212 0.0796019336593916
2.97070707070707 0.0784348466677252
3.02020202020202 0.0772940453056644
3.06969696969697 0.076178378410483
3.11919191919192 0.0750876227259915
3.16868686868687 0.0740206810443127
3.21818181818182 0.0729772689274776
3.26767676767677 0.0719557119324746
3.31717171717172 0.0709558040925081
3.36666666666667 0.069976889540563
3.41616161616162 0.0690183594291242
3.46565656565657 0.0680795783434928
3.51515151515152 0.067159944576583
3.56464646464647 0.0662589560887048
3.61414141414142 0.0653759432844223
3.66363636363636 0.0645105753391787
3.71313131313131 0.0636622553251144
3.76262626262626 0.062830319078108
3.81212121212121 0.0620146400839372
3.86161616161616 0.0612145706612799
3.91111111111111 0.0604296790990533
3.96060606060606 0.0596596127883182
4.01010101010101 0.0589039241500143
4.05959595959596 0.0581622510884332
4.10909090909091 0.0574341983431547
4.15858585858586 0.0567194207408507
4.20808080808081 0.0560175068207571
4.25757575757576 0.0553281610537768
4.30707070707071 0.0546511930849787
4.35656565656566 0.0539861303508251
4.40606060606061 0.0533326660688661
4.45555555555556 0.0526905767093995
4.50505050505051 0.0520595452332223
4.55454545454546 0.0514392706436293
4.6040404040404 0.0508295863346369
4.65353535353535 0.050230092021018
4.7030303030303 0.0496406953199736
4.75252525252525 0.0490609968527118
4.8020202020202 0.0484909392572145
4.85151515151515 0.0479301232789825
4.9010101010101 0.0473783453048053
4.95050505050505 0.0468355195150741
5 0.0463013938168111
};
\addplot [semithick, black, dash pattern=on 1pt off 3pt on 3pt off 3pt]
table {%
0.1 0.305669459650649
0.14949494949495 0.315645087848816
0.198989898989899 0.3260139057028
0.248484848484848 0.337052702517744
0.297979797979798 0.349111685367621
0.347474747474747 0.362713021593463
0.396969696969697 0.378842396088484
0.446464646464647 0.399570148043039
0.495959595959596 0.409848338450404
0.545454545454546 0.388465987842998
0.594949494949495 0.362771716901449
0.644444444444444 0.337534125591916
0.693939393939394 0.313800193236602
0.743434343434343 0.291913414619604
0.792929292929293 0.271940890340598
0.842424242424242 0.253818630317098
0.891919191919192 0.237421102752568
0.941414141414142 0.222598031753858
0.990909090909091 0.209194410670207
1.04040404040404 0.197061042468015
1.08989898989899 0.186059649219736
1.13939393939394 0.176064917883494
1.18888888888889 0.166964842242025
1.23838383838384 0.158660161737486
1.28787878787879 0.151063363407393
1.33737373737374 0.14409751413908
1.38686868686869 0.137695071589848
1.43636363636364 0.131796752524301
1.48585858585859 0.126350495391488
1.53535353535354 0.121310531503993
1.58484848484849 0.11663656542197
1.63434343434343 0.112293058732529
1.68383838383838 0.108248607978788
1.73333333333333 0.104475406436681
1.78282828282828 0.100948779531399
1.83232323232323 0.0976467842224606
1.88181818181818 0.0945499559805102
1.93131313131313 0.0916406000828286
1.98080808080808 0.0889032143690539
2.03030303030303 0.0863237728138785
2.07979797979798 0.0838896630914443
2.12929292929293 0.0815895233684717
2.17878787878788 0.0794132106523916
2.22828282828283 0.0773512437637959
2.27777777777778 0.0753953021433746
2.32727272727273 0.0735377515988338
2.37676767676768 0.0717716475827715
2.42626262626263 0.0700906615890143
2.47575757575758 0.0684890164730995
2.52525252525253 0.0669614295110992
2.57474747474748 0.065503056997427
2.62424242424242 0.0641094756436942
2.67373737373737 0.0627765916268462
2.72323232323232 0.0615006573912575
2.77272727272727 0.0602782146822279
2.82222222222222 0.0591060721845946
2.87171717171717 0.0579812808517719
2.92121212121212 0.0569011118810372
2.97070707070707 0.0558630369964203
3.02020202020202 0.0548647107678411
3.06969696969697 0.0539039547301554
3.11919191919192 0.0529787430965968
3.16868686868687 0.0520871898857559
3.21818181818182 0.0512275373054172
3.26767676767677 0.0503981452542496
3.31717171717172 0.0495974818165704
3.36666666666667 0.0488241146661635
3.41616161616162 0.0480767032271106
3.46565656565657 0.0473539915865162
3.51515151515152 0.0466548029815275
3.56464646464647 0.045978029938222
3.61414141414142 0.0453226350880039
3.66363636363636 0.0446876419676149
3.71313131313131 0.0440721316620819
3.76262626262626 0.0434752386654048
3.81212121212121 0.042896147131572
3.86161616161616 0.0423340874123175
3.91111111111111 0.0417883328801073
3.96060606060606 0.0412581970065288
4.01010101010101 0.0407430306727355
4.05959595959596 0.0402422196903245
4.10909090909091 0.0397551825146341
4.15858585858586 0.0392813681315021
4.20808080808081 0.0388202541118733
4.25757575757576 0.0383713447963874
4.30707070707071 0.0379341696298348
4.35656565656566 0.0375082816097366
4.40606060606061 0.0370932558487216
4.45555555555556 0.0366886892269632
4.50505050505051 0.0362941952624902
4.55454545454546 0.0359094087002257
4.6040404040404 0.0355339806734661
4.65353535353535 0.0351675786289083
4.7030303030303 0.0348098853942834
4.75252525252525 0.0344605983090727
4.8020202020202 0.0341194284124967
4.85151515151515 0.0337860996852957
4.9010101010101 0.0334603483404285
4.95050505050505 0.033141922159906
5 0.032830579873619
};
\end{axis}

\end{tikzpicture}

%% file: Figures/supplementary/n_over_p=10/test_errors_n_over_p=10.0.tex
\begin{tikzpicture}
\tikzstyle{every node}=[font=\tiny]

\definecolor{darkgray176}{RGB}{176,176,176}
\definecolor{gray}{RGB}{128,128,128}
\definecolor{lightgray204}{RGB}{204,204,204}
\definecolor{orange}{RGB}{255,165,0}

\begin{axis}[
height=\figheight,
legend cell align={left},
legend style={fill opacity=0.8, draw opacity=1, text opacity=1, draw=lightgray204},
tick align=outside,
tick pos=left,
width=\figwidth,
x grid style={darkgray176},
xlabel={$p / n$},
xmajorgrids,
xmin=-0.145, xmax=5.245,
xtick style={color=black},
y grid style={darkgray176},
ylabel={\(\displaystyle \varepsilon_g\)},
ymajorgrids,
ymin=0.357829317984596, ymax=0.438228398524729,
ytick style={color=black}
]
\addplot [semithick, black]
table {%
0.1 0.414524946805968
0.14949494949495 0.401335580674657
0.198989898989899 0.402231897652578
0.248484848484848 0.407171644236027
0.297979797979798 0.413286760332811
0.347474747474747 0.419781442618075
0.396969696969697 0.426378180956067
0.446464646464647 0.432259327896778
0.495959595959596 0.434573894863814
0.545454545454546 0.433503498839409
0.594949494949495 0.431705087580364
0.644444444444444 0.429832911005753
0.693939393939394 0.428011953588622
0.743434343434343 0.426271516081781
0.792929292929293 0.424616644597388
0.842424242424242 0.423045020077438
0.891919191919192 0.421552011372751
0.941414141414142 0.420132413826162
0.990909090909091 0.418781090817564
1.04040404040404 0.417493209735481
1.08989898989899 0.41626431144783
1.13939393939394 0.415090318604964
1.18888888888889 0.413967515055053
1.23838383838384 0.412892514982775
1.28787878787879 0.411862230331445
1.33737373737374 0.410873842303901
1.38686868686869 0.409924771246416
1.43636363636364 0.409012652604556
1.48585858585859 0.408135313766983
1.53535353535354 0.407290757308133
1.58484848484849 0.406477133990348
1.63434343434343 0.405692742851848
1.68383838383838 0.404936006828246
1.73333333333333 0.404205461123249
1.78282828282828 0.403499748303743
1.83232323232323 0.402817604272876
1.88181818181818 0.402157849794954
1.93131313131313 0.401519385584678
1.98080808080808 0.400901185654551
2.03030303030303 0.400302288572458
2.07979797979798 0.399721794306386
2.12929292929293 0.399158858813562
2.17878787878788 0.398612689664164
2.22828282828283 0.398082542077856
2.27777777777778 0.397567715318035
2.32727272727273 0.397067549427837
2.37676767676768 0.396581422258071
2.42626262626263 0.396108746746732
2.47575757575758 0.395648968447256
2.52525252525253 0.395201563260104
2.57474747474748 0.394766035360044
2.62424242424242 0.394341915285666
2.67373737373737 0.393928758192315
2.72323232323232 0.393526142243987
2.77272727272727 0.39313366712403
2.82222222222222 0.392750952670371
2.87171717171717 0.392377637609188
2.92121212121212 0.392013378389576
2.97070707070707 0.39165784809567
3.02020202020202 0.391310735446974
3.06969696969697 0.390971740131191
3.11919191919192 0.390640592918101
3.16868686868687 0.390317006015153
3.21818181818182 0.390000732630211
3.26767676767677 0.389691524652115
3.31717171717172 0.389389147206805
3.36666666666667 0.389093375758114
3.41616161616162 0.388803995542199
3.46565656565657 0.388520801039006
3.51515151515152 0.388243595477618
3.56464646464647 0.38797219037378
3.61414141414142 0.387706405248468
3.66363636363636 0.387446066445938
3.71313131313131 0.387191008178291
3.76262626262626 0.386941071338606
3.81212121212121 0.386696102305523
3.86161616161616 0.386455954536167
3.91111111111111 0.386220486522154
3.96060606060606 0.385989562602232
4.01010101010101 0.38576305223572
4.05959595959596 0.38554082984646
4.10909090909091 0.385322774588326
4.15858585858586 0.385108770123975
4.20808080808081 0.384898704415929
4.25757575757576 0.38469246952922
4.30707070707071 0.384489961444832
4.35656565656566 0.384291079883337
4.40606060606061 0.384095728138369
4.45555555555556 0.383903812916124
4.50505050505051 0.383715244189763
4.55454545454546 0.383529935054599
4.6040404040404 0.383347801477521
4.65353535353535 0.383168762753228
4.7030303030303 0.382992740215562
4.75252525252525 0.382819658290001
4.8020202020202 0.382649443800678
4.85151515151515 0.382482025982918
4.9010101010101 0.382317336386594
4.95050505050505 0.382155308777854
5 0.381995879054498
};
\addlegendentry{$\hat{f}_{\erm}(\lambda = 10^{-4})$}
\addplot [semithick, gray]
table {%
0.1 0.387131813667719
0.14949494949495 0.37484184777058
0.198989898989899 0.369714258593329
0.248484848484848 0.367286983631852
0.297979797979798 0.365919856977843
0.347474747474747 0.36505226100385
0.396969696969697 0.364455304711469
0.446464646464647 0.364020337894883
0.495959595959596 0.363689666810282
0.545454545454546 0.363429962035699
0.594949494949495 0.363220675330787
0.644444444444444 0.363048468551683
0.693939393939394 0.362904314775032
0.743434343434343 0.362781890648327
0.792929292929293 0.362676636733289
0.842424242424242 0.362585183644445
0.891919191919192 0.362504988104035
0.941414141414142 0.362434094754257
0.990909090909091 0.36237097575345
1.04040404040404 0.362314420119213
1.08989898989899 0.36226345571623
1.13939393939394 0.362217293190032
1.18888888888889 0.362175284972581
1.23838383838384 0.362136894839051
1.28787878787879 0.362101674980043
1.33737373737374 0.362069248511627
1.38686868686869 0.362039295976674
1.43636363636364 0.362011544814359
1.48585858585859 0.361985761063709
1.53535353535354 0.3619617427674
1.58484848484849 0.361939314682896
1.63434343434343 0.361918323983442
1.68383838383838 0.361898636852867
1.73333333333333 0.361880135587624
1.78282828282828 0.361862716189062
1.83232323232323 0.36184628649024
1.88181818181818 0.361830764501121
1.93131313131313 0.361816077042077
1.98080808080808 0.361802158591752
2.03030303030303 0.361788950311067
2.07979797979798 0.361776399212744
2.12929292929293 0.361764457451643
2.17878787878788 0.361753081716301
2.22828282828283 0.361742232702996
2.27777777777778 0.361731874664867
2.32727272727273 0.361721975017211
2.37676767676768 0.361712503995779
2.42626262626263 0.36170343435813
2.47575757575758 0.361694741122138
2.52525252525253 0.361686401336372
2.57474747474748 0.361678393877908
2.62424242424242 0.361670699273815
2.67373737373737 0.36166329954316
2.72323232323232 0.36165617805682
2.77272727272727 0.361649319412804
2.82222222222222 0.361642709325129
2.87171717171717 0.361636334524545
2.92121212121212 0.361630182669684
2.97070707070707 0.36162424226735
3.02020202020202 0.361618502600901
3.06969696969697 0.361612953665756
3.11919191919192 0.361607586111227
3.16868686868687 0.361602391187963
3.21818181818182 0.361597360700379
3.26767676767677 0.361592486963532
3.31717171717172 0.361587762763964
3.36666666666667 0.361583181324093
3.41616161616162 0.361578736269784
3.46565656565657 0.361574421600773
3.51515151515152 0.361570231663652
3.56464646464647 0.361566161127166
3.61414141414142 0.361562204959591
3.66363636363636 0.361558358407988
3.71313131313131 0.361554616921474
3.76262626262626 0.361550976364626
3.81212121212121 0.361547432654785
3.86161616161616 0.361543981978431
3.91111111111111 0.361540620719655
3.96060606060606 0.361537345447523
4.01010101010101 0.361534152904398
4.05959595959596 0.361531039995137
4.10909090909091 0.361528003777083
4.15858585858586 0.36152504145079
4.20808080808081 0.361522150351418
4.25757575757576 0.361519327940735
4.30707070707071 0.36151657179969
4.35656565656566 0.361513879621499
4.40606060606061 0.361511249205207
4.45555555555556 0.361508678449688
4.50505050505051 0.361506165348049
4.55454545454546 0.361503707982404
4.6040404040404 0.361501304518998
4.65353535353535 0.361498953203642
4.7030303030303 0.361496652357442
4.75252525252525 0.361494400372807
4.8020202020202 0.361492195709698
4.85151515151515 0.361490036892122
4.9010101010101 0.361487922504834
4.95050505050505 0.361485851190244
5 0.361483821645511
};
\addlegendentry{$\hat{f}_{\bo}$}
\addplot [semithick, red]
table {%
0.1 0.413054009921423
0.14949494949495 0.39161620441319
0.198989898989899 0.381995353291297
0.248484848484848 0.376911518040222
0.297979797979798 0.373812840123395
0.347474747474747 0.371734594287127
0.396969696969697 0.370246010155729
0.446464646464647 0.369127947809144
0.495959595959596 0.368257616961895
0.545454545454546 0.367560993756848
0.594949494949495 0.366990838040711
0.644444444444444 0.366515599562566
0.693939393939394 0.366113409532787
0.743434343434343 0.365768634638366
0.792929292929293 0.365469803363001
0.842424242424242 0.36520830823581
0.891919191919192 0.364977564915318
0.941414141414142 0.364772450975104
0.990909090909091 0.364588921792855
1.04040404040404 0.364423740765893
1.08989898989899 0.364274287804979
1.13939393939394 0.364138418673943
1.18888888888889 0.364014362148381
1.23838383838384 0.363900641973327
1.28787878787879 0.363796018079841
1.33737373737374 0.363699441140563
1.38686868686869 0.363610017168593
1.43636363636364 0.363526979874364
1.48585858585859 0.363449668677153
1.53535353535354 0.363377510975784
1.58484848484849 0.363310008285681
1.63434343434343 0.363246724052868
1.68383838383838 0.363187275043619
1.73333333333333 0.363131322649687
1.78282828282828 0.363078567388284
1.83232323232323 0.363028742665679
1.88181818181818 0.362981611123475
1.93131313131313 0.362936959938267
1.98080808080808 0.362894598297252
2.03030303030303 0.362854354811224
2.07979797979798 0.362816074272327
2.12929292929293 0.362779616492759
2.17878787878788 0.362744854263575
2.22828282828283 0.362711672066538
2.27777777777778 0.362679964574302
2.32727272727273 0.362649635679079
2.37676767676768 0.36262059722283
2.42626262626263 0.36259276869693
2.47575757575758 0.362566076037247
2.52525252525253 0.362540450890247
2.57474747474748 0.362515830750333
2.62424242424242 0.362492157408924
2.67373737373737 0.362469377350189
2.72323232323232 0.362447441075658
2.77272727272727 0.362426302271949
2.82222222222222 0.362405918618534
2.87171717171717 0.362386249962272
2.92121212121212 0.362367259593739
2.97070707070707 0.362348912915044
3.02020202020202 0.362331177835866
3.06969696969697 0.362314024190925
3.11919191919192 0.362297423816165
3.16868686868687 0.362281350539785
3.21818181818182 0.362265779396085
3.26767676767677 0.362250687437134
3.31717171717172 0.362236052800677
3.36666666666667 0.3622218548913
3.41616161616162 0.362208074634232
3.46565656565657 0.362194693758797
3.51515151515152 0.362181695208392
3.56464646464647 0.362169062864946
3.61414141414142 0.362156781317837
3.66363636363636 0.362144836280331
3.71313131313131 0.362133214077304
3.76262626262626 0.362121901715443
3.81212121212121 0.362110887224816
3.86161616161616 0.3621001587237
3.91111111111111 0.362089705250754
3.96060606060606 0.36207951649868
4.01010101010101 0.362069582366868
4.05959595959596 0.362059893562762
4.10909090909091 0.362050441072731
4.15858585858586 0.362041216430431
4.20808080808081 0.362032211297441
4.25757575757576 0.362023418042325
4.30707070707071 0.362014829271762
4.35656565656566 0.362006438048083
4.40606060606061 0.361998237373194
4.45555555555556 0.361990220967195
4.50505050505051 0.361982382860359
4.55454545454546 0.361974716925619
4.6040404040404 0.361967217715218
4.65353535353535 0.361959879704564
4.7030303030303 0.361952697834
4.75252525252525 0.36194566721534
4.8020202020202 0.361938783034448
4.85151515151515 0.361932040756696
4.9010101010101 0.361925436100544
4.95050505050505 0.361918964765416
5 0.361912622991912
};
\addlegendentry{$\hat{f}_{\erm}(\lambda_{\rm loss})$}
\addplot [semithick, orange, dashed]
table {%
0.1 0.413954220981321
0.14949494949495 0.39258968835001
0.198989898989899 0.382717394493346
0.248484848484848 0.377421418156971
0.297979797979798 0.374180051966857
0.347474747474747 0.372007180141137
0.396969696969697 0.370454239434425
0.446464646464647 0.369290990651217
0.495959595959596 0.368387975315965
0.545454545454546 0.36766706733153
0.594949494949495 0.367078436407456
0.644444444444444 0.366588863041633
0.693939393939394 0.366175354082989
0.743434343434343 0.365821503465898
0.792929292929293 0.365515296851632
0.842424242424242 0.365247736791221
0.891919191919192 0.365011953347036
0.941414141414142 0.3648026113289
0.990909090909091 0.364615502203441
1.04040404040404 0.364447271945306
1.08989898989899 0.364295201393139
1.13939393939394 0.364157071187739
1.18888888888889 0.364031050137
1.23838383838384 0.363915614154187
1.28787878787879 0.3638094843587
1.33737373737374 0.363711579780008
1.38686868686869 0.363620980511386
1.43636363636364 0.363536898808068
1.48585858585859 0.363458656189464
1.53535353535354 0.363385665139931
1.58484848484849 0.363317414366129
1.63434343434343 0.363253456828098
1.68383838383838 0.363193399947165
1.73333333333333 0.363136897602374
1.78282828282828 0.363083643415237
1.83232323232323 0.363033365223512
1.88181818181818 0.362985820432185
1.93131313131313 0.362940792108233
1.98080808080808 0.362898085675355
2.03030303030303 0.362857526107118
2.07979797979798 0.362818955532501
2.12929292929293 0.362782231184857
2.17878787878788 0.362747223638178
2.22828282828283 0.362713815284807
2.27777777777778 0.362681899016881
2.32727272727273 0.362651377080299
2.37676767676768 0.362622160076206
2.42626262626263 0.362594166086316
2.47575757575758 0.362567319907493
2.52525252525253 0.362541552377027
2.57474747474748 0.362516799778508
2.62424242424242 0.36249300331611
2.67373737373737 0.36247010864751
2.72323232323232 0.362448065474574
2.77272727272727 0.362426827167667
2.82222222222222 0.362406350441343
2.87171717171717 0.36238659505889
2.92121212121212 0.362367523568589
2.97070707070707 0.362349101066897
3.02020202020202 0.362331294985375
3.06969696969697 0.362314074898656
3.11919191919192 0.36229741235103
3.16868686868687 0.362281280699639
3.21818181818182 0.362265654972354
3.26767676767677 0.362250511738893
3.31717171717172 0.362235828993608
3.36666666666667 0.362221586048866
3.41616161616162 0.362207763437825
3.46565656565657 0.3621943428257
3.51515151515152 0.362181306928669
3.56464646464647 0.362168639439634
3.61414141414142 0.36215632546505
3.66363636363636 0.362144349440607
3.71313131313131 0.362132698099329
3.76262626262626 0.362121358432285
3.81212121212121 0.362110318093359
3.86161616161616 0.362099565385961
3.91111111111111 0.362089089215114
3.96060606060606 0.362078879049133
4.01010101010101 0.362068924884407
4.05959595959596 0.362059217212711
4.10909090909091 0.36204974699077
4.15858585858586 0.362040505612456
4.20808080808081 0.362031484882397
4.25757575757576 0.362022676992068
4.30707070707071 0.362014074496991
4.35656565656566 0.362005670295739
4.40606060606061 0.361997457610758
4.45555555555556 0.361989429969639
4.50505050505051 0.361981581188458
4.55454545454546 0.361973905360303
4.6040404040404 0.361966396813254
4.65353535353535 0.361959050159986
4.7030303030303 0.36195186021207
4.75252525252525 0.361944822008899
4.8020202020202 0.361937930796971
4.85151515151515 0.361931182019281
4.9010101010101 0.361924571305221
4.95050505050505 0.36191809446129
5 0.361911747461991
};
\addlegendentry{$\hat{f}_{\empbayes}(\lambda_{\rm evidence})$}
\addplot [semithick, blue, dashed]
table {%
0.1 0.411598336712368
0.14949494949495 0.390189948137028
0.198989898989899 0.380853023777595
0.248484848484848 0.375995152436981
0.297979797979798 0.373058546088313
0.347474747474747 0.371098655454939
0.396969696969697 0.369699333544999
0.446464646464647 0.368650637708133
0.495959595959596 0.367835597979403
0.545454545454546 0.36718398709646
0.594949494949495 0.36665112926618
0.644444444444444 0.366207259635032
0.693939393939394 0.365831788519059
0.743434343434343 0.365510021142214
0.792929292929293 0.365231189692919
0.842424242424242 0.364987224290928
0.891919191919192 0.364771958653365
0.941414141414142 0.364580599165823
0.990909090909091 0.364409364130684
1.04040404040404 0.364254961779504
1.08989898989899 0.364115493421719
1.13939393939394 0.363988677149162
1.18888888888889 0.363872860767518
1.23838383838384 0.36376666794977
1.28787878787879 0.363668942961495
1.33737373737374 0.363578708203799
1.38686868686869 0.36349513104193
1.43636363636364 0.363417498120282
1.48585858585859 0.363345194409215
1.53535353535354 0.36327768724713
1.58484848484849 0.363214512545438
1.63434343434343 0.363155264411871
1.68383838383838 0.363099586191506
1.73333333333333 0.363047163146537
1.78282828282828 0.36299771671822
1.83232323232323 0.362950999066444
1.88181818181818 0.362906789395246
1.93131313131313 0.362864882704153
1.98080808080808 0.362825124277423
2.03030303030303 0.362787331067739
2.07979797979798 0.362751367132694
2.12929292929293 0.362717102109091
2.17878787878788 0.362684417792002
2.22828282828283 0.362653206713756
2.27777777777778 0.362623370955666
2.32727272727273 0.362594821109253
2.37676767676768 0.362567475370554
2.42626262626263 0.362541258747538
2.47575757575758 0.362516102364691
2.52525252525253 0.362491942851196
2.57474747474748 0.362468721800969
2.62424242424242 0.362446385295395
2.67373737373737 0.362424883479711
2.72323232323232 0.362404170186879
2.77272727272727 0.362384202602247
2.82222222222222 0.36236494096423
2.87171717171717 0.362346348296632
2.92121212121212 0.362328390168538
2.97070707070707 0.362311034478681
3.02020202020202 0.362294251261435
3.06969696969697 0.362278012511747
3.11919191919192 0.362262292027248
3.16868686868687 0.362247065265124
3.21818181818182 0.36223230921269
3.26767676767677 0.36221800226972
3.31717171717172 0.362204124135959
3.36666666666667 0.36219065573615
3.41616161616162 0.362177579097682
3.46565656565657 0.362164877292047
3.51515151515152 0.362152534355037
3.56464646464647 0.362140535219025
3.61414141414142 0.362128865650825
3.66363636363636 0.36211751219416
3.71313131313131 0.362106462118018
3.76262626262626 0.362095703367806
3.81212121212121 0.362085224521127
3.86161616161616 0.362075014751596
3.91111111111111 0.362065063771562
3.96060606060606 0.362055361826464
4.01010101010101 0.36204589964287
4.05959595959596 0.362036668403446
4.10909090909091 0.362027659719261
4.15858585858586 0.362018865603932
4.20808080808081 0.362010278449822
4.25757575757576 0.362001891005739
4.30707070707071 0.361993696356263
4.35656565656566 0.361985687902551
4.40606060606061 0.361977859344351
4.45555555555556 0.361970204663314
4.50505050505051 0.361962718107303
4.55454545454546 0.361955394175823
4.6040404040404 0.361948227606414
4.65353535353535 0.361941213361863
4.7030303030303 0.361934346618309
4.75252525252525 0.36192762275379
4.8020202020202 0.361921037338459
4.85151515151515 0.361914586123785
4.9010101010101 0.361908265033877
4.95050505050505 0.361902070156722
5 0.361895997735983
};
\addlegendentry{$\hat{f}_{\empbayes}(\lambda_{\rm error})$}
\addplot [semithick, blue]
table {%
0.1 0.411599403841881
0.14949494949495 0.390191868431978
0.198989898989899 0.380855112024051
0.248484848484848 0.375997187211859
0.297979797979798 0.373060486817564
0.347474747474747 0.371100509521311
0.396969696969697 0.369701108898728
0.446464646464647 0.368652343171522
0.495959595959596 0.36783724011357
0.545454545454546 0.367185572369995
0.594949494949495 0.366652662616754
0.644444444444444 0.366208745543708
0.693939393939394 0.365833230677879
0.743434343434343 0.365511422543909
0.792929292929293 0.365232553012157
0.842424242424242 0.364988552024715
0.891919191919192 0.364773252648688
0.941414141414142 0.364581861450708
0.990909090909091 0.364410596351853
1.04040404040404 0.364256433332786
1.08989898989899 0.364116925559026
1.13939393939394 0.363990071953356
1.18888888888889 0.363874220135584
1.23838383838384 0.363767993624175
1.28787878787879 0.363670236555262
1.33737373737374 0.363579971168996
1.38686868686869 0.363496364815003
1.43636363636364 0.363418703934098
1.48585858585859 0.363346373544972
1.53535353535354 0.36327884075302
1.58484848484849 0.363215641510467
1.63434343434343 0.363156369815819
1.68383838383838 0.363100668963299
1.73333333333333 0.363048224211306
1.78282828282828 0.362998756824167
1.83232323232323 0.362952019065481
1.88181818181818 0.362907789980796
1.93131313131313 0.362865871890536
1.98080808080808 0.362826087403742
2.03030303030303 0.362788276829572
2.07979797979798 0.362752296112571
2.12929292929293 0.36271801488053
2.17878787878788 0.36268531489745
2.22828282828283 0.362654088665186
2.27777777777778 0.362624238246136
2.32727272727273 0.362595674205928
2.37676767676768 0.362568314723777
2.42626262626263 0.362542084782309
2.47575757575758 0.362516915493774
2.52525252525253 0.36249274345993
2.57474747474748 0.362469510263751
2.62424242424242 0.362447161974885
2.67373737373737 0.362425648717533
2.72323232323232 0.362404924312802
2.77272727272727 0.362384945929297
2.82222222222222 0.362365673799395
2.87171717171717 0.3623470709285
2.92121212121212 0.362329102877959
2.97070707070707 0.362311737533966
3.02020202020202 0.362294944917926
3.06969696969697 0.362278697023707
3.11919191919192 0.362262967631837
3.16868686868687 0.362247732193089
3.21818181818182 0.362232967683999
3.26767676767677 0.362218652499362
3.31717171717172 0.36220476633375
3.36666666666667 0.362191290097371
3.41616161616162 0.362178205817544
3.46565656565657 0.362165496550793
3.51515151515152 0.362153146334616
3.56464646464647 0.362141140091138
3.61414141414142 0.362129463579956
3.66363636363636 0.362118103345052
3.71313131313131 0.362107046643418
3.76262626262626 0.362096281420764
3.81212121212121 0.362085796246167
3.86161616161616 0.362075580285666
3.91111111111111 0.362065623255159
3.96060606060606 0.362055915394861
4.01010101010101 0.362046447417044
4.05959595959596 0.362037210515097
4.10909090909091 0.362028196282292
4.15858585858586 0.362019396744748
4.20808080808081 0.362010804278748
4.25757575757576 0.362002411625858
4.30707070707071 0.361994211880233
4.35656565656566 0.361986198431996
4.40606060606061 0.361978364978551
4.45555555555556 0.361970705500048
4.50505050505051 0.361963214242446
4.55454545454546 0.361955885701572
4.6040404040404 0.361948714612022
4.65353535353535 0.361941695933334
4.7030303030303 0.361934824838357
4.75252525252525 0.361928096709829
4.8020202020202 0.361921507106382
4.85151515151515 0.361915051783965
4.9010101010101 0.361908726661422
4.95050505050505 0.361902527820039
5 0.361896451513079
};
\addlegendentry{$\hat{f}_{\erm}(\lambda_{\rm error})$}
\end{axis}

\end{tikzpicture}

%% file: Figures/supplementary/n_over_p=10/calibration_p=0.75_n_over_p=10.0.tex
\begin{tikzpicture}
\tikzstyle{every node}=[font=\tiny]

\definecolor{darkgray176}{RGB}{176,176,176}
\definecolor{gray}{RGB}{128,128,128}
\definecolor{lightgray204}{RGB}{204,204,204}
\definecolor{orange}{RGB}{255,165,0}

\begin{axis}[
height=\figheight,
legend cell align={left},
legend style={
  fill opacity=0.8,
  draw opacity=1,
  text opacity=1,
  at={(0.91,0.5)},
  anchor=east,
  draw=lightgray204
},
tick align=outside,
tick pos=left,
width=\figwidth,
x grid style={darkgray176},
xlabel={$p / n$},
xmajorgrids,
xmin=-0.145, xmax=5.245,
xtick style={color=black},
y grid style={darkgray176},
ylabel={\(\displaystyle \Delta_{0.75}\)},
ymajorgrids,
ymin=-0.0120941433492986, ymax=0.253977010335264,
ytick style={color=black}
]
\addplot [semithick, black]
table {%
0.1 0.140616598791678
0.14949494949495 0.155027487867312
0.198989898989899 0.174056541085374
0.248484848484848 0.191343125635679
0.297979797979798 0.20618705235553
0.347474747474747 0.218902600199479
0.396969696969697 0.229802458977012
0.446464646464647 0.238260249796193
0.495959595959596 0.241855438130541
0.545454545454546 0.241882866985966
0.594949494949495 0.241080753373309
0.644444444444444 0.240113490321688
0.693939393939394 0.239119394379561
0.743434343434343 0.238137602857418
0.792929292929293 0.237181106284979
0.842424242424242 0.236254240024664
0.891919191919192 0.235358040251538
0.941414141414142 0.234492178533705
0.990909090909091 0.233655745538224
1.04040404040404 0.232847596737293
1.08989898989899 0.232066508248217
1.13939393939394 0.231311255551662
1.18888888888889 0.230580648055035
1.23838383838384 0.229873545846848
1.28787878787879 0.22918886544508
1.33737373737374 0.228525582870293
1.38686868686869 0.227882730864713
1.43636363636364 0.227259397276281
1.48585858585859 0.226654721970383
1.53535353535354 0.226067894982611
1.58484848484849 0.225498146703233
1.63434343434343 0.224944756694425
1.68383838383838 0.22440704176811
1.73333333333333 0.22388435589023
1.78282828282828 0.223376087582048
1.83232323232323 0.222881657624957
1.88181818181818 0.222400517371883
1.93131313131313 0.221932146345937
1.98080808080808 0.221476050351773
2.03030303030303 0.221031760367718
2.07979797979798 0.220598830598508
2.12929292929293 0.22017683713721
2.17878787878788 0.219765376651876
2.22828282828283 0.219364065161519
2.27777777777778 0.218972536888525
2.32727272727273 0.218590443204766
2.37676767676768 0.218217451647511
2.42626262626263 0.217853244995445
2.47575757575758 0.217497520416345
2.52525252525253 0.217149988666334
2.57474747474748 0.216810373344818
2.62424242424242 0.216478410193944
2.67373737373737 0.216153846447176
2.72323232323232 0.215836440217861
2.77272727272727 0.21552595992509
2.82222222222222 0.215222183757037
2.87171717171717 0.214924899166495
2.92121212121212 0.214633902396787
2.97070707070707 0.21434899803743
3.02020202020202 0.214069998605591
3.06969696969697 0.213796720839405
3.11919191919192 0.21352900392383
3.16868686868687 0.213266665842211
3.21818181818182 0.21300955652303
3.26767676767677 0.212757520617831
3.31717171717172 0.212510410693077
3.36666666666667 0.212268084919833
3.41616161616162 0.212030406812531
3.46565656565657 0.211797244982089
3.51515151515152 0.211568472902455
3.56464646464647 0.21134396868977
3.61414141414142 0.211123614898279
3.66363636363636 0.210907298297261
3.71313131313131 0.210694909730399
3.76262626262626 0.210486343908323
3.81212121212121 0.210281499217404
3.86161616161616 0.210080277614769
3.91111111111111 0.209882584429067
3.96060606060606 0.209688328243221
4.01010101010101 0.209497420746966
4.05959595959596 0.209309776608933
4.10909090909091 0.209125313353045
4.15858585858586 0.208943951240964
4.20808080808081 0.208765613160247
4.25757575757576 0.208590224517906
4.30707070707071 0.208417713139055
4.35656565656566 0.208248009170384
4.40606060606061 0.208081044988184
4.45555555555556 0.207916755110642
4.50505050505051 0.207755076114346
4.55454545454546 0.207595946554446
4.6040404040404 0.207439307000006
4.65353535353535 0.20728509940438
4.7030303030303 0.207133268149872
4.75252525252525 0.206983758867498
4.8020202020202 0.206836518930701
4.85151515151515 0.206691497283467
4.9010101010101 0.206548644382056
4.95050505050505 0.206407912141442
5 0.206269253879754
};
\addplot [semithick, gray]
table {%
0.1 -1.11022302462516e-16
0.14949494949495 0
0.198989898989899 -1.11022302462516e-16
0.248484848484848 0
0.297979797979798 0
0.347474747474747 0
0.396969696969697 0
0.446464646464647 0
0.495959595959596 0
0.545454545454546 0
0.594949494949495 1.11022302462516e-16
0.644444444444444 0
0.693939393939394 0
0.743434343434343 -2.22044604925031e-16
0.792929292929293 0
0.842424242424242 1.11022302462516e-16
0.891919191919192 0
0.941414141414142 0
0.990909090909091 0
1.04040404040404 0
1.08989898989899 0
1.13939393939394 0
1.18888888888889 1.11022302462516e-16
1.23838383838384 0
1.28787878787879 1.11022302462516e-16
1.33737373737374 0
1.38686868686869 0
1.43636363636364 1.11022302462516e-16
1.48585858585859 1.11022302462516e-16
1.53535353535354 0
1.58484848484849 -2.22044604925031e-16
1.63434343434343 1.11022302462516e-16
1.68383838383838 0
1.73333333333333 0
1.78282828282828 2.22044604925031e-16
1.83232323232323 1.11022302462516e-16
1.88181818181818 -1.11022302462516e-16
1.93131313131313 0
1.98080808080808 0
2.03030303030303 1.11022302462516e-16
2.07979797979798 0
2.12929292929293 1.11022302462516e-16
2.17878787878788 0
2.22828282828283 0
2.27777777777778 0
2.32727272727273 1.11022302462516e-16
2.37676767676768 0
2.42626262626263 0
2.47575757575758 -2.22044604925031e-16
2.52525252525253 0
2.57474747474748 -2.22044604925031e-16
2.62424242424242 1.11022302462516e-16
2.67373737373737 0
2.72323232323232 0
2.77272727272727 0
2.82222222222222 0
2.87171717171717 -1.11022302462516e-16
2.92121212121212 0
2.97070707070707 0
3.02020202020202 1.11022302462516e-16
3.06969696969697 0
3.11919191919192 1.11022302462516e-16
3.16868686868687 1.11022302462516e-16
3.21818181818182 1.11022302462516e-16
3.26767676767677 0
3.31717171717172 0
3.36666666666667 0
3.41616161616162 1.11022302462516e-16
3.46565656565657 1.11022302462516e-16
3.51515151515152 -1.11022302462516e-16
3.56464646464647 0
3.61414141414142 1.11022302462516e-16
3.66363636363636 -1.11022302462516e-16
3.71313131313131 0
3.76262626262626 0
3.81212121212121 0
3.86161616161616 0
3.91111111111111 0
3.96060606060606 1.11022302462516e-16
4.01010101010101 -3.33066907387547e-16
4.05959595959596 0
4.10909090909091 0
4.15858585858586 0
4.20808080808081 -1.11022302462516e-16
4.25757575757576 1.11022302462516e-16
4.30707070707071 -1.11022302462516e-16
4.35656565656566 1.11022302462516e-16
4.40606060606061 0
4.45555555555556 0
4.50505050505051 0
4.55454545454546 0
4.6040404040404 0
4.65353535353535 1.11022302462516e-16
4.7030303030303 0
4.75252525252525 1.11022302462516e-16
4.8020202020202 0
4.85151515151515 0
4.9010101010101 0
4.95050505050505 -1.11022302462516e-16
5 0
};
\addplot [semithick, orange, dashed]
table {%
0.1 0.0324988421874998
0.14949494949495 0.0155631376865124
0.198989898989899 0.00959326266178917
0.248484848484848 0.0068369423758935
0.297979797979798 0.00528811673817464
0.347474747474747 0.0043030923030315
0.396969696969697 0.00362324073494624
0.446464646464647 0.00312670589825559
0.495959595959596 0.00274802396263973
0.545454545454546 0.00244989192711276
0.594949494949495 0.00220952707422872
0.644444444444444 0.00201139357225222
0.693939393939394 0.00184529192912197
0.743434343434343 0.00170407402871009
0.792929292929293 0.00158254665131263
0.842424242424242 0.00147686580052508
0.891919191919192 0.00138412571999502
0.941414141414142 0.00130208939517806
0.990909090909091 0.00122925510954242
1.04040404040404 0.00116372851101665
1.08989898989899 0.00110465232116042
1.13939393939394 0.00105111742180164
1.18888888888889 0.00100238556987031
1.23838383838384 0.000957833166263078
1.28787878787879 0.000916946689105247
1.33737373737374 0.000879291819283146
1.38686868686869 0.000844500215786526
1.43636363636364 0.000812257434496977
1.48585858585859 0.000782293425184433
1.53535353535354 0.000754374968890081
1.58484848484849 0.000728299607934191
1.63434343434343 0.000703891047529215
1.68383838383838 0.000680996774300002
1.73333333333333 0.000659477904800809
1.78282828282828 0.000639214272817767
1.83232323232323 0.000620098977713979
1.88181818181818 0.000602037226854923
1.93131313131313 0.00058494428358391
1.98080808080808 0.000568744257704035
2.03030303030303 0.000553368982037394
2.07979797979798 0.00053875705417139
2.12929292929293 0.000524853017429994
2.17878787878788 0.000511606658863184
2.22828282828283 0.000498972403501163
2.27777777777778 0.000486908792487473
2.32727272727273 0.000475378056984943
2.37676767676768 0.00046434560239883
2.42626262626263 0.000453779866266824
2.47575757575758 0.00044365181501016
2.52525252525253 0.000433934825386117
2.57474747474748 0.000424604375106163
2.62424242424242 0.000415637852772432
2.67373737373737 0.000407014663803551
2.72323232323232 0.000398714918495191
2.77272727272727 0.000390721014282724
2.82222222222222 0.000383016355186783
2.87171717171717 0.000375585522459532
2.92121212121212 0.000368414176140752
2.97070707070707 0.000361488958938505
3.02020202020202 0.000354797415281216
3.06969696969697 0.000348327914948543
3.11919191919192 0.00034206958890004
3.16868686868687 0.000336012262662777
3.21818181818182 0.000330146410326138
3.26767676767677 0.000324463092957972
3.31717171717172 0.000318953922927245
3.36666666666667 0.000313611016878435
3.41616161616162 0.000308426960653474
3.46565656565657 0.00030339477508734
3.51515151515152 0.000298507883218613
3.56464646464647 0.000293760083120653
3.61414141414142 0.000288975679155978
3.66363636363636 0.000284487312503856
3.71313131313131 0.000280125295313405
3.76262626262626 0.000275878505711757
3.81212121212121 0.000271744554955267
3.86161616161616 0.000267719022206125
3.91111111111111 0.000263797684205813
3.96060606060606 0.000259976553306784
4.01010101010101 0.000256251832386756
4.05959595959596 0.000252619904086981
4.10909090909091 0.000249077375335305
4.15858585858586 0.000245620960818482
4.20808080808081 0.000242247589543942
4.25757575757576 0.000238954271298164
4.30707070707071 0.000235738190265566
4.35656565656566 0.00023259670297171
4.40606060606061 0.000229527184468559
4.45555555555556 0.000226527223262663
4.50505050505051 0.000223594474750288
4.55454545454546 0.000220724776531345
4.6040404040404 0.000217923719857738
4.65353535353535 0.000215179585170255
4.7030303030303 0.0002124943148194
4.75252525252525 0.000209866010358106
4.8020202020202 0.000207292885331944
4.85151515151515 0.000204773219275145
4.9010101010101 0.000202305379928158
4.95050505050505 0.000199887750205296
5 0.000197518831721255
};
\addplot [semithick, blue]
table {%
0.1 0.102770860961542
0.14949494949495 0.080178526995685
0.198989898989899 0.0678313117484063
0.248484848484848 0.0604866333399062
0.297979797979798 0.0556103183717084
0.347474747474747 0.0520960699395439
0.396969696969697 0.0494070974007256
0.446464646464647 0.0472551655593881
0.495959595959596 0.0454725492835596
0.545454545454546 0.0439552997116017
0.594949494949495 0.0426356441117984
0.644444444444444 0.0414675591433086
0.693939393939394 0.0404186948460539
0.743434343434343 0.0394656538299866
0.792929292929293 0.0385911086404146
0.842424242424242 0.0377819327197273
0.891919191919192 0.0370279971131858
0.941414141414142 0.0363213581092364
0.990909090909091 0.0356556740241174
1.04040404040404 0.035025871641444
1.08989898989899 0.0344278065657759
1.13939393939394 0.0338579835648153
1.18888888888889 0.0333135617361411
1.23838383838384 0.0327921423209421
1.28787878787879 0.0322916731942019
1.33737373737374 0.0318102997769151
1.38686868686869 0.031346630027952
1.43636363636364 0.0308992380416231
1.48585858585859 0.0304670404483303
1.53535353535354 0.030048986257349
1.58484848484849 0.0296442066973881
1.63434343434343 0.0292517521581636
1.68383838383838 0.0288711128546837
1.73333333333333 0.0285015072598687
1.78282828282828 0.0281423154201117
1.83232323232323 0.0277930818581821
1.88181818181818 0.0274532544653614
1.93131313131313 0.0271224140961437
1.98080808080808 0.0268002848644098
2.03030303030303 0.0264861154583568
2.07979797979798 0.0261798844169668
2.12929292929293 0.0258811450492779
2.17878787878788 0.025589624061378
2.22828282828283 0.0253049686662207
2.27777777777778 0.0250269814569111
2.32727272727273 0.0247554017314695
2.37676767676768 0.0244899774013915
2.42626262626263 0.0242304460920884
2.47575757575758 0.0239767001483642
2.52525252525253 0.0237283856068605
2.57474747474748 0.0234854042699687
2.62424242424242 0.0232476495558187
2.67373737373737 0.0230148476964028
2.72323232323232 0.0227868774648842
2.77272727272727 0.0225635234728309
2.82222222222222 0.0223447622673787
2.87171717171717 0.0221303347854165
2.92121212121212 0.0219201779245642
2.97070707070707 0.0217141308629495
3.02020202020202 0.0215120161046422
3.06969696969697 0.02131389725702
3.11919191919192 0.0211195017244606
3.16868686868687 0.02092877270553
3.21818181818182 0.0207415553866458
3.26767676767677 0.020557818774636
3.31717171717172 0.0203773971690405
3.36666666666667 0.020200291371071
3.41616161616162 0.0200264472417596
3.46565656565657 0.0198555802508023
3.51515151515152 0.0196878047782828
3.56464646464647 0.0195229406895995
3.61414141414142 0.0193608777154277
3.66363636363636 0.0192016868926825
3.71313131313131 0.0190451272760727
3.76262626262626 0.0188912735572452
3.81212121212121 0.0187399624876549
3.86161616161616 0.0185911838026934
3.91111111111111 0.0184448275591345
3.96060606060606 0.0183009492439867
4.01010101010101 0.0181592735104683
4.05959595959596 0.0180200136877365
4.10909090909091 0.0178827853766758
4.15858585858586 0.0177480241720923
4.20808080808081 0.0176153008841788
4.25757575757576 0.0174844407487549
4.30707070707071 0.0173557976662446
4.35656565656566 0.0172290770951716
4.40606060606061 0.0171042539587238
4.45555555555556 0.0169813312104248
4.50505050505051 0.016860297219796
4.55454545454546 0.0167410782029123
4.6040404040404 0.0166236225104964
4.65353535353535 0.0165078647668788
4.7030303030303 0.0163937361310268
4.75252525252525 0.0162814440779194
4.8020202020202 0.016170638997124
4.85151515151515 0.0160615098503739
4.9010101010101 0.0159539127309375
4.95050505050505 0.0158476381273318
5 0.0157430888269137
};
\addplot [semithick, red]
table {%
0.1 0.0504249218223812
0.14949494949495 0.0335425147430458
0.198989898989899 0.025056094264808
0.248484848484848 0.0199779397413156
0.297979797979798 0.0166027716183277
0.347474747474747 0.0142009324300398
0.396969696969697 0.0124064592954553
0.446464646464647 0.0110158221534286
0.495959595959596 0.00990696939120006
0.545454545454546 0.00900238164983136
0.594949494949495 0.0082505240276628
0.644444444444444 0.00761581502065489
0.693939393939394 0.00707291056865711
0.743434343434343 0.00660325452805532
0.792929292929293 0.00619297940425301
0.842424242424242 0.00583151051363551
0.891919191919192 0.00551064086279129
0.941414141414142 0.00522390277749896
0.990909090909091 0.00496611280828108
1.04040404040404 0.00473311895047068
1.08989898989899 0.00452149633112631
1.13939393939394 0.00432845124732761
1.18888888888889 0.00415163259737039
1.23838383838384 0.00398908170980405
1.28787878787879 0.00383913741764041
1.33737373737374 0.00370038303319919
1.38686868686869 0.00357161544831086
1.43636363636364 0.00345179734279633
1.48585858585859 0.00334002590994331
1.53535353535354 0.00323552743854794
1.58484848484849 0.00313758915480111
1.63434343434343 0.00304564053601442
1.68383838383838 0.00295911499021961
1.73333333333333 0.00287758328331167
1.78282828282828 0.00280059402195842
1.83232323232323 0.00272780907324033
1.88181818181818 0.00265885367047325
1.93131313131313 0.00259347089841211
1.98080808080808 0.00253139836111149
2.03030303030303 0.00247233262542035
2.07979797979798 0.00241610148161631
2.12929292929293 0.00236250385982417
2.17878787878788 0.00231136957914868
2.22828282828283 0.0022625172676729
2.27777777777778 0.00221579675038919
2.32727272727273 0.00217105695299336
2.37676767676768 0.00212821062862689
2.42626262626263 0.00208711000612272
2.47575757575758 0.00204764692504456
2.52525252525253 0.00200977346460596
2.57474747474748 0.0019733287994147
2.62424242424242 0.00193828765573689
2.67373737373737 0.00190455394779165
2.72323232323232 0.00187202334836634
2.77272727272727 0.00184070679065618
2.82222222222222 0.00181043471551823
2.87171717171717 0.00178126248850652
2.92121212121212 0.00175306066610714
2.97070707070707 0.00172581064576749
3.02020202020202 0.00169943900512071
3.06969696969697 0.00167392790966248
3.11919191919192 0.0016492458125813
3.16868686868687 0.00162530202467237
3.21818181818182 0.00160214012742677
3.26767676767677 0.00157965577955432
3.31717171717172 0.00155784321382113
3.36666666666667 0.00153671073862938
3.41616161616162 0.00151616781149055
3.46565656565657 0.00149622507585112
3.51515151515152 0.00147683775195673
3.56464646464647 0.00145797008280257
3.61414141414142 0.00143965214338726
3.66363636363636 0.00142181789742546
3.71313131313131 0.00140445913139131
3.76262626262626 0.00138758560256214
3.81212121212121 0.00137108530821184
3.86161616161616 0.00135504259538843
3.91111111111111 0.00133942721634916
3.96060606060606 0.00132417264078855
4.01010101010101 0.00130933031136315
4.05959595959596 0.0012948335094678
4.10909090909091 0.00128068402588599
4.15858585858586 0.00126683989133114
4.20808080808081 0.00125336914909846
4.25757575757576 0.00124021380262451
4.30707070707071 0.00122736395758494
4.35656565656566 0.0012147626753809
4.40606060606061 0.00120250961718249
4.45555555555556 0.00119053520951051
4.50505050505051 0.0011787572605827
4.55454545454546 0.00116727404468897
4.6040404040404 0.00115600514166692
4.65353535353535 0.00114500677650864
4.7030303030303 0.00113423377933608
4.75252525252525 0.00112365993869956
4.8020202020202 0.00111331782033708
4.85151515151515 0.00110320331431057
4.9010101010101 0.00109328131516595
4.95050505050505 0.00108360615769931
5 0.00107405925692061
};
\addplot [semithick, blue, dashed]
table {%
0.1 0.101905481798863
0.14949494949495 0.0786989882656534
0.198989898989899 0.0660750203831522
0.248484848484848 0.058595417579009
0.297979797979798 0.0536485790159289
0.347474747474747 0.0500930585275551
0.396969696969697 0.0473805959107824
0.446464646464647 0.0452155709383434
0.495959595959596 0.0434276718064533
0.545454545454546 0.0419095431167699
0.594949494949495 0.0405926230029667
0.644444444444444 0.0394293134675325
0.693939393939394 0.0383867455428617
0.743434343434343 0.0374414761173449
0.792929292929293 0.0365756565028775
0.842424242424242 0.035775414009099
0.891919191919192 0.0350316555279575
0.941414141414142 0.0343349328796002
0.990909090909091 0.0336794354455612
1.04040404040404 0.033050489329373
1.08989898989899 0.0324631784172976
1.13939393939394 0.031904051899585
1.18888888888889 0.0313702778185838
1.23838383838384 0.0308594302986774
1.28787878787879 0.0303693888965559
1.33737373737374 0.0298985859987368
1.38686868686869 0.0294449185000655
1.43636363636364 0.0290078211444692
1.48585858585859 0.0285850288529035
1.53535353535354 0.0281769058066815
1.58484848484849 0.027781538786034
1.63434343434343 0.027398389548248
1.68383838383838 0.0270269140397292
1.73333333333333 0.0266657105324443
1.78282828282828 0.0263156735160662
1.83232323232323 0.0259745439999914
1.88181818181818 0.0256432268641599
1.93131313131313 0.0251637344918643
1.98080808080808 0.0250167471127288
2.03030303030303 0.0247101597218521
2.07979797979798 0.0244115688666192
2.12929292929293 0.0241202211116277
2.17878787878788 0.0238359041612314
2.22828282828283 0.0235583354214238
2.27777777777778 0.0232872516037722
2.32727272727273 0.0230224038274388
2.37676767676768 0.0227635592298966
2.42626262626263 0.0225104986985396
2.47575757575758 0.0222630153701678
2.52525252525253 0.0220209144430432
2.57474747474748 0.0217840097276014
2.62424242424242 0.0215521270202955
2.67373737373737 0.0213250989940581
2.72323232323232 0.0211027681935577
2.77272727272727 0.0208849844914791
2.82222222222222 0.020671603687863
2.87171717171717 0.0204624892260601
2.92121212121212 0.0202575105193868
2.97070707070707 0.0200565424558192
3.02020202020202 0.0198594663619639
3.06969696969697 0.0196661671445715
3.11919191919192 0.0194765362355392
3.16868686868687 0.0192904675689917
3.21818181818182 0.0191078606631314
3.26767676767677 0.0189286188422559
3.31717171717172 0.0187524925612269
3.36666666666667 0.0185797033640795
3.41616161616162 0.0184100101830834
3.46565656565657 0.0182433310270385
3.51515151515152 0.0180795852321922
3.56464646464647 0.0179186966860002
3.61414141414142 0.0177605947270661
3.66363636363636 0.0176051995773447
3.71313131313131 0.0174524469515659
3.76262626262626 0.0173022693800385
3.81212121212121 0.017154602856632
3.86161616161616 0.0170095493206632
3.91111111111111 0.0168667201509629
3.96060606060606 0.0167262214950431
4.01010101010101 0.0165879970276442
4.05959595959596 0.0164519917664321
4.10909090909091 0.0163181540923424
4.15858585858586 0.016186431924866
4.20808080808081 0.0160567765023371
4.25757575757576 0.0159291398147237
4.30707070707071 0.0158034743688282
4.35656565656566 0.0156797354575064
4.40606060606061 0.0155578795029645
4.45555555555556 0.0154378654827294
4.50505050505051 0.0153196524199128
4.55454545454546 0.0152031997391139
4.6040404040404 0.015088468791856
4.65353535353535 0.0149754216053399
4.7030303030303 0.0148640231766749
4.75252525252525 0.014754236612143
4.8020202020202 0.0146460282382918
4.85151515151515 0.0145393652996656
4.9010101010101 0.0144342135960033
4.95050505050505 0.0143305423912063
5 0.0142283216934058
};
\end{axis}

\end{tikzpicture}

%% file: Figures/supplementary/n_over_p=10/conditional_variance_bo_p=0.75_n_over_p=10.0.tex
\begin{tikzpicture}
\tikzstyle{every node}=[font=\tiny]

\definecolor{darkgray176}{RGB}{176,176,176}
\definecolor{lightgray204}{RGB}{204,204,204}
\definecolor{orange}{RGB}{255,165,0}

\begin{axis}[
height=\figheight,
legend cell align={left},
legend style={fill opacity=0.8, draw opacity=1, text opacity=1, draw=lightgray204},
tick align=outside,
tick pos=left,
width=\figwidth,
x grid style={darkgray176},
xlabel={$p / n$},
xmajorgrids,
xmin=-0.145, xmax=5.245,
xtick style={color=black},
y grid style={darkgray176},
ylabel={\(\displaystyle {\rm Var}(\hat{f}_{\rm bo} | \hat{f} = 0.75)\)},
ymajorgrids,
ymin=-0.000993631967394773, ymax=0.0234164433579886,
ytick style={color=black}
]
\addplot [semithick, black]
table {%
0.1 0.00820791846185825
0.14949494949495 0.00918955938078353
0.198989898989899 0.0116781206960656
0.248484848484848 0.0142530746685616
0.297979797979798 0.0165738017264032
0.347474747474747 0.0186033618728172
0.396969696969697 0.0203604039342803
0.446464646464647 0.0217381941563241
0.495959595959596 0.0223068944795621
0.545454545454546 0.0222117232229971
0.594949494949495 0.0219537821145608
0.644444444444444 0.0216565182887251
0.693939393939394 0.0213475355006242
0.743434343434343 0.0210362039013621
0.792929292929293 0.0207266324693841
0.842424242424242 0.0204209189453021
0.891919191919192 0.0201202304137075
0.941414141414142 0.019825241547707
0.990909090909091 0.0195363406764404
1.04040404040404 0.0192537383753939
1.08989898989899 0.0189775286332045
1.13939393939394 0.0187077271631509
1.18888888888889 0.0184442957100403
1.23838383838384 0.0181871581375274
1.28787878787879 0.0179362115300685
1.33737373737374 0.0176913347480395
1.38686868686869 0.0174523937110549
1.43636363636364 0.0172192459635442
1.48585858585859 0.0169917436410637
1.53535353535354 0.0167697366639723
1.58484848484849 0.0165530718905219
1.63434343434343 0.0163415987333854
1.68383838383838 0.0161351671739109
1.73333333333333 0.0159336286201562
1.78282828282828 0.0157368380986364
1.83232323232323 0.0155446533533157
1.88181818181818 0.015356935114649
1.93131313131313 0.0151735480407068
1.98080808080808 0.0149943608768271
2.03030303030303 0.01481924573257
2.07979797979798 0.014648078765809
2.12929292929293 0.0144807400210836
2.17878787878788 0.0143171134039858
2.22828282828283 0.0141570866327603
2.27777777777778 0.0140005511694927
2.32727272727273 0.013847402142608
2.37676767676768 0.0136975382595016
2.42626262626263 0.0135508617089791
2.47575757575758 0.0134072780632994
2.52525252525253 0.0132666961747891
2.57474747474748 0.013129028072678
2.62424242424242 0.0129941888565255
2.67373737373737 0.0128620965927949
2.72323232323232 0.0127326722121794
2.77272727272727 0.0126058394058951
2.82222222222222 0.0124815245270193
2.87171717171717 0.0123596564921942
2.92121212121212 0.0122401666873432
2.97070707070707 0.0121229888737945
3.02020202020202 0.0120080590993558
3.06969696969697 0.0118953143379359
3.11919191919192 0.0117846995590544
3.16868686868687 0.0116761509839322
3.21818181818182 0.0115696165951908
3.26767676767677 0.0114650418432855
3.31717171717172 0.0113623747723679
3.36666666666667 0.0112615651657955
3.41616161616162 0.0111625644792229
3.46565656565657 0.01106532577635
3.51515151515152 0.0109698036671317
3.56464646464647 0.010875954248669
3.61414141414142 0.0107837351014565
3.66363636363636 0.0106931049643264
3.71313131313131 0.0106040242044164
3.76262626262626 0.0105164543790934
3.81212121212121 0.0104303580350898
3.86161616161616 0.0103456992333951
3.91111111111111 0.0102624429152345
3.96060606060606 0.010180555245178
4.01010101010101 0.0101000034071401
4.05959595959596 0.0100207555981453
4.10909090909091 0.00994278099100793
4.15858585858586 0.00986604969855692
4.20808080808081 0.00979053273933639
4.25757575757576 0.00971620200471346
4.30707070707071 0.00964303022733826
4.35656565656566 0.00957099095091679
4.40606060606061 0.00950005850135061
4.45555555555556 0.00943020795801419
4.50505050505051 0.00936141512854577
4.55454545454546 0.00929365652215353
4.6040404040404 0.00922690928308956
4.65353535353535 0.00916115137811335
4.7030303030303 0.00909636115169477
4.75252525252525 0.00903251772639169
4.8020202020202 0.00896960077088699
4.85151515151515 0.00890759052213341
4.9010101010101 0.00884646776625386
4.95050505050505 0.00878621381927719
5 0.00872681051049001
};
\addplot [semithick, blue]
table {%
0.1 0.00688024631964118
0.14949494949495 0.00480200234584627
0.198989898989899 0.00358438399120009
0.248484848484848 0.00282462162181041
0.297979797979798 0.00231982016162041
0.347474747474747 0.00196408829511741
0.396969696969697 0.001701191069231
0.446464646464647 0.00149947878035889
0.495959595959596 0.00134003214979472
0.545454545454546 0.00121092807553447
0.594949494949495 0.00110430930785621
0.644444444444444 0.00101479996258835
0.693939393939394 0.000938603193751719
0.743434343434343 0.000872963831720974
0.792929292929293 0.000815835534814346
0.842424242424242 0.000765667284364779
0.891919191919192 0.000721262519386756
0.941414141414142 0.000681683596948646
0.990909090909091 0.00064618544690509
1.04040404040404 0.000614168804130566
1.08989898989899 0.000585146145966053
1.13939393939394 0.000558716668364156
1.18888888888889 0.000534548187758443
1.23838383838384 0.000512362896966656
1.28787878787879 0.000491926666891795
1.33737373737374 0.00047304061552611
1.38686868686869 0.00045553533459064
1.43636363636364 0.000439264809675044
1.48585858585859 0.000424103226307038
1.53535353535354 0.000409941125944058
1.58484848484849 0.000396683109979978
1.63434343434343 0.000384245249640869
1.68383838383838 0.000372554255279822
1.73333333333333 0.000361544789470258
1.78282828282828 0.000351159070259288
1.83232323232323 0.000341345829972695
1.88181818181818 0.000332059063720469
1.93131313131313 0.000323257693814805
1.98080808080808 0.000314904967727458
2.03030303030303 0.00030696704948463
2.07979797979798 0.000299414343408966
2.12929292929293 0.000292219434533281
2.17878787878788 0.000285357633007965
2.22828282828283 0.000278806342354443
2.27777777777778 0.00027254513011965
2.32727272727273 0.000266555224091469
2.37676767676768 0.00026081942967604
2.42626262626263 0.000255321941108888
2.47575757575758 0.00025004839959164
2.52525252525253 0.000244985287884858
2.57474747474748 0.00024012041726329
2.62424242424242 0.00023544250909191
2.67373737373737 0.000230940936641999
2.72323232323232 0.000226606008600139
2.77272727272727 0.00022242862547095
2.82222222222222 0.000218400524378159
2.87171717171717 0.000214513770839031
2.92121212121212 0.000210761163738749
2.97070707070707 0.000207135884438547
3.02020202020202 0.000203631543946425
3.06969696969697 0.000200242395285044
3.11919191919192 0.000196962738354345
3.16868686868687 0.000193787428458214
3.21818181818182 0.000190711546755873
3.26767676767677 0.000187730583150358
3.31717171717172 0.0001848401760568
3.36666666666667 0.000182036365414806
3.41616161616162 0.000179315372442757
3.46565656565657 0.000176673436800723
3.51515151515152 0.000174107344570462
3.56464646464647 0.000171613815244664
3.61414141414142 0.000169189804326453
3.66363636363636 0.00016683257775052
3.71313131313131 0.000164539314733148
3.76262626262626 0.000162307540108375
3.81212121212121 0.00016013477143384
3.86161616161616 0.000158018754057854
3.91111111111111 0.000155957271155582
3.96060606060606 0.000153948340279819
4.01010101010101 0.000151989830165511
4.05959595959596 0.000150080073378245
4.10909090909091 0.000148217046974097
4.15858585858586 0.000146399411804721
4.20808080808081 0.000144625279146049
4.25757575757576 0.000142893025387947
4.30707070707071 0.000141201467269481
4.35656565656566 0.000139549030957364
4.40606060606061 0.00013793439530263
4.45555555555556 0.000136356316288166
4.50505050505051 0.000134813594560801
4.55454545454546 0.000133305041751686
4.6040404040404 0.00013182953459201
4.65353535353535 0.000130385988421189
4.7030303030303 0.000128973361680162
4.75252525252525 0.000127590830041102
4.8020202020202 0.000126237255369199
4.85151515151515 0.00012491187835284
4.9010101010101 0.000123613766445985
4.95050505050505 0.000122341984311047
5 0.000121096000813004
};
\addplot [semithick, blue, dashed]
table {%
0.1 0.00706081992944563
0.14949494949495 0.00500142608585313
0.198989898989899 0.00376727750202122
0.248484848484848 0.00298577310718262
0.297979797979798 0.00246173043786047
0.347474747474747 0.00209009023020729
0.396969696969697 0.00181417765255265
0.446464646464647 0.00160173882657783
0.495959595959596 0.00143335775420872
0.545454545454546 0.00129671602609333
0.594949494949495 0.00118366789396379
0.644444444444444 0.00108861401711119
0.693939393939394 0.00100759084856822
0.743434343434343 0.000937715984730414
0.792929292929293 0.00087684167651858
0.842424242424242 0.000823335842973449
0.891919191919192 0.000775942918915584
0.941414141414142 0.000733669554705552
0.990909090909091 0.000695731279342082
1.04040404040404 0.00066137965322477
1.08989898989899 0.000630334544434286
1.13939393939394 0.000602049829513285
1.18888888888889 0.000576173540078717
1.23838383838384 0.000552410983687079
1.28787878787879 0.000530513510037689
1.33737373737374 0.000510270539402169
1.38686868686869 0.000491500439910142
1.43636363636364 0.000474049820975819
1.48585858585859 0.000457782353076031
1.53535353535354 0.000442584402491342
1.58484848484849 0.000428352267022813
1.63434343434343 0.000414997334122369
1.68383838383838 0.00040244139608242
1.73333333333333 0.000390613701163833
1.78282828282828 0.000379455237308368
1.83232323232323 0.000368908203061735
1.88181818181818 0.000358925927150255
1.93131313131313 0.000349219906773279
1.98080808080808 0.000340497066367074
2.03030303030303 0.000331958894479789
2.07979797979798 0.000323833963950193
2.12929292929293 0.000316092507273047
2.17878787878788 0.000308708210635444
2.22828282828283 0.000301657001251954
2.27777777777778 0.000294916922524369
2.32727272727273 0.000288467902466616
2.37676767676768 0.000282291559404957
2.42626262626263 0.000276371026641464
2.47575757575758 0.000270690799109241
2.52525252525253 0.000265236599668262
2.57474747474748 0.000259995256973511
2.62424242424242 0.000254954605105717
2.67373737373737 0.000250103384687228
2.72323232323232 0.000245431163869214
2.77272727272727 0.000240928261693418
2.82222222222222 0.000236585680751134
2.87171717171717 0.000232395050141609
2.92121212121212 0.000228348570675552
2.97070707070707 0.000224438966776375
3.02020202020202 0.000220659444594928
3.06969696969697 0.000217003650355152
3.11919191919192 0.000213465638382204
3.16868686868687 0.000210039835365383
3.21818181818182 0.000206721014716305
3.26767676767677 0.000203504268672661
3.31717171717172 0.000200384838708767
3.36666666666667 0.000197358676808812
3.41616161616162 0.000194421550325008
3.46565656565657 0.000191569608856468
3.51515151515152 0.000188799218887148
3.56464646464647 0.00018610695155985
3.61414141414142 0.000183489569474116
3.66363636363636 0.000180944001463068
3.71313131313131 0.000178467370773205
3.76262626262626 0.000176056893340104
3.81212121212121 0.000173709985877624
3.86161616161616 0.000171424312481627
3.91111111111111 0.000169197264924481
3.96060606060606 0.000167026752727284
4.01010101010101 0.000164910662034501
4.05959595959596 0.000162846982414311
4.10909090909091 0.000160833802184013
4.15858585858586 0.000158869299692665
4.20808080808081 0.000156951740843736
4.25757575757576 0.000155079472149544
4.30707070707071 0.000153250915947312
4.35656565656566 0.000151464567734516
4.40606060606061 0.000149718990905989
4.45555555555556 0.000148012813999787
4.50505050505051 0.000146344725394099
4.55454545454546 0.000144713470312974
4.6040404040404 0.000143117848823149
4.65353535353535 0.000141556712090374
4.7030303030303 0.000140028961229777
4.75252525252525 0.000138533541469865
4.8020202020202 0.000137069443410631
4.85151515151515 0.000135635698955583
4.9010101010101 0.000134231377916239
4.95050505050505 0.000132855589179304
5 0.000131507477087012
};
\addplot [semithick, red]
table {%
0.1 0.00614156479609435
0.14949494949495 0.00440195432668711
0.198989898989899 0.0033252276892346
0.248484848484848 0.00263335589364189
0.297979797979798 0.00216776168936639
0.347474747474747 0.00183760916704034
0.396969696969697 0.00159280169568654
0.446464646464647 0.00140461568711536
0.495959595959596 0.00125570007118414
0.545454545454546 0.00113504898741801
0.594949494949495 0.00103537880917814
0.644444444444444 0.000951691504507912
0.693939393939394 0.000880449805713002
0.743434343434343 0.00081908302898781
0.792929292929293 0.000765679687299459
0.842424242424242 0.000718789804090791
0.891919191919192 0.00067729408127104
0.941414141414142 0.000640315010452586
0.990909090909091 0.000607155110045277
1.04040404040404 0.000577253146141876
1.08989898989899 0.000550152466161635
1.13939393939394 0.000525477792812867
1.18888888888889 0.000502917888027588
1.23838383838384 0.000482212485376166
1.28787878787879 0.000463142279687911
1.33737373737374 0.000445521182425135
1.38686868686869 0.000429190276041713
1.43636363636364 0.000414013032524974
1.48585858585859 0.000399871508237171
1.53535353535354 0.000386663292799905
1.58484848484849 0.000374299025768465
1.63434343434343 0.000362700409606176
1.68383838383838 0.000351798508890222
1.73333333333333 0.000341532427615032
1.78282828282828 0.000331848160004755
1.83232323232323 0.000322697649720483
1.88181818181818 0.000314037974785464
1.93131313131313 0.000305830705689014
1.98080808080808 0.000298041314956587
2.03030303030303 0.000290638688759115
2.07979797979798 0.000283594741563031
2.12929292929293 0.000276884034577707
2.17878787878788 0.000270483479704264
2.22828282828283 0.000264372068728536
2.27777777777778 0.000258530650173783
2.32727272727273 0.000252941724761757
2.37676767676768 0.000247589279653182
2.42626262626263 0.000242458618666785
2.47575757575758 0.000237536241720293
2.52525252525253 0.000232809728562322
2.57474747474748 0.000228267612798927
2.62424242424242 0.000223899320191401
2.67373737373737 0.000219695061162373
2.72323232323232 0.000215645765803685
2.77272727272727 0.000211743032404788
2.82222222222222 0.000207979036316752
2.87171717171717 0.000204346522447829
2.92121212121212 0.00020083871405574
2.97070707070707 0.000197449302470121
3.02020202020202 0.000194172391839498
3.06969696969697 0.000191002476211288
3.11919191919192 0.000187934401634005
3.16868686868687 0.00018496333422402
3.21818181818182 0.00018208475288195
3.26767676767677 0.000179294400118102
3.31717171717172 0.000176588283870127
3.36666666666667 0.000173962650419135
3.41616161616162 0.000171413957498756
3.46565656565657 0.000168938875617863
3.51515151515152 0.000166534260168549
3.56464646464647 0.000164197143649236
3.61414141414142 0.000161924727806517
3.66363636363636 0.000159714359676988
3.71313131313131 0.000157563549211925
3.76262626262626 0.000155469893471483
3.81212121212121 0.00015343115233446
3.86161616161616 0.000151445200605504
3.91111111111111 0.000149510013551901
3.96060606060606 0.00014762366610277
4.01010101010101 0.000145784337040267
4.05959595959596 0.000143990286876616
4.10909090909091 0.000142239865161264
4.15858585858586 0.000140531497887175
4.20808080808081 0.000138863692644375
4.25757575757576 0.000137235019153081
4.30707070707071 0.000135644116390332
4.35656565656566 0.000134089682894634
4.40606060606061 0.0001325704853683
4.45555555555556 0.000131085335893788
4.50505050505051 0.0001296330985322
4.55454545454546 0.000128212697239505
4.6040404040404 0.000126823092697737
4.65353535353535 0.000125463297589556
4.7030303030303 0.000124132361033036
4.75252525252525 0.000122829373088607
4.8020202020202 0.000121553464653479
4.85151515151515 0.000120303800457577
4.9010101010101 0.000119079577676873
4.95050505050505 0.000117880030239381
5 0.000116704414183877
};
\addplot [semithick, orange, dashed]
table {%
0.1 0.00589007610532433
0.14949494949495 0.00427766743194824
0.198989898989899 0.00326227899134013
0.248484848484848 0.00259533539684609
0.297979797979798 0.00214140329573553
0.347474747474747 0.00181760574664835
0.396969696969697 0.00157670706400714
0.446464646464647 0.00139115005943469
0.495959595959596 0.00124412137329566
0.545454545454546 0.00112489152695461
0.594949494949495 0.00102633166762855
0.644444444444444 0.000943536369346187
0.693939393939394 0.000873027475098653
0.743434343434343 0.000812273723029167
0.792929292929293 0.000759390923480541
0.842424242424242 0.000712948632926924
0.891919191919192 0.000671841914822724
0.941414141414142 0.000635204066011141
0.990909090909091 0.000602345965714068
1.04040404040404 0.000572712711040646
1.08989898989899 0.000545852851401185
1.13939393939394 0.000521395218849974
1.18888888888889 0.000499031923629101
1.23838383838384 0.000478505441639521
1.28787878787879 0.000459598741431755
1.33737373737374 0.000442127637105383
1.38686868686869 0.000425934810432627
1.43636363636364 0.000410885090580737
1.48585858585859 0.000396861693930717
1.53535353535354 0.000383763203917087
1.58484848484849 0.000371501126564588
1.63434343434343 0.000359997904590181
1.68383838383838 0.000349185262212948
1.73333333333333 0.00033900286617794
1.78282828282828 0.000329397220075633
1.83232323232323 0.000320320705391564
1.88181818181818 0.000311730801738186
1.93131313131313 0.000303589422772754
1.98080808080808 0.000295862353279497
2.03030303030303 0.000288518769824497
2.07979797979798 0.000281530830868348
2.12929292929293 0.000274873324887515
2.17878787878788 0.000268523367145779
2.22828282828283 0.000262460137454457
2.27777777777778 0.000256664652597438
2.32727272727273 0.000251119568202562
2.37676767676768 0.000245809005684716
2.42626262626263 0.000240718400690954
2.47575757575758 0.00023583436991792
2.52525252525253 0.000231144593834665
2.57474747474748 0.000226637713078404
2.62424242424242 0.000222303236734334
2.67373737373737 0.000218131460951398
2.72323232323232 0.000214113396414639
2.77272727272727 0.000210240703929165
2.82222222222222 0.000206505636618948
2.87171717171717 0.000202900988289123
2.92121212121212 0.000199420047080068
2.97070707070707 0.000196056553804169
3.02020202020202 0.000192804664425394
3.06969696969697 0.000189658916209812
3.11919191919192 0.000186614197134838
3.16868686868687 0.000183665718193837
3.21818181818182 0.000180808988285652
3.26767676767677 0.000178039791411178
3.31717171717172 0.000175354165929731
3.36666666666667 0.000172748385665278
3.41616161616162 0.00017021894267466
3.46565656565657 0.000167762531500526
3.51515151515152 0.000165376034776621
3.56464646464647 0.00016305651003401
3.61414141414142 0.000160801166254387
3.66363636363636 0.000158607398243138
3.71313131313131 0.000156472724025414
3.76262626262626 0.000154394757487686
3.81212121212121 0.00015237128354606
3.86161616161616 0.000150400186733401
3.91111111111111 0.000148479459667228
3.96060606060606 0.000146607196235937
4.01010101010101 0.000144781585291964
4.05959595959596 0.000143000904811119
4.10909090909091 0.000141263516482359
4.15858585858586 0.00013956786067415
4.20808080808081 0.000137912451776745
4.25757575757576 0.000136295873855441
4.30707070707071 0.00013471677661725
4.35656565656566 0.000133173871654679
4.40606060606061 0.000131665928930547
4.45555555555556 0.000130191773522248
4.50505050505051 0.000128750282559209
4.55454545454546 0.000127340382276131
4.6040404040404 0.000125961045960077
4.65353535353535 0.000124611290017063
4.7030303030303 0.000123290173286716
4.75252525252525 0.000121996793918333
4.8020202020202 0.000120730287533788
4.85151515151515 0.00011948982529586
4.9010101010101 0.000118274612104341
4.95050505050505 0.0001170838848914
5 0.000115916911031744
};
\end{axis}

\end{tikzpicture}

%% file: Figures/supplementary/n_over_p=10/calibration_p=0.75_laplace_n_over_p=10.0.tex
\begin{tikzpicture}
\tikzstyle{every node}=[font=\tiny]

\definecolor{darkgray176}{RGB}{176,176,176}
\definecolor{gray}{RGB}{128,128,128}
\definecolor{lightgray204}{RGB}{204,204,204}
\definecolor{orange}{RGB}{255,165,0}

\begin{axis}[
height=\figheight,
legend cell align={left},
legend style={fill opacity=0.8, draw opacity=1, text opacity=1, draw=lightgray204},
tick align=outside,
tick pos=left,
width=\figwidth,
x grid style={darkgray176},
xlabel={$p / n$},
xmajorgrids,
xmin=-0.145, xmax=5.245,
xtick style={color=black},
y grid style={darkgray176},
ylabel={\(\displaystyle \Delta_{0.75}\)},
ymajorgrids,
ymin=-0.180557936040151, ymax=0.261409579473669,
ytick style={color=black}
]
\addplot [semithick, black, dash pattern=on 1pt off 3pt on 3pt off 3pt]
table {%
0.1 0.102569946229699
0.14949494949495 0.122385273955911
0.198989898989899 0.142631663674062
0.248484848484848 0.160654458194268
0.297979797979798 0.176518443978141
0.347474747474747 0.190784654816235
0.396969696969697 0.20390229352648
0.446464646464647 0.214581214313238
0.495959595959596 0.215818287164479
0.545454545454546 0.207185171942901
0.594949494949495 0.195668502905237
0.644444444444444 0.183434466868779
0.693939393939394 0.171078072141928
0.743434343434343 0.158836539165061
0.792929292929293 0.146828157665613
0.842424242424242 0.135119263566017
0.891919191919192 0.123748867135352
0.941414141414142 0.112739685900934
0.990909090909091 0.102103879866874
1.04040404040404 0.0918463977426938
1.08989898989899 0.0819670685701516
1.13939393939394 0.0724620436383075
1.18888888888889 0.0633247953781064
1.23838383838384 0.0545468554273751
1.28787878787879 0.04611836492093
1.33737373737374 0.0380285037236102
1.38686868686869 0.0302658116899611
1.43636363636364 0.022818444141444
1.48585858585859 0.0156743699661878
1.53535353535354 0.00882153837566946
1.58484848484849 0.00224794188906485
1.63434343434343 -0.00405819295509491
1.68383838383838 -0.0101084097933422
1.73333333333333 -0.0159139504306202
1.78282828282828 -0.0214857145759145
1.83232323232323 -0.02683423473345
1.88181818181818 -0.0319696579757471
1.93131313131313 -0.0369017363271451
1.98080808080808 -0.0416398214937398
2.03030303030303 -0.0461928665170334
2.07979797979798 -0.0505694293204857
2.12929292929293 -0.0547776800014425
2.17878787878788 -0.0588254100072785
2.22828282828283 -0.0627200429789304
2.27777777777778 -0.0664686468525389
2.32727272727273 -0.0700779466558671
2.37676767676768 -0.0735543378399959
2.42626262626263 -0.076903899943932
2.47575757575758 -0.0801324102573421
2.52525252525253 -0.0832453574606438
2.57474747474748 -0.086247955048979
2.62424242424242 -0.0891451545155831
2.67373737373737 -0.0919416581481235
2.72323232323232 -0.0946419314372062
2.77272727272727 -0.0972502150680711
2.82222222222222 -0.0997705364286675
2.87171717171717 -0.102206720657213
2.92121212121212 -0.104562401205939
2.97070707070707 -0.106841029926085
3.02020202020202 -0.109045886672806
3.06969696969697 -0.111180088450193
3.11919191919192 -0.113246591096104
3.16868686868687 -0.115248232454054
3.21818181818182 -0.117187670366689
3.26767676767677 -0.119067459845901
3.31717171717172 -0.120890025211603
3.36666666666667 -0.122657673682195
3.41616161616162 -0.124372601636175
3.46565656565657 -0.12603690052947
3.51515151515152 -0.127652562485327
3.56464646464647 -0.129221485573651
3.61414141414142 -0.130745478742268
3.66363636363636 -0.132226266783476
3.71313131313131 -0.133665494304614
3.76262626262626 -0.13506473028972
3.81212121212121 -0.136425471880178
3.86161616161616 -0.137749148404751
3.91111111111111 -0.139037124465666
3.96060606060606 -0.140290703541568
4.01010101010101 -0.141511131064884
4.05959595959596 -0.142699597386267
4.10909090909091 -0.143857240573485
4.15858585858586 -0.144985149054392
4.20808080808081 -0.146084364113027
4.25757575757576 -0.147155882247415
4.30707070707071 -0.148200657397091
4.35656565656566 -0.149219603048151
4.40606060606061 -0.150213594221964
4.45555555555556 -0.151183469359057
4.50505050505051 -0.152130032092719
4.55454545454546 -0.153054052934529
4.6040404040404 -0.153956271805439
4.65353535353535 -0.154837394838173
4.7030303030303 -0.155698105209202
4.75252525252525 -0.156539055081473
4.8020202020202 -0.157360871585577
4.85151515151515 -0.158164157089604
4.9010101010101 -0.158949490347405
4.95050505050505 -0.159717427580604
5 -0.160468503516795
};
\addlegendentry{$\hat{f}_{\lap}(\lambda = 1e-4)$}
\addplot [semithick, blue, dash pattern=on 1pt off 3pt on 3pt off 3pt]
table {%
0.1 0.0585154931695748
0.14949494949495 0.0430861634554737
0.198989898989899 0.0330242103932116
0.248484848484848 0.0266397218689262
0.297979797979798 0.0222722144853904
0.347474747474747 0.0190738995229732
0.396969696969697 0.0166038232738401
0.446464646464647 0.0146159132733045
0.495959595959596 0.0129634713603565
0.545454545454546 0.0115540686072686
0.594949494949495 0.0103267213985492
0.644444444444444 0.00923964509690611
0.693939393939394 0.00826326500908858
0.743434343434343 0.00737607346777924
0.792929292929293 0.0065620721756815
0.842424242424242 0.00580909794777595
0.891919191919192 0.00510773765443773
0.941414141414142 0.00445059120949254
0.990909090909091 0.00383173946798321
1.04040404040404 0.00324643924322421
1.08989898989899 0.00269081073595967
1.13939393939394 0.00216157958770991
1.18888888888889 0.00165608135440409
1.23838383838384 0.00117206511342183
1.28787878787879 0.000707604729125011
1.33737373737374 0.000260960642471719
1.38686868686869 -0.000169176004830462
1.43636363636364 -0.000584142157818546
1.48585858585859 -0.000984955744663174
1.53535353535354 -0.00137260242750348
1.58484848484849 -0.0017478975597871
1.63434343434343 -0.00211172943383486
1.68383838383838 -0.00246457857859683
1.73333333333333 -0.00280717524941498
1.78282828282828 -0.00314009935888337
1.83232323232323 -0.0034637778907981
1.88181818181818 -0.00377872630647846
1.93131313131313 -0.00408533649323894
1.98080808080808 -0.00438386728772477
2.03030303030303 -0.0046750165423991
2.07979797979798 -0.00495880692268336
2.12929292929293 -0.00523565383883462
2.17878787878788 -0.00550581204730327
2.22828282828283 -0.00576960954631767
2.27777777777778 -0.00602723037497865
2.32727272727273 -0.00627891680921744
2.37676767676768 -0.00652490335822287
2.42626262626263 -0.00676543333139201
2.47575757575758 -0.00700060711784423
2.52525252525253 -0.00723075284651897
2.57474747474748 -0.00745596184423036
2.62424242424242 -0.00767633308650828
2.67373737373737 -0.00789212025298447
2.72323232323232 -0.00810343575274075
2.77272727272727 -0.00831047913343985
2.82222222222222 -0.00851327216126152
2.87171717171717 -0.00871205475993775
2.92121212121212 -0.00890688534276907
2.97070707070707 -0.00909791280574923
3.02020202020202 -0.00928530148615525
3.06969696969697 -0.00946899232809484
3.11919191919192 -0.00964923774728055
3.16868686868687 -0.00982609022205394
3.21818181818182 -0.00999969304614023
3.26767676767677 -0.0101700747617521
3.31717171717172 -0.0103373887532712
3.36666666666667 -0.0105016340794871
3.41616161616162 -0.0106628607081382
3.46565656565657 -0.0108213322013166
3.51515151515152 -0.0109769423129301
3.56464646464647 -0.0111298578386492
3.61414141414142 -0.011280180817774
3.66363636363636 -0.0114278451653512
3.71313131313131 -0.0115730741083951
3.76262626262626 -0.0117157981714341
3.81212121212121 -0.0118561685517115
3.86161616161616 -0.011994194567539
3.91111111111111 -0.0121299779989558
3.96060606060606 -0.0122634671727615
4.01010101010101 -0.0123949172700751
4.05959595959596 -0.0125241303700053
4.10909090909091 -0.0126514627486091
4.15858585858586 -0.0127765101154845
4.20808080808081 -0.0128996705139166
4.25757575757576 -0.0130211059593679
4.30707070707071 -0.013140487869847
4.35656565656566 -0.0132580894110675
4.40606060606061 -0.013373933699619
4.45555555555556 -0.0134880178372061
4.50505050505051 -0.0136003524606438
4.55454545454546 -0.0137110059026351
4.6040404040404 -0.0138200259594492
4.65353535353535 -0.0139274731778507
4.7030303030303 -0.0140334113315362
4.75252525252525 -0.0141376476642642
4.8020202020202 -0.0142405066235463
4.85151515151515 -0.0143418126420634
4.9010101010101 -0.0144416992070058
4.95050505050505 -0.0145403607458198
5 -0.0146374231257039
};
\addlegendentry{$\hat{f}_{\lap}(\lambda_{\rm error})$}
\addplot [semithick, red, dash pattern=on 1pt off 3pt on 3pt off 3pt]
table {%
0.1 0.00107491406729421
0.14949494949495 -0.00389312342751147
0.198989898989899 -0.00875089344895685
0.248484848484848 -0.0123369862864542
0.297979797979798 -0.0149381005539594
0.347474747474747 -0.0168744570858662
0.396969696969697 -0.0183604851637683
0.446464646464647 -0.0195324978395057
0.495959595959596 -0.020478601488456
0.545454545454546 -0.0212574520815698
0.594949494949495 -0.0219093093600849
0.644444444444444 -0.0224626175467273
0.693939393939394 -0.022937989891349
0.743434343434343 -0.0233507248218595
0.792929292929293 -0.0237123763273009
0.842424242424242 -0.0240318314192278
0.891919191919192 -0.0243160372657054
0.941414141414142 -0.0245705022018223
0.990909090909091 -0.0247996652882624
1.04040404040404 -0.0250070952384722
1.08989898989899 -0.0251957498885858
1.13939393939394 -0.0253680477146675
1.18888888888889 -0.0255260322782976
1.23838383838384 -0.0256714095450542
1.28787878787879 -0.0258056304774601
1.33737373737374 -0.0259299351422598
1.38686868686869 -0.0260453782052917
1.43636363636364 -0.0261528707077698
1.48585858585859 -0.0262532071161368
1.53535353535354 -0.0263470686641092
1.58484848484849 -0.0264350857720638
1.63434343434343 -0.0265177608867023
1.68383838383838 -0.0265955970344748
1.73333333333333 -0.0266689723071792
1.78282828282828 -0.0267382885150477
1.83232323232323 -0.0268038440661573
1.88181818181818 -0.0268659737727718
1.93131313131313 -0.0269249040769914
1.98080808080808 -0.026980868123669
2.03030303030303 -0.0270341387696684
2.07979797979798 -0.0270848673635384
2.12929292929293 -0.0271332332539015
2.17878787878788 -0.0271793878021608
2.22828282828283 -0.0272234936025157
2.27777777777778 -0.0272656847687294
2.32727272727273 -0.027306096851531
2.37676767676768 -0.0273448062840409
2.42626262626263 -0.0273819465242128
2.47575757575758 -0.027417614428752
2.52525252525253 -0.0274518509422872
2.57474747474748 -0.0274848028250632
2.62424242424242 -0.0275164904889129
2.67373737373737 -0.0275470007675193
2.72323232323232 -0.0275764284268641
2.77272727272727 -0.0276047607793629
2.82222222222222 -0.0276321539284677
2.87171717171717 -0.0276585540550286
2.92121212121212 -0.0276840801201375
2.97070707070707 -0.027708747698624
3.02020202020202 -0.027732623696567
3.06969696969697 -0.0277557232225234
3.11919191919192 -0.0277780743015439
3.16868686868687 -0.0277997602877044
3.21818181818182 -0.0278207392367832
3.26767676767677 -0.0278411074235205
3.31717171717172 -0.0278608691471761
3.36666666666667 -0.0278800156366158
3.41616161616162 -0.0278986304868691
3.46565656565657 -0.0279167028702364
3.51515151515152 -0.0279342737639672
3.56464646464647 -0.0279513757700932
3.61414141414142 -0.0279679800739218
3.66363636363636 -0.0279841476385013
3.71313131313131 -0.0279998854808104
3.76262626262626 -0.0280151838790677
3.81212121212121 -0.0280301468172975
3.86161616161616 -0.0280446950060678
3.91111111111111 -0.0280588561691388
3.96060606060606 -0.0280726919268391
4.01010101010101 -0.0280861537965814
4.05959595959596 -0.028099303623876
4.10909090909091 -0.0281121393235076
4.15858585858586 -0.0281246996796461
4.20808080808081 -0.0281369208067794
4.25757575757576 -0.0281488564889831
4.30707070707071 -0.0281605156268855
4.35656565656566 -0.0281719510476297
4.40606060606061 -0.0281830694268462
4.45555555555556 -0.0281939354156493
4.50505050505051 -0.0282046254526137
4.55454545454546 -0.0282150475522542
4.6040404040404 -0.0282252765240751
4.65353535353535 -0.0282352596546478
4.7030303030303 -0.0282450388749187
4.75252525252525 -0.0282546384307184
4.8020202020202 -0.0282640277220155
4.85151515151515 -0.0282732103873126
4.9010101010101 -0.0282822190000398
4.95050505050505 -0.0282910026721278
5 -0.0282996718980814
};
\addlegendentry{$\hat{f}_{\lap}(\lambda_{\rm loss})$}
\addplot [semithick, black]
table {%
0.1 0.117535701338803
0.14949494949495 0.142813751821537
0.198989898989899 0.166427384386137
0.248484848484848 0.186301514413506
0.297979797979798 0.202854751986546
0.347474747474747 0.216799695504573
0.396969696969697 0.228613888142763
0.446464646464647 0.237630724917748
0.495959595959596 0.241320146950313
0.545454545454546 0.241308947780918
0.594949494949495 0.240460354104288
0.644444444444444 0.239444876398697
0.693939393939394 0.238403208621239
0.743434343434343 0.237375266427874
0.792929292929293 0.236374229790198
0.842424242424242 0.235404419037166
0.891919191919192 0.234466788514516
0.941414141414142 0.233560911406578
0.990909090909091 0.2326857817659
1.04040404040404 0.231840166144243
1.08989898989899 0.231022761208722
1.13939393939394 0.230232271690298
1.18888888888889 0.229467444187019
1.23838383838384 0.228727082732922
1.28787878787879 0.228010053877471
1.33737373737374 0.227315288157645
1.38686868686869 0.226641777428537
1.43636363636364 0.225988572099575
1.48585858585859 0.225354777692585
1.53535353535354 0.224739552755093
1.58484848484849 0.224142097983132
1.63434343434343 0.22356166587762
1.68383838383838 0.222997547746128
1.73333333333333 0.222449073793183
1.78282828282828 0.221915610007014
1.83232323232323 0.221396556150166
1.88181818181818 0.220891343310629
1.93131313131313 0.220399432402014
1.98080808080808 0.2199203111813
2.03030303030303 0.21945349358904
2.07979797979798 0.21899851761222
2.12929292929293 0.218554943889098
2.17878787878788 0.218122354341999
2.22828282828283 0.217700350904894
2.27777777777778 0.217288554331776
2.32727272727273 0.216886603103894
2.37676767676768 0.216494152410299
2.42626262626263 0.216110873191255
2.47575757575758 0.215736451256773
2.52525252525253 0.21537058645879
2.57474747474748 0.215012991921451
2.62424242424242 0.214663393317579
2.67373737373737 0.214321528196235
2.72323232323232 0.213987145351702
2.77272727272727 0.213660004231106
2.82222222222222 0.213339874380809
2.87171717171717 0.21302653492615
2.92121212121212 0.212719774082453
2.97070707070707 0.212419388696903
3.02020202020202 0.212125183816854
3.06969696969697 0.211836972282295
3.11919191919192 0.211554575599688
3.16868686868687 0.21127781731619
3.21818181818182 0.211006535200633
3.26767676767677 0.210740568411908
3.31717171717172 0.210479763447976
3.36666666666667 0.210223972611412
3.41616161616162 0.209973053739933
3.46565656565657 0.209726869951667
3.51515151515152 0.209485289404295
3.56464646464647 0.209248185067131
3.61414141414142 0.209015434513192
3.66363636363636 0.208786919676669
3.71313131313131 0.208562526728342
3.76262626262626 0.208342145838125
3.81212121212121 0.208125671026131
3.86161616161616 0.207912999955292
3.91111111111111 0.207704033838041
3.96060606060606 0.207498677247223
4.01010101010101 0.207296837983863
4.05959595959596 0.207098426942766
4.10909090909091 0.206903357984714
4.15858585858586 0.206711547814908
4.20808080808081 0.206522915867294
4.25757575757576 0.206337384194456
4.30707070707071 0.206154877362754
4.35656565656566 0.205975322352428
4.40606060606061 0.205798648462407
4.45555555555556 0.205624787219517
4.50505050505051 0.205453672292023
4.55454545454546 0.205285239406948
4.6040404040404 0.205119425989631
4.65353535353535 0.204956172496856
4.7030303030303 0.204795419528265
4.75252525252525 0.204637110574448
4.8020202020202 0.204481190542976
4.85151515151515 0.204327605977366
4.9010101010101 0.204176304996649
4.95050505050505 0.204027237239866
5 0.203880353808368
};
\addplot [semithick, blue]
table {%
0.1 0.0721408281901177
0.14949494949495 0.0599594539303405
0.198989898989899 0.0513769679448255
0.248484848484848 0.0457692269210502
0.297979797979798 0.0418603350441332
0.347474747474747 0.0389541557572574
0.396969696969697 0.0366793443208608
0.446464646464647 0.0348258344615852
0.495959595959596 0.0332673091934015
0.545454545454546 0.0319237039266931
0.594949494949495 0.0307419530010086
0.644444444444444 0.0296855701160202
0.693939393939394 0.0287286532944564
0.743434343434343 0.0278523176609469
0.792929292929293 0.0270424868521534
0.842424242424242 0.0262884341598165
0.891919191919192 0.0255818399617493
0.941414141414142 0.0249161499475105
0.990909090909091 0.0242861059964994
1.04040404040404 0.0236874937941812
1.08989898989899 0.0231168571685114
1.13939393939394 0.0225712564116239
1.18888888888889 0.0220483076732675
1.23838383838384 0.0215459895545735
1.28787878787879 0.0210625638108921
1.33737373737374 0.0205964364441507
1.38686868686869 0.0201464416507607
1.43636363636364 0.0197113396622025
1.48585858585859 0.0192902109846868
1.53535353535354 0.0188821427674506
1.58484848484849 0.0184863873284196
1.63434343434343 0.0181020952714329
1.68383838383838 0.0177288535397528
1.73333333333333 0.0173659564097767
1.78282828282828 0.0170128528454351
1.83232323232323 0.0166691509966678
1.88181818181818 0.0163343513411145
1.93131313131313 0.0160080836464768
1.98080808080808 0.015690118415804
2.03030303030303 0.0153797326267049
2.07979797979798 0.0150769476482602
2.12929292929293 0.0147813441303092
2.17878787878788 0.0144926769727095
2.22828282828283 0.0142106162384779
2.27777777777778 0.0139349884921469
2.32727272727273 0.0136655528762909
2.37676767676768 0.0134020746632677
2.42626262626263 0.0131443071672726
2.47575757575758 0.0128921600797486
2.52525252525253 0.0126452890252251
2.57474747474748 0.0124036103701488
2.62424242424242 0.0121670305174277
2.67373737373737 0.0119352833134124
2.72323232323232 0.0117082578627888
2.77272727272727 0.0114857457949459
2.82222222222222 0.0112677344996986
2.87171717171717 0.0110539691547774
2.92121212121212 0.0108443949598963
2.97070707070707 0.0106388563295959
3.02020202020202 0.010437179996298
3.06969696969697 0.0102394390397397
3.11919191919192 0.0100453617101536
3.16868686868687 0.00985489677699025
3.21818181818182 0.00966789221308439
3.26767676767677 0.00948432235911001
3.31717171717172 0.00930402327406377
3.36666666666667 0.0091270011176352
3.41616161616162 0.00895320547561729
3.46565656565657 0.00878234984874249
3.51515151515152 0.00861455574061221
3.56464646464647 0.00844964298248729
3.61414141414142 0.00828750268588296
3.66363636363636 0.00812821125385721
3.71313131313131 0.00797152567199633
3.76262626262626 0.00781752567865368
3.81212121212121 0.00766604740267496
3.86161616161616 0.00751708331747325
3.91111111111111 0.00737052377786085
3.96060606060606 0.0072264282025517
4.01010101010101 0.00708451757429929
4.05959595959596 0.00694501226795308
4.10909090909091 0.0068075216667064
4.15858585858586 0.00667249330703124
4.20808080808081 0.00653949036424228
4.25757575757576 0.00640833612985781
4.30707070707071 0.00627939432181379
4.35656565656566 0.00615236557383314
4.40606060606061 0.00602722595581451
4.45555555555556 0.00590398011078619
4.50505050505051 0.00578261769370825
4.55454545454546 0.00566306475785783
4.6040404040404 0.00554526992256599
4.65353535353535 0.00542916772054303
4.7030303030303 0.00531468909105115
4.75252525252525 0.00520204730281804
4.8020202020202 0.00509088624823784
4.85151515151515 0.00498140016746118
4.9010101010101 0.00487344308836424
4.95050505050505 0.00476680198237345
5 0.00466188940844703
};
\addplot [semithick, red]
table {%
0.1 0.0504249218223812
0.14949494949495 0.0335425147430458
0.198989898989899 0.025056094264808
0.248484848484848 0.0199779397413156
0.297979797979798 0.0166027716183277
0.347474747474747 0.0142009324300398
0.396969696969697 0.0124064592954553
0.446464646464647 0.0110158221534286
0.495959595959596 0.00990696939120006
0.545454545454546 0.00900238164983136
0.594949494949495 0.0082505240276628
0.644444444444444 0.00761581502065489
0.693939393939394 0.00707291056865711
0.743434343434343 0.00660325452805532
0.792929292929293 0.00619297940425301
0.842424242424242 0.00583151051363551
0.891919191919192 0.00551064086279129
0.941414141414142 0.00522390277749896
0.990909090909091 0.00496611280828108
1.04040404040404 0.00473311895047068
1.08989898989899 0.00452149633112631
1.13939393939394 0.00432845124732761
1.18888888888889 0.00415163259737039
1.23838383838384 0.00398908170980405
1.28787878787879 0.00383913741764041
1.33737373737374 0.00370038303319919
1.38686868686869 0.00357161544831086
1.43636363636364 0.00345179734279633
1.48585858585859 0.00334002590994331
1.53535353535354 0.00323552743854794
1.58484848484849 0.00313758915480111
1.63434343434343 0.00304564053601442
1.68383838383838 0.00295911499021961
1.73333333333333 0.00287758328331167
1.78282828282828 0.00280059402195842
1.83232323232323 0.00272780907324033
1.88181818181818 0.00265885367047325
1.93131313131313 0.00259347089841211
1.98080808080808 0.00253139836111149
2.03030303030303 0.00247233262542035
2.07979797979798 0.00241610148161631
2.12929292929293 0.00236250385982417
2.17878787878788 0.00231136957914868
2.22828282828283 0.0022625172676729
2.27777777777778 0.00221579675038919
2.32727272727273 0.00217105695299336
2.37676767676768 0.00212821062862689
2.42626262626263 0.00208711000612272
2.47575757575758 0.00204764692504456
2.52525252525253 0.00200977346460596
2.57474747474748 0.0019733287994147
2.62424242424242 0.00193828765573689
2.67373737373737 0.00190455394779165
2.72323232323232 0.00187202334836634
2.77272727272727 0.00184070679065618
2.82222222222222 0.00181043471551823
2.87171717171717 0.00178126248850652
2.92121212121212 0.00175306066610714
2.97070707070707 0.00172581064576749
3.02020202020202 0.00169943900512071
3.06969696969697 0.00167392790966248
3.11919191919192 0.0016492458125813
3.16868686868687 0.00162530202467237
3.21818181818182 0.00160214012742677
3.26767676767677 0.00157965577955432
3.31717171717172 0.00155784321382113
3.36666666666667 0.00153671073862938
3.41616161616162 0.00151616781149055
3.46565656565657 0.00149622507585112
3.51515151515152 0.00147683775195673
3.56464646464647 0.00145797008280257
3.61414141414142 0.00143965214338726
3.66363636363636 0.00142181789742546
3.71313131313131 0.00140445913139131
3.76262626262626 0.00138758560256214
3.81212121212121 0.00137108530821184
3.86161616161616 0.00135504259538843
3.91111111111111 0.00133942721634916
3.96060606060606 0.00132417264078855
4.01010101010101 0.00130933031136315
4.05959595959596 0.0012948335094678
4.10909090909091 0.00128068402588599
4.15858585858586 0.00126683989133114
4.20808080808081 0.00125336914909846
4.25757575757576 0.00124021380262451
4.30707070707071 0.00122736395758494
4.35656565656566 0.0012147626753809
4.40606060606061 0.00120250961718249
4.45555555555556 0.00119053520951051
4.50505050505051 0.0011787572605827
4.55454545454546 0.00116727404468897
4.6040404040404 0.00115600514166692
4.65353535353535 0.00114500677650864
4.7030303030303 0.00113423377933608
4.75252525252525 0.00112365993869956
4.8020202020202 0.00111331782033708
4.85151515151515 0.00110320331431057
4.9010101010101 0.00109328131516595
4.95050505050505 0.00108360615769931
5 0.00107405925692061
};
\end{axis}

\end{tikzpicture}

%% file: Figures/supplementary/n_over_p=10/calibration_temp_scaling_p=0.75_n_over_p=10.0.tex
\begin{tikzpicture}
\tikzstyle{every node}=[font=\tiny]

\definecolor{darkgray176}{RGB}{176,176,176}
\definecolor{lightgray204}{RGB}{204,204,204}
\definecolor{orange}{RGB}{255,165,0}

\begin{axis}[
height=\figheight,
legend cell align={left},
legend style={fill opacity=0.8, draw opacity=1, text opacity=1, draw=lightgray204},
tick align=outside,
tick pos=left,
width=\figwidth,
x grid style={darkgray176},
xlabel={\(\displaystyle p/n\)},
xmajorgrids,
xmin=-0.145, xmax=5.245,
xtick style={color=black},
y grid style={darkgray176},
ylabel={\(\displaystyle \Delta\)0.75},
ymajorgrids,
ymin=-0.00321621348422218, ymax=0.0341995591242484,
ytick style={color=black}
]
\addplot [semithick, black, dashdotdotted]
table {%
0.1 -0.00105213528664416
0.14949494949495 -0.000729606734907962
0.198989898989899 -0.000751577393657499
0.248484848484848 -0.000872412909214471
0.297979797979798 -0.00102189494686622
0.347474747474747 -0.00117841002627395
0.396969696969697 -0.00133288741551718
0.446464646464647 -0.0014651489848293
0.495959595959596 -0.00151549654747352
0.545454545454546 -0.00149234611946347
0.594949494949495 -0.00145296876977175
0.644444444444444 -0.00141130201867667
0.693939393939394 -0.0013701896280488
0.743434343434343 -0.00133042732879307
0.792929292929293 -0.00129220701639998
0.842424242424242 -0.0012555308756742
0.891919191919192 -0.00122043424722851
0.941414141414142 -0.00118678062247235
0.990909090909091 -0.00115456071682962
1.04040404040404 -0.00112367905891053
1.08989898989899 -0.00109408661671395
1.13939393939394 -0.00106570548735641
1.18888888888889 -0.00103844932759345
1.23838383838384 -0.00101231518468803
1.28787878787879 -0.000987212728929654
1.33737373737374 -0.000963073536129788
1.38686868686869 -0.000939877164848557
1.43636363636364 -0.000917524388994173
1.48585858585859 -0.000896078077298412
1.53535353535354 -0.000875375642678256
1.58484848484849 -0.000855412410100453
1.63434343434343 -0.000836184186083666
1.68383838383838 -0.000817667735636407
1.73333333333333 -0.000799758833677333
1.78282828282828 -0.000782487692223
1.83232323232323 -0.000765782873859089
1.88181818181818 -0.000749677552522088
1.93131313131313 -0.000734143409689358
1.98080808080808 -0.000719051413098359
2.03030303030303 -0.00070447785113581
2.07979797979798 -0.000690354275115657
2.12929292929293 -0.000676718989071734
2.17878787878788 -0.000663488197740092
2.22828282828283 -0.00065066512150691
2.27777777777778 -0.000638227883437259
2.32727272727273 -0.000626175779505744
2.37676767676768 -0.000614487706387612
2.42626262626263 -0.000603147657474712
2.47575757575758 -0.000592140125808216
2.52525252525253 -0.000581457540798858
2.57474747474748 -0.000571087659506753
2.62424242424242 -0.000560972009124971
2.67373737373737 -0.000551156967560362
2.72323232323232 -0.000541613107177152
2.77272727272727 -0.000532329979885282
2.82222222222222 -0.000523297788826649
2.87171717171717 -0.000514507024132804
2.92121212121212 -0.000505948477744655
2.97070707070707 -0.00049761537525328
3.02020202020202 -0.000489521023286477
3.06969696969697 -0.00048162789154127
3.11919191919192 -0.000473828049963987
3.16868686868687 -0.000466310292792871
3.21818181818182 -0.000459075914945206
3.26767676767677 -0.000451924173418217
3.31717171717172 -0.00044494649304383
3.36666666666667 -0.000438137360519231
3.41616161616162 -0.000431491356062685
3.46565656565657 -0.000424902314043463
3.51515151515152 -0.000418566102473816
3.56464646464647 -0.000412377281237108
3.61414141414142 -0.000406330753233264
3.66363636363636 -0.000400457693556899
3.71313131313131 -0.000394720097818824
3.76262626262626 -0.000389035251872838
3.81212121212121 -0.00038353688529269
3.86161616161616 -0.000378094644688742
3.91111111111111 -0.000372804486911238
3.96060606060606 -0.000367681265384001
4.01010101010101 -0.000362572858224053
4.05959595959596 -0.000357611871326879
4.10909090909091 -0.00035274826071674
4.15858585858586 -0.000347972092367921
4.20808080808081 -0.000343326536770605
4.25757575757576 -0.000338749472661992
4.30707070707071 -0.000334368319382272
4.35656565656566 -0.000329938822067
4.40606060606061 -0.000325650936222277
4.45555555555556 -0.000321441867321415
4.50505050505051 -0.000317314827432247
4.55454545454546 -0.000313267546468876
4.6040404040404 -0.000309297152802257
4.65353535353535 -0.000305402990324
4.7030303030303 -0.000301589755739795
4.75252525252525 -0.00029785349664091
4.8020202020202 -0.000294186234269245
4.85151515151515 -0.000290586104422119
4.9010101010101 -0.000287051353060397
4.95050505050505 -0.000283473737872075
5 -0.000280065876387292
};
\addlegendentry{$\hat{f}_{\erm}(\lambda = 10^{-4})$ - TS}
\addplot [semithick, blue, dashdotdotted]
table {%
0.1 -0.000980747073132449
0.14949494949495 -0.000463276591117001
0.198989898989899 -0.000255878671691967
0.248484848484848 -0.000156906320004313
0.297979797979798 -0.000100958720860023
0.347474747474747 -6.51143219553285e-05
0.396969696969697 -4.08705095256146e-05
0.446464646464647 -2.3148585095889e-05
0.495959595959596 -9.72256863918108e-06
0.545454545454546 7.86344466541244e-07
0.594949494949495 9.22818682613791e-06
0.644444444444444 1.61541304198876e-05
0.693939393939394 2.1936466350736e-05
0.743434343434343 2.68283188902263e-05
0.792929292929293 3.12819802146258e-05
0.842424242424242 3.46862858242591e-05
0.891919191919192 3.7877656296148e-05
0.941414141414142 4.06943848291208e-05
0.990909090909091 4.31986808864249e-05
1.04040404040404 4.54396892090037e-05
1.08989898989899 4.74567648502955e-05
1.13939393939394 4.92819633005626e-05
1.18888888888889 5.09413618640098e-05
1.23838383838384 5.24565371404773e-05
1.28787878787879 5.38455779948155e-05
1.33737373737374 5.51235449075538e-05
1.38686868686869 5.63163221267704e-05
1.43636363636364 5.7205130448712e-05
1.48585858585859 5.81995183182027e-05
1.53535353535354 5.91254517878781e-05
1.58484848484849 5.99897076053457e-05
1.63434343434343 6.07982814870844e-05
1.68383838383838 6.15564212974595e-05
1.73333333333333 6.22687281964796e-05
1.78282828282828 6.35456675270873e-05
1.83232323232323 6.41791122609758e-05
1.88181818181818 6.47775403941342e-05
1.93131313131313 6.5343802129858e-05
1.98080808080808 6.58804477277064e-05
2.03030303030303 6.63897470050134e-05
2.07979797979798 6.68737658063856e-05
2.12929292929293 6.73343520932734e-05
2.17878787878788 6.77731817587679e-05
2.22828282828283 6.81917707517909e-05
2.27777777777778 6.85915025021577e-05
2.32727272727273 6.89736333320257e-05
2.37676767676768 6.93393073716297e-05
2.42626262626263 6.96895727230196e-05
2.47575757575758 7.00253970722375e-05
2.52525252525253 7.03476563383987e-05
2.57474747474748 7.06571634175868e-05
2.62424242424242 7.09546713247855e-05
2.67373737373737 7.12408642153939e-05
2.72323232323232 7.15163812893271e-05
2.77272727272727 7.17818146707128e-05
2.82222222222222 7.20377110691173e-05
2.87171717171717 7.2284580201476e-05
2.92121212121212 7.25228931659494e-05
2.97070707070707 7.27530904940377e-05
3.02020202020202 7.29755804587118e-05
3.06969696969697 7.31907511649643e-05
3.11919191919192 7.33992651723581e-05
3.16868686868687 7.36015736439999e-05
3.21818181818182 7.37975391457146e-05
3.26767676767677 7.39874625803605e-05
3.31717171717172 7.4171611350704e-05
3.36666666666667 7.43502537936536e-05
3.41616161616162 7.45236286197049e-05
3.46565656565657 7.4691969448204e-05
3.51515151515152 7.48554898156728e-05
3.56464646464647 7.5014392097672e-05
3.61414141414142 7.51688715864285e-05
3.66363636363636 7.53191146057874e-05
3.71313131313131 7.5465284002374e-05
3.76262626262626 7.56075529960709e-05
3.81212121212121 7.57460681526378e-05
3.86161616161616 7.58809830341267e-05
3.91111111111111 7.60124317379374e-05
3.96060606060606 7.61405444301699e-05
4.01010101010101 7.62654528416729e-05
4.05959595959596 7.63872687900014e-05
4.10909090909091 7.65061131391409e-05
4.15858585858586 7.66220873069656e-05
4.20808080808081 7.67352935224785e-05
4.25757575757576 7.68458339430733e-05
4.30707070707071 7.69537988932756e-05
4.35656565656566 7.70592786781821e-05
4.40606060606061 7.71623586695025e-05
4.45555555555556 7.72631148496794e-05
4.50505050505051 7.73616320759452e-05
4.55454545454546 7.74579780526974e-05
4.6040404040404 7.75522292383313e-05
4.65353535353535 7.76444458197023e-05
4.7030303030303 7.77347003694251e-05
4.75252525252525 7.78230489756337e-05
4.8020202020202 7.79095557394971e-05
4.85151515151515 7.79942772176634e-05
4.9010101010101 7.80772651449713e-05
4.95050505050505 7.81585737821278e-05
5 7.8238255697749e-05
};
\addlegendentry{$\hat{f}_{\erm}(\lambda_{\rm error})$ - TS}
\addplot [semithick, red, dashdotdotted]
table {%
0.1 -0.00101624592336991
0.14949494949495 -0.000496719590648009
0.198989898989899 -0.000279915735099778
0.248484848484848 -0.000174778382660001
0.297979797979798 -0.000114767549019512
0.347474747474747 -7.65215975516709e-05
0.396969696969697 -5.02046697625769e-05
0.446464646464647 -3.10650171988502e-05
0.495959595959596 -1.65541342738162e-05
0.545454545454546 -5.1918971663012e-06
0.594949494949495 3.9364798950503e-06
0.644444444444444 1.14251969896006e-05
0.693939393939394 1.76762300762023e-05
0.743434343434343 2.2970853211679e-05
0.792929292929293 2.77513887811942e-05
0.842424242424242 3.1474122881292e-05
0.891919191919192 3.49075282124467e-05
0.941414141414142 3.79378699196797e-05
0.990909090909091 4.06318106360715e-05
1.04040404040404 4.29969531232643e-05
1.08989898989899 4.51067279985118e-05
1.13939393939394 4.70186467962375e-05
1.18888888888889 4.92322274291412e-05
1.23838383838384 5.08304514954538e-05
1.28787878787879 5.22964564346751e-05
1.33737373737374 5.36459536241152e-05
1.38686868686869 5.48922680237984e-05
1.43636363636364 5.60467670563947e-05
1.48585858585859 5.71192123469677e-05
1.53535353535354 5.8118033826049e-05
1.58484848484849 5.90505353766924e-05
1.63434343434343 5.99231032043512e-05
1.68383838383838 6.0741314236501e-05
1.73333333333333 6.15100917180644e-05
1.78282828282828 6.22337690160402e-05
1.83232323232323 6.29161975395087e-05
1.88181818181818 6.35607949893702e-05
1.93131313131313 6.41706229598249e-05
1.98080808080808 6.47484132312304e-05
2.03030303030303 6.52966141828637e-05
2.07979797979798 6.58174473414652e-05
2.12929292929293 6.63129034378374e-05
2.17878787878788 6.67847966253632e-05
2.22828282828283 6.72347633201564e-05
2.27777777777778 6.76642895034441e-05
2.32727272727273 6.80747379077085e-05
2.37676767676768 6.84673498474453e-05
2.42626262626263 6.88432661505312e-05
2.47575757575758 6.92035193657903e-05
2.52525252525253 6.95490787949726e-05
2.57474747474748 6.98808156781583e-05
2.62424242424242 7.01995452402349e-05
2.67373737373737 7.05060161878546e-05
2.72323232323232 7.08009198011617e-05
2.77272727272727 7.10849036049677e-05
2.82222222222222 7.13585531256777e-05
2.87171717171717 7.16224335082227e-05
2.92121212121212 7.18770504446464e-05
2.97070707070707 7.21228840432353e-05
3.02020202020202 7.23603783550075e-05
3.06969696969697 7.25899565328092e-05
3.11919191919192 7.28120056624482e-05
3.16868686868687 7.30268865571926e-05
3.21818181818182 7.32349370963226e-05
3.26767676767677 7.3436492922796e-05
3.31717171717172 7.36318638405731e-05
3.36666666666667 7.38213428873591e-05
3.41616161616162 7.40051514892537e-05
3.46565656565657 7.41835423225812e-05
3.51515151515152 7.43567712383486e-05
3.56464646464647 7.45250212357274e-05
3.61414141414142 7.46885390948693e-05
3.66363636363636 7.48475016987271e-05
3.71313131313131 7.50020937548834e-05
3.76262626262626 7.51525223413596e-05
3.81212121212121 7.5298926874634e-05
3.86161616161616 7.5441464706949e-05
3.91111111111111 7.55802891413415e-05
3.96060606060606 7.57155760943151e-05
4.01010101010101 7.58848610555862e-05
4.05959595959596 7.6063620348954e-05
4.10909090909091 7.62363487160078e-05
4.15858585858586 7.64034125479984e-05
4.20808080808081 7.65648256491058e-05
4.25757575757576 7.67209833771876e-05
4.30707070707071 7.6872114747939e-05
4.35656565656566 7.70185634371146e-05
4.40606060606061 7.64672003640277e-05
4.45555555555556 7.66044376037112e-05
4.50505050505051 7.67376319998903e-05
4.55454545454546 7.68666967505549e-05
4.6040404040404 7.69919664149743e-05
4.65353535353535 7.71134544658159e-05
4.7030303030303 7.72313971260097e-05
4.75252525252525 7.73459710077251e-05
4.8020202020202 7.74572304304e-05
4.85151515151515 7.7565304527516e-05
4.9010101010101 7.76703658298405e-05
4.95050505050505 7.76313243610049e-05
5 7.81800249044506e-05
};
\addlegendentry{$\hat{f}_{\erm}(\lambda_{\rm loss})$ - TS}
\addplot [semithick, orange, dashed]
table {%
0.1 0.0324988421874998
0.14949494949495 0.0155631376865124
0.198989898989899 0.00959326266178917
0.248484848484848 0.0068369423758935
0.297979797979798 0.00528811673817464
0.347474747474747 0.0043030923030315
0.396969696969697 0.00362324073494624
0.446464646464647 0.00312670589825559
0.495959595959596 0.00274802396263973
0.545454545454546 0.00244989192711276
0.594949494949495 0.00220952707422872
0.644444444444444 0.00201139357225222
0.693939393939394 0.00184529192912197
0.743434343434343 0.00170407402871009
0.792929292929293 0.00158254665131263
0.842424242424242 0.00147686580052508
0.891919191919192 0.00138412571999502
0.941414141414142 0.00130208939517806
0.990909090909091 0.00122925510954242
1.04040404040404 0.00116372851101665
1.08989898989899 0.00110465232116042
1.13939393939394 0.00105111742180164
1.18888888888889 0.00100238556987031
1.23838383838384 0.000957833166263078
1.28787878787879 0.000916946689105247
1.33737373737374 0.000879291819283146
1.38686868686869 0.000844500215786526
1.43636363636364 0.000812257434496977
1.48585858585859 0.000782293425184433
1.53535353535354 0.000754374968890081
1.58484848484849 0.000728299607934191
1.63434343434343 0.000703891047529215
1.68383838383838 0.000680996774300002
1.73333333333333 0.000659477904800809
1.78282828282828 0.000639214272817767
1.83232323232323 0.000620098977713979
1.88181818181818 0.000602037226854923
1.93131313131313 0.00058494428358391
1.98080808080808 0.000568744257704035
2.03030303030303 0.000553368982037394
2.07979797979798 0.00053875705417139
2.12929292929293 0.000524853017429994
2.17878787878788 0.000511606658863184
2.22828282828283 0.000498972403501163
2.27777777777778 0.000486908792487473
2.32727272727273 0.000475378056984943
2.37676767676768 0.00046434560239883
2.42626262626263 0.000453779866266824
2.47575757575758 0.00044365181501016
2.52525252525253 0.000433934825386117
2.57474747474748 0.000424604375106163
2.62424242424242 0.000415637852772432
2.67373737373737 0.000407014663803551
2.72323232323232 0.000398714918495191
2.77272727272727 0.000390721014282724
2.82222222222222 0.000383016355186783
2.87171717171717 0.000375585522459532
2.92121212121212 0.000368414176140752
2.97070707070707 0.000361488958938505
3.02020202020202 0.000354797415281216
3.06969696969697 0.000348327914948543
3.11919191919192 0.00034206958890004
3.16868686868687 0.000336012262662777
3.21818181818182 0.000330146410326138
3.26767676767677 0.000324463092957972
3.31717171717172 0.000318953922927245
3.36666666666667 0.000313611016878435
3.41616161616162 0.000308426960653474
3.46565656565657 0.00030339477508734
3.51515151515152 0.000298507883218613
3.56464646464647 0.000293760083120653
3.61414141414142 0.000288975679155978
3.66363636363636 0.000284487312503856
3.71313131313131 0.000280125295313405
3.76262626262626 0.000275878505711757
3.81212121212121 0.000271744554955267
3.86161616161616 0.000267719022206125
3.91111111111111 0.000263797684205813
3.96060606060606 0.000259976553306784
4.01010101010101 0.000256251832386756
4.05959595959596 0.000252619904086981
4.10909090909091 0.000249077375335305
4.15858585858586 0.000245620960818482
4.20808080808081 0.000242247589543942
4.25757575757576 0.000238954271298164
4.30707070707071 0.000235738190265566
4.35656565656566 0.00023259670297171
4.40606060606061 0.000229527184468559
4.45555555555556 0.000226527223262663
4.50505050505051 0.000223594474750288
4.55454545454546 0.000220724776531345
4.6040404040404 0.000217923719857738
4.65353535353535 0.000215179585170255
4.7030303030303 0.0002124943148194
4.75252525252525 0.000209866010358106
4.8020202020202 0.000207292885331944
4.85151515151515 0.000204773219275145
4.9010101010101 0.000202305379928158
4.95050505050505 0.000199887750205296
5 0.000197518831721255
};
\addlegendentry{$\hat{f}_{\empbayes}(\lambda_{\rm evidence})$}
\addplot [semithick, red]
table {%
0.1 0.0504249218223812
0.14949494949495 0.0335425147430458
0.198989898989899 0.025056094264808
0.248484848484848 0.0199779397413156
0.297979797979798 0.0166027716183277
0.347474747474747 0.0142009324300398
0.396969696969697 0.0124064592954553
0.446464646464647 0.0110158221534286
0.495959595959596 0.00990696939120006
0.545454545454546 0.00900238164983136
0.594949494949495 0.0082505240276628
0.644444444444444 0.00761581502065489
0.693939393939394 0.00707291056865711
0.743434343434343 0.00660325452805532
0.792929292929293 0.00619297940425301
0.842424242424242 0.00583151051363551
0.891919191919192 0.00551064086279129
0.941414141414142 0.00522390277749896
0.990909090909091 0.00496611280828108
1.04040404040404 0.00473311895047068
1.08989898989899 0.00452149633112631
1.13939393939394 0.00432845124732761
1.18888888888889 0.00415163259737039
1.23838383838384 0.00398908170980405
1.28787878787879 0.00383913741764041
1.33737373737374 0.00370038303319919
1.38686868686869 0.00357161544831086
1.43636363636364 0.00345179734279633
1.48585858585859 0.00334002590994331
1.53535353535354 0.00323552743854794
1.58484848484849 0.00313758915480111
1.63434343434343 0.00304564053601442
1.68383838383838 0.00295911499021961
1.73333333333333 0.00287758328331167
1.78282828282828 0.00280059402195842
1.83232323232323 0.00272780907324033
1.88181818181818 0.00265885367047325
1.93131313131313 0.00259347089841211
1.98080808080808 0.00253139836111149
2.03030303030303 0.00247233262542035
2.07979797979798 0.00241610148161631
2.12929292929293 0.00236250385982417
2.17878787878788 0.00231136957914868
2.22828282828283 0.0022625172676729
2.27777777777778 0.00221579675038919
2.32727272727273 0.00217105695299336
2.37676767676768 0.00212821062862689
2.42626262626263 0.00208711000612272
2.47575757575758 0.00204764692504456
2.52525252525253 0.00200977346460596
2.57474747474748 0.0019733287994147
2.62424242424242 0.00193828765573689
2.67373737373737 0.00190455394779165
2.72323232323232 0.00187202334836634
2.77272727272727 0.00184070679065618
2.82222222222222 0.00181043471551823
2.87171717171717 0.00178126248850652
2.92121212121212 0.00175306066610714
2.97070707070707 0.00172581064576749
3.02020202020202 0.00169943900512071
3.06969696969697 0.00167392790966248
3.11919191919192 0.0016492458125813
3.16868686868687 0.00162530202467237
3.21818181818182 0.00160214012742677
3.26767676767677 0.00157965577955432
3.31717171717172 0.00155784321382113
3.36666666666667 0.00153671073862938
3.41616161616162 0.00151616781149055
3.46565656565657 0.00149622507585112
3.51515151515152 0.00147683775195673
3.56464646464647 0.00145797008280257
3.61414141414142 0.00143965214338726
3.66363636363636 0.00142181789742546
3.71313131313131 0.00140445913139131
3.76262626262626 0.00138758560256214
3.81212121212121 0.00137108530821184
3.86161616161616 0.00135504259538843
3.91111111111111 0.00133942721634916
3.96060606060606 0.00132417264078855
4.01010101010101 0.00130933031136315
4.05959595959596 0.0012948335094678
4.10909090909091 0.00128068402588599
4.15858585858586 0.00126683989133114
4.20808080808081 0.00125336914909846
4.25757575757576 0.00124021380262451
4.30707070707071 0.00122736395758494
4.35656565656566 0.0012147626753809
4.40606060606061 0.00120250961718249
4.45555555555556 0.00119053520951051
4.50505050505051 0.0011787572605827
4.55454545454546 0.00116727404468897
4.6040404040404 0.00115600514166692
4.65353535353535 0.00114500677650864
4.7030303030303 0.00113423377933608
4.75252525252525 0.00112365993869956
4.8020202020202 0.00111331782033708
4.85151515151515 0.00110320331431057
4.9010101010101 0.00109328131516595
4.95050505050505 0.00108360615769931
5 0.00107405925692061
};
\end{axis}

\end{tikzpicture}

%% file: Figures/supplementary/n_over_p=10/conditional_variance_bo_temp_scaling_n_over_p=10.0.tex
\begin{tikzpicture}
\tikzstyle{every node}=[font=\tiny]

\definecolor{darkgray176}{RGB}{176,176,176}
\definecolor{lightgray204}{RGB}{204,204,204}

\begin{axis}[
height=\figheight,
legend cell align={left},
legend style={fill opacity=0.8, draw opacity=1, text opacity=1, draw=lightgray204},
tick align=outside,
tick pos=left,
width=\figwidth,
x grid style={darkgray176},
xlabel={$p / n$},
xmajorgrids,
xmin=-0.145, xmax=5.245,
xtick style={color=black},
y grid style={darkgray176},
ylabel={Var(\(\displaystyle \hat{f}_{\rm bo} | \hat{f}_{\rm erm} = 0.75\))},
ymajorgrids,
ymin=-0.000546659256332005, ymax=0.0139374341008725,
ytick style={color=black}
]
\addplot [semithick, blue, dashdotdotted]
table {%
0.1 0.0046535257393876
0.14949494949495 0.00345642916033428
0.198989898989899 0.00268113411915827
0.248484848484848 0.00216524262581574
0.297979797979798 0.00180882850989639
0.347474747474747 0.00155101118140855
0.396969696969697 0.00135684780792977
0.446464646464647 0.00120573054171846
0.495959595959596 0.00108492731199827
0.545454545454546 0.00098621791587572
0.594949494949495 0.000904081896325737
0.644444444444444 0.000834684197421498
0.693939393939394 0.000775282001816202
0.743434343434343 0.000723863838372529
0.792929292929293 0.0006789231663622
0.842424242424242 0.000639305717554217
0.891919191919192 0.000604119763294553
0.941414141414142 0.000572659341466109
0.990909090909091 0.000544361238730806
1.04040404040404 0.000518770214306463
1.08989898989899 0.000495514345352666
1.13939393939394 0.000474286828568204
1.18888888888889 0.000454832349286671
1.23838383838384 0.000436936723556935
1.28787878787879 0.000420418943640311
1.33737373737374 0.000405124995067618
1.38686868686869 0.000390923048155689
1.43636363636364 0.000377699071718585
1.48585858585859 0.000365355583520133
1.53535353535354 0.000353806644235122
1.58484848484849 0.000342977459699578
1.63434343434343 0.000332802356454387
1.68383838383838 0.000323223428715491
1.73333333333333 0.000314189413570642
1.78282828282828 0.000305655792040893
1.83232323232323 0.000297579956759719
1.88181818181818 0.000289926645463789
1.93131313131313 0.000282663309127185
1.98080808080808 0.000275760670134795
2.03030303030303 0.00026919230788125
2.07979797979798 0.000262934342749266
2.12929292929293 0.000256965115208341
2.17878787878788 0.000251264944549634
2.22828282828283 0.00024581590309869
2.27777777777778 0.000240601629292958
2.32727272727273 0.000235607155638862
2.37676767676768 0.000230818767045005
2.42626262626263 0.000226223868020181
2.47575757575758 0.000221810875853379
2.52525252525253 0.000217569113677984
2.57474747474748 0.000213488731163003
2.62424242424242 0.000209560622093208
2.67373737373737 0.000205776352858789
2.72323232323232 0.000202128105923349
2.77272727272727 0.000198608621054031
2.82222222222222 0.000195211150367358
2.87171717171717 0.000191929408106195
2.92121212121212 0.000188757537275985
2.97070707070707 0.000185690069908095
3.02020202020202 0.00018272189617996
3.06969696969697 0.000179848238763936
3.11919191919192 0.000177064619440115
3.16868686868687 0.000174366842039309
3.21818181818182 0.000171750966568185
3.26767676767677 0.000169213293335857
3.31717171717172 0.000166750341574873
3.36666666666667 0.000164358836792666
3.41616161616162 0.000162035693293427
3.46565656565657 0.000159777998796629
3.51515151515152 0.000157583009004036
3.56464646464647 0.000155448128548064
3.61414141414142 0.000153370904624528
3.66363636363636 0.000151349018483682
3.71313131313131 0.000149380287131073
3.76262626262626 0.000147462599137227
3.81212121212121 0.000145593995393489
3.86161616161616 0.000143772603520387
3.91111111111111 0.000141996644637765
3.96060606060606 0.000140264430014359
4.01010101010101 0.00013857435023501
4.05959595959596 0.00013692488064676
4.10909090909091 0.000135314562251021
4.15858585858586 0.000133742013169624
4.20808080808081 0.000132205908741523
4.25757575757576 0.000130704987835184
4.30707070707071 0.000129238051485236
4.35656565656566 0.000127803949542593
4.40606060606061 0.00012640158551791
4.45555555555556 0.000125029911628549
4.50505050505051 0.000123687925795424
4.55454545454546 0.000122374668773517
4.6040404040404 0.000121089222534509
4.65353535353535 0.000119830707782986
4.7030303030303 0.000118598282023208
4.75252525252525 0.000117391139728196
4.8020202020202 0.000116208504213455
4.85151515151515 0.000115049634391684
4.9010101010101 0.000113913816392142
4.95050505050505 0.000112800364299726
5 0.000111708623540929
};
\addlegendentry{$\lambda_{\rm error}$, TS}
\addplot [semithick, red, dashdotdotted]
table {%
0.1 0.00489022091354641
0.14949494949495 0.0037504656570958
0.198989898989899 0.00293957324393157
0.248484848484848 0.00238240363567743
0.297979797979798 0.00199254276981498
0.347474747474747 0.00170869868026757
0.396969696969697 0.00149412153809925
0.446464646464647 0.00132670588492645
0.495959595959596 0.00119265377741506
0.545454545454546 0.00108299588159422
0.594949494949495 0.000991680476469137
0.644444444444444 0.000914488996675122
0.693939393939394 0.000848396343829494
0.743434343434343 0.000791179451787749
0.792929292929293 0.000741170213710118
0.842424242424242 0.00069708908709476
0.891919191919192 0.000657945803145599
0.941414141414142 0.000622956032011235
0.990909090909091 0.000591493381205233
1.04040404040404 0.000563050896981832
1.08989898989899 0.000537214441935219
1.13939393939394 0.000513642083252885
1.18888888888889 0.000492050374734654
1.23838383838384 0.000472197790103435
1.28787878787879 0.000453883587093862
1.33737373737374 0.000436935811567118
1.38686868686869 0.000421207171491522
1.43636363636364 0.000406570792947036
1.48585858585859 0.000392916810317279
1.53535353535354 0.000380149581612366
1.58484848484849 0.000368185543855981
1.63434343434343 0.000356951219989465
1.68383838383838 0.00034638191693015
1.73333333333333 0.000336420247513014
1.78282828282828 0.000327015360487848
1.83232323232323 0.000318121819293804
1.88181818181818 0.000309699104439232
1.93131313131313 0.000301710751664852
1.98080808080808 0.000294124009415753
2.03030303030303 0.000286909428166671
2.07979797979798 0.000280040231395828
2.12929292929293 0.000273492208764159
2.17878787878788 0.000267243344228341
2.22828282828283 0.000261273621327662
2.27777777777778 0.000255564751027171
2.32727272727273 0.00025010002386483
2.37676767676768 0.000244864067105222
2.42626262626263 0.000239842849647176
2.47575757575758 0.000235023433227677
2.52525252525253 0.00023039384089385
2.57474747474748 0.000225943151629604
2.62424242424242 0.000221661134507833
2.67373737373737 0.00021753840005867
2.72323232323232 0.000213566258663644
2.77272727272727 0.000209736532185922
2.82222222222222 0.000206041802852663
2.87171717171717 0.000202474950403264
2.92121212121212 0.000199029520685068
2.97070707070707 0.000195699389830328
3.02020202020202 0.000192478891133829
3.06969696969697 0.000189362676525251
3.11919191919192 0.00018634575353027
3.16868686868687 0.000183423497766011
3.21818181818182 0.00018059145465199
3.26767676767677 0.000177845570255464
3.31717171717172 0.000175181951247771
3.36666666666667 0.000172596923751933
3.41616161616162 0.000170087110854977
3.46565656565657 0.000167649247946544
3.51515151515152 0.000165280300795878
3.56464646464647 0.000162977398181496
3.61414141414142 0.000160737776075948
3.66363636363636 0.000158558895652527
3.71313131313131 0.000156438326218256
3.76262626262626 0.000154373706503952
3.81212121212121 0.00015236293516252
3.86161616161616 0.000150403860116888
3.91111111111111 0.000148494522520859
3.96060606060606 0.000146633088505155
4.01010101010101 0.000144817761021265
4.05959595959596 0.00014304686634059
4.10909090909091 0.000141318771461396
4.15858585858586 0.000139631963842701
4.20808080808081 0.000137984927760071
4.25757575757576 0.00013637630340213
4.30707070707071 0.000134804761983021
4.35656565656566 0.00013326906661304
4.40606060606061 0.000131767439822794
4.45555555555556 0.000130299759762198
4.50505050505051 0.000128864467461631
4.55454545454546 0.000127460436577964
4.6040404040404 0.000126086701160544
4.65353535353535 0.0001247422522066
4.7030303030303 0.000123426185544728
4.75252525252525 0.000122137623505658
4.8020202020202 0.000120875689604727
4.85151515151515 0.000119639565151375
4.9010101010101 0.000118428483018573
4.95050505050505 0.000117241565249393
5 0.000116078617321658
};
\addlegendentry{$\lambda_{\rm loss}$, TS}
\addplot [semithick, black, dashdotdotted]
table {%
0.1 0.00512511465098908
0.14949494949495 0.00563820684226202
0.198989898989899 0.00704277619215732
0.248484848484848 0.00850963422419082
0.297979797979798 0.00984811378896056
0.347474747474747 0.0110381712691096
0.396969696969697 0.0120881240020264
0.446464646464647 0.0129273918782212
0.495959595959596 0.0132790662209996
0.545454545454546 0.0132207181821032
0.594949494949495 0.0130613274564932
0.644444444444444 0.0128776054685474
0.693939393939394 0.0126867471461066
0.743434343434343 0.0124945952968722
0.792929292929293 0.0123037127575081
0.842424242424242 0.0121154044366635
0.891919191919192 0.0119303857971477
0.941414141414142 0.0117490718219989
0.990909090909091 0.0115716898881603
1.04040404040404 0.0113983609322101
1.08989898989899 0.0112291306090596
1.13939393939394 0.0110639977896808
1.18888888888889 0.0109029287346558
1.23838383838384 0.0107458623515573
1.28787878787879 0.0105927263662524
1.33737373737374 0.0104434367347382
1.38686868686869 0.0102978996824024
1.43636363636364 0.0101560210352896
1.48585858585859 0.0100176958509599
1.53535353535354 0.00988282979266963
1.58484848484849 0.00975131940835272
1.63434343434343 0.00962306350908926
1.68383838383838 0.00949796283548643
1.73333333333333 0.00937592336180548
1.78282828282828 0.00925684749194156
1.83232323232323 0.00914064454922459
1.88181818181818 0.00902722137135925
1.93131313131313 0.00891649050282428
1.98080808080808 0.00880837182059047
2.03030303030303 0.00870277852594081
2.07979797979798 0.00859963306615241
2.12929292929293 0.00849885598041844
2.17878787878788 0.00840037571926944
2.22828282828283 0.00830411962151001
2.27777777777778 0.00821001865700444
2.32727272727273 0.00811800541906682
2.37676767676768 0.00802801577058421
2.42626262626263 0.00793998770668569
2.47575757575758 0.00785386147011968
2.52525252525253 0.00776957918302434
2.57474747474748 0.00768708527246509
2.62424242424242 0.00760632797400984
2.67373737373737 0.00752725344060046
2.72323232323232 0.00744981264797573
2.77272727272727 0.00737395765480142
2.82222222222222 0.00729964224243651
2.87171717171717 0.00722682186009027
2.92121212121212 0.00715545355861846
2.97070707070707 0.00708549584935536
3.02020202020202 0.00701690801971466
3.06969696969697 0.00694965183920948
3.11919191919192 0.00688369638259301
3.16868686868687 0.00681899520822637
3.21818181818182 0.0067555156675897
3.26767676767677 0.00669323055512883
3.31717171717172 0.00663210422787996
3.36666666666667 0.00657210551187604
3.41616161616162 0.00651320430276248
3.46565656565657 0.0064553748340318
3.51515151515152 0.00639858238270075
3.56464646464647 0.0063428031242142
3.61414141414142 0.00628801085526876
3.66363636363636 0.00623417896406242
3.71313131313131 0.00618128408190399
3.76262626262626 0.00612930496167308
3.81212121212121 0.00607821419896848
3.86161616161616 0.00602799360318373
3.91111111111111 0.00597861851379267
3.96060606060606 0.00593006759315129
4.01010101010101 0.00588232518675635
4.05959595959596 0.00583536764357584
4.10909090909091 0.00578917746255836
4.15858585858586 0.00574373651281235
4.20808080808081 0.00569902567713609
4.25757575757576 0.0056550294774228
4.30707070707071 0.00561172746079541
4.35656565656566 0.00556911040612296
4.40606060606061 0.00552715707119478
4.45555555555556 0.0054858539588224
4.50505050505051 0.00544518613329781
4.55454545454546 0.00540513925406871
4.6040404040404 0.00536569939435882
4.65353535353535 0.0053268530772651
4.7030303030303 0.00528858692018397
4.75252525252525 0.00525088826179765
4.8020202020202 0.00521374488776161
4.85151515151515 0.00517714476991682
4.9010101010101 0.0051410762186006
4.95050505050505 0.00510553067357233
5 0.00507049143795357
};
\addlegendentry{$\lambda = 0$, TS}
\end{axis}

\end{tikzpicture}

%% file: Figures/supplementary/n_over_p=20/test_errors_n_over_p=20.0.tex
\begin{tikzpicture}
\tikzstyle{every node}=[font=\tiny]

\definecolor{darkgray176}{RGB}{176,176,176}
\definecolor{gray}{RGB}{128,128,128}
\definecolor{lightgray204}{RGB}{204,204,204}
\definecolor{orange}{RGB}{255,165,0}

\begin{axis}[
height=\figheight,
legend cell align={left},
legend style={fill opacity=0.8, draw opacity=1, text opacity=1, draw=lightgray204},
tick align=outside,
tick pos=left,
width=\figwidth,
x grid style={darkgray176},
xlabel={$p / n$},
xmajorgrids,
xmin=-0.145, xmax=5.245,
xtick style={color=black},
y grid style={darkgray176},
ylabel={\(\displaystyle \varepsilon_g\)},
ymajorgrids,
ymin=0.0873477382663537, ymax=0.149895613528467,
ytick style={color=black}
]
\addplot [semithick, black]
table {%
0.1 0.14705252828928
0.45 0.123447910499816
0.8 0.110092049985006
1.15 0.104351508644367
1.5 0.101138088722832
1.85 0.099079988527156
2.2 0.0976478702579267
2.55 0.096593384739216
2.9 0.0957843639186934
3.25 0.0951439503941508
3.6 0.0946243582065108
3.95 0.0941943129038997
4.3 0.093832486895555
4.65 0.0935238305059264
5 0.093257419861089
};
\addlegendentry{$\hat{f}_{\erm}(\lambda = 10^{-4})$}
\addplot [semithick, gray]
table {%
0.1 0.132858814013325
0.45 0.0971091793012206
0.8 0.093693722765965
1.15 0.092403452241728
1.5 0.0917252999097901
1.85 0.0913071495639692
2.2 0.091023517916674
2.55 0.0908184839362717
2.9 0.0906633515271793
3.25 0.0905418787471711
3.6 0.0904441825441139
3.95 0.0903639041452583
4.3 0.0902967666940624
4.65 0.0902397876267469
5 0.0901908235055407
};
\addlegendentry{$\hat{f}_{\bo}$}
\addplot [semithick, red]
table {%
0.1 0.137589968852639
0.45 0.100923769444448
0.8 0.0966856637913891
1.15 0.0948882191465593
1.5 0.0938616289164673
1.85 0.0931875076978088
2.2 0.0927071716724154
2.55 0.0923459634599993
2.9 0.0920637034395676
3.25 0.091836676694035
3.6 0.0916499084007281
3.95 0.0914934439573224
4.3 0.0913603911288212
4.65 0.0912458171217036
5 0.0911460938942609
};
\addlegendentry{$\hat{f}_{\erm}(\lambda_{\rm loss})$}
\addplot [semithick, blue]
table {%
0.1 0.137799479499874
0.45 0.099787013044243
0.8 0.0956450544223995
1.15 0.0940056260016959
1.5 0.0931132072589547
1.85 0.0925466447914887
2.2 0.0921525195327662
2.55 0.0918611793473844
2.9 0.0916362855385032
3.25 0.0914569551702855
3.6 0.0913103104070202
3.95 0.0911879461739863
4.3 0.0910841462217404
4.65 0.0909948764899563
5 0.0909172060513411
};
\addlegendentry{$\hat{f}_{\erm}(\lambda_{\rm error})$}
\addplot [semithick, orange, dashed]
table {%
0.1 0.147036692718623
0.45 0.11488882000905
0.8 0.103644555490741
1.15 0.099348086557094
1.5 0.097081252410513
1.85 0.0956764736209821
2.2 0.0947217774727732
2.55 0.0940291899780826
2.9 0.0935058008370034
3.25 0.0930943078067685
3.6 0.0927641892324531
3.95 0.0924918033200538
4.3 0.0922638390759112
4.65 0.0920703027617972
5 0.0919039687594985
};
\addlegendentry{$\hat{f}_{\empbayes}(\lambda_{\rm evidence})$}
\addplot [semithick, blue, dashed]
table {%
0.1 0.137821996978902
0.45 0.0998534165045269
0.8 0.095713950550578
1.15 0.0940734711671764
1.5 0.093178714132605
1.85 0.092609527149394
2.2 0.0922126798824099
2.55 0.0919187006369474
2.9 0.0916922185998854
3.25 0.0915104428875782
3.6 0.0913623822690686
3.95 0.0912378124145296
4.3 0.0911327601605914
4.65 0.0910423336753896
5 0.0909627493625637
};
\addlegendentry{$\hat{f}_{\empbayes}(\lambda_{\rm error})$}
\end{axis}

\end{tikzpicture}

%% file: Figures/supplementary/n_over_p=20/calibration_p=0.75_n_over_p=20.0.tex
\begin{tikzpicture}
\tikzstyle{every node}=[font=\tiny]

\definecolor{darkgray176}{RGB}{176,176,176}
\definecolor{lightgray204}{RGB}{204,204,204}
\definecolor{orange}{RGB}{255,165,0}

\begin{axis}[
height=\figheight,
legend cell align={left},
legend style={fill opacity=0.8, draw opacity=1, text opacity=1, draw=lightgray204},
tick align=outside,
tick pos=left,
width=\figwidth,
x grid style={darkgray176},
xlabel={$p / n$},
xmajorgrids,
xmin=-0.145, xmax=5.245,
xtick style={color=black},
y grid style={darkgray176},
ylabel={\(\displaystyle \Delta_{0.75}\)},
ymajorgrids,
ymin=-0.113101820590194, ymax=0.201168325504033,
ytick style={color=black}
]
\addplot [semithick, black]
table {%
0.1 0.107927852076012
0.45 0.186883318863386
0.8 0.168728067561377
1.15 0.159826505683534
1.5 0.154479219084355
1.85 0.15089794122368
2.2 0.148327444500133
2.55 0.146391055912861
2.9 0.144879146886293
3.25 0.143665604871376
3.6 0.142669797788127
3.95 0.141837872888306
4.3 0.141132376590407
4.65 0.140526483747188
5 0.140000466025213
};
\addplot [semithick, blue]
table {%
0.1 -0.0398842699444951
0.45 -0.0860753851322875
0.8 -0.0805423109445055
1.15 -0.0741669530971241
1.5 -0.0686715181442888
1.85 -0.0640252535066516
2.2 -0.0600567714267223
2.55 -0.0566188304326924
2.9 -0.0536004041398175
3.25 -0.0509131897722155
3.6 -0.0485122904573859
3.95 -0.0463339211719047
4.3 -0.0443470407913288
4.65 -0.04252295809461
5 -0.0408387777697945
};
\addplot [semithick, red]
table {%
0.1 -0.0173635249663473
0.45 -0.0183207992085063
0.8 -0.012103497697697
1.15 -0.00804115635145564
1.5 -0.00522574397118292
1.85 -0.00316710218796878
2.2 -0.00159943600948087
2.55 -0.00036750401292418
2.9 0.000625134981122266
3.25 0.00144146760217401
3.6 0.00212423927619965
3.95 0.00270351336328178
4.3 0.00320103893451473
4.65 0.00363286995544565
5 0.00401111870020232
};
\addplot [semithick, orange, dashed]
table {%
0.1 -0.0156839243888788
0.45 -0.0132357919196072
0.8 -0.0133569305137629
1.15 -0.0135805244403967
1.5 -0.0137193862758735
1.85 -0.0137258633094517
2.2 -0.0136310589363685
2.55 -0.013530232940737
2.9 -0.0133412373475514
3.25 -0.0146148242638778
3.6 -0.0159279303991774
3.95 -0.0170605658712411
4.3 -0.0179675742550474
4.65 -0.018740669267151
5 -0.0194057251778565
};
\addplot [semithick, blue, dashed]
table {%
0.1 -0.0475607183319438
0.45 -0.0988168139495473
0.8 -0.0956378113882264
1.15 -0.0910514847361321
1.5 -0.0869514157889513
1.85 -0.0835653946489198
2.2 -0.0806518595585091
2.55 -0.0781479336172007
2.9 -0.0760132040672916
3.25 -0.0740924997554584
3.6 -0.0724368070385896
3.95 -0.070906847385801
4.3 -0.0695646508641224
4.65 -0.0683523191474319
5 -0.0672020477558956
};
\end{axis}

\end{tikzpicture}

%% file: Figures/supplementary/n_over_p=20/conditional_variance_bo_p=0.75_n_over_p=20.0.tex
\begin{tikzpicture}
\tikzstyle{every node}=[font=\tiny]

\definecolor{darkgray176}{RGB}{176,176,176}
\definecolor{lightgray204}{RGB}{204,204,204}
\definecolor{orange}{RGB}{255,165,0}

\begin{axis}[
height=\figheight,
legend cell align={left},
legend style={fill opacity=0.8, draw opacity=1, text opacity=1, draw=lightgray204},
tick align=outside,
tick pos=left,
width=\figwidth,
x grid style={darkgray176},
xlabel={$p / n$},
xmajorgrids,
xmin=-0.145, xmax=5.245,
xtick style={color=black},
ylabel={\(\displaystyle {\rm Var}(\hat{f}_{\rm bo} | \hat{f} = 0.75)\)},
ymajorgrids,
ymin=-0.00181409454883606, ymax=0.0641862654701949,
ytick style={color=black}
]
\addplot [semithick, black]
table {%
0.1 0.025950516466403
0.45 0.0611862491056935
0.8 0.0437280412561672
1.15 0.034048189160786
1.5 0.0279027742800436
1.85 0.023653333918339
2.2 0.0205388601771112
2.55 0.0181578159961969
2.9 0.0162782122258597
3.25 0.0147566733778146
3.6 0.0134997233019415
3.95 0.0124438317157602
4.3 0.0115443170774449
4.65 0.0107688266271428
5 0.010093358872353
};
\addplot [semithick, blue]
table {%
0.1 0.00527747812401902
0.45 0.00280059457627568
0.8 0.00222715225943426
1.15 0.00196264841651006
1.5 0.00179552390377336
1.85 0.00167436150778943
2.2 0.00157983114825455
2.55 0.00150270170062705
2.9 0.00143785158068344
3.25 0.00138219520024663
3.6 0.00133349361772461
3.95 0.00129046917597087
4.3 0.00125202369045818
4.65 0.00121737496370566
5 0.00118592181566535
};
\addplot [semithick, blue, dashed]
table {%
0.1 0.00610336984650794
0.45 0.00386902954240187
0.8 0.00321405095168625
1.15 0.00289873402304697
1.5 0.00269381447017214
1.85 0.00253861926515242
2.2 0.00241528524076884
2.55 0.0023123986560718
2.9 0.00222566227479404
3.25 0.00214855851211659
3.6 0.00208157391613073
3.95 0.00202018551329652
4.3 0.00196622020196746
4.65 0.00191709834289289
5 0.00187047734115942
};
\addplot [semithick, red]
table {%
0.1 0.00582434840590695
0.45 0.00657565125453186
0.8 0.00557410205108311
1.15 0.00481992782579199
1.5 0.00425111208907736
1.85 0.00380878760790115
2.2 0.00345539335350475
2.55 0.00316672085599201
2.9 0.00292657303930421
3.25 0.00272372420092337
3.6 0.00255015007172921
3.95 0.00239996829845057
4.3 0.00226876970803336
4.65 0.00215318301354051
5 0.00205058772040712
};
\addplot [semithick, orange, dashed]
table {%
0.1 0.024824282873214
0.45 0.0367017049228104
0.8 0.0203827513931052
1.15 0.0139298487743965
1.5 0.010550464319913
1.85 0.00848885472071215
2.2 0.00710701686467363
2.55 0.00611766900028177
2.9 0.00539668417414074
3.25 0.00486811046514457
3.6 0.00440156871406949
3.95 0.00401994301118447
4.3 0.00370271673504341
4.65 0.00343462844357378
5 0.00320552086966375
};
\end{axis}

\end{tikzpicture}

%% file: Figures/supplementary/n_over_p=20/calibration_laplace_p=0.75_n_over_p=20.0.tex
\begin{tikzpicture}
\tikzstyle{every node}=[font=\tiny]

\definecolor{darkgray176}{RGB}{176,176,176}
\definecolor{lightgray204}{RGB}{204,204,204}

\begin{axis}[
height=\figheight,
legend cell align={left},
legend style={
  fill opacity=0.8,
  draw opacity=1,
  text opacity=1,
  at={(0.5,0.91)},
  anchor=north,
  draw=lightgray204
},
tick align=outside,
tick pos=left,
width=\figwidth,
x grid style={darkgray176},
xlabel={$p / n$},
xmajorgrids,
xmin=-0.145, xmax=5.245,
xtick style={color=black},
y grid style={darkgray176},
ylabel={\(\displaystyle \Delta_{0.75}\)},
ymajorgrids,
ymin=-0.270719185830655, ymax=0.208673914325007,
ytick style={color=black}
]
\addplot [semithick, black, dash pattern=on 1pt off 3pt on 3pt off 3pt]
table {%
0.1 0.0600679819012684
0.45 -0.186402449992992
0.8 -0.236804220313708
1.15 -0.244241038556245
1.5 -0.246453168303106
1.85 -0.247407594538201
2.2 -0.247916135795306
2.55 -0.248225239241184
2.9 -0.248430531926633
3.25 -0.248575762394409
3.6 -0.248683452228364
3.95 -0.248766247131381
4.3 -0.248831754675017
4.65 -0.248884802077458
5 -0.248928590369034
};
\addlegendentry{$\hat{f}_{\lap}(\lambda = 1e-4)$}
\addplot [semithick, blue, dash pattern=on 1pt off 3pt on 3pt off 3pt]
table {%
0.1 -0.0684807774780954
0.45 -0.129084352899877
0.8 -0.130198303109129
1.15 -0.128883046977121
1.5 -0.127499488950033
1.85 -0.126324743981109
2.2 -0.125365011006587
2.55 -0.124585258936707
2.9 -0.12394991893608
3.25 -0.123427850756378
3.6 -0.123001880593275
3.95 -0.122648855525052
4.3 -0.122356922178655
4.65 -0.122115407737043
5 -0.121915886524648
};
\addlegendentry{$\hat{f}_{\lap}(\lambda_{\rm error})$}
\addplot [semithick, red, dash pattern=on 1pt off 3pt on 3pt off 3pt]
table {%
0.1 -0.0493310443692409
0.45 -0.0898704361617135
0.8 -0.0989360213883236
1.15 -0.104225663087576
1.5 -0.107903678626277
1.85 -0.110658297097667
2.2 -0.112812464004737
2.55 -0.114547732250084
2.9 -0.115977029843495
3.25 -0.117175283190375
3.6 -0.11819450632901
3.95 -0.119072078633275
4.3 -0.119835612364849
4.65 -0.120505959354408
5 -0.121099169119331
};
\addlegendentry{$\hat{f}_{\lap}(\lambda_{\rm loss})$}
\addplot [semithick, black]
table {%
0.1 0.107927852076012
0.45 0.186883318863386
0.8 0.168728067561377
1.15 0.159826505683534
1.5 0.154479219084355
1.85 0.15089794122368
2.2 0.148327444500133
2.55 0.146391055912861
2.9 0.144879146886293
3.25 0.143665604871376
3.6 0.142669797788127
3.95 0.141837872888306
4.3 0.141132376590407
4.65 0.140526483747188
5 0.140000466025213
};
\addplot [semithick, blue]
table {%
0.1 -0.0398842699444951
0.45 -0.0860753851322875
0.8 -0.0805423109445055
1.15 -0.0741669530971241
1.5 -0.0686715181442888
1.85 -0.0640252535066516
2.2 -0.0600567714267223
2.55 -0.0566188304326924
2.9 -0.0536004041398175
3.25 -0.0509131897722155
3.6 -0.0485122904573859
3.95 -0.0463339211719047
4.3 -0.0443470407913288
4.65 -0.04252295809461
5 -0.0408387777697945
};
\addplot [semithick, red]
table {%
0.1 -0.0173635249663473
0.45 -0.0183207992085063
0.8 -0.012103497697697
1.15 -0.00804115635145564
1.5 -0.00522574397118292
1.85 -0.00316710218796878
2.2 -0.00159943600948087
2.55 -0.00036750401292418
2.9 0.000625134981122266
3.25 0.00144146760217401
3.6 0.00212423927619965
3.95 0.00270351336328178
4.3 0.00320103893451473
4.65 0.00363286995544565
5 0.00401111870020232
};
\end{axis}

\end{tikzpicture}

%% file: Figures/supplementary/n_over_p=20/calibration_temp_scaling_p=0.75_n_over_p=20.0.tex
\begin{tikzpicture}
\tikzstyle{every node}=[font=\tiny]

\definecolor{darkgray176}{RGB}{176,176,176}
\definecolor{lightgray204}{RGB}{204,204,204}
\definecolor{orange}{RGB}{255,165,0}

\begin{axis}[
height=\figheight,
legend cell align={left},
legend style={fill opacity=0.8, draw opacity=1, text opacity=1, draw=lightgray204},
tick align=outside,
tick pos=left,
width=\figwidth,
x grid style={darkgray176},
xlabel={\(\displaystyle p/n\)},
xmajorgrids,
xmin=-0.145, xmax=5.245,
xtick style={color=black},
y grid style={darkgray176},
ylabel={\(\displaystyle \Delta\)0.75},
ymajorgrids,
ymin=-0.0250279069121927, ymax=0.0986317711376744,
ytick style={color=black}
]
\addplot [semithick, black, dashdotdotted]
table {%
0.1 0.0151265515980198
0.45 0.0150257803147456
0.8 0.0140028587054122
1.15 0.0131685959693377
1.5 0.0125518915244869
1.85 0.01208802132066
2.2 0.0117291923683863
2.55 0.0114443014856004
2.9 0.0112130063976452
3.25 0.0110216675302565
3.6 0.0108608824849439
3.95 0.0107238750493203
4.3 0.010605868828454
4.65 0.0105030774016839
5 0.0104126119249154
};
\addlegendentry{$\hat{f}_{\erm}(\lambda = 10^{-4})$ - TS}
\addplot [semithick, blue, dashdotdotted]
table {%
0.1 0.0152589797456389
0.45 0.0122544385774268
0.8 0.0111720929288718
1.15 0.0106629579085867
1.5 0.010362976810903
1.85 0.0101644488021516
2.2 0.0100220758166865
2.55 0.00991460130671029
2.9 0.00983036729182918
3.25 0.00976231806685612
3.6 0.00970611102773733
3.95 0.00965890624272181
4.3 0.00961867569895647
4.65 0.00958375356294461
5 0.00955310986823521
};
\addlegendentry{$\hat{f}_{\erm}(\lambda_{\rm error})$ - TS}
\addplot [semithick, red, dashdotdotted]
table {%
0.1 0.0152599404315407
0.45 0.0125059998489114
0.8 0.011470046026855
1.15 0.0109432479814194
1.5 0.0106151682717757
1.85 0.0103885332058733
2.2 0.0102212601250304
2.55 0.010092428617406
2.9 0.00998956142036811
3.25 0.00990552149822288
3.6 0.00983540099859548
3.95 0.00977614089397361
4.3 0.00972529435897007
4.65 0.00968117293533899
5 0.0096425150960352
};
\addlegendentry{$\hat{f}_{\erm}(\lambda_{\rm loss})$ - TS}
\addplot [semithick, orange, dashed]
table {%
0.1 -0.0156839243888788
0.45 -0.0132357919196072
0.8 -0.0133569305137629
1.15 -0.0135805244403967
1.5 -0.0137193862758735
1.85 -0.0137258633094517
2.2 -0.0136310589363685
2.55 -0.013530232940737
2.9 -0.0133412373475514
3.25 -0.0146148242638778
3.6 -0.0159279303991774
3.95 -0.0170605658712411
4.3 -0.0179675742550474
4.65 -0.018740669267151
5 -0.0194057251778565
};
\addlegendentry{$\hat{f}_{\empbayes}(\lambda_{\rm evidence}$}
\end{axis}

\end{tikzpicture}

%% file: Figures/supplementary/n_over_p=20/conditional_variance_bo_temp_scaling_n_over_p=20.0.tex
\begin{tikzpicture}
\tikzstyle{every node}=[font=\tiny]

\definecolor{darkgray176}{RGB}{176,176,176}
\definecolor{lightgray204}{RGB}{204,204,204}

\begin{axis}[
height=\figheight,
legend cell align={left},
legend style={fill opacity=0.8, draw opacity=1, text opacity=1, draw=lightgray204},
tick align=outside,
tick pos=left,
width=\figwidth,
x grid style={darkgray176},
xlabel={$p / n$},
xmajorgrids,
xmin=-0.145, xmax=5.245,
xtick style={color=black},
y grid style={darkgray176},
ylabel={Var(\(\displaystyle \hat{f}_{\rm bo} | \hat{f}_{\rm erm} = 0.75\))},
ymajorgrids,
ymin=-0.000476965169818861, ymax=0.0453895689225406,
ytick style={color=black}
]
\addplot [semithick, blue, dashdotdotted]
table {%
0.1 0.00721752034284318
0.45 0.00547838043047499
0.8 0.0041450978739902
1.15 0.00345356394230589
1.5 0.00301527343851704
1.85 0.00270614434148575
2.2 0.00247328574271088
2.55 0.00228990388395012
2.9 0.0021407738262591
3.25 0.00201650919467167
3.6 0.00191098644184384
3.95 0.00181998270571004
4.3 0.00174050608270493
4.65 0.00167035581735941
5 0.00160787728892475
};
\addplot [semithick, red, dashdotdotted]
table {%
0.1 0.00692130699169802
0.45 0.00773360843980953
0.8 0.00630508692354093
1.15 0.0053210288417288
1.5 0.0046161453914374
1.85 0.00408667299672294
2.2 0.00367408841748473
2.55 0.00334337359292558
2.9 0.00307227971585999
3.25 0.00284598984633122
3.6 0.00265423047641811
3.95 0.00248965719255057
4.3 0.00234687082935137
4.65 0.00222181460797877
5 0.00211138021491908
};
\addplot [semithick, black, dashdotdotted]
table {%
0.1 0.0193811041789516
0.45 0.043304726463797
0.8 0.0304807232692247
1.15 0.0234645742006905
1.5 0.019064749028184
1.85 0.016053770467099
2.2 0.0138660300930583
2.55 0.0122056158212503
2.9 0.0109029475497564
3.25 0.00985400104038159
3.6 0.00899141998715725
3.95 0.00826970205619215
4.3 0.0076570300518668
4.65 0.00713047324987282
5 0.00667309856153986
};
\end{axis}

\end{tikzpicture}

%% file: main.bbl
\begin{thebibliography}{}

\bibitem[Abbasi et~al., 2019]{NEURIPS2019_dffbb6ef}
Abbasi, E., Salehi, F., and Hassibi, B. (2019).
\newblock Universality in learning from linear measurements.
\newblock In Wallach, H., Larochelle, H., Beygelzimer, A., d\textquotesingle
  Alch\'{e}-Buc, F., Fox, E., and Garnett, R., editors, {\em Advances in Neural
  Information Processing Systems}, volume~32. Curran Associates, Inc.

\bibitem[Abdar et~al., 2021]{ABDAR2021243}
Abdar, M., Pourpanah, F., Hussain, S., Rezazadegan, D., Liu, L., Ghavamzadeh,
  M., Fieguth, P., Cao, X., Khosravi, A., Acharya, U.~R., Makarenkov, V., and
  Nahavandi, S. (2021).
\newblock A review of uncertainty quantification in deep learning: Techniques,
  applications and challenges.
\newblock {\em Information Fusion}, 76:243--297.

\bibitem[Aizenman et~al., 2006]{aizenman2006mean}
Aizenman, M., Sims, R., and Starr, S.~L. (2006).
\newblock Mean-field spin glass models from the cavity--rost perspective.
\newblock {\em arXiv preprint math-ph/0607060}.

\bibitem[Alexos et~al., 2022]{alexos_structured_2022}
Alexos, A., Boyd, A.~J., and Mandt, S. (2022).
\newblock Structured stochastic gradient {MCMC}.
\newblock In Chaudhuri, K., Jegelka, S., Song, L., Szepesvari, C., Niu, G., and
  Sabato, S., editors, {\em Proceedings of the 39th International Conference on
  Machine Learning}, volume 162 of {\em Proceedings of Machine Learning
  Research}, pages 414--434. PMLR.

\bibitem[Aubin et~al., 2020]{aubin_generalization_2020}
Aubin, B., Krzakala, F., Lu, Y., and Zdeborov\'{a}, L. (2020).
\newblock Generalization error in high-dimensional perceptrons: Approaching
  bayes error with convex optimization.
\newblock In Larochelle, H., Ranzato, M., Hadsell, R., Balcan, M., and Lin, H.,
  editors, {\em Advances in Neural Information Processing Systems}, volume~33,
  pages 12199--12210. Curran Associates, Inc.

\bibitem[Aubin et~al., 2021]{9240945}
Aubin, B., Loureiro, B., Maillard, A., Krzakala, F., and Zdeborová, L. (2021).
\newblock The spiked matrix model with generative priors.
\newblock {\em IEEE Transactions on Information Theory}, 67(2):1156--1181.

\bibitem[Aubin et~al., 2018]{NEURIPS2018_84f0f204}
Aubin, B., Maillard, A., barbier, j., Krzakala, F., Macris, N., and
  Zdeborov\'{a}, L. (2018).
\newblock The committee machine: Computational to statistical gaps in learning
  a two-layers neural network.
\newblock In Bengio, S., Wallach, H., Larochelle, H., Grauman, K.,
  Cesa-Bianchi, N., and Garnett, R., editors, {\em Advances in Neural
  Information Processing Systems}, volume~31. Curran Associates, Inc.

\bibitem[Bai et~al., 2021]{bai_dont_2021}
Bai, Y., Mei, S., Wang, H., and Xiong, C. (2021).
\newblock Don’t just blame over-parametrization for over-confidence:
  Theoretical analysis of calibration in binary classification.
\newblock In Meila, M. and Zhang, T., editors, {\em Proceedings of the 38th
  International Conference on Machine Learning}, volume 139 of {\em Proceedings
  of Machine Learning Research}, pages 566--576. PMLR.

\bibitem[Bai and Zhou, 2008]{10.2307/24308489}
Bai, Z. and Zhou, W. (2008).
\newblock Large sample covariance matrices without independence structures in
  columns.
\newblock {\em Statist. Sinica}, 18(2):425--442.

\bibitem[Barbier et~al., 2021a]{barbier2021performance}
Barbier, J., Chen, W.-K., Panchenko, D., and S{\'a}enz, M. (2021a).
\newblock Performance of bayesian linear regression in a model with mismatch.
\newblock {\em arXiv preprint arXiv:2107.06936}.

\bibitem[Barbier et~al., 2019]{barbier_optimal_2019}
Barbier, J., Krzakala, F., Macris, N., Miolane, L., and Zdeborová, L. (2019).
\newblock Optimal errors and phase transitions in high-dimensional generalized
  linear models.
\newblock 116(12):5451--5460.

\bibitem[Barbier et~al., 2018]{barbier2018mutual}
Barbier, J., Macris, N., Maillard, A., and Krzakala, F. (2018).
\newblock The mutual information in random linear estimation beyond iid
  matrices.
\newblock In {\em 2018 IEEE International Symposium on Information Theory
  (ISIT)}, pages 1390--1394. IEEE.

\bibitem[Barbier et~al., 2021b]{JeanOverlap}
Barbier, J., Panchenko, D., and Sáenz, M. (2021b).
\newblock {Strong replica symmetry for high-dimensional disordered log-concave
  Gibbs measures}.
\newblock {\em Information and Inference: A Journal of the IMA},
  11(3):1079--1108.

\bibitem[Bartlett et~al., 2020]{Bartlett2020}
Bartlett, P.~L., Long, P.~M., Lugosi, G., and Tsigler, A. (2020).
\newblock Benign overfitting in linear regression.
\newblock {\em Proceedings of the National Academy of Sciences},
  117(48):30063--30070.

\bibitem[Bayati and Montanari, 2011]{bayati2011dynamics}
Bayati, M. and Montanari, A. (2011).
\newblock The dynamics of message passing on dense graphs, with applications to
  compressed sensing.
\newblock {\em IEEE Transactions on Information Theory}, 57(2):764--785.

\bibitem[Bayati and Montanari, 2012]{6069859}
Bayati, M. and Montanari, A. (2012).
\newblock The lasso risk for gaussian matrices.
\newblock {\em IEEE Transactions on Information Theory}, 58(4):1997--2017.

\bibitem[Belkin et~al., 2019]{Belkin2019}
Belkin, M., Hsu, D., Ma, S., and Mandal, S. (2019).
\newblock Reconciling modern machine-learning practice and the classical
  bias-variance trade-off.
\newblock {\em Proceedings of the National Academy of Sciences},
  116(32):15849--15854.

\bibitem[Benigni and P{\'e}ch{\'e}, 2021]{Benigni2021}
Benigni, L. and P{\'e}ch{\'e}, S. (2021).
\newblock {Eigenvalue distribution of some nonlinear models of random
  matrices}.
\newblock {\em Electronic Journal of Probability}, 26(none):1 -- 37.

\bibitem[Berthier et~al., 2019]{Berthier2019}
Berthier, R., Montanari, A., and Nguyen, P.-M. (2019).
\newblock {State evolution for approximate message passing with non-separable
  functions}.
\newblock {\em Information and Inference: A Journal of the IMA}, 9(1):33--79.

\bibitem[Brosse et~al., 2020]{brosse_last-layer_2020}
Brosse, N., Riquelme, C., Martin, A., Gelly, S., and Moulines, E. (2020).
\newblock On last-layer algorithms for classification: Decoupling
  representation from uncertainty estimation.

\bibitem[Bruce and Saad, 1994]{Bruce_1994}
Bruce, A.~D. and Saad, D. (1994).
\newblock Statistical mechanics of hypothesis evaluation.
\newblock {\em Journal of Physics A: Mathematical and General},
  27(10):3355--3363.

\bibitem[Candes and Sur, 2018]{candes_phase_2018}
Candes, E.~J. and Sur, P. (2018).
\newblock The phase transition for the existence of the maximum likelihood
  estimate in high-dimensional logistic regression.

\bibitem[Carmona and Hu, 2004]{Carmona2004}
Carmona, P. and Hu, Y. (2004).
\newblock Universality in sherrington-kirkpatrick's spin glass model.

\bibitem[Celentano et~al., 2020]{celentano2020estimation}
Celentano, M., Montanari, A., and Wu, Y. (2020).
\newblock The estimation error of general first order methods.
\newblock In {\em Conference on Learning Theory}, pages 1078--1141. PMLR.

\bibitem[Chatterjee, 2005]{Chatterjee2005}
Chatterjee, S. (2005).
\newblock A simple invariance theorem.

\bibitem[Chouard, 2022]{chouard2022quantitative}
Chouard, C. (2022).
\newblock Quantitative deterministic equivalent of sample covariance matrices
  with a general dependence structure.
\newblock {\em arXiv:2211.13044}.

\bibitem[Clart{\'{e}} et~al., 2022]{clarte_theoretical_2022}
Clart{\'{e}}, L., Loureiro, B., Krzakala, F., and Zdeborov{\'{a}}, L. (2022).
\newblock Theoretical characterization of uncertainty in high-dimensional
  linear classification.

\bibitem[Cornacchia et~al., 2022]{Cornacchia2022}
Cornacchia, E., Mignacco, F., Veiga, R., Gerbelot, C., Loureiro, B., and
  Zdeborová, L. (2022).
\newblock Learning curves for the multi-class teacher-student perceptron.

\bibitem[Dandi et~al., 2023]{dandi_universality_2023}
Dandi, Y., Stephan, L., Krzakala, F., Loureiro, B., and Zdeborová, L. (2023).
\newblock Universality laws for gaussian mixtures in generalized linear models.

\bibitem[D'Ascoli et~al., 2020]{Ascoli2020}
D'Ascoli, S., Refinetti, M., Biroli, G., and Krzakala, F. (2020).
\newblock Double trouble in double descent: Bias and variance(s) in the lazy
  regime.
\newblock In III, H.~D. and Singh, A., editors, {\em Proceedings of the 37th
  International Conference on Machine Learning}, volume 119 of {\em Proceedings
  of Machine Learning Research}, pages 2280--2290. PMLR.

\bibitem[Daxberger et~al., 2021]{daxberger_laplace_2021}
Daxberger, E., Kristiadi, A., Immer, A., Eschenhagen, R., Bauer, M., and
  Hennig, P. (2021).
\newblock Laplace {Redux} - {Effortless} {Bayesian} {Deep} {Learning}.
\newblock In Ranzato, M., Beygelzimer, A., Dauphin, Y., Liang, P.~S., and
  Vaughan, J.~W., editors, {\em Advances in {Neural} {Information} {Processing}
  {Systems}}, volume~34, pages 20089--20103. Curran Associates, Inc.

\bibitem[Deng et~al., 2022]{deng2022model}
Deng, Z., Kammoun, A., and Thrampoulidis, C. (2022).
\newblock A model of double descent for high-dimensional binary linear
  classification.
\newblock {\em Information and Inference: A Journal of the IMA},
  11(2):435--495.

\bibitem[Dhifallah and Lu, 2020]{Dhifallah2020}
Dhifallah, O. and Lu, Y.~M. (2020).
\newblock A precise performance analysis of learning with random features.

\bibitem[Dobriban and Wager, 2018]{Dobriban2018}
Dobriban, E. and Wager, S. (2018).
\newblock {High-dimensional asymptotics of prediction: Ridge regression and
  classification}.
\newblock {\em The Annals of Statistics}, 46(1):247 -- 279.

\bibitem[Donoho and Montanari, 2016]{donoho_high_2013}
Donoho, D. and Montanari, A. (2016).
\newblock High dimensional robust m-estimation: Asymptotic variance via
  approximate message passing.
\newblock {\em Probability Theory and Related Fields}, 166(3):935--969.

\bibitem[Donoho and Tanner, 2009]{Donoho_2009}
Donoho, D. and Tanner, J. (2009).
\newblock Observed universality of phase transitions in high-dimensional
  geometry, with implications for modern data analysis and signal processing.
\newblock {\em Philosophical Transactions of the Royal Society A: Mathematical,
  Physical and Engineering Sciences}, 367(1906):4273--4293.

\bibitem[Erdos et~al., 2009]{Erdos2009}
Erdos, L., Schlein, B., and Yau, H.-T. (2009).
\newblock Universality of random matrices and local relaxation flow.

\bibitem[Erdos et~al., 2010]{Erdos2010}
Erdos, L., Yau, H.-T., and Yin, J. (2010).
\newblock Bulk universality for generalized wigner matrices.

\bibitem[Gabri{\'e} et~al., 2018]{gabrie2018entropy}
Gabri{\'e}, M., Manoel, A., Luneau, C., Macris, N., Krzakala, F.,
  Zdeborov{\'a}, L., et~al. (2018).
\newblock Entropy and mutual information in models of deep neural networks.
\newblock {\em Advances in Neural Information Processing Systems}, 31.

\bibitem[Gal and Ghahramani, 2016]{gal_dropout_2016}
Gal, Y. and Ghahramani, Z. (2016).
\newblock Dropout as a bayesian approximation: Representing model uncertainty
  in deep learning.

\bibitem[Gawlikowski et~al., 2022]{gawlikowski_survey_2022}
Gawlikowski, J., Tassi, C. R.~N., Ali, M., Lee, J., Humt, M., Feng, J., Kruspe,
  A., Triebel, R., Jung, P., Roscher, R., Shahzad, M., Yang, W., Bamler, R.,
  and Zhu, X.~X. (2022).
\newblock A survey of uncertainty in deep neural networks.

\bibitem[Geiger et~al., 2019]{PhysRevE.100.012115}
Geiger, M., Spigler, S., d'Ascoli, S., Sagun, L., Baity-Jesi, M., Biroli, G.,
  and Wyart, M. (2019).
\newblock Jamming transition as a paradigm to understand the loss landscape of
  deep neural networks.
\newblock {\em Phys. Rev. E}, 100:012115.

\bibitem[Geman et~al., 1992]{geman1992neural}
Geman, S., Bienenstock, E., and Doursat, R. (1992).
\newblock Neural networks and the bias/variance dilemma.
\newblock {\em Neural computation}, 4(1):1--58.

\bibitem[Gerace et~al., 2022]{gerace2022gaussian}
Gerace, F., Krzakala, F., Loureiro, B., Stephan, L., and Zdeborová, L. (2022).
\newblock Gaussian universality of linear classifiers with random labels in
  high-dimension.

\bibitem[Gerace et~al., 2020]{gerace_generalisation_2020}
Gerace, F., Loureiro, B., Krzakala, F., Mezard, M., and Zdeborova, L. (2020).
\newblock Generalisation error in learning with random features and the hidden
  manifold model.
\newblock In III, H.~D. and Singh, A., editors, {\em Proceedings of the 37th
  International Conference on Machine Learning}, volume 119 of {\em Proceedings
  of Machine Learning Research}, pages 3452--3462. PMLR.

\bibitem[Gerbelot and Berthier, 2021]{Gerbelot21}
Gerbelot, C. and Berthier, R. (2021).
\newblock Graph-based approximate message passing iterations.

\bibitem[Goldt et~al., 2022]{goldt_gaussian_2021}
Goldt, S., Loureiro, B., Reeves, G., Krzakala, F., Mezard, M., and Zdeborova,
  L. (2022).
\newblock The gaussian equivalence of generative models for learning with
  shallow neural networks.
\newblock In Bruna, J., Hesthaven, J., and Zdeborova, L., editors, {\em
  Proceedings of the 2nd Mathematical and Scientific Machine Learning
  Conference}, volume 145 of {\em Proceedings of Machine Learning Research},
  pages 426--471. PMLR.

\bibitem[Goldt et~al., 2020]{Goldt2020}
Goldt, S., M\'ezard, M., Krzakala, F., and Zdeborov\'a, L. (2020).
\newblock Modeling the influence of data structure on learning in neural
  networks: The hidden manifold model.
\newblock {\em Phys. Rev. X}, 10:041044.

\bibitem[Graves, 2011]{NIPS2011_7eb3c8be}
Graves, A. (2011).
\newblock Practical variational inference for neural networks.
\newblock In Shawe-Taylor, J., Zemel, R., Bartlett, P., Pereira, F., and
  Weinberger, K., editors, {\em Advances in Neural Information Processing
  Systems}, volume~24. Curran Associates, Inc.

\bibitem[Guo et~al., 2017]{guo_calibration_2017}
Guo, C., Pleiss, G., Sun, Y., and Weinberger, K.~Q. (2017).
\newblock On calibration of modern neural networks.
\newblock In Precup, D. and Teh, Y.~W., editors, {\em Proceedings of the 34th
  International Conference on Machine Learning}, volume~70 of {\em Proceedings
  of Machine Learning Research}, pages 1321--1330. PMLR.

\bibitem[Hachem et~al., 2007]{hachem2007deterministic}
Hachem, W., Loubaton, P., and Najim, J. (2007).
\newblock Deterministic equivalents for certain functionals of large random
  matrices.

\bibitem[Hastie et~al., 2022]{10.1214/21-AOS2133}
Hastie, T., Montanari, A., Rosset, S., and Tibshirani, R.~J. (2022).
\newblock {Surprises in high-dimensional ridgeless least squares
  interpolation}.
\newblock {\em The Annals of Statistics}, 50(2):949 -- 986.

\bibitem[Hein et~al., 2019]{Hein_2019_CVPR}
Hein, M., Andriushchenko, M., and Bitterwolf, J. (2019).
\newblock Why relu networks yield high-confidence predictions far away from the
  training data and how to mitigate the problem.
\newblock In {\em Proceedings of the IEEE/CVF Conference on Computer Vision and
  Pattern Recognition (CVPR)}.

\bibitem[Hu and Lu, 2020]{Hu2020}
Hu, H. and Lu, Y.~M. (2020).
\newblock Universality laws for high-dimensional learning with random features.

\bibitem[Jospin et~al., 2022]{jospin_hands_2022}
Jospin, L.~V., Laga, H., Boussaid, F., Buntine, W., and Bennamoun, M. (2022).
\newblock Hands-on bayesian neural networks—a tutorial for deep learning
  users.
\newblock {\em IEEE Computational Intelligence Magazine}, 17(2):29--48.

\bibitem[Karoui, 2009]{Karoui2009}
Karoui, N.~E. (2009).
\newblock {Concentration of measure and spectra of random matrices:
  Applications to correlation matrices, elliptical distributions and beyond}.
\newblock {\em The Annals of Applied Probability}, 19(6):2362 -- 2405.

\bibitem[Karoui, 2010]{10.1214/08-AOS648}
Karoui, N.~E. (2010).
\newblock {The spectrum of kernel random matrices}.
\newblock {\em The Annals of Statistics}, 38(1):1 -- 50.

\bibitem[Karoui et~al., 2013]{Karoui2013}
Karoui, N.~E., Bean, D., Bickel, P.~J., Lim, C., and Yu, B. (2013).
\newblock On robust regression with high-dimensional predictors.
\newblock {\em Proceedings of the National Academy of Sciences},
  110(36):14557--14562.

\bibitem[Knowles and Yin, 2017]{Knowles2017}
Knowles, A. and Yin, J. (2017).
\newblock Anisotropic local laws for random matrices.
\newblock {\em Probability Theory and Related Fields}, 169(1):257--352.

\bibitem[Korada and Montanari, 2010]{Montanari2010}
Korada, S.~B. and Montanari, A. (2010).
\newblock Applications of lindeberg principle in communications and statistical
  learning.

\bibitem[Kristiadi et~al., 2020]{kristiadi_being_2020}
Kristiadi, A., Hein, M., and Hennig, P. (2020).
\newblock Being bayesian, even just a bit, fixes overconfidence in {R}e{LU}
  networks.
\newblock In III, H.~D. and Singh, A., editors, {\em Proceedings of the 37th
  International Conference on Machine Learning}, volume 119 of {\em Proceedings
  of Machine Learning Research}, pages 5436--5446. PMLR.

\bibitem[Krzakala et~al., 2012]{Krzakala_2012}
Krzakala, F., M{\'{e}}zard, M., Sausset, F., Sun, Y., and Zdeborov{\'{a}}, L.
  (2012).
\newblock Probabilistic reconstruction in compressed sensing: algorithms, phase
  diagrams, and threshold achieving matrices.
\newblock {\em Journal of Statistical Mechanics: Theory and Experiment},
  2012(08):P08009.

\bibitem[Lakshminarayanan et~al., 2017]{NIPS2017_9ef2ed4b}
Lakshminarayanan, B., Pritzel, A., and Blundell, C. (2017).
\newblock Simple and scalable predictive uncertainty estimation using deep
  ensembles.
\newblock In Guyon, I., Luxburg, U.~V., Bengio, S., Wallach, H., Fergus, R.,
  Vishwanathan, S., and Garnett, R., editors, {\em Advances in Neural
  Information Processing Systems}, volume~30. Curran Associates, Inc.

\bibitem[Liao and Couillet, 2018]{Liao2018}
Liao, Z. and Couillet, R. (2018).
\newblock On the spectrum of random features maps of high dimensional data.
\newblock In Dy, J. and Krause, A., editors, {\em Proceedings of the 35th
  International Conference on Machine Learning}, volume~80 of {\em Proceedings
  of Machine Learning Research}, pages 3063--3071. PMLR.

\bibitem[Liao and Mahoney, 2021]{NEURIPS2021_a7d8ae45}
Liao, Z. and Mahoney, M.~W. (2021).
\newblock Hessian eigenspectra of more realistic nonlinear models.
\newblock In Ranzato, M., Beygelzimer, A., Dauphin, Y., Liang, P., and Vaughan,
  J.~W., editors, {\em Advances in Neural Information Processing Systems},
  volume~34, pages 20104--20117. Curran Associates, Inc.

\bibitem[Liu et~al., 2020]{NEURIPS2020_543e8374}
Liu, J., Lin, Z., Padhy, S., Tran, D., Bedrax~Weiss, T., and Lakshminarayanan,
  B. (2020).
\newblock Simple and principled uncertainty estimation with deterministic deep
  learning via distance awareness.
\newblock In Larochelle, H., Ranzato, M., Hadsell, R., Balcan, M., and Lin, H.,
  editors, {\em Advances in Neural Information Processing Systems}, volume~33,
  pages 7498--7512. Curran Associates, Inc.

\bibitem[Louart et~al., 2018]{louart2018random}
Louart, C., Liao, Z., and Couillet, R. (2018).
\newblock {A} random matrix approach to neural networks.
\newblock {\em Ann. Appl. Probab.}, 28(2):1190--1248.

\bibitem[Loureiro et~al., 2021a]{loureiro_learning_2021}
Loureiro, B., Gerbelot, C., Cui, H., Goldt, S., Krzakala, F., Mezard, M., and
  Zdeborov\'{a}, L. (2021a).
\newblock Learning curves of generic features maps for realistic datasets with
  a teacher-student model.
\newblock In Ranzato, M., Beygelzimer, A., Dauphin, Y., Liang, P., and Vaughan,
  J.~W., editors, {\em Advances in Neural Information Processing Systems},
  volume~34, pages 18137--18151. Curran Associates, Inc.

\bibitem[Loureiro et~al., 2022]{loureiro_fluctuations_2022}
Loureiro, B., Gerbelot, C., Refinetti, M., Sicuro, G., and Krzakala, F. (2022).
\newblock Fluctuations, bias, variance \& ensemble of learners: Exact
  asymptotics for convex losses in high-dimension.
\newblock In Chaudhuri, K., Jegelka, S., Song, L., Szepesvari, C., Niu, G., and
  Sabato, S., editors, {\em Proceedings of the 39th International Conference on
  Machine Learning}, volume 162 of {\em Proceedings of Machine Learning
  Research}, pages 14283--14314. PMLR.

\bibitem[Loureiro et~al., 2021b]{NEURIPS2021_543e8374}
Loureiro, B., Sicuro, G., Gerbelot, C., Pacco, A., Krzakala, F., and
  Zdeborov\'{a}, L. (2021b).
\newblock Learning gaussian mixtures with generalized linear models: Precise
  asymptotics in high-dimensions.
\newblock In Ranzato, M., Beygelzimer, A., Dauphin, Y., Liang, P., and Vaughan,
  J.~W., editors, {\em Advances in Neural Information Processing Systems},
  volume~34, pages 10144--10157. Curran Associates, Inc.

\bibitem[MacKay, 1992]{10.1162/neco.1992.4.3.415}
MacKay, D. J.~C. (1992).
\newblock {Bayesian Interpolation}.
\newblock {\em Neural Computation}, 4(3):415--447.

\bibitem[MacKay, 1996]{MacKay1996}
MacKay, D. J.~C. (1996).
\newblock {\em Hyperparameters: Optimize, or Integrate Out?}, pages 43--59.
\newblock Springer Netherlands, Dordrecht.

\bibitem[Maddox et~al., 2019]{maddox_simple_2019}
Maddox, W.~J., Garipov, T., Izmailov, P., Vetrov, D.~P., and Wilson, A.~G.
  (2019).
\newblock A simple baseline for bayesian uncertainty in deep learning.
\newblock In {\em NeurIPS}.

\bibitem[Mai et~al., 2019]{8683376}
Mai, X., Liao, Z., and Couillet, R. (2019).
\newblock A large scale analysis of logistic regression: Asymptotic performance
  and new insights.
\newblock In {\em ICASSP 2019 - 2019 IEEE International Conference on
  Acoustics, Speech and Signal Processing (ICASSP)}, pages 3357--3361.

\bibitem[Marion and Saad, 1994]{marion_hyperparameters_1994}
Marion, G. and Saad, D. (1994).
\newblock Hyperparameters evidence and generalisation for an unrealisable rule.
\newblock In Tesauro, G., Touretzky, D., and Leen, T., editors, {\em Advances
  in Neural Information Processing Systems}, volume~7. MIT Press.

\bibitem[Marion and Saad, 1995]{Marion_1995}
Marion, G. and Saad, D. (1995).
\newblock A statistical mechanical analysis of a bayesian inference scheme for
  an unrealizable rule.
\newblock {\em Journal of Physics A: Mathematical and General},
  28(8):2159--2171.

\bibitem[Mattei, 2019]{Mattei2019}
Mattei, P.-A. (2019).
\newblock A parsimonious tour of bayesian model uncertainty.

\bibitem[Mei and Montanari, 2022]{mei_generalization_2022}
Mei, S. and Montanari, A. (2022).
\newblock The generalization error of random features regression: Precise
  asymptotics and the double descent curve.
\newblock {\em Communications on Pure and Applied Mathematics}, 75(4):667--766.

\bibitem[Mezard and Montanari, 2009]{mezard2009information}
Mezard, M. and Montanari, A. (2009).
\newblock {\em Information, physics, and computation}.
\newblock Oxford University Press.

\bibitem[M{\'e}zard et~al., 1987]{mezard1987spin}
M{\'e}zard, M., Parisi, G., and Virasoro, M.~A. (1987).
\newblock {\em Spin glass theory and beyond: An Introduction to the Replica
  Method and Its Applications}, volume~9.
\newblock World Scientific Publishing Company.

\bibitem[Montanari et~al., 2020]{montanari2020generalization}
Montanari, A., Ruan, F., Sohn, Y., and Yan, J. (2020).
\newblock The generalization error of max-margin linear classifiers:
  High-dimensional asymptotics in the overparametrized regime.
\newblock {\em arXiv:1911.01544 [math.ST]}.

\bibitem[Montanari and Saeed, 2022]{Montanari2022}
Montanari, A. and Saeed, B.~N. (2022).
\newblock Universality of empirical risk minimization.
\newblock In Loh, P.-L. and Raginsky, M., editors, {\em Proceedings of Thirty
  Fifth Conference on Learning Theory}, volume 178 of {\em Proceedings of
  Machine Learning Research}, pages 4310--4312. PMLR.

\bibitem[Mukhoti et~al., 2020]{NEURIPS2020_aeb7b30e}
Mukhoti, J., Kulharia, V., Sanyal, A., Golodetz, S., Torr, P., and Dokania, P.
  (2020).
\newblock Calibrating deep neural networks using focal loss.
\newblock In Larochelle, H., Ranzato, M., Hadsell, R., Balcan, M., and Lin, H.,
  editors, {\em Advances in Neural Information Processing Systems}, volume~33,
  pages 15288--15299. Curran Associates, Inc.

\bibitem[Nakkiran et~al., 2020]{Nakkiran2020}
Nakkiran, P., Kaplun, G., Bansal, Y., Yang, T., Barak, B., and Sutskever, I.
  (2020).
\newblock Deep double descent: Where bigger models and more data hurt.
\newblock In {\em 8th International Conference on Learning Representations,
  {ICLR} 2020, Addis Ababa, Ethiopia, April 26-30, 2020}. OpenReview.net.

\bibitem[Nakkiran et~al., 2021]{nakkiran2021optimal}
Nakkiran, P., Venkat, P., Kakade, S.~M., and Ma, T. (2021).
\newblock Optimal regularization can mitigate double descent.
\newblock In {\em International Conference on Learning Representations}.

\bibitem[Nguyen et~al., 2015]{7298640}
Nguyen, A., Yosinski, J., and Clune, J. (2015).
\newblock Deep neural networks are easily fooled: High confidence predictions
  for unrecognizable images.
\newblock In {\em 2015 IEEE Conference on Computer Vision and Pattern
  Recognition (CVPR)}, pages 427--436.

\bibitem[Niculescu-Mizil and Caruana, 2005]{10.1145/1102351.1102430}
Niculescu-Mizil, A. and Caruana, R. (2005).
\newblock Predicting good probabilities with supervised learning.
\newblock ICML '05, page 625–632, New York, NY, USA. Association for
  Computing Machinery.

\bibitem[Panahi and Hassibi, 2017]{NIPS2017_136f9513}
Panahi, A. and Hassibi, B. (2017).
\newblock A universal analysis of large-scale regularized least squares
  solutions.
\newblock In Guyon, I., Luxburg, U.~V., Bengio, S., Wallach, H., Fergus, R.,
  Vishwanathan, S., and Garnett, R., editors, {\em Advances in Neural
  Information Processing Systems}, volume~30. Curran Associates, Inc.

\bibitem[Pennington and Worah, 2017]{NIPS2017_0f3d014e}
Pennington, J. and Worah, P. (2017).
\newblock Nonlinear random matrix theory for deep learning.
\newblock In Guyon, I., Luxburg, U.~V., Bengio, S., Wallach, H., Fergus, R.,
  Vishwanathan, S., and Garnett, R., editors, {\em Advances in Neural
  Information Processing Systems}, volume~30. Curran Associates, Inc.

\bibitem[Platt, 2000]{platt_2000}
Platt, J. (2000).
\newblock Probabilistic outputs for support vector machines and comparisons to
  regularized likelihood methods.
\newblock {\em Adv. Large Margin Classif.}, 10.

\bibitem[Rahimi and Recht, 2007]{Rahimi2007}
Rahimi, A. and Recht, B. (2007).
\newblock Random features for large-scale kernel machines.
\newblock In Platt, J., Koller, D., Singer, Y., and Roweis, S., editors, {\em
  Advances in Neural Information Processing Systems}, volume~20. Curran
  Associates, Inc.

\bibitem[Rangan, 2011]{rangan2011generalized}
Rangan, S. (2011).
\newblock Generalized approximate message passing for estimation with random
  linear mixing.
\newblock In {\em 2011 IEEE International Symposium on Information Theory
  Proceedings}. IEEE.

\bibitem[Ritter et~al., 2018]{ritter_scalable_2018}
Ritter, H., Botev, A., and Barber, D. (2018).
\newblock A scalable laplace approximation for neural networks.
\newblock In {\em International Conference on Learning Representations}.

\bibitem[Rosset et~al., 2003]{NIPS2003_0fe47339}
Rosset, S., Zhu, J., and Hastie, T. (2003).
\newblock Margin maximizing loss functions.
\newblock In Thrun, S., Saul, L., and Sch\"{o}lkopf, B., editors, {\em Advances
  in Neural Information Processing Systems}, volume~16. MIT Press.

\bibitem[Schröder et~al., 2023]{Schroder2023}
Schröder, D., Cui, H., Dmitriev, D., and Loureiro, B. (2023).
\newblock Deterministic equivalent and error universality of deep random
  features learning.

\bibitem[Sollich, 1998]{NIPS1998_5cbdfd0d}
Sollich, P. (1998).
\newblock Learning curves for gaussian processes.
\newblock In Kearns, M., Solla, S., and Cohn, D., editors, {\em Advances in
  Neural Information Processing Systems}, volume~11. MIT Press.

\bibitem[Sollich, 2001]{NIPS2001_d68a1827}
Sollich, P. (2001).
\newblock Gaussian process regression with mismatched models.
\newblock In Dietterich, T., Becker, S., and Ghahramani, Z., editors, {\em
  Advances in Neural Information Processing Systems}, volume~14. MIT Press.

\bibitem[Spigler et~al., 2019]{Spigler_2019}
Spigler, S., Geiger, M., d'Ascoli, S., Sagun, L., Biroli, G., and Wyart, M.
  (2019).
\newblock A jamming transition from under- to over-parametrization affects
  generalization in deep learning.
\newblock {\em Journal of Physics A: Mathematical and Theoretical},
  52(47):474001.

\bibitem[Stojnic, 2013]{Stojnic2013}
Stojnic, M. (2013).
\newblock A framework to characterize performance of lasso algorithms.

\bibitem[Sur and Cand\`es, 2019]{sur_modern_2018}
Sur, P. and Cand\`es, E.~J. (2019).
\newblock A modern maximum-likelihood theory for high-dimensional logistic
  regression.
\newblock {\em Proceedings of the National Academy of Sciences},
  116(29):14516--14525.

\bibitem[Tao and Vu, 2012]{Tao2012}
Tao, T. and Vu, V. (2012).
\newblock Random matrices: universal properties of eigenvectors.
\newblock {\em Random Matrices: Theory and Applications}, 01(01):1150001.

\bibitem[Thrampoulidis et~al., 2018]{8365826}
Thrampoulidis, C., Abbasi, E., and Hassibi, B. (2018).
\newblock Precise error analysis of regularized $m$ -estimators in high
  dimensions.
\newblock {\em IEEE Transactions on Information Theory}, 64(8):5592--5628.

\bibitem[Thrampoulidis et~al., 2015]{Thrampoulidis15}
Thrampoulidis, C., Oymak, S., and Hassibi, B. (2015).
\newblock Regularized linear regression: A precise analysis of the estimation
  error.
\newblock In Grünwald, P., Hazan, E., and Kale, S., editors, {\em Proceedings
  of The 28th Conference on Learning Theory}, volume~40 of {\em Proceedings of
  Machine Learning Research}, pages 1683--1709, Paris, France. PMLR.

\bibitem[Wasserman, 2013]{wasserman2013all}
Wasserman, L. (2013).
\newblock {\em All of Statistics: A Concise Course in Statistical Inference}.
\newblock Springer Texts in Statistics. Springer New York.

\bibitem[Welling and Teh, 2011]{welling_bayesian_2011}
Welling, M. and Teh, Y.~W. (2011).
\newblock Bayesian learning via stochastic gradient langevin dynamics.
\newblock In {\em Proceedings of the 28th International Conference on
  International Conference on Machine Learning}, ICML'11, page 681–688,
  Madison, WI, USA. Omnipress.

\bibitem[Zdeborov{\'a} and Krzakala, 2016]{zdeborova2016statistical}
Zdeborov{\'a}, L. and Krzakala, F. (2016).
\newblock Statistical physics of inference: Thresholds and algorithms.
\newblock {\em Advances in Physics}.

\end{thebibliography}
